\documentclass[11pt]{article}
\usepackage{amsmath}
\usepackage{graphicx} 
\usepackage{fullpage}
\usepackage{float}
\usepackage[utf8]{inputenc}
\usepackage{amsmath,amssymb}
\usepackage{bbm}
\usepackage{upgreek} 

\DeclareMathOperator*{\argmax}{arg\,max}
\DeclareMathOperator*{\argmin}{arg\,min}
 
\newenvironment{proofof}[1]{\vspace{0.8em}\par{\noindent \textit{Proof of #1.}}}{\hspace*{\fill} $\qed$ \par}

\usepackage{hyperref}
\usepackage{cleveref}

\usepackage{amsthm}
\newtheorem{theorem}{Theorem}[section]
\newtheorem{lemma}[theorem]{Lemma}

\newtheorem{definition}[theorem]{Definition}
\newtheorem{corollary}[theorem]{Corollary}

\newtheorem{remark}[theorem]{Remark}

\newtheorem{assumption}{Assumption}

\usepackage{comment}
\usepackage{mathtools}

\usepackage[shortlabels]{enumitem}

\newcommand{\E}{\mathbb{E}}

\usepackage[shortlabels]{enumitem}

\newcommand{\ignore}[1]{}

\usepackage[shortlabels]{enumitem}

\usepackage[shortlabels]{enumitem}

\usepackage{caption}
\usepackage{xcolor}

\def\Rset{\mathbb{R}}

\let\Pr\undefined

\DeclareMathOperator*{\Pr}{\mathbb{P}}

\DeclareMathOperator{\VCdim}{VCdim}

\newcommand{\cA}{\mathcal{A}}
\newcommand{\cB}{\mathcal{B}}
\newcommand{\cC}{\mathcal{C}}
\newcommand{\cD}{\mathcal{D}}
\newcommand{\cF}{\mathcal{F}}
\newcommand{\cG}{\mathcal{G}}
\newcommand{\cH}{\mathcal{H}}

\newcommand{\cO}{\mathcal{O}}
\newcommand{\cP}{\mathcal{P}}

\newcommand{\cX}{\mathcal{X}}
\newcommand{\cY}{\mathcal{Y}}
\newcommand{\cZ}{\mathcal{Z}}

\newcommand{\bbZ}{\mathbb{Z}}
\newcommand{\bbR}{\mathbb{R}}

\newcommand{\round}{\text{Round}}
\newcommand{\rh}{\tilde{h}}

\newcommand{\eps}{{\varepsilon}}

\DeclareMathOperator{\BR}{\mathsf{BR}}

\newcommand{\inv}{^{-1}}

\newcommand{\bfhaty}{ {\hat{ {\bf y}}} }

\newcommand{\SQE}{\mathrm{SQE}}
\newcommand{\SQErr}{\mathrm{SQErr}}
\newcommand{\CalDist}{\mathrm{CalDist}}
\newcommand{\ECE}{\mathrm{ECE}}
\newcommand{\hy}{\hat{y}}

\newcommand{\Tk}[1]{T^{\ge #1}}

\newcommand{\alice}{A}
\newcommand{\bob}{B}

\newcommand{\yak}[2]{\hat{y}^{#1,#2}_A}
\newcommand{\ybk}[2]{\hat{y}^{#1,#2}_B}
\newcommand{\ymk}{\yak}
\newcommand{\yhk}{\ybk}

\newcommand{\byak}[2]{\bfhaty^{#1,#2}_A}

\newcommand{\bymk}[2]{\byak}
\newcommand{\byhk}[2]{\byhk}

\newcommand{\pak}[2]{\hat{y}^{#1,#2}_A}
\newcommand{\pbk}[2]{\hat{y}^{#1,#2}_B}
\newcommand{\pmk}{\pak}
\newcommand{\phk}{\pbk}
\newcommand{\barpak}[2]{\bar{p}^{#1,#2}_A}
\newcommand{\barpbk}[2]{\bar{p}^{#1,#2}_B}
\newcommand{\barpmk}[2]{\barpak}
\newcommand{\barphk}[2]{\barpbk}

\newcommand{\1}{\mathbbm{1}}
\renewcommand{\BR}{\mathrm{BR}}

\newcommand{\aliceModel}[1]{f_A^{#1}}
\newcommand{\aliceOracle}{\cO_{\cH_A}}
\newcommand{\aliceSample}{S_A}

\newcommand{\aliceData}{x_A}
\newcommand{\alicePreds}[1]{P_A^{#1}}

\newcommand{\aliceModelBoost}[2]{f_{A}^{#1,v,#2}}
\newcommand{\aliceG}{g_{r+1,v,k+1}^{v',i}}
\newcommand{\aliceOracleModel}[1]{h_A^{r+1,v,#1,v'}}
\newcommand{\aliceTranscript}[1]{\pi_A^{#1}}

\newcommand{\bobModel}[1]{f_B^{#1}}

\newcommand{\bobOracle}{\cO_{\cH_B}}
\newcommand{\bobSample}{S_B}

\newcommand{\bobData}{x_B}
\newcommand{\bobPreds}[1]{P_B^{#1}}

\newcommand{\bobTranscript}[1]{\pi_B^{#1}}

\newcommand{\playerOracle}{\cO_{\cH_\bullet}}
\newcommand{\playerSample}{S_{\bullet}}
\newcommand{\playerSampleLS}[2]{S_\bullet^{#1,#2}}
\newcommand{\playerSampleLSBoost}[1]{S_\bullet^{r,v,#1,v'}}
\newcommand{\playerModel}[1]{f_{\bullet}^{#1}}
\newcommand{\playerModelLS}[2]{f_{\bullet}^{#1,#2}}
\newcommand{\playerModelLStilde}[2]{\tilde{f}_{\bullet}^{#1,#2}}
\newcommand{\playerData}{x_\bullet}
\newcommand{\playerModelBoost}[1]{f_{\bullet}^{r,v,#1}}

\newcommand{\playerOracleModel}[1]{h_\bullet^{r,v,#1,v'}}
\newcommand{\playerPreds}[1]{P_{\circ}^{#1}}

\newcommand{\playerDecisionRule}[1]{\rho_{\bullet}^{#1}}

\newcommand{\err}{\text{err}}
\newcommand{\ls}{\text{LS}}
\newcommand{\lsboost}{\textrm{INTERNAL-BOOST}}
\newcommand{\lsboosteval}{\textrm{INTERNAL-BOOST-EVAL}}
\newcommand{\crossboost}{\textrm{CROSS-BOOST}}
\newcommand{\crossboosteval}{\textrm{CROSS-BOOST-EVAL}}
\newcommand{\ndim}{\textup{Ndim}}
\newcommand{\pdim}{\textup{Pdim}}
\newcommand{\inducedLS}{\cG_{\cH_J, \eta}}
\newcommand{\inducedThreshold}{\cG_{\cH_J, \eta}^\le}

\newcommand{\transcript}[1]{\pi^{#1}}
\newcommand{\lsTranscript}[1]{\pi^{r,v}}
\newcommand{\lsSubtranscript}[1]{\pi^{r,v,#1}}

\newcommand{\sampleExp}{\E_{(\playerData, y) \sim S_\bullet}}
\newcommand{\sampleExpLS}{\E_{(\playerData, y) \sim S_\bullet^{r,v}}}
\newcommand{\roundStuff}{\round\left(\tilde{f}_\bullet^{r,v,K}(\playerData);m^2\right)}
\newcommand{\unroundedModel}{\tilde{f}_\bullet^{r,v,K}(\playerData)}
\newcommand{\roundedModel}{f_\bullet^{r,v,K}(\playerData)}

\newcommand{\miniquad}{\quad\quad\quad}
\newcommand{\megaquad}{\quad\quad\quad\quad\quad\quad\quad\quad\quad\quad}
\newcommand{\megamegaquad}{\megaquad\megaquad}

\usepackage{nicematrix}

\usepackage{natbib}

\usepackage{algorithm}
\usepackage[noend]{algorithmic}
\usepackage{amsmath,amsthm}
\usepackage{amsfonts}
\usepackage{bbm}
\usepackage{xcolor}
\usepackage{setspace} 

\allowdisplaybreaks
\makeatletter
\newenvironment{protocol}[1][htb]{%
    \renewcommand{\ALG@name}{Protocol}
   \begin{algorithm}[#1]%
  }{\end{algorithm}}

\title{Collaborative Prediction: \\ Tractable Information Aggregation via Agreement}

\author{Natalie Collina \vspace{4pt} \\ Varun Gupta \and Ira Globus-Harris  \vspace{4pt} \\ Aaron Roth \vspace{4pt}  \\ 
  \textit{University of Pennsylvania} 
  \and Surbhi Goel \vspace{4pt} \\ Mirah Shi \\
}

\begin{document}
\pagestyle{empty} 

\maketitle
\thispagestyle{empty}
\begin{abstract}
We give efficient ``collaboration protocols'' through which two parties, who observe different features about the same instances, can interact to arrive at predictions that are more accurate than either could have obtained on their own. The parties only need to iteratively share and update their own label predictions---without either party ever having to share the actual features that they observe. Our protocols are efficient reductions to the problem of learning on each party's feature space alone, and so can be used even in settings in which each party's feature space is illegible to the other---which arises in models of human/AI interaction and in multi-modal learning. The communication requirements of our protocols are independent of the dimensionality of the data. In an online adversarial setting we show how to give regret bounds on the predictions that the parties arrive at with respect to a class of benchmark policies defined on the joint feature space of the two parties, despite the fact that neither party has access to this joint feature space. We also give simpler algorithms for the same task in the ``batch'' setting in which we assume that there is a fixed but unknown data distribution.  We  generalize our protocols to a decision theoretic setting with high dimensional outcome spaces---the parties in this setting do not need to communicate their (high dimensional) predictions about the outcome, but can instead communicate only ``best response actions'' with respect to a known utility function and their predicted outcome distribution. 

Our theorems give a computationally and statistically tractable generalization of past work on information aggregation amongst Bayesians who share a common and correct prior, as part of a literature studying ``agreement'' in the style of Aumann's agreement theorem. Our results require no knowledge of (or even the existence of) a prior distribution and are computationally efficient. Nevertheless we show how to lift our theorems back to this classical Bayesian setting, and in doing so, give new information aggregation theorems for Bayesian agreement. In particular we give the first distribution-agnostic information aggregation theorems that do not require making assumptions on the prior distribution, but instead are able to give worst-case accuracy guarantees with respect to restricted classes of functions on the parties' joint feature spaces.

\end{abstract}

\newpage
\tableofcontents
\clearpage
\pagestyle{plain}
\pagenumbering{arabic}
\setcounter{page}{1}

\section{Introduction}

Imagine that there are multiple parties who hold different kinds of information about the same examples, and would like to collaboratively learn an optimal label predictor on their joint feature space. The straightforward solution would be for them to  pool their features and attempt to train an optimal predictor on the data containing the pooled features. But there are settings in which this approach is infeasible. For example, the features legible to one party might not be legible to another. This is the case in e.g. models of human/AI interaction in which one of the parties is a human being and another is a predictive model \citep{alur2024human,collina2025tractable}. In this case the human being is in possession of qualitative features that are difficult to encode for a model (in a medical application, e.g. observations about patient demeanor, mood, smell, etc. that are difficult to formalize) and the model in turn has been trained on enormous amounts of data that cannot be easily used by the human being. In other settings, directly sharing the data might be impossible because of legal or contractual obligations, which is the case in  regulated industries like healthcare. This setting motivates the field of ``vertically federated learning'' \citep{wei2022vertical}. It may also be that the data itself is very high dimensional, and communication bandwidth constraints preclude sharing it in its entirety --- which motivates the study of the communication complexity of distributed learning \citep{balcan2012distributed}. Finally, the technical expertise to train on different kinds of features might be siloed: for example, the data held by different parties might be multi-modal; one party might hold image data (e.g. CT scans) while another might hold text data (e.g. physician notes). Each party might have the tooling necessary to learn on data in their own modality, but none may have the tooling to be easily able to learn on all of the data together.

Aumann's agreement theorem suggests a tempting general (if stylized) solution to these problems: it states that two perfectly informed Bayesians with a common prior, different observations, and common knowledge of each other's posteriors must share the same posterior \citep{aumann1976}. Subsequent work has given finite-time convergence protocols through which the different parties  engage in a conversation about their beliefs about the outcome \citep{geanakoplos1982we,aaronson2004complexity} without ever sharing their observations. In particular, \cite{aaronson2004complexity} showed that for 1-dimensional real-valued outcomes, two Bayesian will reach approximate agreement quickly --- in a number of rounds that depends only (polynomially) on the error parameters characterizing ``approximate'' agreement, independently of the dimensionality or complexity of the prior distribution. Agreement on its own does not in general solve the collaborative learning problem --- it has been known since  \cite{geanakoplos1982we} that \emph{agreement} does not imply \emph{information aggregation}\footnote{Consider a joint distribution on bits $x_A$, $x_B$, and $y$ that are all marginally uniform, but such that $y = x_A + x_B \mod 2$. If Alice is in possession of $x_A$ and Bob is in possession of $x_B$ they will agree that $\Pr[y = 1] = \frac{1}{2}$, even though they would know $y$ with certainty if they shared their data with each other.}. In other words, although two Bayesians engaging in an agreement protocol can only improve the accuracy of their beliefs, they may \emph{agree} on beliefs that are less accurate than those they would have arrived at had they instead shared their information and formed a posterior belief conditional on their joint observations.  Nevertheless \cite{kong2023false} and \cite{bo2023agreement} have studied conditions (on the prior distribution) under which agreement \emph{does} imply full information aggregation --- i.e. optimal learning on the joint feature space.  Unfortunately, since it studies perfectly informed Bayesians, this literature makes implausible computational and epistemic assumptions (Why do the two parties share the same, perfect prior knowledge? How do they perform Bayes updates in complex settings?) which makes these approaches seemingly far from algorithmic solutions. Recently, \cite{collina2025tractable} showed how to recover and generalize  quantitative agreement theorems without making any distributional assumptions (i.e. in online adversarial settings) using only computationally and statistically tractable calibration conditions that substantially relax Bayesian rationality. But the work of \cite{collina2025tractable} says nothing about information aggregation, and so does not provide a solution to the collaborative learning problem.

In this paper we generalize the connection between agreement and information aggregation and  give computationally efficient protocols that provably result in information aggregation after only a small number of rounds of communication, that need not make any distributional assumptions at all.  In our model, different parties hold different (possibly overlapping) features of the same examples, and may interact over a small number of rounds to share (only) their label predictions, computed from their own features, with each other. Because they hold different information from each other, different parties will likely initially make different predictions about the same example. Nevertheless, during the interaction, they may update their predictions in response to the predictions of their counterparty in the collaboration protocol. We would like them to converge on predictions that are more accurate than any single party could have obtained on their own --- and ideally predictions that are \emph{optimal} with respect to some benchmark class of predictors that are defined on their joint feature space, despite the fact that every participant in the protocol only has access to their own feature space. Because the parties only need to communicate their label predictions (to bounded precision) at each round, the communication complexity of our protocol is independent of the dimensionality of the data. 

We give two variants of our protocol: one for prediction in the online adversarial setting, in which we make no distributional assumptions at all --- and a simpler protocol for the batch setting, in which data is assumed to be drawn i.i.d. from a fixed but unknown and arbitrary distribution. In both cases we guarantee that the collaboration protocol is accuracy-improving, and give conditions under which the predictions are provably \emph{optimal} with respect to a  benchmark class of models defined on the pooled features. These conditions are frequentist ``weak-learning'' assumptions that substantially generalize the ``information substitutes'' condition on prior distributions used by \cite{bo2023agreement} in Bayesian settings.  Moreover our protocols are computationally efficient to run in the sense that they are computationally efficient reductions from the problem of multi-party learning to the problem of single-party learning, and therefore efficient in the worst case whenever the single-party learning problem can be efficiently solved. Each party only needs to run their own learning algorithm, tailored to data from their own modality, on their own data a bounded number of times in order to engage in the protocol. 

Finally, we show that all of our results ``lift'' back to the Bayesian setting of \cite{aumann1976,aaronson2004complexity,bo2023agreement}, resulting in new theorems about agreement and information aggregation in the classical setting in which examples are assumed to be drawn from a fixed and known prior, and a conversation between two perfect Bayesians occurs for a single draw from this prior. Among other things, we show that Bayesian agreement implies accuracy at least that of the best linear function on the joint feature space of the two parties, independently of any assumptions on the prior distribution --- the first such distribution-independent information aggregation theorem we are aware of in the agreement literature.

\subsection{Our Model and Results}
We begin by describing our results in the online-adversarial setting, when our goal is to solve a one-dimensional regression problem; we then describe extensions to more complex prediction tasks and to simpler algorithms in the batch setting. Finally we describe how our results ``lift'' to one-shot interactions if both parties are perfect Bayesians and share a common and correct prior --- the traditional setting for Aumann's agreement theorem \citep{aumann1976}.

There is a feature space $\cX = \cX_A \times \cX_B$ partitioned into parts $\cX_A$ and $\cX_B$ which may each be arbitrary, as well as a label space $\cY$ that initially we take to be $\cY = [0,1]$. At each day $t$, an arbitrary adaptive adversarial process chooses an example $x^t = (x_A^t,x_B^t) \in \cX$ and a label $y^t \in \cY$. Party 1 (Alice) receives $x_A^t$ and Party 2 (Bob) receives $x_B^t$. Each day, Alice and Bob then engage in a ``collaboration protocol'' which is an interaction that takes place across $K$ rounds. In odd rounds $k$, Alice produces a prediction $\hat{y}^{t,k} \in \cY$ that may be a function of $x^t_A$ as well as all previously observed history (including Bob's predictions at previous rounds on the same day). Similarly, in even rounds $k$, Bob produces a prediction $\hat{y}^{t,k} \in \cY$ that may be a function of $x_B^t$ and all previously observed history. Crucially Alice and Bob never share their feature vectors with one another---only label predictions. At the final round $K$ each day, they fix their final prediction $\hat y^t = \hat y^{t,K}$, at which point both Alice and Bob learn the true label $y^t$, and time proceeds to the next day. 

Our goal in interaction is to arrive at a set of predictions $\hat y^1,\ldots,\hat y^T$ that have squared error that is as low as the \emph{best predictor in hindsight} in some class of models $\cH_J$ defined on the \emph{joint} feature space of Alice and Bob: i.e. each function $h \in H_J$ has the form $h:\cX\rightarrow \mathbb{R}$ and produces a prediction $h(x_A,x_B)$ that is a function of the features available to both parties. For example, we might take $\cH_J$ to be the set of all norm-bounded linear functions on the joint feature space. In other words, we want to be able to guarantee, against an arbitrary adversarial sequence:
\begin{definition}[Predictions have No (External) Regret to $\cH_J$] 
The final predictions $\hat y^1,\ldots \hat y^T$ have no (external) regret to $\cH_J$ if for every $h \in \cH_J$:
$$\sum_{t=1}^T (\hat y^t - y^t)^2 \leq \sum_{t=1}^T (h(x^t) - y^t)^2$$
\end{definition}

The difficulty is that there may be no predictor defined over $\cX_A$ or $\cX_B$ individually that can obtain this --- collaborating to use information from \emph{both} parties is essential. What each party can try to do instead is to make predictions $\hat{y}^{t,k}$ during their own rounds of conversation that have no regret with respect to classes of functions $\cH_A$ and $\cH_B$ that are respectively defined only on their own feature spaces --- i.e. each function $h_A \in \cH_A$ is defined as $h_A:\cX_A \rightarrow \mathbb{R}$ and each function $h_B \in \cH_B$ is defined as $h_B:\cX_B \rightarrow \mathbb{R}$. For example, $\cH_A$ and $\cH_B$ might be the set of unit-norm linear functions defined only over $\cX_A$ and $\cX_B$ respectively. We take as an intermediate goal to produce a \emph{single} sequence of predictions at some round $k \in \{1,\ldots,K\}$ that has no \emph{swap}-regret with respect to $\cH_A$ and $\cH_B$ simultaneously.

\begin{definition}[Swap Regret (Informal Version of Definition \ref{def:swap})]
A sequence of predictions $\hat y^{1,k},\ldots,\hat y^{T,k}$ has no swap regret with respect to a class $\cH$ if for every value $v \in \{\hat y^{1,k},\ldots,\hat y^{T,k}\}$ and for every $h \in \cH$:
$$\sum_{t=1}^{T}\mathbbm{1}[\hat y^{t,k} = v](\hat y^{t,k} - y^{t})^{2} \leq
    \min_{h \in \cH}\left( \sum_{t = 1}^{T}\mathbbm{1}[\hat y^{t,k} = v](h(x^{t}) - y^{t})^{2}\right)  $$
\end{definition}

Swap regret is a stronger condition than external regret --- it requires that our predictions $\hat y^t_k$ be as accurate as the best model $h \in \cH$ not just marginally, but conditionally on the value of our own predictions. As an intermediate step towards our goal of obtaining predictions with no regret to $\cH_J$, we will hope to produce a single sequence of predictions that has no swap regret with respect to  classes $\cH_A$ and $\cH_B$ which are individually weaker than $\cH_J$ (as they are defined only on $\cX_A$ and $\cX_B$ respectively). To relate this condition to our ultimate goal,  we define a weak learning condition (related to the weak learning condition used to characterize the relationship between multicalibration and boosting by \cite{globus-harris2023multicalibration}) that relates $\cH_A, \cH_B$ and $\cH_J$. 

\begin{definition}[Weak Learning for Regression (Informal Version of Definition~\ref{def:joint-weak})]
We say that $\cH_A$ and $\cH_B$ jointly satisfy the \emph{weak learning} condition with respect to $\cH_J$ if for any joint distribution $\cD$ over $\cX\times \cY$ we have that if:
$$\min_{h_J \in \cH_J}\E_\cD[(h_J(x)-y)^2] < \E_\cD[(\mu(\cD)-y)^2]$$
Then there either exists $h_A \in \cH_A$ such that:
$$\E_\cD[(h_A(x_A)-y)^2] < \E_\cD[(\mu(\cD)-y)^2]$$
or there exists $h_B \in \cH_B$ such that:
$$\E_\cD[(h_B(x_B)-y)^2] < \E_\cD[(\mu(\cD)-y)^2]$$
where $\mu(\cD) = \E_\cD[y]$ is  the label mean over the distribution.
\end{definition}
 In other words, the weak learning condition requires that if there is any model in the joint class $\cH_J$ that obtains predictions with lower error than a constant predictor, there must be some model in either Alice's class $\cH_A$ or Bob's class $\cH_B$ that also obtains lower error than a constant predictor. This is a weak learning condition because the model in the joint class might obtain \emph{much} better error than a trivial constant predictor, but we only require that $\cH_A$ or $\cH_B$ obtain \emph{slightly} better than trivial error. 

 In Section \ref{sec:boost} we prove a ``boosting'' theorem: if $\cH_A$ and $\cH_B$ satisfy the weak learning condition with respect to $\cH_J$, and a sequence of predictions  $\{\hat y^{1,k},\ldots,\hat y^{T,k}\}$ has no swap regret with respect to $\cH_A$ and $\cH_B$, then it must in fact have no external regret with respect to the joint class $\cH_J$.  
 In fact, we show that a slightly weaker condition than swap regret suffices. It is enough that the sequence of predictions has low \emph{distance} to no-swap regret with respect to $\cH_A$ and $\cH_B$ --- i.e. that it is possible to perturb the sequence of predictions by a small amount in the $L_1$ norm such that the perturbed predictions have no swap regret. This is related to recently studied notions of ``distance to calibration'' \citep{blasiok2023unifying,qiao2024distance,arunachaleswaran2025elementary}, and will be easier for us to satisfy. 
  It then follows that any other sequence of predictions  $\{\hat y^1,\ldots,\hat y^T\}$ that has lower squared error than predictions  $\{\hat y^{1,k},\ldots,\hat y^{T,k}\}$ must in turn have no (external) regret with respect to $\cH_J$. 

  Given classes $\cH_A$ and $\cH_B$ that satisfy our weak learning condition with respect to a class of models $\cH_J$ defined on the joint feature space $\cX$, our problem is thus reduced to giving a collaboration protocol that quickly converges to a sequence of predictions that simultaneously has no swap regret to both $\cH_\cA$ and $\cH_B$. Towards this end, we ask that both Alice and Bob satisfy a condition that we call ``conversation swap regret'' relative to $\cH_A$ and $\cH_B$ respectively.
  \begin{definition}[Conversation Swap Regret (Definition \ref{def:CSR})]
  We say that Bob's predictions have no \emph{conversation swap regret} with respect to $\cH_B$ if for every even round of conversation $k$ and for every pair of values $v \in \{\hat y^{1,k},\ldots,\hat y^{T,k}\}$ and $v' \in \{\hat y^{1,k-1},\ldots,\hat y^{T,k-1}\}$:
$$\sum_{t=1}^{T}\mathbbm{1}[\hat y^{t,k-1} = v']\mathbbm{1}[\hat y^{t,k} = v](\hat y^{t,k} - y^{t})^{2} \leq
    \min_{h \in \cH}\left( \sum_{t = 1}^{T}\mathbbm{1}[\hat y^{t,k-1} = v']\mathbbm{1}[\hat y^{t,k} = v](h(x^{t}) - y^{t})^{2}\right)  $$
    If Alice satisfies a symmetric condition on odd rounds $k$ with respect to $\cH_A$ we say that Alice has no conversation swap regret with respect to $\cH_A$.
  \end{definition}
  In other words, conversation swap regret requires that Alice and Bob satisfy the no swap regret condition (with respect to their respective model classes $\cH_A$ and $\cH_B$) not just marginally over the whole sequence, but on each subsequence defined by the other party's prediction \emph{at the previous round of interaction}. Whenever $\cH_A$ and $\cH_B$ contain all constant functions with range in $[0,1]$, having no conversation swap regret implies satisfying the ``conversation calibration'' condition defined in \cite{collina2025tractable}. 
 
  In Section \ref{sec:online-alg}, we show that when both Alice and Bob make predictions with no conversation swap regret with respect to $\cH_A$ and $\cH_B$ respectively, then if the collaboration protocol runs for sufficiently  many rounds $K$, there must exist some round $k \leq K$ at which the sequence of predictions $\{\hat y^{1,k},\ldots,\hat y^{T,k}\}$ has low distance to swap regret with respect to \emph{both} $\cH_A$ and $\cH_B$ simultaneously. Although this round $k$ may not be the final round $K$, we also show that the final set of predictions has only lower squared error than the predictions made at any previous round $k$:
  $$\sum_{t=1}^T (\hat y^{t,K}-y^t)^2 \leq \sum_{t=1}^T (\hat y^{t,k} - y^t)^2$$
  Thus, by applying our ``boosting'' theorem from Section \ref{sec:boost}, we can conclude that if $\cH_A$ and $\cH_B$ satisfy the weak learning condition with respect to a joint class $\cH_J$, then the final sequence of predictions $\{\hat y^{1},\ldots,\hat y^{T}\}$ has no (external) regret with respect to the joint class $\cH_J$. 

  It remains to ask: for which classes $\cH_A$ and $\cH_B$ do there exist efficient algorithms for satisfying the no-conversation-swap regret condition, and are there examples of classes $(\cH_A, \cH_B, \cH_J)$ that satisfy the weak learning condition? In Sections~\ref{sec:linear} and \ref{sec:algorithmic} we provide answers to these questions. \cite{garg2024oracle} gave an  efficient reduction (in the style of \cite{blum2007external})  from the problem of obtaining no swap regret with respect to an arbitrary class of functions $\cH$ to the problem of obtaining no external regret with respect to $\cH$. We in turn give an efficient reduction from the problem of obtaining no conversation swap regret with respect to an arbitrary class of functions $\cH$ to the problem of obtaining no swap regret with respect to $\cH$. In combination, these results mean that there are computationally efficient algorithms for engaging in our collaboration protocol for any class of models $\cH_A, \cH_B$ that admit standard efficient online learning algorithms with regret guarantees --- and ``oracle efficient'' algorithms for any class of models for which there are online learning algorithms with good (external) regret bounds, even if they are not computationally efficient in the worst case. Because there exist computationally efficient algorithms for online adversarial norm-bounded linear regression \cite{azoury2001relative,vovk2001competitive} and related problems (e.g. squared error regression over reproducing kernel Hilbert spaces \cite{vovk2006line}), this immediately implies efficient algorithms for obtaining no conversation swap regret with respect to classes $\cH_A, \cH_B$ representing norm-bounded linear functions over $\cX_A$ and $\cX_B$ respectively. Moreover, we show in Section \ref{sec:linear}  that norm-bounded linear functions over $\cX_A$ and $\cX_B$ respectively satisfy the weak learning condition with respect to norm-bounded linear functions on the joint feature space $\cX$. In fact we show a more general theorem for any class of functions $\cH_J$ that can be represented as the Minkowski sum of classes $\cH_A$ and $\cH_B$ that are themselves bounded and star-shaped. Moreover we show (also in Section \ref{sec:linear}) that the weak learning condition we prove is quantitatively tight for linear functions.  
  
  All together, this means that we have a computationally and statistically efficient collaboration protocol for learning predictors that are as accurate as the best linear function on the joint feature space (and more general classes of functions). 

  \begin{theorem}[Informal statement of Theorem \ref{thm:sublinear-to-sublinear}]
      Fix any triple of hypothesis classes $\cH_A,\cH_B,$ and $\cH_J$. Suppose $\cH_A$ and $\cH_B$ consist of functions with bounded range and admit efficient online algorithms guaranteeing no external regret with respect to $\cH_A$ and $\cH_B$ respectively. If $\cH_A$ and $\cH_B$ satisfy the weak learning condition with respect to $\cH_J$, and the conversation length $K$ is sublinear in $T$ (but not constant), then there is an efficient collaboration protocol such that:
      \[
      \sum_{t=1}^T (\hat y^{t} - y^t)^2 - \min_{h_J\in\cH_J} \sum_{t=1}^T (h_J(x^t) - y^t)^2 \leq o(T)
      \]
      In particular, this is true for the classes of norm-bounded linear functions. 
  \end{theorem}

\subsubsection{Tightness of Our Approach}
In \Cref{sec:lower-bounds} we give several lower bounds intended to illustrate the tightness of various aspects of our approach, answering several questions:

\paragraph{Is interaction necessary?} Perhaps for sufficiently simple classes of functions (e.g. linear functions) that satisfy our weak learning condition, no interaction is necessary --- maybe the optimal linear predictor on $\cX_A$ and $\cX_B$ already contains enough information to compete with the best linear predictor on the full feature space. We show that this is not the case, by exhibiting a lower bound instance (\Cref{thm:interaction-necessity}) such that the Bayes optimal predictors $h^*(x_A), h^*(x_B)$, and $h^*(x)$ defined on $\cX_A$, $\cX_B$, and $\cX$ are all linear, and yet no function of $h^*(x_A)$ and $h^*(x_B)$ has squared error competitive with $h^*(x)$. 

\paragraph{Is our weak learning condition necessary?} Can we relax our weak learning condition? We show that the answer is no, at least for any similar approach. Our boosting theorem demonstrates that the weak learning condition is sufficient for no swap regret with respect to $\cH_A$ and $\cH_B$ to imply no external regret with respect to $\cH_J$. We give a lower bound instance (\Cref{thm:weak-learning-necessity}) showing that it is also necessary: for any triple of function classes $\cH_A,\cH_B,\cH_J$ that fail to satisfy the weak learning condition, there is a distribution and a sequence of predictions such that the predictions have no swap regret with respect to $\cH_A$ and $\cH_B$ and yet have positive external regret with respect to $\cH_J$. We also show that our weak learning condition is strictly weaker than the ``information substitutes'' condition studied in \cite{bo2023agreement}, and that indeed linear functions do \emph{not} satisfy the information substitutes condition on all distributions (\Cref{lem:weak-is-weaker} and \Cref{cor:linear-weak-IS}).  

\paragraph{Is swap regret necessary?} Our collaboration protocol is designed to converge to a single sequence of predictions that has low (distance to) swap regret with respect to $\cH_A$ and $\cH_B$ simultaneously --- despite the fact that our ultimate goal is simply to have no \emph{external} regret with respect to $\cH_J$. Might it instead suffice to converge to a single sequence of predictions that has no \emph{external} regret with respect to $\cH_A$ and $\cH_B$? No. We give a lower bound instance (\Cref{lem:swap-necessary}) exhibiting that even for linear functions $\cH_A$ and $\cH_B$ (which satisfy the weak learning condition relative to linear functions on the joint feature space $\cH_J$), predictions that have no external regret with respect to $\cH_A$ and $\cH_B$ can still have positive external regret with respect to $\cH_J$. 

\subsubsection{A Decision Theoretic Extension for Higher Dimensional Outcome Spaces}
In Section \ref{sec:action} we give a decision theoretic extension of our setting to high dimensional outcome spaces. Now  the outcome space $\cY \subseteq [0,1]^d$ is $d$ dimensional, and we model a decision maker with a finite action space $\cA$ and a utility function $u:\cA\times \cY \rightarrow [0,1]$ that maps an action and an outcome to a utility. The natural extension of our one-dimensional solution to a $d$-dimensional outcome space---by asking for swap regret with respect to outcome predictions themselves---would inherit exponential dependencies on $d$. We circumvent this difficulty by not  communicating predictions $\hat y$ of the outcome $y \in \cY$ itself. Instead, in each round $k$, parties produce predictions $\hat y^{t,k} \in \cY$ but   communicate only \emph{actions} $a^{t,k} \in \cA$ that are utility maximizing given their predictions: $a^{t,k} = \BR_u(\hat y^{t,k})$ where the best response function is defined as:
$$ \BR_u(\hat y) = \arg\max_{a \in \cA} u(a,\hat y)$$
We use a definition of decision calibration sufficient to guarantee swap regret of the best response actions first used by \cite{noarov2023high}, generalizing the original definition given by \cite{zhao2021calibrating} (the definition from \cite{zhao2021calibrating} does imply swap regret bounds). 
\begin{definition}[Decision Calibration (Definition \ref{def:decision-calibration})]
 Fix an action space $\cA$ and a utility function $u:\cA\times \cY \rightarrow [0,1]$. 
 A sequence of outcome predictions $\{\hat y^{1,k},\ldots \hat y^{T,k}\}$ is decision calibrated  if for every action $a \in \cA$:
 $$\left\|\sum_{t=1}^T \mathbbm{1}[\BR_u(\hat y^{t,k}) = a](\hat y^{t,k} - y) \right\| = 0$$
 \end{definition}
We also use a definition of decision \emph{cross-calibration} first used by \cite{lu2025sample}:
\begin{definition}[Decision Cross Calibration (Definition \ref{def:decision-cross-calibration})]
 Fix an action space $\cA$, a utility function $u:\cA\times \cY \rightarrow [0,1]$, and a class of benchmark policies $\cC$ containing functions $c:\cX\rightarrow \cA$ mapping contexts to actions.  
 A sequence of outcome predictions $\{\hat y^{1,k},\ldots \hat y^{T,k}\}$ is decision cross-calibrated with respect to $\cC$ if for every pair of actions $a, a' \in \cA$ and for every $c \in \cC$:
 $$\left\|\sum_{t=1}^T \mathbbm{1}[\BR_u(\hat y^{t,k}) = a]\mathbbm{1}[c(x^t) = a'](\hat y^{t,k} - y) \right\| = 0$$
 \end{definition}
 If a sequence of predictions $\{\hat y^{1,k},\ldots \hat y^{T,k}\}$ is simultaneously decision calibrated and decision cross-calibrated with respect to $\cC$, then the corresponding sequence of actions $a^{t,k} = \BR_u(\hat y^{t,k})$ have no swap regret with respect to $\cC$ --- i.e. for every $c \in \cC$ and for every $a \in \cA$:
 $$\sum_{t=1}^T \mathbbm{1}[a^{t,k} = a]u(a^{t,k},y^t) \geq \max_{c \in \cC}\sum_{t=1}^T \mathbbm{1}[a^{t,k}=a]u(c(x^t),y^t)$$
 We ask that both Alice and Bob are decision calibrated and decision cross calibrated conditional on the action that the other communicated at the previous round --- which implies that both parties have no conversation swap regret with respect to $\cC_A$ and $\cC_B$ respectively on their own rounds. It also allows us to invoke a fast agreement theorem from \cite{collina2025tractable} which lets us establish fast convergence to a round of predicted actions that simultaneously has no swap regret to $\cC_A$ and $\cC_B$. This lets us apply a similar boosting theorem to the one we develop in Section \ref{sec:boost} to establish that the sequence of \emph{actions} $a^1,\ldots,a^T$ that result from the collaboration protocol have no regret with respect to a collection of action policies defined on the joint feature space. Our regret bounds and our communication requirements depend only polynomially on the dimension of the outcome space and the cardinality of the action space.

\subsubsection{Simpler Algorithms in the Batch Setting}
In the bulk of this paper, we study the collaborative learning problem in the difficult online adversarial setting, in which examples are assumed to arrive adversarially. Of course the problem is still interesting in the more standard \emph{batch} setting, in which examples $(x,y)$ are assumed to be drawn i.i.d. from a fixed but unknown distribution. In Section \ref{sec:batch} we give a simpler algorithm for this setting, which can be viewed as a two-party generalization of the ``level-set boosting'' algorithm given in \cite{globus-harris2023multicalibration}. This algorithm is a reduction to the problem of squared error regression over the classes $\cH_A$ and $\cH_B$ respectively; we prove fast convergence and out-of-sample generalization theorems for it. Our algorithm in this setting uses \emph{test-time} compute to make predictions on new instances: the two parties engage in a polynomial-length interaction exchanging and updating predictions about each test-time instance before agreeing on a final prediction.

\subsubsection{Lifting to the One-Shot Bayesian Setting}
Finally in Section \ref{sec:bayes} we show that the theorems we prove in the online adversarial section can be ``lifted'' to the one-shot Bayesian setting in which agreement theorems have been traditionally studied \citep{aumann1976,geanakoplos1982we,aaronson2004complexity,kong2023false,bo2023agreement}. This is, informally, because Bayesians with correct priors have beliefs that are unbiased conditional on \emph{any} event, and in particular their predictions are guaranteed in expectation to have no conversation swap regret with respect to any fixed collection of benchmark functions. For any class of benchmark functions for which empirical squared error converges uniformly to expected squared error (e.g. any class of functions with bounded fat shattering dimension) this means that they are guaranteed to satisfy the conditions of our boosting theorems on any sufficiently long sequence of instances drawn from a known prior. Thus we can imagine that Bayesian Alice and Bayesian Bob engage in an interaction for an arbitrarily long sequence of examples drawn i.i.d. from their commonly shared prior, and apply our theorems to bound the accuracy of the predictions that result along this imagined sequence. But  when examples are drawn i.i.d. from a fixed prior the final predictions at each day in this imagined sequence will also be i.i.d. Thus our theorems, which generically apply to the average error of predictions over a sequence, actually in this case apply to the expected squared error of the predictions that result from the collaboration protocol on the \emph{first} day of the sequence, and hence apply in the one-shot setting. 

The result is new information aggregation theorems in the classical Bayesian setting. In comparison to the information aggregation theorem given by \cite{bo2023agreement}, our theorem is of an ``agnostic learning'' sort: \cite{bo2023agreement} assume an ``information substitutes'' condition on the Bayes optimal predictors on $\cX_A,\cX_B$, and $\cX$ respectively, and under this assumption show that agreement implies Bayes optimality. In contrast, our theorem makes no assumption on the underlying distribution $\cD$ at all, and implies that for every data distribution, agreement implies accuracy at least that of the best predictor in any benchmark class of functions that satisfy our weak learning condition and has bounded fat-shattering dimension. This includes bounded norm linear functions among other things.

\subsection{Related Work}
\paragraph{Agreement.} 
Aumann's classic ``agreement theorem'' \citep{aumann1976} states that two Bayesians with a common and correct prior, who have \emph{common knowledge} of each other's posterior expectation of any predicate must have the same posterior expectation of that predicate. ``Common Knowledge'' is the limit of an infinite exchange of information, but Geanakoplos and Polemarchakis \citep{geanakoplos1982we} showed that whenever the underlying state space is finite, then agreement occurs after a finite number rounds in which the information exchanged in each round is the posterior expectation of each party. Aaronson \citep{aaronson2004complexity} showed  that for 1-dimensional expectations, $\epsilon$-approximate agreement can be obtained (with probability $1-\delta$ over the draw from the prior distribution) after the parties exchange only $O(1/\epsilon^2\delta)$ messages. Two papers \citep{kong2023false,bo2023agreement} study conditions under which Aumannian agreement implies information aggregation --- i.e. when ``agreement'' is reached at the same posterior belief that would have resulted had the two parties shared all of their information, rather than interacting within an agreement protocol. These papers all assume perfect Bayes updates based on a correctly specified and commonly known prior distribution, and so in general do not correspond to computationally tractable algorithms. \cite{collina2025tractable} generalizes \cite{aaronson2004complexity} and proves agreement theorems without making any distributional assumptions (i.e. in an online adversarial setting as in this paper), and using tractable calibration conditions that relax Bayesian rationality --- but says nothing about information aggregation. Our paper extends the work of \cite{collina2025tractable} to be able to give information-aggregation like statements in an online adversarial setting --- in particular, regret bounds with respect to a class of models defined on the joint feature space across the two parties. When applied to the Bayes optimal predictors, our ``weak learning'' condition is strictly weaker than the ``information substitutes'' condition given by \cite{bo2023agreement}, and our weak learning condition can be applied to any other class of models (not necessarily Bayes optimal). Our results can be lifted back to the Bayesian setting of \citep{aumann1976,aaronson2004complexity,bo2023agreement} to give new information aggregation theorems.

\paragraph{Vertically Federated Learning.} Vertically federated learning (see e.g. \cite{wei2022vertical}) studies distibuted learning problems in which features are distributed amongst parties, just as we do. The goal in this literature is to simulate learning on the shared feature space without sharing the data in the clear. Standard techniques in this literature involve running stochastic gradient descent over the full feature space over a cryptographic substrate --- see e.g \cite{hardy2017private} who give an algorithm for solving logistic regression over the joint feature space using additively homomorphic encryption and \cite{cheng2021secureboost} who give similar results for tree based models. In contrast to this line of work, our protocols require only learning on one's own data and communicating only predictions. This is what allows us to lift our results to the Bayesian agreement setting (all of the learning conditions we need are satisfied by Bayesian reasoners),  gives us protocols whose communication complexity is independent of the data dimension, and gives our protocols the form of direct reductions from multi-party learning to single-party learning, with no cryptographic overhead.

\paragraph{Calibration, Ensembling, and Boosting.} Beyond \cite{collina2025tractable} which replaces the assumption of Bayesian rationality with tractable calibration conditions in the context of Aumann's agreement theorem, several papers \citep{camara2020mechanisms,collina2023efficient} have replaced traditional assumptions of Bayesian rationality (and common prior assumptions) with calibration assumptions in \emph{principal agent} problems arising e.g. in contract theory and Bayesian Persuasion. In particular, \cite{collina2023efficient} shows how to do this with tractable decision calibration conditions. 

Our weak learning condition is a generalization of the weak learning condition given by \cite{globus-harris2023multicalibration}, which they showed characterizes when \emph{multicalibration} \citep{hebert2018multicalibration} with respect to one class of functions implies error optimality with respect to another. An important step in our analysis is that agents with ``conversation swap regret'' converge quickly to predictions that agree on most days, which we obtain by showing that conversation swap regret implies conversation calibration as defined in \cite{collina2025tractable}, which in turn implies fast agreement. The fact that swap regret with respect to squared loss implies low calibration error is a classical result originally due to \cite{FV99}. In the ``action setting'' in Section \ref{sec:action}, the conditions we require on each party are that they be decision calibrated and decision ``cross-calibrated'' with respect to a benchmark class of functions --- conditions that were recently used in \cite{lu2025sample}. These conditions are variants of decision calibration as studied by \cite{zhao2021calibrating,noarov2023high} and ``decision outcome indistinguishability'' as studied by \cite{gopalan2023loss}. We use the algorithm of \cite{noarov2023high} to constructively enforce these conditions. ``Cross calibration'' conditions have also been used to ensemble models in accuracy improving ways \citep{roth2023reconciling,alur2024human} --- but with the exception of \cite{globus-harris2023multicalibration} (which gives results in a single-party setting) these methods do not promise to compete with a benchmark class of models that is \emph{strictly} more accurate than the initial models.

\paragraph{Human-AI Collaboration.} The HCI literature on human-AI interaction has identified \emph{complementarity} as a core goal --- that a team consisting of a human and a model should perform measurably better than either of them could perform alone \cite{bansal2021does}. In particular, collaboration in the form of interaction is an explicit design goal \cite{gomez2025human}, although one that has been hard to realize. Peng, Garg, and Kleinberg \citep{peng2024no} prove a ``no-free-lunch'' theorem for human-AI collaboration, showing that for protocols that do not engage in an interaction (i.e. are just a post-processing of individual static predictors), non-trivial aggregation schemes (that do not always follow the prediction of a single model) must sometimes perform \emph{worse} than the worst single predictor. Other empirical and theoretical studies of human-AI collaboration with the goal of improving over the best individual model include \citep{green2019principles,donahue2022human,noti2025ai}. We give a protocol involving interaction (thus circumventing the barrier result proven by \cite{peng2024no}) that guarantees that a collaborative team can do strictly better than either alone. Additionally, common empirical approaches to human-AI collaboration often use insights into the model's reasoning through 'explanations' as a form of communication. However, empirical studies show mixed results \cite{bansal2021does,goh2024large}; explanations can sometimes be ineffective or even misleading, potentially hindering human understanding or team performance, particularly if the explanations themselves are flawed. Our framework explores a different pathway for collaboration, that circumvents the need for explanations by replacing them with sharing only predictions.
\paragraph{Multi-modal Learning.} Effectively integrating information across modalities like vision and language is a key challenge in multi-modal learning (see \cite{baltruvsaitis2018multimodal,li2024multimodal} and citations within). Standard techniques often involve either early fusion, combining representations before joint processing, or late fusion, typically averaging predictions from unimodal models. Early fusion may require complex joint models and careful feature alignment, while our theoretical results suggest simple late fusion can be suboptimal. In contrast, our protocols utilize iterative prediction or action exchange, requiring only learning on native data modalities. This mechanism avoids feature-level fusion entirely, enables communication complexity independent of data dimensionality, and represents a direct reduction to single-party learning, thus sidestepping the need for explicit feature alignment or joint model training overhead.  

\paragraph{Other Related Work.}
The setting we study, in which different parties hold different features about the same example and want to coordinate on a single learning task resembles \emph{co-training} as studied by \cite{blum1998combining,balcan2004co}. Models of co-training generally assume that the features each party hold are sufficient to learn a perfect model, but that labels are scarce: co-training protocols seek to use agreement with the other party as a regularization technique that allows them to learn with only small amounts of labeled data (together with larger amounts of unlabeled data). In contrast, our interest is in the setting in which each individual's features are not sufficient to learn an accurate model, and the goal is to collaboratively learn a model that is more accurate than could be learned by any party alone, even with arbitrarily many samples. \cite{blum2017collaborative} define \emph{collaborative learning}, later studied by \citep{haghtalab2022demand,donahue2021model,blum2021one,zhang2024optimal,peng2024no,haghtalab2023unifying}. In the collaborative learning setting, multiple parties have data from different distributions that are all labeled with the same function, and are interested in collaborating to learn their shared label function with fewer samples than it would take for each party to learn the function only from their own data. In contrast, in our setting, there is a single distribution (or no distribution, in the online adversarial setting), and it is the features that are distributed amongst parties.

\section{Preliminaries}
\label{sec:prelim}
We study a setting with two parties, Alice and Bob. Both parties are able to make predictions about a label not only given their observed features, but as a function of an interaction that they have had with their counterparty. With the exception of Sections \ref{sec:batch} and \ref{sec:bayes}, we consider the adaptive, online setting where Alice and Bob interact to make label predictions over a sequence of days $t= 1, \ldots, T$. We let $\cX_A$ and $\cX_B$ denote feature spaces for Alice and Bob, respectively, and we let $\cX=\cX_A\times\cX_B$ denote the joint feature space. We let $\cY$ represent the outcome (label) space which we will take to be $\cY = [0,1]$ for much of the paper, generalizing it to higher dimensions in Section \ref{sec:action}. 

On each day $t$, the parties converse for exactly $K$ rounds about their predictions of that day's outcome $y^t$ based on the features they each see: $x^t_A$ and $x^t_B$, respectively. At each round $k$ when they are speaking, an agent makes a prediction of the label, denoted $\ymk{t}{k}$ and $\yhk{t}{k}$ respectively. This prediction can be a function of everything the agent has observed so far --- the features relevant to the instance, the predictions sent by the other party, and past outcomes on previous days. 

They will alternate speaking, and we suppose that  Alice (Party 1) acts in odd numbered rounds; Bob (Party 2) acts in even numbered rounds. In an odd round $k$, Alice sends her prediction $\ymk{t}{k}$, and then in the next round $k+1$; Bob responds with a prediction $\yhk{t}{k+1}$. We use the subscript $\alice$ and $\bob$ for readability, so there is a clear distinction between Alice and Bob's messages when possible. However, since which party is speaking is simply a function of the parity of the round $k$, we can also write $\hat{y}^{t,k}$ as shorthand for $\hat{y}_B^{t, k}$ or $\hat{y}_A^{t,k}$ when the round $k$ is even or odd, respectively. 


We formalize the interaction between the two agents in Protocol \ref{alg:general-agreement}---a generic ``collaboration protocol."

\begin{protocol}[ht]
\begin{algorithmic}
    \STATE{ {\bf Input} $(\cX, \cY, K, T)$}
    \FOR{each day $t = 1, \ldots,T$}
        \STATE Receive $x^t = (x^t_A,x^t_B)$. Alice sees $x^t_A$ and Bob sees $x^t_B$.
        \FOR{each round $k = 1, 2, \ldots,K$}
            \IF { $k$ is odd}
                \STATE Alice predicts $\ymk{t}{k} \in \cY$, and sends Bob $\ymk{t}{k}$. 
            \ENDIF
            \IF{ $k$ is even}
                \STATE Bob  predicts $\yhk{t}{k} $, and sends Alice $\yhk{t}{k}$. 
                \ENDIF
        \ENDFOR
        \STATE{Alice and Bob observe $y^t \in \cY$.}
    \ENDFOR

\end{algorithmic}
\caption{\textsc{Online Collaboration Protocol}}  \label{alg:general-agreement}
\end{protocol}

We informally refer to the history of interaction \emph{within} any given day $t$ as a ``conversation." This is, the sequence of predictions exchanged by Alice and Bob specifically about the currently unknown label $y^t$. We refer to the history of interaction \emph{across} multiple days as a ``conversation transcript." It is an object that records the interactions between the agents and is visible to both, and which they can use to make their predictions.

\begin{definition}[Conversation Transcript $\pi^{1:T,1:K}$] \label{def:prediction-transcript}
    A conversation transcript $\pi^{1:T,1:K} \in \left\{ \cY^{K+1} \right\}^T $ is a sequence of tuples of predictions over rounds made by Alice and Bob (alternating across rounds), and the outcome, over $T$ days:
    \begin{align*}
        \pi^{1:T, 1:K} = \left\{
        \left(\ymk{1}{1}, \yhk{1}{2}, \ymk{1}{3}, \ldots \ymk{1}{K}, y^1\right), 
        \ldots,  
        \left(\ymk{T}{1}, \yhk{T}{2}, \ymk{T}{3}, \ldots \ymk{T}{K}, y^T\right)        \right\}.
    \end{align*}
    We define $\pi^{1:T: k}$ to be the restriction to only round $k$ of conversation across days as follows:

    \begin{align*}
        \pi^{1:T:k} = \begin{cases}
        \{ (\yak{t}{k}, y^t) \}_{t \in [T]} & \text{if } k \text{ is odd,}\vspace{1ex}\\
        \{(\ybk{t}{k}, y^t)\}_{t \in [T]}& \text{otherwise.}
        \end{cases}
    \end{align*}


\end{definition}

We will use the notation $\pi^{1:T}$ to refer to a single sequence of predictions over $T$ days, outside the context of a conversation.

\begin{definition}[Prediction Transcript $\pi^{1:T}$]
    A prediction transcript $\pi^{1:T} \in \left\{ \cY^2 \right\}^T $ is a sequence of tuples of predictions and outcomes over $T$ days:
    \[
    \pi^{1:T} = \left\{
        \left(\hat{y}^1, y^1\right), 
        \ldots,  
        \left(\hat{y}^T, y^T\right)        \right\}
    \]
\end{definition}

\subsection{Information Aggregation}

Our focus is on giving algorithms in this collaborative learning setting that give strong information aggregation guarantees, in the sense that the parties, using only their own sets of features individually, converge on predictions that are optimal with respect to a benchmark class of predictors is defined with respect to \emph{both} parties' features. 

In order to state such guarantees, we need to define a benchmark class. We first define the class of benchmark functions that map each of Alice and Bob's features, individually, to predictions, and then the class of benchmark functions defined on their joint feature space.

\begin{definition}[Individual Hypothesis Classes $\cH_A,\cH_B$]
\label{def:hypothesis}
    Let $\cH_A: \{h:\cX_A \mapsto \mathbb{R}\}$  be a set of functions mapping from Alice's feature set to $\mathbb{R}$. Analogously, let $\cH_B: \{h:\cX_B \mapsto \mathbb{R}\}$ be a set of functions mapping from Bob's features to $\mathbb{R}$.
\end{definition}

\begin{definition}[Joint Hypothesis Class $\cH_J$]
\label{def:joint-hypothesis}
    Let $\cH_J: \{h:\cX \mapsto \mathbb{R}\}$ be a set of functions mapping from the joint feature set $\cX = \cX_A \times \cX_B$ to $\mathbb{R}$. 
\end{definition}

For simplicity, it will be convenient for us to assume that the hypothesis classes $\cH_A,\cH_B,\cH_J$ contain constant functions (this is the case for most natural concept classes and is easy to enforce for any class for which it is not true originally)

\begin{assumption}\label{assumption:constant-functions}
    We assume that the hypothesis classes $\cH$ we work with contain the set of all constant functions $\{h(x)=v\}_{v\in[0,1]}$. 
\end{assumption}

The goal of our collaboration protocol will be to guarantee that the sequence of predictions resulting from the interaction have error that is competitive with the best model in $\cH_J$. In service of this, we will leverage the ability of Alice and Bob to make predictions that have low swap regret with respect to their individual hypothesis classes $\cH_A$ and $\cH_B$ respectively:

\begin{definition}[$(f, \cH)$-Swap Regret]
\label{def:swap}
Fix an error function $f:\{1, \ldots, T\} \rightarrow \mathbb{R}$ and a hypothesis class $\cH$. A transcript $\pi^{1:T}$ has $(f, \cH)$-swap regret if: 
\begin{align*}
\sum_{t=1}^{T}(\hat{y}^t - y^{t})^{2} -
    \sum_{v}\min_{h \in \cH}\left( \sum_{t = 1}^{T}\1[\hat{y}^t = v](h(x^{t}) - y^{t})^{2}\right) \leq f(T) 
\end{align*}
Here $v$ ranges over values of the predictions: $v \in \{\hat y^1,\ldots,\hat y^T\}$.
\end{definition}

It will also be useful to have a notion of \emph{distance} to swap regret. Distance to swap regret, which we define below, is analogous to the recently defined measure of \emph{distance to calibration} \cite{blasiok2023unifying}. A sequence of predictions has low distance to swap regret, informally, if they are close (in $\ell_1$ distance) to a sequence of predictions that itself has low swap regret. 

\begin{definition}[$(q, f, \cH)$-Distance to Swap Regret]
Fix an error functions $f, q:\{1, \ldots, T\} \rightarrow \mathbb{R}$ and a hypothesis class $\cH$. Let $Q_{f, \cH}$ be the set of prediction sequences $p^{1:T}$ that have $(f, \cH)$-swap regret. A transcript $\pi^{1:T}$ has $(q, f, \cH)$-distance to swap regret if:
\begin{align*}
\min_{p^{1:T} \in Q_{f, \cH}}||\hat{y}^{1:T} - p^{1:T}||_{1} \leq q(T)
\end{align*}
\end{definition}

\subsection{Conversation Swap Regret}
Our collaboration protocols involve ``conversations'' over $k$ rounds. An important condition for us in our construction is called ``conversation swap regret'', which informally requires that the predictions that Alice (resp. Bob) make at each round of conversation have no swap regret with respect to $\cH_A$ (resp. $\cH_B$) not just marginally, but \emph{conditionally} on the prediction that their counter-party made at the round before. 

\begin{definition}[$(f, g, \cH)$-Conversation Swap Regret]\label{def:CSR}
Fix an error function $f:\{1, \ldots, T\} \rightarrow \mathbb{R}$, a bucketing function $g:\{1, \ldots, T\} \rightarrow \mathbb{R}$ and a prediction class $\cH_{B}$. Let $v$ range over the values $\in \{\hat y^1_k,\ldots,\hat y^T_k\}$. Given a conversation transcript $\pi^{1:T, 1:K}$ from an interaction in the Collaboration Protocol (Protocol \ref{alg:general-agreement}), Bob has $(f, g, \cH_{B})$-swap regret if for all even rounds $k$ and buckets $i \in \left\{1,\ldots,\frac{1}{g(T)}\right\}$: 
\begin{align*}
\sum_{t \in T_{A}(k-1,i)}(\hat{y}^{t,k} - y^{t})^{2} -
    \sum_{v}\min_{h \in \cH_{B}}\left( \sum_{t \in T_{A}(k-1,i)}\mathbb{I}[\hat{y}^{t,k} = v](h(x^{t}) - y^{t})^{2}\right) \leq f(|T_{A}(k-1,i)|).
\end{align*}

Where $T_{A}(k-1,i)=\{t:\hat{y}^{t,k-1}\in[(i-1)g(T), ig(T))\}$ is the subsequence of days where the predictions of Alice in round $k-1$ fall in bucket $i$.

If Alice satisfies a symmetric condition on odd rounds $k$ with respect to $\cH_{A}$, we say that Alice has $(f, g, \cH_{A})$-Conversation Swap Regret with respect to $\cH_{A}$.  
\end{definition}

\begin{assumption}
    We assume that all error functions $f(\cdot)$ are concave.
\end{assumption}


\section{Boosting for Collaboration}
\label{sec:boost}

In this section we give a weak learning condition that characterizes when swap regret guarantees with respect to $\cH_A$ and $\cH_B$ on a single sequence of predictions imply regret guarantees with respect to a richer hypothesis class $\cH_J$. We also show that linear functions (and substantial generalizations) over $\cX_A$ and $\cX_B$ indeed satisfy the weak learning condition with respect to $\cH_J$, linear functions over the joint feature space $\cX$. This justifies the algorithmic approach we pursue in Sections \ref{sec:online-alg}, \ref{sec:action} and \ref{sec:batch}, giving collaboration protocols whose aim is to arrive at a sequence of predictions that have no swap regret with respect to both $\cH_A$ and $\cH_B$ --- the final accuracy guarantees in those sections will then follow from applying the boosting theorem we will prove here.




We first state our weak learning condition, which roughly speaking requires that on every distribution, if there is any model in $\cH_J$ that is able to obtain error lower than that of a constant predictor (by any margin $\gamma$), then there must also be a model in either $\cH_A$ or $\cH_B$ that can obtain error better than a constant predictor (by some smaller margin $w(\gamma)$). This is a generalization of a condition given in \cite{globus-harris2023multicalibration} in the context of studying the boosting properties of multicalibration. Our definition below generalizes that of \cite{globus-harris2023multicalibration} to multiple parties, and to a general margin function $w$ (rather than just a linear function $w(\gamma) = \gamma$ as stated in \cite{globus-harris2023multicalibration}). This generalization is important because as we will see, linear functions satisfy the weak learning condition only with the margin $w(\gamma) = \Theta(\gamma^2)$.


\begin{definition}[$w(\cdot)$-Weak Learning Condition]
\label{def:joint-weak}
    Let $\cH_A = \{h_A:\cX_A\to\cY\}$ and $\cH_B = \{h_B:\cX_B\to\cY\}$ be hypothesis classes over $\cX_A$ and $\cX_B$ respectively. Let $\cH_J$ be a hypothesis class of over the joint feature space $\cX = \cX_A \times \cX_B$. Let $w:[0,1]\to[0,1]$ be a strictly increasing, continuous, convex function that satisfies $w(\gamma)\leq \gamma$. We say that $\cH_A$ and $\cH_B$ jointly satisfy the $w(\cdot)$-weak learning condition with respect to $\cH_J$ if for any distribution $\cD$ over $\cX_A \times \cX_B \times \cY$, and any $\gamma\in[0,1]$, if:
    \[
     \min_{c \in \Rset} \E_\cD[(c-y)^2] - \min_{h_J \in \cH_J} \E_\cD[(h_J(x)-y)^2] \ge \gamma,
    \]

    \noindent then there exists either $h_A \in \cH_A$ or $h_B \in \cH_B$ such that:
    \[
    \min_{c \in \Rset}\E_\cD[(c-y)^2] - \E_\cD[(h_A(x_A)-y)^2] \ge w(\gamma)
    \]
    or:
    \[
    \min_{c \in \Rset} \E_\cD[ (c-y)^2] - \E_\cD[(h_B(x_B)-y)^2] \ge w(\gamma)
    \]
\end{definition}


\begin{remark}
We note that the conditions that $w$ is convex and satisfies $w(\gamma)\leq \gamma$ is without loss. Indeed, if $\cH_A$ and $\cH_B$ jointly improve by a margin $w'$ that is non-convex, there exists a convex function $w$ such that $w(\gamma)\leq w'(\gamma)$ for all $\gamma\in[0,1]$, and thus, $\cH_A$ and $\cH_B$ also jointly improve by the margin $w$. Similarly, if $w(\gamma) > \gamma$ for some $\gamma$ --- i.e. $\cH_A$ and $\cH_B$ jointly improve by more than $\gamma$ --- then they certainly improve by at least $\gamma$. We impose these conditions for technical reasons in the proof of Theorem \ref{thm:weak-learning}. 
\end{remark}

We now state our ``boosting'' theorem. In fact, we will not need that our predictions have low swap regret --- it will suffice that they have low \emph{distance to swap regret}, which will be an easier condition to obtain. If we have a single sequence of predictions such that those predictions have low distance to swap regret with respect to $\cH_A$ \emph{and} $\cH_B$, and $\cH_A$ and $\cH_B$ satisfy our weak learning condition with respect to a stronger joint class of functions $\cH_J$, then in fact the sequence of predictions has no regret with respect to $\cH_J$. 

\begin{theorem}\label{thm:weak-learning}
    Let $\cH_J$ be a hypothesis class over the joint feature space $\cX$. Let $\cH_A = \{h_A:\cX_A\to\cY\}$ and $\cH_B = \{h_B:\cX_B\to\cY\}$ be hypothesis classes over $\cX_A$ and $\cX_B$ respectively. Let $\cD\in\Delta(\cX\times\cY)$ be the empirical distribution over a sequence $(x^t,y^t)_{t=1}^T$.
    If:
    \begin{itemize}
        \item Predictions $\hat{y}^{1:T}$ have $(q, f, \cH_A\cup\cH_B)$-distance to swap regret over $\cD$, and
        \item $\cH_A$ and $\cH_B$ jointly satisfy the $w(\cdot)$-weak learning condition with respect to $\cH_J$
    \end{itemize}
    Then: 
    \[
    \E_\cD[(\hat y - y)^2] - \min_{h_J\in\cH_J}\E_\cD[(h_J(x)-y)^2] \leq 2w\inv\left(\frac{f(T)}{T}\right) + 3\frac{q(T)}{T}
    \]
    whenever the inverse of $w$ exists.
\end{theorem}

We first show that if our predictions have \textit{no} distance to swap regret, then the weak learning condition implies low external regret with respect to $\cH_J$. We will then argue that perturbing the predictions by a small amount cannot increase external regret by very much. 

\begin{lemma}\label{lem:weak-learning-nodist}
    Let $\cH_J$ be a hypothesis class over the joint feature space $\cX$. Let $\cH_A = \{h_A:\cX_A\to\cY\}$ and $\cH_B = \{h_B:\cX_B\to\cY\}$ be hypothesis classes over $\cX_A$ and $\cX_B$ respectively. Let $\cD\in\Delta(\cX\times\cY)$ be the empirical distribution over a sequence $(x^t,y^t)_{t=1}^T$. 
    If:
    \begin{itemize}
        \item Predictions $\hat{y}^{1:T}$ have $(f, \cH_A\cup\cH_B)$-swap regret over $\cD$, and
        \item $\cH_A$ and $\cH_B$ jointly satisfy the $w(\cdot)$-weak learning condition with respect to $\cH_J$
    \end{itemize}
    Then: 
    \[
    \E_\cD[(\hat y - y)^2] - \min_{h_J\in\cH_J}\E_\cD[(h_J(x)-y)^2] \leq 2w\inv \left(\frac{f(T)}{T}\right)
    \]
    whenever the inverse of $w$ exists. 
\end{lemma}


\begin{proof}
    We show the contrapositive. Suppose there exists $h_J\in\cH_J$ such that:
    \[
    \frac{1}{T} \sum_v \sum_{t=1}^T \1[\hat{y}^t=v] (h_J(x^t) - y^t)^2 < \frac{1}{T} \sum_v \sum_{t=1}^T \1[\hat{y}^t=v] (v -y^t)^2 - 2w\inv\left(\frac{f(T)}{T}\right)
    \]
    Since a swap benchmark is only stronger, there exists a collection $\{h_{J,v}\}_v\subseteq \cH_J$ such that: $$\frac{1}{T} \sum_v \sum_{t=1}^T \1[\hat{y}^t=v] (h_{J,v}(x^t) - y^t)^2 \leq \frac{1}{T} \sum_v \sum_{t=1}^T \1[\hat{y}^t=v] (h_J(x^t) - y^t)^2$$ and thus:
    \[
    \frac{1}{T} \sum_v \sum_{t=1}^T \1[\hat{y}^t=v] (h_{J,v}(x^t) - y^t)^2 < \frac{1}{T} \sum_v \sum_{t=1}^T \1[\hat{y}^t=v] (v -y^t)^2 - 2w\inv\left(\frac{f(T)}{T}\right)
    \]
    Let $S_v=\{t: \hat{y}^t = v\}$ be the level set corresponding to the subset of the domain that the prediction is $v$. Let $\bar{y}_v = \frac{1}{|S_v|}\sum_{t=1}^T \1[\hat{y}^t=v] y^t$ be the label mean of this subset. 
    By Assumption \ref{assumption:constant-functions}, $\cH_A\cup\cH_B$ contains the set of all constant functions in $[0,1]$. Let $\cH_C \subset \cH_A\cup\cH_B$ denote the set of constant functions. 
    Since, for every $v$, $h_c(x)=\bar{y}_v$  is the constant function that minimizes squared error, and $\hat{y}^{1:T}$ has $(f, \cH_C)$-swap regret, we have that the average swap regret with respect to $\cH_C$ is bounded by:
    \begin{align*}
        &\frac{1}{T} \sum_v \sum_{t=1}^T \1[\hat{y}^t=v] (v-y^t)^2 - \frac{1}{T} \sum_v \sum_{t=1}^T \1[\hat{y}^t=v] (\bar{y}_v-y^t)^2 \\
        &= \frac{1}{T} \sum_v \sum_{t=1}^T \1[\hat{y}^t=v] (v-y^t)^2 - \frac{1}{T} \sum_v \min_{h_c\in\cH_c}\sum_{t=1}^T \1[\hat{y}^t=v] (h_c(x^t)-y^t)^2 \\
        &\leq \frac{f(T)}{T} \\
        &\leq w\inv\left( \frac{f(T)}{T} \right)
    \end{align*}
    In the last step, we use the fact that $w(\gamma)\leq \gamma$, and so $\gamma\leq w\inv(\gamma)$.
    Then, since the squared error of $\{h_{J,v}\}_v$ is less than the squared error of $\hat{y}^{1:T}$, and the squared error of $\hat{y}^{1:T}$ is close to the squared error of the label mean $\bar{y}_v$ on each level set, we have that:
    \begin{align*}
        \frac{1}{T} \sum_v \sum_{t=1}^T \1[\hat{y}^t=v] (h_{J,v}(x^t) - y^t)^2 
        &< \frac{1}{T} \sum_v \sum_{t=1}^T \1[\hat{y}^t=v] (v-y^t)^2 - 2w\inv\left(\frac{f(T)}{T}\right) \\
        &\leq \frac{1}{T} \sum_v \sum_{t=1}^T \1[\hat{y}^t=v] (\bar{y}_v-y^t)^2 + w\inv\left(\frac{f(T)}{T}\right) - 2w\inv\left(\frac{f(T)}{T}\right) \\
        &= \frac{1}{T} \sum_v \sum_{t=1}^T \1[\hat{y}^t=v] (\bar{y}_v-y^t)^2 - w\inv\left(\frac{f(T)}{T}\right)
    \end{align*}
    Letting $$\gamma_v = \frac{1}{|S_v|}\sum_{t=1}^T \1[\hat{y}^t=v] (\bar{y}_v-y^t)^2 - \frac{1}{|S_v|}\sum_{t=1}^T \1[\hat{y}^t=v] (h_{J,v}(x^t)-y^t)^2,$$ 
    we can rewrite the expression above as:
    \begin{align*}
        &\frac{1}{T} \sum_v \sum_{t=1}^T \1[\hat{y}^t=v] (\bar{y}_v-y^t)^2 - \frac{1}{T} \sum_v \sum_{t=1}^T \1[\hat{y}^t=v] (h_{J,v}(x^t)-y^t)^2 \\
        &= \frac{1}{T} \sum_v |S_v| \cdot \frac{1}{|S_v|} \sum_{t=1}^T \1[\hat{y}^t=v] (\bar{y}_v-y^t)^2 - \frac{1}{T} \sum_v |S_v| \cdot \frac{1}{|S_v|} \sum_{t=1}^T \1[\hat{y}^t=v] (h_{J,v}(x^t)-y^t)^2 \\
        &= \frac{1}{T} \sum_v |S_v| \gamma_v \\
        &> w\inv\left(\frac{f(T)}{T}\right)
    \end{align*}
    Observe that since $\cH_J$ contains the set of all constant functions (Assumption \ref{assumption:constant-functions}), there is always a choice of $\{h_{J,v}\}_v$ such that $\gamma_v$ is non-negative for all $v$.

    Thus, by the $w(\cdot)$-weak learning condition applied to the empirical distribution over the sequence on which $\hat{y}^t = v$ for any level set $v$, if $h_{J,v}$ improves over the best constant prediction $\bar{y}_v$ by $\gamma_v$, there is some $h_v\in \cH_A\cup\cH_B$ that improves over $\bar{y}_v$ by $w(\gamma_v)$. 
    That is, there exists a collection $\{h_v\}\subseteq \cH_A\cup\cH_B$ such that:
    \begin{align*}
        & \frac{1}{T} \sum_v \sum_{t=1}^T \1[\hat{y}^t=v] (\bar{y}_v-y^t)^2 - \frac{1}{T} \sum_v \sum_{t=1}^T \1[\hat{y}^t=v] (h_v(x^t)-y^t)^2 \\
        &= \frac{1}{T} \sum_v |S_v| \cdot \frac{1}{|S_v|} \sum_{t=1}^T \1[\hat{y}^t=v] (\bar{y}_v-y^t)^2 - \frac{1}{T} \sum_v |S_v|\cdot \frac{1}{|S_v|} \sum_{t=1}^T \1[\hat{y}^t=v] (h_{v}(x^t)-y^t)^2 \\
        &\geq \frac{1}{T} \sum_v |S_v| w(\gamma_v) \tag{by the $w$-weak learning condition} \\
        &\geq w\left( \frac{1}{T} \sum_v |S_v| \gamma_v \right) \tag{by convexity of $w$ and Jensen's inequality} \\
        &> w\left( w\inv\left(\frac{f(T)}{T}\right) \right) \tag{by monotonicity of $w$} \\
        &= \frac{f(T)}{T}
    \end{align*}

    In particular, this implies that:
    \begin{align*}
        \sum_v \min_{h_v^*\in\cH_A\cup\cH_B} \sum_{t=1}^T \1[\hat{y}^t=v] (h_v^*(x^t)-y^t)^2 &\leq \sum_v \sum_{t=1}^T \1[\hat{y}^t=v] (h_v(x^t)-y^t)^2 \\
        &< \sum_v \sum_{t=1}^T \1[\hat{y}^t=v] (\bar{y}_v-y^t)^2 - f(T) \\
        &\leq \sum_v \sum_{t=1}^T \1[\hat{y}^t=v] (v-y^t)^2 - f(T) \\
        &= \sum_{t=1}^T (\hat{y}^t-y^t)^2 - f(T)
    \end{align*}
    Here, the third line follows from the fact that on level set $v$, the squared error of the constant prediction $v$ is at least the squared error of the best constant prediction $\bar{y}_v$. This violates the $(f, \cH_A\cup\cH_B)$-swap regret condition, which completes the proof. 
    
\end{proof}
We can now complete the proof by noting that squared error is Lipschitz in the predictions --- so perturbing predictions that have low swap regret to those that merely have low \emph{distance} to swap regret does not affect the final error bound by much:

\begin{proofof}{Theorem \ref{thm:weak-learning}}
     By definition of distance to swap regret, there is a sequence $p^{1:T}$ with $(f, \cH_A\cup\cH_B)$-swap regret such that $\|\hat{y}^{1:T}-p^{1:T}\|_1\leq q(T)$. Furthermore, by Lemma \ref{lem:weak-learning-nodist}, $p^{1:T}$ satisfies:
     \[
     \frac{1}{T}\sum_{t=1}^T (p^t-y^t)^2 - \min_{h_J\in\cH_J} \frac{1}{T}\sum_{t=1}^T (h_J(x^t) - y^t)^2 \leq 2w\inv\left(\frac{f(T)}{T}\right)
     \]
     Applying Lemma \ref{lem:bound_error_diff}, we can conclude:
     \begin{align*}
        &\frac{1}{T}\sum_{t=1}^T (\hat{y}^t-y^t)^2 - \min_{h_J\in\cH_J} \frac{1}{T}\sum_{t=1}^T (h_J(x^t) - y^t)^2 \\
        &\leq \frac{1}{T}\sum_{t=1}^T (q^t-y^t)^2 - \min_{h_J\in\cH_J} \frac{1}{T}\sum_{t=1}^T (h_J(x^t) - y^t)^2 + 3\frac{q(T)}{T} \\
        &\leq 2w\inv\left(\frac{f(T)}{T}\right) + 3\frac{q(T)}{T}
     \end{align*}
     as desired.
\end{proofof}

\subsection{Function Classes Satisfying the Weak Learning Guarantee}
\label{sec:linear}
Next, we show that a broad set of function classes satisfy our weak learning assumption. What we require is that $\cH_A$ and $\cH_B$ be ``star shaped'' (i.e. closed under downward scaling), bounded, and closed under additive shifts, and that $\cH_J$ be representable as the Minkowski sum of $\cH_A$ and $\cH_B$ --- that is, for every $h_j \in \cH_J$ there should be $h_A \in \cH_A$ and $h_B \in \cH_B$ such that $h_J(x) = h_A(x_A) + h_B(x_B)$. In particular, the class of linear functions over the feature spaces of Alice and Bob respectively satisfy our weak learning assumption relative to linear functions on their joint feature space.

In order to define Alice and Bob's function classes, let us first define a few useful properties.
\begin{definition}[Bounded and star-shaped function class]
  For any class $\cF = \{f : \cX \rightarrow \mathbb{R}\}$ on domain $\cX$, we say it is
\begin{enumerate}
    \item \textit{$C$-bounded}: if there exists $C > 0$ such that $\sup_{f \in \cF, x \in \cX} |f(x)|  \le C$
    \item \textit{Star-shaped:} if $f \in \cF$ then $\alpha f \in \cF$ for all $0 \le \alpha \le 1$.
\end{enumerate}  
\end{definition}
Note that the function class of linear functions with bounded norms $\cF = \{x \mapsto \theta^\top x: \|\theta\|_2 \le C\}$  over bounded inputs $\cX = \{x \in \mathbb{R}^d: \|x\|_2 \le 1\}$ is $C$-bounded and star-shaped.

We now state our weak-learnability guarantees with respect to the Minowski sum of our base function classes satisfying the above properties.
\begin{theorem}\label{thm:linear-weak-learning}
Let $\cH_A = \{f_A + b_A : f_A \in \cF_A, b_A \in \mathbb{R}\}$ and $\cH_B = \{f_B + b_B : f_B \in \cF_B, b_B \in \mathbb{R}\}$ where $\cF_A = \{f_A: \cX_A \to \mathbb{R}\}$ and $\cF_B = \{f_B: \cX_B \to \mathbb{R}\}$ are $C$-bounded and star-shaped. Let $\cH_J = \{h_A + h_B : h_A \in \cH_A, h_B \in \cH_B\}$ be the Minkowski sum of $\cH_A$ and $\cH_B$. If $C \ge 1/2$, then $\cH_A$ and $\cH_B$ jointly satisfy the $w(\cdot)$-weak learning condition with respect to $\cH_J$ for:
$$w(\gamma) = \frac{\gamma^2}{16C^2}$$
\end{theorem}
The key idea is to show that if a function in the joint class $h_J(x) = h_A(x_A) + h_B(x_B)$ improves over the constant predictor then this translates to at least one of the base functions $h_A(x_A)$ or $h_B(x_B)$ having non-trivial correlation with the label $y$. Now appropriately choosing the scaling of the base function allows us to transfer this correlation to an improvement in squared loss over the constant predictor. This transfer is not exact and leads to the weaker $\gamma^2$ improvement, which we later show is actually tight!

\begin{proof}[Proof of \Cref{thm:linear-weak-learning}]
Consider a distribution $\cD$ over $\cX_A \times \cX_B \times \cY$ with $\mu = \E_\cD[y]$. Define $\bar{y} = y - \mu$ so that $\E_\cD[\bar{y}] = 0$, and for any predictor $h(x)$, define the centered predictor $\bar{h}(x) = h(x) - \mu$. The best constant predictor for predicting $\bar{y}$ is 0 with error $\E_\cD[\bar{y}^2]$, and for any predictor, $\E_\cD[(h(x)-y)^2] = \E_\cD[(\bar{h}(x)-\bar{y})^2]$.

To prove the weak learnability condition, assume that there exists $h_J(x) = h_A(x_A) + h_B(x_B) = f_A(x_A) + f_B(x_B) + b_A + b_B$ and its corresponding centered version $\bar{h}_J(x) = h_J(x) -\mu$ such that
\begin{align*}
\E_\cD[(\bar{h}_J(x)-\bar{y})^2] \leq \E_\cD[\bar{y}^2] - \gamma \implies - \E_\cD[(\bar{h}_J(x))^2] + 2\E_\cD[\bar{h}_J(x)\bar{y}]  \ge \gamma.
\end{align*}
Given that $\E_\cD[(\bar{h}_J(x))^2] \geq 0$, we have:
\begin{align*}
&\E_\cD[\bar{h}_J(x)\bar{y}] \geq \frac{\gamma}{2}\\
&\implies \E_\cD[(f_A(x_A) + f_B(x_B) + b_A + b_B - \mu)\bar{y}] \geq \frac{\gamma}{2}\\
&\implies  \E_\cD[f_A(x_A)\bar{y}] + \E_\cD[f_B(x_B)\bar{y}] \geq \frac{\gamma}{2},
\end{align*}
where the last inequality follows from the fact that $\E_\cD[\bar{y}] = 0$. This implies that either 
\begin{align*}
\E_\cD[f_A(x_A) \bar{y}] \geq \frac{\gamma}{4} \quad \text{or} \quad \E_\cD[f_B(x_B)\bar{y}] \geq \frac{\gamma}{4}.
\end{align*}

Without loss of generality, assume $\E_\cD[f_A(x_A)\bar{y}] \geq \frac{\gamma}{4}$. Now let us construct $h_A(x_A) = \alpha f_A(x_A) + \mu$ and corresponding centered predictor $\bar{h}_1(x_1) = \alpha f_A(x_A)$ for $\alpha = \frac{\gamma}{4 C^2}$. Note that $h_A \in \cH_A$ since $\alpha  \le 1$ (by assumption) and $\cF_A$ is star-shaped.

Now let us compute the error of $h_A(x_A)$.
\begin{align*}
\E_\cD[\bar{y}^2] - \E_\cD[(\bar{h}_A(x_A)-\bar{y})^2] &= -\E_\cD[(\bar{h}_A(x_A))^2] + 2\E_\cD[\bar{h}_A(x_A)\bar{y}] \\
&= -\E_\cD[(\alpha f_A(x_A))^2] + 2\E_\cD[\alpha f_A(x_A) \bar{y}] \\
&\ge -\alpha^2C^2 + \frac{\alpha \gamma}{2}= \frac{\gamma^2}{16C^2}.
\end{align*}
where the inequality follows from $C$-boundedness of $\cF_A$ and $\E_\cD[f_A(x_A)\bar{y}] \geq \frac{\gamma}{4}$. Removing the centering gives us,
\begin{align*}
\E_\cD[(\mu-y)^2] - \E_\cD[(h_A(x_A)-y)^2] = \E_\cD[\bar{y}^2] - \E_\cD[(\bar{h}_A(x_A)-\bar{y})^2] \geq \frac{\gamma^2}{16C^2}.
\end{align*}
\end{proof}

We next establish the tightness of various aspects of our theorem, both qualitatively and quantitatively. First, we have assumed that our function classes are bounded. This is necessary:
\begin{theorem}\label{thm:bounded-necessary}
There exists classes $\cF_A = \{f_A: \cX_A \to \mathbb{R}\}$ and $\cF_B = \{f_B: \cX_B \to \mathbb{R}\}$ that are star-shaped but unbounded over some domain $\cX_A, \cX_B$ such that $\cH_A = \{f_A + b_A : f_A \in \cF_A, b_A \in \mathbb{R}\}$ and $\cH_B = \{f_B + b_B : f_B \in \cF_B, b_B \in \mathbb{R}\}$ do not jointly satisfy $w(\cdot)$-weak learning with respect to $\cH_J = \{h_A + h_B: h_A \in \cH_A, h_B \in \cH_B\}$ for any strictly increasing $w$.
\end{theorem}
To prove this theorem, we construct a simple distribution ($\cX_A$ and $\cX_B$ are one-dimensional and $\cH_A$, $\cH_B$, and $\cH_J$ are linear functions) where both the features to the individual parties $x_A$ and $x_B$ have a small signal to noise ratio and hence cannot predict the label $y$ very accurately, but their difference can exactly cancel out the noise to recover a scaled down version of the signal. Now scaling it up can recover the label $y$ exactly. The signal to noise ratio for the individual parties is inversely proportional to the norm of the joint predictor, therefore we can make this arbitrarily small if the norms are allowed to be unbounded. 


Using the same construction, we establish that the quadratic dependence on the weak learning margin $w(\gamma) = \Theta(\gamma^2)$ cannot be improved. In particular, despite bounding the norm of the predictors, the perfect canceling of noise allows the joint predictor to do an order of magnitude better than any individual predictor on noisy features.
\begin{theorem}\label{thm:quadratic}
There exists classes $\cF_A = \{f_A: \cX_A \to \mathbb{R}\}$ and $\cF_B = \{f_B: \cX_B \to \mathbb{R}\}$ that are star-shaped and 1-bounded over some domain $\cX_A, \cX_B$ such that $\cH_A = \{f_A + b_A : f_A \in \cF_A, b_A \in \mathbb{R}\}$ and $\cH_B = \{f_B + b_B : f_B \in \cF_B, b_B \in \mathbb{R}\}$ do not jointly satisfy $w(\cdot)$-weak learning with respect to $\cH_J = \{h_A + h_B: h_A \in \cH_A, h_B \in \cH_B\}$ for any strictly increasing $w$ such that $w(\gamma) = \omega(\gamma^2)$.
\end{theorem}



\section{Collaboration in the Online Setting}
\label{sec:online-alg}
In Section~\ref{sec:boost}, we established that if $\cH_{A}$ and $\cH_{B}$ satisfy our weak learning condition with respect to $\cH_{J}$, a sequence of predictions that has low distance to swap regret with respect to $\cH_{A} \cup \cH_{B}$ has low external regret with respect to $\cH_{J}$. In this section, we show how to arrive at such a prediction sequence via a collaboration protocol. 

The high level idea of the proof is straightforward, but the details are surprisingly subtle.~\cite{collina2025tractable} defines a notion called Conversation Calibration in settings (such as our collaboration protocol) in which two parties engage in conversations about predictions of a real valued outcome. This notion is formally defined in Appendix~\ref{app:conv_calib}. Informally speaking, conversation calibration requires that at each round $k$ of the conversation, the sequence of predictions made over the $T$ days is unbiased relative to the outcomes, conditional both on the prediction made at round $k$ \emph{and} on the prediction made by the other party at round $k-1$. \cite{collina2025tractable} show that if both parties satisfy the conversation calibration condition across all rounds, then most conversations must quickly converge to approximate agreement. The conversation swap regret condition we require of our parties implies that the predictions also satisfy conversation calibration, and so the theorem of \cite{collina2025tractable} implies fast approximate agreement in our setting as well. The idea at a high level is that if Alice's predictions have no swap regret with respect to $\cH_A$ at every round, and Bob's predictions have no swap regret with respect to $\cH_B$ at every round, then when they agree, we will have a single sequence of predictions that has no swap regret with respect to both $\cH_A$ and $\cH_B$ simultaneously, exactly the condition that we need in order to invoke our boosting theorem.

However, several difficulties arise. First, the agreement theorem of \cite{collina2025tractable} states informally that conversations on most days must reach agreement quickly, but they might reach agreement at different rounds on different days. Just because the predictions at each round satisfy swap-regret guarantees does not mean that the sequence of final ``agreed upon'' predictions --- stitched together from different rounds at different days --- will have the same guarantee. To solve this problem, we use a different protocol than \cite{collina2025tractable}: rather than halting conversation at agreement, we continue each conversation for $K$ rounds even if agreement is reached earlier. We generalize the agreement theorem of \cite{collina2025tractable} to show that (even if it is not the final round), for sufficiently large $K$ there must \emph{exist} a round $k$ at which Alice's predictions at round $k$ are close to Bob's predictions at round $k-1$:



\begin{theorem} \label{thm:rounds}If Alice has $(f_{A},g_{A},\cH_{A})$-conversation swap regret and Bob has $(f_{B},g_{B},\cH_{B})$-conversation swap regret and they engage in a Collaboration Protocol (Protocol \ref{alg:general-agreement}) for $K$ rounds, then for any $\epsilon \in (0,1)$, there is at least one round $k$ such that \[ \frac{1}{T} \sum_{t=1}^{T} \mathbb{I}[|\hat{y}^{t,k} - \hat{y}^{t,k-1}| \geq \epsilon] \leq \frac{1}{2K\epsilon^{2}} + \frac{\beta(T,f_{A},f_{B})}{2\epsilon^{2}}
\] 
That is, the fraction of predictions in round $k$ that are $\epsilon$-away from those in round $k-1$ is at most $$\frac{1}{2K\epsilon^{2}} + \frac{\beta(T,f_{A},f_{B})}{2\epsilon^{2}}$$

Here, and for the other theorems following, we let  $\beta(T,f_{A},f_{B}) = \frac{f_{A}(g_{A}(T)\cdot T)}{Tg_{A}(T)} + \frac{f_{B}(g_{B}(T)\cdot T)}{Tg_{B}(T)} + g_{A}(T) + g_{B}(T)$, $f_{A}'(x) = \sqrt{x \cdot f_{A}(x)}$ and $f_{B}'(x) = \sqrt{x \cdot f_{B}(x)}$. 
\end{theorem}

The proof for this theorem (and all other theorems this section) can be found in Appendix~\ref{sec:main_app}. \\

If on Alice's rounds, she has low swap regret with respect to $\cH_A$ and on Bob's rounds, he has low swap regret with respect to $\cH_B$, then if on a pair of adjacent rounds, they made \emph{exactly} the same predictions, then on (both) of these rounds, the predictions would have no swap regret with respect to $\cH_A$ and $\cH_B$ simultaneously. Unfortunately Theorem \ref{thm:rounds} does not guarantee a pair of rounds on which Alice and Bob's predictions are exactly the same --- it only guarantees a pair of adjacent rounds on which the predictions are \emph{close} on \emph{most days}. Naively, this gives us two sequences, one of which has low swap regret with respect to $\cH_A$ and low distance to swap regret with respect to $\cH_B$, and the other of which has low swap regret with respect to $\cH_B$ and low distance to swap regret with respect to $\cH_A$. But to apply our boosting theorem, we need a single sequence of predictions that simultaneously has low distance to swap regret with respect to both $\cH_A$ and $\cH_B$. The following theorem (Theorem \ref{thm:main}) shows that in fact the round $k$ identified in Theorem \ref{thm:rounds} has this property:


\begin{theorem}\label{thm:main}If Alice has $(f_{A},g_{A},\cH_{A})$-conversation swap regret and Bob has $(f_{B},g_{B},\cH_{B})$-conversation swap regret, and they engage in a Collaboration Protocol (Protocol \ref{alg:general-agreement}) for $K$ rounds, then there exists a round $k$ of the protocol such that the transcript $\pi^{1:T,k}$ at round $k$ has $(q, f, \cH_{A} \cup \cH_{B})$-distance to swap regret, where 
\[q = \frac{T}{2}(g_{A}(T) + g_{B}(T))\] and \[ f= 8T\left(\frac{\beta(T,f_{A}',f'_{B}) + 1/K}{2}\right)^{\frac{1}{3}} + \frac{11}{2}T\beta(T,f_{A},f_{B})\]
\end{theorem}

We have thus established that there must be a sequence of predictions corresponding to \emph{some} round in the collaboration protocol which we can apply our boosting theorem to. However, this will not necessarily be the final round, and so 
the accuracy guarantees that we get from our boosting theorem will not necessarily apply to the final sequence of predictions. We show in the following theorem (Theorem \ref{thm:last-round}) that, while the final sequence of predictions do not necessarily have swap regret guarantees with respect to $\cH_A$ and $\cH_B$, it nevertheless has external regret guarantees with respect to $\cH_J$, the joint function class.

\begin{theorem}\label{thm:last-round} 
Let $\cH_J$ be a hypothesis class over the joint feature space $\cX$. Let $\cH_A = \{h_A:\cX_1\to\cY\}$ and $\cH_B = \{h_B:\cX_2\to\cY\}$ be hypothesis classes over $\cX_A$ and $\cX_B$ respectively. Consider some transcript $\pi^{1:T, 1:K}$ generated via the Collaboration Protocol (Protocol \ref{alg:general-agreement}) between Alice and Bob over $K$ rounds. If:
     \begin{itemize}
        \item Alice has $(f_{A}, g_{A}, \cH_A)$-conversation swap regret
        \item Bob has $(f_{B}, g_{B}, \cH_B)$-conversation swap regret
        \item $\cH_A$ and $\cH_B$ jointly satisfy the $w(\cdot)$-weak learning condition with respect to $\cH_J$ 
    \end{itemize}
    The transcript $\pi^{1:T,K}$ on the last round $K$ satisfies:
    \[
    \sum_{t=1}^T (\hat{y}^{t,k}-y^t)^2 - \min_{h_J\in\cH_J} \sum_{t=1}^T (h_J(x^t) - y^t)^2 \leq \\  \]
     \[
    2Tw\inv\left(8\left(\frac{\beta(T,f'_{A},f'_{B}) + 1/K}{2}\right)^{\frac{1}{3}} + \frac{11}{2}\beta(T,f_{A},f_{B})\right) + \frac{3}{2}T(g_{A}(T) + g_{B}(T)) + 3K\beta(T,f'_{A},f'_{B})
    \]

\end{theorem}

To prove this theorem, we apply our boosting theorem (Theorem~\ref{thm:weak-learning}) to the round $k$ identified in Theorem \ref{thm:main}, which establishes an external regret guarantee with respect to $\cH_J$ for the predictions made at round $k$. We then show that the swap regret conditions we assume of Alice and Bob also imply that the squared error cannot substantially increase at any subsequent round, which allows is to conclude that the error of our predictions at the final round $K$ is not much larger than it is at the round $k$ at which our boosting theorem applied. External regret (unlike swap regret) is monotone in the squared error of our predictions, which thus allows us to conclude that our final predictions satisfy the claimed external regret bound with respect to $\cH_J$.


\subsection{Reducing Conversation Swap Regret to External Regret}
\label{sec:algorithmic}

We have now established that two agents, engaging in our collaboration protocol, will arrive at predictions that have no external regret to $\cH_J$ if their predictions have no conversation swap regret with respect to classes $\cH_A$ and $\cH_B$ respectively. We now turn to reducing the algorithmic problem of engaging in our collaboration protocol with conversation swap regret guarantees with respect to a hypothesis class $\cH$ to the well studied problem of making predictions in an adversarial environment that simply have no \emph{external} regret with respect to $\cH$. \citet{garg2024oracle} give a generic reduction that efficiently transforms an algorithm guaranteeing no external regret with respect to $\cH$ into an algorithm that guarantees no \textit{swap} regret with respect to $\cH$. We in turn show how to transform any algorithm guaranteeing no swap regret with respect to $\cH$ into one that can engage in a collaboration protocol and guarantee no  \textit{conversation} swap regret with respect to $\cH$. \citet{collina2025tractable} use a similar reduction from conversation calibration to calibration. Together, this gives an efficient reduction from the problem of interacting with collaboration protocol \ref{alg:general-agreement} with no conversation swap regret guarantees (what is needed to invoke Theorem \ref{thm:last-round}) to the problem of making no (external) regret predictions. As we will see, whenever we start with an algorithm that guarantees sublinear external regret rates, we obtain an algorithm that guarantees sublinear conversation swap regret rates.

We begin by quoting the result of \citet{garg2024oracle}.

\begin{theorem}[Theorem 3.1 of~\cite{garg2024oracle}]\label{thm:external-to-swap}
Fix a hypothesis class $\cH$. If:
\begin{itemize}
    \item All $h \in \cH$ satisfy $h(x)^2\leq B$ for all $x\in\cX$
    \item $\cH$ has finite sequential fat-shattering dimension (Definition \ref{def:sequential-fat})
    \item There exists an efficient online algorithm producing predictions $\hat{y}^1,...,\hat{y}^T$ that achieve, for any sequence of outcomes $y^1,...,y^T$, external regret with respect to $\cH$ bounded by $r(T)$, i.e.:
    \[
    \sum_{t=1}^T (\hat{y}^t-y^t)^2 - \min_{h\in\cH} \sum_{t=1}^T (h(x^t)-y^t)^2 \leq r(T)
    \]
    where $r(T)$ is a concave function. 
\end{itemize}
 Then, for any $m > 0$, there exists an efficient online algorithm which, with probability $1 - \rho$, guarantees $(f,\cH)$-swap regret, where 
 \[
 f(T) \leq m \cdot r\left(\frac{T}{m}\right)+ \frac{3T}{m} + m + \max(8B,2\sqrt{B}) \cdot m \cdot C_{\cH} \cdot \sqrt{T\log\left(\frac{4m}{\rho}\right)}
 \]
 Here, $C_{\cH}$ is a constant that depends on the sequential fat-shattering dimension of $\cH$.
\end{theorem}

\begin{algorithm}[ht]
\begin{algorithmic}
    \STATE {\bf Input} External regret algorithm $M_0$, hypothesis class $\cH$, bucketing function $g$
    \vspace{0.5em}

    Let $M$ be the swap regret algorithm given by Theorem \ref{thm:external-to-swap}, when initiated with $M_0$.
    
    For every odd $k\in\{3,...,K\}$ and bucket $i\in\{1,...,1/g(T)\}$, instantiate a copy of $M$, called $M_{k,i}$. For the first round $k=1$, instantiate a copy of $M$, called $M_1$.

    Let $\pi^{1:t,k|i}$ denote the transcript on round $k$ up until day $t$, restricted to $\{t:\hat{y}^{t,k-1}\in[(i-1)g(T), ig(T))\}$, the subsequence where the previously communicated predictions falls into bucket $i$.

    Let $M(\pi^{1:t,k|i}, \cH)$ denote the output of $M$ given this transcript and hypothesis class $\cH$. 
    
    \FOR{each day $t = 1, \ldots, T$}
        \STATE Receive $x^t_A$
        \STATE Make prediction $\hat{y}^{t,1}_A = M_{1}(\pi^{1:t-1,1}, \cH)$
        \STATE Send to Bob $\hat{y}^{t,1}_A$
        \FOR{each odd round $k = 3, 5, \ldots,K$}
            \STATE Observe Bob's prediction from the previous round $\hat{y}^{t,k-1}_B$ and let $i$ be an integer such that $\hat{y}^{t,k-1}_B \in [(i-1)g(T), ig(T))$.
            \STATE Make prediction $\hat{y}^{t,k}_A = M_{k, i}(\pi^{1:t-1,k|i}, \cH)$
            \STATE Send to Bob $\hat{y}^{t,k}_A$
        \ENDFOR
        \STATE{Observe $y^t \in \cY$.}
    \ENDFOR
\end{algorithmic}
\caption{A reduction from a conversation swap regret algorithm to an external regret algorithm}  \label{alg:csr-algorithm}
\end{algorithm}

We formalize our reduction from conversation swap regret to external regret in Algorithm~\ref{alg:csr-algorithm} and prove its correctness in Theorem~\ref{thm:alg-conversation-sr}. We state the algorithm from the perspective of Alice; Bob's is symmetric. 
\begin{theorem}\label{thm:alg-conversation-sr}
    Fix a hypothesis class $\cH$. If:
    \begin{itemize}
        \item All $h \in \cH$ satisfy $h(x)^2\leq B$ for all $x\in\cX$
        \item $\cH$ has finite sequential fat-shattering dimension
        \item There exists an efficient online algorithm guaranteeing external regret with respect to $\cH$ bounded by $r(T)$
        where $r(T)$ is a concave function. 
    \end{itemize}
    Then, for any $m>0$ and bucketing function $g$, Algorithm \ref{alg:csr-algorithm} guarantees, with probability $1-\rho$ (over the internal randomness of the algorithm), $(f,g,\cH)$-conversation swap regret for:
    \[
    f(|T(k-1,i)|) \leq m \cdot r\left(\frac{|T(k-1,i)|}{m}\right)+ \frac{3|T(k-1,i)|}{m} + m + \max(8B,2\sqrt{B}) \cdot m \cdot C_{\cH} \cdot \sqrt{|T(k-1,i)| \log\left(\frac{2mK}{g(T)\rho}\right)}
    \]
    where $T(k-1,i)$ is the subsequence of days where the predictions of round $k-1$ fall into bucket $i$ and $C_\cH$ is a constant that depends on the sequential fat-shattering dimension of $\cH$. 
\end{theorem}

\subsection{End-to-End Results}
Now we are able to state our end-to-end reduction which starts with algorithms with external regret guarantees to $\cH_{A}$ and $\cH_{B}$ respectively and instantiates a collaboration protocol  with external regret guarantees to $\cH_{J}$. In Theorem~\ref{thm:sublinear-to-sublinear}, we show that as long as the external regret bounds we start with  are sublinear in $T$ and the number of rounds $K$ that parameterize the collaboration protocol grows sublinearly with $T$ (but is not constant), we obtain sublinear regret guarantees with respect to $\cH_{J}$. 

\begin{theorem}\label{thm:sublinear-to-sublinear}
    Fix any tuple of hypothesis classes $\cH_A,\cH_B,$ and $\cH_J$. If:
    \begin{itemize}
        \item All $h\in\cH_A$ and $h\in\cH_B$ satisfy $h(x)^2\leq B$ for some constant $B$, for all $x\in\cX$.
        \item $\cH_A$ and $\cH_B$ have finite sequential fat-shattering dimension
        \item There exists an efficient online algorithm guaranteeing external regret with respect to $\cH_A$ bounded by $r_A(T)$, and there exists an efficient online algorithm achieving external regret with respect to $\cH_B$ bounded by $r_B(T)$, where $r_A(T)\leq \tilde{O}(T^{\alpha_A})$ and $r_B(T)\leq \tilde{O}(T^{\alpha_B})$, $\alpha_1,\alpha_2\in(0,1)$, are sublinear in $T$
        \item $\cH_A$ and $\cH_B$ jointly satisfy the $w(\cdot)$-weak learning condition with respect to $\cH_J$
    \end{itemize}
    Then, there is an efficient online algorithm such that if Alice and Bob both use the algorithm to interact in the Collaboration Protocol (Protocol \ref{alg:general-agreement}), then the transcript $\pi^{1:T,K}$ at the last round $K$ satisfies, with probability $1-\rho$:
    \begin{align*}
        &\sum_{t=1}^T (\hat{y}^{t,K}-y^t)^2 - \min_{h_J\in\cH_J} \sum_{t=1}^T (h_J(x^t) - y^t)^2 \\ &\leq 2Tw\inv\left(\tilde{O}\left( T^{-\alpha'} \sqrt{\log \left( \frac{K}{\rho} \right)} +\frac{1}{K^{1/3}} \right) \right) + \tilde{O}\left( KT^{1-\alpha''} \log ^{1/4}\left( \frac{K}{\rho} \right)\right) + O(T^{\alpha})
    \end{align*}
    for some constants $\alpha, \alpha', \alpha'' \in(0,1)$.
    
    Moreover, if $K=\omega(1)$ and $K = o(T^{\alpha''})$, then the transcript $\pi^{1:T,K}$ satisfies, with probability $1-\rho$:
    \begin{align*}
        \sum_{t=1}^T (\hat{y}^{t,K}-y^t)^2 - \min_{h_J\in\cH_J} \sum_{t=1}^T (h_J(x^t) - y^t)^2 \leq \tilde{O}\left( T^{\alpha'''} \log ^{1/4}\left( \frac{1}{\rho} \right)\right) + o(T)
    \end{align*}
    for some constant $\alpha'''\in(0,1)$ and $T$ sufficiently large (larger than a constant that depends on $w, \alpha_A, \alpha_B$, and $\rho$). Here, $o(T)$ is a sublinear term that depends on $w, K, \alpha_A,$ and $\alpha_B$.
    That is, the transcript on the last round achieves sublinear regret.
  \end{theorem}
\begin{remark} 
    Observe that Theorem \ref{thm:sublinear-to-sublinear} allows us to trade off $K$, the parameter controlling the length of the conversation at each day in our collaboration protocol, with the final regret bound. Increasing $K$ can improve the regret bound, at the cost of increasing the amount of daily communication and computation. There is a range of choices of $K$, growing with $T$, that guarantee regret that grows only sublinearly with $T$. The algorithm itself is an efficient reduction to the external regret algorithms for $\cH_A$ and $\cH_B$ that we start with.
    
\end{remark}

Finally, we derive concrete regret bounds when $\cH_A,\cH_B,$ and $\cH$ are norm-bounded linear functions over the domains $\cX_A,\cX_B\subseteq\mathbb{R}^d$ and $\cX_J \subseteq \mathbb{R}^{2d}$ respectively (recall that these classes satisfy the weak learning condition). First, for linear functions there indeed exists an efficient algorithm due to \cite{azoury1999relative} that achieves diminishing external regret --- and thus conversation swap regret --- and so we can apply our reductions to get worst-case polynomial-time algorithms to interact in our collaboration protocol.

\begin{theorem}\citep{azoury1999relative}\label{thm:linear-external-regret}
    There exists an efficient online algorithm producing predictions such that for $x^t\in\mathbb{R}^d, \|x^t\|_2\leq 1$ and for all parameter vectors $\theta\in\mathbb{R}^d$:
    \[
    \sum_{t=1}^T (\hat{y}^t - y^t)^2 - \sum_{t=1}^T (\langle\theta, x^t\rangle, y^t)^2 \leq 2d\ln(T+1) + \|\theta\|^2
    \]
\end{theorem}

Recall that norm-bounded linear functions satisfy the weak learning condition with margin  $w(\gamma) = \Omega(\gamma^2)$ (Theorem \ref{thm:linear-weak-learning}). Together with the conversation swap regret rates we have just derived, we can instantiate Theorem \ref{thm:last-round} for norm-bounded linear functions on the joint feature space. Our result is  Theorem~\ref{thm:linear}.

\begin{theorem}\label{thm:linear} Let $\cX_A = \cX_B =\{x\in\mathbb{R}^d : \|x\|_2\leq 1\}$. Let $\cH_A = \{x_A \mapsto\langle\theta,x_A\rangle : \|\theta\|_2 \leq C\}$ and $\cH_B = \{x_B \mapsto\langle\theta,x_B\rangle : \|\theta\|_2 \leq C\}$ be the sets of all linear functions with bounded norm over $\cX_1$ and $\cX_2$ respectively, for $C\geq 1/2$. Let $\cH_J = \{h_A + h_B : h_A \in \cH_A, h_B \in \cH_B\}$ be the Minowski sum of $\cH_A$ and $\cH_B$.
Consider some transcript $\pi^{1:T, 1:K}$ generated via the Collaboration Protocol between Alice and Bob over $K$ rounds (Protocol~\ref{alg:general-agreement}). There exists an online algorithm (Algorithm \ref{alg:csr-algorithm}, instantiated with the algorithm of Theorem \ref{thm:linear-external-regret}) such that the transcript $\pi^{1:T,K}$ at the last round $K$ satisfies, with probability $1-\rho$:
    \[
    \sum_{t=1}^T (\hat{y}^{t,K}-y^t)^2 - \min_{h_J\in\cH_J} \sum_{t=1}^T (h_J(x^t) - y^t)^2 \leq   \]
    \[
    \tilde{O} \left( T^{47/48}\sqrt{\max(C^{2},C)d\log\left(\frac{KT^{1/8}}{\rho}\right)} + TK^{-\frac{1}{6}} + KT^{\frac{7}{8}}\sqrt{\max(C^{2},C)d\log\left(\frac{KT^{1/8}}{\rho}\right)}\right) 
    \]
\end{theorem}
\begin{remark}
    By setting $K = T^{\frac{3}{28}}$,  the external regret is sublinear:  \[
    \tilde{O} \left( T^{47/48}\sqrt{\max(C^{2},C)d\log\left(\frac{T^{15/56}}{\rho}\right)} + T^{55/56} + T^{55/56}\sqrt{\max(C^{2},C)d\log\left(\frac{T^{15/56}}{\rho}\right)}\right) 
    \] 
    \[
     = \tilde{O} \left( T^{55/56}\sqrt{\max(C^{2},C)d\log\left(\frac{T^{15/56}}{\rho}\right)}\right) 
    \] 
\end{remark}

\section{Collaboration via Decisions}
\label{sec:action}

Thus far we have focused on real valued outcome spaces $\cY = [0,1]$ in which we evaluate predictions by their squared error. Next we turn to an extension where the outcome space $\cY=[0,1]^d$ is $d$-dimensional. The number of possible predictions (up to any reasonable discretization) now grows exponentially in $d$, and so the natural extension of our previous approach of asking the two parties to obtain no swap regret with respect to our predictions becomes infeasible --- all known algorithms for obtaining this would have both run-time and regret bounds scaling exponentially with $d$ or else regret bounds diminishing exponentially slowly with $T$. To circumvent this issue, we model Alice and Bob as decision makers who use predictions to inform downstream actions. More concretely, Alice and Bob have an action set $\cA$ and a utility function $u: \cA\times\cY\to[0,1]$ taking as input an action and outcome. As before, both parties will maintain predictions of the real-valued underlying outcome. However, rather than communicating their estimates of the state directly, they will now simply communicate \textit{actions} --- specifically, the utility-maximizing action relative to their prediction.

\subsection{Decision Preliminaries}
\begin{definition}[Best Response Action]
    Fix a utility function $u:\cA\times\cY\to[0,1]$ and an outcome/prediction $y\in\cY$. The best response to $y$ according to $u$ is the action $\BR_u(y) = \argmax_{a\in\cA}u(a, y)$. 
\end{definition}

Throughout this section, will assume that the utility function $u$ is linear and Lipschitz in the outcome. 

\begin{assumption}
\label{def:utility}
    We assume that the utility function $u: \mathcal{A} \times \mathcal{Y} \rightarrow [0,1]$ satisfies: for every action $a \in \mathcal{A}$,

    \begin{itemize}
        \item $u(a,\cdot)$ is linear in its second argument: for all $\alpha_{1},\alpha_{2}\in \mathbb{R}$, $y_{1},y_{2} \in [0,1]^{d}$, $$u(a,\alpha_{1}y_{1} + \alpha_{2}y_{2}) = \alpha_{1}u(a,y_{1}) + \alpha_{2}u(a,y_{2})$$
        \item $u(a,\cdot)$ is $L$-Lipschitz in its second argument in the $L_\infty$-norm: for all $y_{1}$, $y_{2} \in [0,1]^{d}$, $$|u(a,y_{1}) - u(a,y_{2})| \leq L \|y_{1} - y_{2} \|_{\infty}. $$
    \end{itemize}
\end{assumption}

\begin{remark}
    One natural special case is when $y$ represents a probability distribution over $d$ discrete outcomes $c_{1},\ldots,c_{d}$, such that there is an arbitrary mapping $M(a,c)$ from action/outcome pairs to utilities $[0,1]$. In this case, $u(a,y)$ represents the expected utility of the action $a$ over the outcome distribution, which is linear in $y$ by the linearity of expectation.  The utility function is $L$-Lipschitz in the $L_\infty$-norm, where $L = \max_{a,c_{1},c_{2}}(M(a,c_{1}) - M(a,c_{2})) \leq 1$. So our assumption is satisfied by any risk neutral (expectation maximizing) decision maker with arbitrary utilities over $d$ payoff relevant states---and is only more general.
\end{remark}

\begin{protocol}[ht]
\begin{algorithmic}
    \STATE {\bf Input} $\cX,\cY,K,T$, action space $\cA$, utility function $u:\cA\times\cY\to[0,1]$ 
    \FOR{each day $t = 1, \ldots, T$}
        \STATE Receive $x^t = (x^t_A,x^t_B)$. Alice sees $x^t_A$ and Bob sees $x^t_B$.
        \FOR{each round $k = 1, 2, \ldots,K$}
            \IF { $k$ is odd}
                \STATE Alice predicts $\ymk{t}{k} \in \cY$, and sends Bob $a^{t,k}_A=\BR_u(\ymk{t}{k})$. 
            \ENDIF
            \IF{ $k$ is even}
                \STATE Bob  predicts $\yhk{t}{k} $, and sends Alice $a^{t,k}_B=\BR_u(\yhk{t}{k})$. 
                \ENDIF
        \ENDFOR
        \STATE{Alice and Bob observe $y^t \in \cY$.}
    \ENDFOR
\end{algorithmic}
\caption{\textsc{Online Collaboration Protocol via Decisions}}  \label{alg:collaboration-protocol-decisions}
\end{protocol}

The interaction between Alice and Bob is formalized in Protocol \ref{alg:collaboration-protocol-decisions} (we will sometimes omit the subscripts $A$ and $B$ when it is not important). The history of interaction is similarly captured by a conversation transcript, which now additionally contains the actions communicated by both parties. 

\begin{definition}[Conversation Transcript $\pi^{1:T,1:K}$] \label{def:prediction-transcript-action}
    A conversation transcript $\pi^{1:T,1:K} \in \left\{ \cY^{K+1} \times \cA^K \right\}^T $ is a sequence of tuples of predictions made and actions chosen over rounds by Alice and Bob (alternating across rounds), and the outcome, over $T$ days:
    \begin{align*}
        \pi^{1:T, 1:K} = \left\{
        \left(\ymk{1}{1}, a^{1,1}_A, \yhk{1}{2}, a^{1,2}_B, \ldots \ymk{1}{K}, a^{1,K}_A, y^1\right), 
        \ldots,  
        \left(\ymk{T}{1}, a^{T,1}_A, \yhk{T}{2}, a^{T,2}_B, \ldots \ymk{T}{K}, a^{T,K}_A, y^T\right)  \right\}.
    \end{align*}
    We define $\pi^{1:T: k}$ to be the restriction to only round $k$ of conversation across days as follows:
    \begin{align*}
        \pi^{1:T:k} = \begin{cases}
        \{ (\yak{t}{k}, a^{t,k}_A, y^t) \}_{t \in [T]} & \text{if } k \text{ is odd,}\vspace{1ex}\\
        \{(\ybk{t}{k}, a^{t,k}_B, y^t)\}_{t \in [T]}& \text{otherwise.}
        \end{cases}
    \end{align*}


\end{definition}

Similarly, we will use the notation $\pi^{1:T}$ to refer to a single sequence of predictions and actions over $T$ days, outside the context of a conversation.

\begin{definition}[Prediction Transcript $\pi^{1:T}$]
    A prediction transcript $\pi^{1:T} \in \left\{ \cY^2 \times \cA \right\}^T $ is a sequence of tuples of predictions, actions, and outcomes over $T$ days:
    \[
    \pi^{1:T} = \left\{
        \left(\hat{y}^1, a^{1}, y^1\right), 
        \ldots,  
        \left(\hat{y}^T, a^T, y^T\right)        \right\}
    \]
\end{definition}

Our goal is still to effectively aggregate information --- in that the sequence of \emph{actions} that results from interaction between two parties only with access to their own features has \emph{utility} comparable to the best function mapping the parties \emph{joint} feature space to actions in some benchmark policy class. Below, we  define benchmark classes for our setting as a collection of policies mapping contexts to actions. 

\begin{definition}[Individual Policy Classes $\cC_A, \cC_B$]
\label{def:policy}
    Let $\cC_A: \{\cX_A \mapsto \cA\}$  be a set of functions mapping from Alice's feature set to an action in $\cA$. We analogously refer to $\cC_B$ for Bob. 
\end{definition}

\begin{definition}[Joint Policy Class $\cC_J$]
\label{def:joint-policy}
    Let $\cC_J: \{\cX \mapsto \cA\}$ be a set of functions mapping from the entire feature set $\cX = \cX_A \times \cX_B$ to an action in $\cA$. 
\end{definition}

\begin{assumption}\label{assumption:decision-constant}
    As before, we assume that all classes $\cC$ contain the set of all constant functions $\{c(x)=a\}_{a\in\cA}$.
\end{assumption}

\subsection{Decision Calibration and Regret}

We will appeal to a coarse notion of calibration suitable for high dimensional prediction problems called ``decision calibration" \citep{zhao2021calibrating,noarov2023high,gopalan2023loss}. For a single sequence of predictions, decision calibration asks that the predictions are unbiased conditional not on the predictions themselves, but on the actions induced by best responding to the predictions. The variant we use here is from \cite{noarov2023high}.

\begin{definition}[$f$-Decision Calibration]\label{def:decision-calibration}
    Fix a utility function $u:\cA\times\cY\to[0,1]$. Fix an error function $f:[T] \rightarrow \mathbb{R}$. We say that a transcript $\pi^{1:T}$ is $f$-decision calibrated with respect to $u$ if for all $a\in\cA$:
    \[
    \left\| \sum_{t=1}^T \1[a^t=a] (\hat{y}_t - y_t) \right\|_\infty \leq f(|T(a)|)
    \]
    where $a^t=\BR_u(\hat{y}^t)$ and $T(a) = \{t : a^t=a\}$ is the subsequence of days in which the best response to $\hat{y}^t$ according to $u$ is $a$. 
\end{definition}

When we are interested in competing with a benchmark class $\cC$, another condition is also useful: decision \emph{cross}-calibration asks that predictions be unbiased conditional on the policy that best responds to our predictions, \emph{and} the decision made by each benchmark policy in $\cC$:

\begin{definition}[$(f, \cC)$-Decision Cross Calibration]\label{def:decision-cross-calibration}
    Fix a utility function $u:\cA\times\cY\to[0,1]$ and a policy class $\cC:\{c:\cX\to\cA\}$. Fix an error function $f:[T] \rightarrow \mathbb{R}$. We say that a transcript $\pi^{1:T}$ is $(f, \cC)$-decision cross calibrated with respect to $u$ if for all $a, a'\in\cA$ and all $c\in\cC$:
    \[
    \left\| \sum_{t=1}^T \1[a^t=a, c(x^t)=a'] (\hat{y}_t - y_t) \right\|_\infty \leq f(|T(a, a')|)
    \]
    where $a^t=\BR_u(\hat{y}^t)$ and $T(a, a') = \{t : a^t=a, c(x^t)=a' \}$ is the subsequence of days in which the best response to $\hat{y}^t$ according to $u$ is $a$ and the action suggested by policy $c$ is $a'$. 
\end{definition}

We can also define an analogous notion of swap regret with respect to a policy class $\cC$, which we will call \textit{decision swap regret}. Decision swap regret compares the utility of best response actions induced by predictions $\hat{y}^t$ to the counterfactual utility of actions suggested by policies in $\cC$.

\begin{definition}[$(f^S, \cC)$-Decision Swap Regret]\label{def:dec-SR}
Fix a utility function $u:\cA\times\cY\to[0,1]$ and a policy class $\cC:\{c:\cX\to\cA\}$. Fix an error function $f^S:[T] \rightarrow \mathbb{R}$. A transcript $\pi^{1:T}$ has $(f^S, \cC)$-decision swap regret if:
\begin{align*}
    \sum_{a\in\cA}\max_{c \in \cC}\left( \sum_{t = 1}^{T}\1[a^{t} = a]u(c(x^{t}), y^{t})\right) - \sum_{t=1}^{T}u(a^{t}, y^{t}) \leq f^S(T) 
\end{align*}
\end{definition}

\begin{remark}
    This is the same as the notion of decision swap regret defined in \citet{lu2025sample},  restricted to a single utility function (\cite{lu2025sample} ask for this condition to hold over a class of utility functions).
\end{remark}

\cite{lu2025sample} relate decision calibration and decision swap calibration (conditions on \emph{predictions}) to decision swap regret on the sequence of actions that result from best-responding to the predictions:

\begin{theorem}[Theorem 1 of \citep{lu2025sample}]\label{thm:decision-sr-bound}
    Fix a utility function $u:\cA\times\cY\to[0,1]$ and a policy class $\cC:\{c:\cX\to\cA\}$. If a transcript $\pi^{1:T}$ is $f$-decision calibrated and $(f', \cC)$-decision cross calibrated, and $a^t=\BR_u(\hat{y}^t)$ for all $t\in[T]$, then $\pi^{1:T}$ has $(f^S,\cC)$-decision swap regret, where:
    \[
    f^S(T) \leq L|\cA|f\left(\frac{T}{|\cA|}\right) + L|\cA|^2f'\left(\frac{T}{|\cA|^2}\right)
    \]
\end{theorem}

\begin{remark}
    We remark that under the assumption that the class $\cC$ contains constant functions, $(f,\cC)$-decision cross calibration implies $f$-decision calibration (in fact, decision cross calibration implies decision calibration even if $\cC$ does not contain constant functions, but with a loss of a factor of $|\cA|$ in decision calibration error). Thus, $(f,\cC)$-decision cross calibration alone suffices to guarantee diminishing decision swap regret.    
\end{remark}

Moving to the collaboration setting, we define decision conversation calibration following \citet{collina2025tractable}; this condition asks for decision calibration conditional on the previous action sent by the other party. In other words, the predictions that each party makes should be unbiased conditional on both  their own best response action \textit{and} the best response action communicated at the previous round. 

\begin{definition}[$f$-Decision Conversation Calibration]
\label{def:conversation-decision}
Fix a error function $f^S:[T] \rightarrow \mathbb{R}$. Given a transcript $\pi^{1:T, 1:K}$ from an interaction in the Collaboration Protocol (Protocol \ref{alg:collaboration-protocol-decisions}), Alice is $f$-decision conversation calibrated if for all odd rounds $k$ and all pairs of actions $a,a^\prime \in \cA$: 
\begin{align*}
    \left\| \sum_{t=1}^T \mathbbm{1}[a^{t,k}_A = a, a^{t,k-1}_B=a']  (\ymk{t}{k} - y^t) \right\|_\infty \leq  f(|T(k, a,a^\prime)|)
\end{align*}
where $T(k,a,a^\prime)= \{t : a^{t,k}_A = a \text{ and } a^{t,k-1}_B = a^\prime \}$ is the subsequence of days in which Alice communicates action $a$ on round $k$ and Bob communicates $a'$ on round $k-1$.

Symmetrically, Bob is $f$-decision conversation calibrated if for all even rounds $k$ and all pairs of actions $a,a^\prime \in \cA$:  
\begin{align*}
    \left\| \sum_{t=1}^T \mathbbm{1}[a^{t,k}_B = a, a^{t,k-1}_A=a']  (\yhk{t}{k} - y^t) \right\|_\infty \leq  f(|T(k, a,a^\prime)|)
\end{align*}
where $T(k,a,a^\prime)= \{t : a^{t,k}_B = a \text{ and } a^{t,k-1}_A = a^\prime \}$ is the subsequence of days in which Bob communicates action $a$ on round $k$ and Alice communicates $a'$ on round $k-1$.
\end{definition}

Similarly, we extend conversation swap regret to \textit{decision conversation swap regret}, which is the decision swap regret conditional on the action chosen by the other party in the previous round.

\begin{definition}[$(f^S, \cC)$-Decision Conversation Swap Regret]\label{def:DCSR}
Fix a utility function $u:\cA\times\cY\to[0,1]$. Fix an error function $f^S:[T] \rightarrow \mathbb{R}$ and a policy class $\cC_{A}$. 
Given a transcript $\pi^{1:T,1:K}$ from an interaction in the Collaboration Protocol (Protocol \ref{alg:collaboration-protocol-decisions}), Alice has $(f^S, \cC_{A})$-decision conversation swap regret if for all odd rounds $k$ and all $a'\in\cA$:
\begin{align*}
\sum_{a\in\cA}\max_{c \in \cC_{A}}\left( \sum_{t \in T_{B}(k-1,a')}\mathbb{I}[a^{t,k}_A=a] u(c(x^t_A), y^t) \right) - \sum_{t \in T_{B}(k-1,a')} u(a^{t,k}_A, y^t)
     \leq f^S(|T_{B}(k-1,a')|).
\end{align*}
where $a^{t,k}_A=\BR_u(\hat{y}^{t,k}_A)$ and  $T_{B}(k-1,a')=\{ t: \BR_u(\hat{y}^{t,k-1}_B) =a' \}$ is the subsequence of days where Bob's action in round $k-1$ is $a'$. 

If Bob satisfies a symmetric condition on even rounds $k$ with respect to $\cH_{B}$, we say that Bob has $(f, \cH_{B})$-decision conversation swap regret.
\end{definition}

\begin{assumption}
    As before, we assume that all error functions $f:[T]\to\mathbb{R}$ are concave. 
\end{assumption}

Our approach will be different compared to the one we took in Section \ref{sec:online-alg} for real valued outcomes. There, we argued that swap regret (with respect to the predictions) implied conversation calibration, and hence fast agreement. In the action setting, decision swap regret does \emph{not} necessarily imply decision calibration, which is what is needed to invoke the fast agreement theorems of \cite{collina2025tractable}. Instead we argue that decision calibration and decision cross calibration together imply both decision conversation swap regret and decision conversation calibration.

\subsection{A Boosting Theorem for Decisions}

We now give a weak learning condition that parallels Definition \ref{def:joint-weak}. Whereas Definition \ref{def:joint-weak} requires that $\cC_A$ and $\cC_B$ jointly improve on the squared error of the best constant prediction whenever $\cC_J$ does, the condition now requires that $\cC_A$ and $\cC_B$ jointly improve on the utility of the best constant \textit{action} whenever $\cC_J$ does.

\begin{definition}[$w(\cdot)$-Weak Learning Condition for Decisions]
\label{def:joint-weak-decisions}
    Fix a utility function $u:\cA\times\cY\to[0,1]$. Let $\cC_J$ be a policy class over the joint feature space $\cX$. Let $\cC_A = \{c_A:\cX_A\to\cA\}$ and $\cC_B = \{c_B:\cX_B\to\cA\}$ be policy classes over $\cX_A$ and $\cX_B$ respectively. Let $w:[0,1]\to[0,1]$ be a strictly increasing, continuous, and convex function that satisfies $w(\gamma)\leq \gamma$. We say that $\cC_A$ and $\cC_B$ jointly satisfy the $w(\cdot)$-weak learning condition with respect to $\cC_J$ if for any sequence of contexts $x^{1:T}$ and labels $y^{1:T}$, any $S \subseteq [T]$, and any $\gamma\in[0,1]$, if:
    \[
     \max_{c_J \in \cC_J} \frac{1}{|S|} \sum_{t\in S} u(c_J(x^t), y^t) - \max_{a \in \cA} \frac{1}{|S|}\sum_{t\in S} u(a, y^t) \ge \gamma,
    \]

    \noindent then there exists either $c_A \in \cC_A$ or $c_B \in \cC_B$ such that:
    \[
     \frac{1}{|S|} \sum_{t\in S} u(c_A(x_A^t), y^t) - \max_{a \in \cA} \frac{1}{|S|} \sum_{t\in S} u(a, y^t) \ge w(\gamma)
    \]
    or:
    \[
    \frac{1}{|S|} \sum_{t\in S} u(c_B(x_B^t), y^t) - \max_{a \in \cA} \frac{1}{|S|} \sum_{t\in S} u(a, y^t) \ge w(\gamma)
    \]
\end{definition}

Next we show that if $\cC_A$ and $\cC_B$ satisfy the weak learning condition with respect to $\cC_J$, then low decision swap regret with respect to the classes $\cC_A$ and $\cC_B$ implies that the best response action obtains utility as high as any policy $c_J\in\cC_J$ (up to regret terms). The proof mostly mirrors that of Theorem \ref{thm:weak-learning} and can be found in Appendix \ref{app:action-proofs}. 


\begin{theorem}\label{thm:decision-weak-learning}
    Fix a utility function $u:\cA\times\cY\to[0,1]$. Let $\cC_J$ be a policy class over the joint feature space $\cX$. Let $\cC_A = \{c_A:\cX_A\to\cA\}$ and $\cC_B = \{c_B:\cX_B\to\cA\}$ be policy classes over $\cX_A$ and $\cX_B$ respectively. Fix a transcript $\pi^{1:T}$. If:
    \begin{itemize}
        \item $\pi^{1:T}$ has $(f^S, \cC_A\cup\cC_B)$-decision swap regret (Definition \ref{def:dec-SR})
        \item $\cC_A$ and $\cC_B$ jointly satisfy the $w(\cdot)$-weak learning condition with respect to $\cC_J$ (Definition \ref{def:joint-weak-decisions})
    \end{itemize}
    Then, $\pi^{1:T}$ has $\left(2Tw\inv\left(\frac{f^S(T)}{T}\right),  \cC_J\right)$-decision swap regret when choosing the best response action. That is:
    \[
    \sum_{a\in\cA} \max_{c_J\in\cC_J} \sum_{t=1}^T \1[\BR_u(\hat{y}^t)=a] u(c_J(x^t), y^t) - \sum_{t=1}^T  u(\BR_u(\hat{y}^t), y^t) \leq 2Tw\inv\left(\frac{f^S(T)}{T}\right)
    \]
    whenever the inverse of $w$ exists. 
\end{theorem}

\subsection{Online Decision Collaboration}
We now extend our collaboration protocol to the action setting. We show that if both parties make predictions that have low decision conversation swap regret with respect to $\cC_A$ and $\cC_B$ respectively and are decision conversation calibrated, then they must quickly converge to a sequence of predictions at some round $k$ (not necessarily the final round) at which they have low decision swap regret to both $\cC_A$ and $\cC_B$ simultaneously. At this round, if $\cC_A$ and $\cC_B$ satisfy the weak learning condition relative to a joint class $\cC_J$, then we can argue that the predictions have utility as high as the best policy in the joint class. We then go on to show that the final sequence of predictions must have utility not much lower than the predictions at round $k$, and therefore also the best policy in $\cC_J$. 

We begin by arguing that if Alice and Bob have low decision swap regret and are decision conversation calibrated with respect to their individual policy classes $\cC_A$ and $\cC_B$, the best response actions at some round $k$ will have low decision swap regret to both $\cC_A$ and $\cC_B$. The argument will closely follow that of Theorem \ref{thm:main}. Since both Alice and Bob are decision conversation calibrated, there will exist some round $k$ (assume for now that Alice communicates on round $k$) such that on most days, their actions $\eps$-agree --- that is, Alice's action at round $k$ is an $\eps$-approximate best response for Bob at round $k+1$, and vice versa (Lemma \ref{lem:agreement-round-action}). Our goal is to show that on round $k$, Alice has bounded decision swap regret simultaneously against $\cC_A$ \textit{and} against $\cC_B$. The first is simple: on round $k$, Alice has low decision conversation swap regret with respect to $\cC_A$, and thus she has low decision swap regret with respect to $\cC_A$ (Lemma \ref{lem:dec-sr}). To argue the second: on round $k+1$, Bob has bounded decision conversation swap regret with respect to $\cC_B$. In particular, this means that conditioned on Alice's action on round $k$, Bob's actions are competitive against any policy in $\cC_B$\footnote{Notice that the decision conversation swap regret condition is in fact stronger, since it guarantees that Bob's actions are competitive conditioned on \textit{both} Alice's action on round $k$ \textit{and} Bob's action on round $k+1$. We will only use the weaker ``external" regret guarantee at this step.}. We will additionally show that since they agree, Alice's actions at round $k$ obtain similar utility to Bob's actions at round $k+1$ (Lemma \ref{lem:agreement-regret-action}). Thus, conditioned on Alice's action on round $k$, Alice's actions are also competitive against any policy in $\cC_B$. Since this is true for any action that Alice chooses, Alice must also have low decision swap regret with respect to $\cC_B$ on this round.

\begin{theorem}\label{thm:decision-online}
Suppose Alice has $(f_A^S, \cC_A)$-decision conversation swap regret and $f_A$-decision conversation calibration. Similarly, suppose Bob has $(f_B^S, \cC_B)$-decision conversation swap regret and $f_B$-decision conversation calibration. If they engage in Protocol \ref{alg:collaboration-protocol-decisions} for $T$ days, with $K$ rounds each day, then there exists a round $k$ of the protocol such that the transcript $\pi^{1:T,k}$ has $(\max\{\lambda_A, \lambda_B\}, \cC_A\cup\cC_B)$-decision swap regret, where:
\[
\lambda_{A} \leq |\cA| f^S_A\left(\frac{T}{|\cA|}\right) + L|\cA|^2 f_A\left(\frac{T}{|\cA|^2}\right) + 2T\left(\frac{1}{(K-1)}+\beta(T)\right)^{1/2}
\]
and
\[
\lambda_B \leq |\cA| f^S_B\left(\frac{T}{|\cA|}\right) + L|\cA|^2 f_B\left(\frac{T}{|\cA|^2}\right) + 2T\left(\frac{1}{(K-1)}+\beta(T)\right)^{1/2}
\]
Here, $\beta(T) = \frac{L|\cA|^2}{T}\left( f_A\left( \frac{T}{|\cA|^2} \right) + f_B\left( \frac{T}{|\cA|^2} \right) \right)$. 
\end{theorem}


Theorem \ref{thm:decision-online} shows that at some intermediate round, the transcript has bounded decision swap regret with respect to $\cC_A\cup\cC_B$. Our boosting result (Theorem \ref{thm:decision-weak-learning}) states that if, additionally, $\cC_A$ and $\cC_B$ are weak learners for $\cC_J$, then the transcript also has bounded decision swap regret with respect to $\cC_J$. Together, these results imply that at an intermediate round, the transcript has bounded decision swap regret with respect to $\cC_J$. 

One difficulty is that Alice and Bob will not know a priori which intermediate round will have these guarantees --- and so it is not clear a priori which downstream action to take on any day. However, we will use a similar argument as we did in the proof of Theorem \ref{thm:last-round} to argue that the transcript on the \textit{last} round inherits an external regret guarantee. That is, as long as Alice and Bob act according to the last round, they are sure to to achieve bounded external regret with respect to $\cC_J$.

\begin{theorem}\label{thm:decision-final-regret}
    Fix a utility function $u:\cA\times\cY\to[0,1]$. Let $\cC_J$ be a policy class over the joint feature space $\cX$. Let $\cC_A = \{c_A:\cX_A\to\cA\}$ and $\cC_B = \{c_B:\cX_B\to\cA\}$ be policy classes over $\cX_A$ and $\cX_B$ respectively. Fix a transcript $\pi^{1:T,1:K}$ generated via Protocol \ref{alg:collaboration-protocol-decisions}. If:
    \begin{itemize}
        \item Alice has $(f^S_A,\cC_A)$-decision conversation swap regret and $f_A$-decision conversation calibration
        \item Bob has $(f^S_B,\cC_B)$-decision conversation swap regret and $f_B$-decision conversation calibration
        \item $\cC_A$ and $\cC_B$ jointly satisfy the $w(\cdot)$-weak learning condition with respect to $\cC_J$ 
    \end{itemize}
    Then, there exists a round $k$ of the protocol such that the transcript $\pi^{1:T,k}$ has $\left(2Tw\inv\left(\frac{\max\{\lambda_A, \lambda_B\}}{T}\right), \cC_J\right)$-decision swap regret, whenever the inverse of $w$ exists.
    Moreover, on the last round $K$, the transcript $\pi^{1:T,K}$ satisfies:
    \[
    \max_{c_J\in\cC_J} \sum_{t=1}^T u(c_J(x^t), y^t) - \sum_{t=1}^T u(a^{t,K}, y^t) \leq 2Tw\inv\left(\frac{\max\{\lambda_A, \lambda_B\}}{T}\right) + (K-1)T\beta(T)
    \]
    whenever the inverse of $w$ exists. Here, 
    \[
\lambda_{A} \leq |\cA| f^S_A\left(\frac{T}{|\cA|}\right) + L|\cA|^2 f_A\left(\frac{T}{|\cA|^2}\right) + 2T\left(\frac{1}{(K-1)}+\beta(T)\right)^{1/2},
\]
\[
\lambda_B \leq |\cA| f^S_B\left(\frac{T}{|\cA|}\right) + L|\cA|^2 f_B\left(\frac{T}{|\cA|^2}\right) + 2T\left(\frac{1}{(K-1)}+\beta(T)\right)^{1/2}
\]
where $\beta(T) = \frac{L|\cA|^2}{T}\left( f_A\left( \frac{T}{|\cA|^2} \right) + f_B\left( \frac{T}{|\cA|^2} \right) \right)$.
\end{theorem}

\subsection{Achieving Conversation Decision Cross Calibration Algorithmically}

Finally, we turn attention to an algorithm that obtains low decision conversation swap regret and low decision conversation calibration; this will allow us to instantiate our results with concrete regret bounds. We use the algorithm of \citet{lu2025sample}, which guarantees diminishing decision calibration and decision cross calibration error and thus, by Theorem \ref{thm:decision-sr-bound}, diminishing decision swap regret. 

\begin{theorem}[Theorem 2 of \citet{lu2025sample}]\label{thm:decision-sr-alg}
    Fix a utility function $u:\cA\times\cY\to[0,1]$. Fix a policy class $\cC$. There is an algorithm that with probability $1-\rho$, for any sequence of outcomes $y^1,...,y^T$, outputs predictions $\hat{y}^1,...,\hat{y}^T$ that are $f$-decision calibrated and $(f,\cC)$-decision cross calibrated, where:
    \[
    f(\tau) \leq O\left( \ln(d|\cA||\cC|T) + \sqrt{T \ln\left(\frac{d|\cA||\cC|T}{\rho}\right)} \right)
    \]
    for any $\tau\in[T]$.
\end{theorem}

To guarantee diminishing decision conversation swap regret and decision conversation calibration, we instantiate a copy of this algorithm for each pair of rounds $k$ and actions $a$. On round $k$ of day $t$, we call on the copy corresponding to that round and the action chosen in the previous round on that day. This gives us precisely what we want: diminishing decision swap regret and decision calibration, conditioned on every round \textit{and} the most recently communicated action. This reduction is formalized in Algorithm \ref{alg:algorithm-decisions} (here, we take the perspective of Alice; Bob's is symmetric).

\begin{algorithm}[ht]
\begin{algorithmic}
    \STATE {\bf Input} Algorithm $M$, policy class $\cC$
    \vspace{0.5em}
    
    For every odd $k\in[K]$ and $a\in\cA$, instantiate a copy of $M$, called $M_{k,a}$. For the first round $k=1$, instantiate a copy of $M$, called $M_1$.

    Let $\pi^{1:t,k|a}$ denote the transcript on round $k$ up until day $t$, restricted to $\{t:a^{t,k-1}=a\}$, the subsequence where the previously communicated action was $a$.

    Let $M(\pi^{1:t,k|a}, \cC)$ denote the output of $M$ given this transcript. 
    
    \FOR{each day $t = 1, \ldots, T$}
        \STATE Receive $x^t_A$
        \STATE Make prediction $\hat{y}^{t,1}_A = M_{1}(\pi^{1:t-1,1}, \cC)$
        \STATE Send to Bob $a^{t,1}_A = \BR_u(\hat{y}^{t,1}_A)$

        \FOR{each odd round $k = 3, 5, \ldots,K$}
            \STATE Observe Bob's action from the previous round $a^{t,k-1}_B$
            \STATE Make prediction $\hat{y}^{t,k}_A = M_{k, a^{t,k-1}_B}(\pi^{1:t-1,k|a^{t,k-1}_B}, \cC)$
            \STATE Send to Bob $a^{t,k}_A = \BR_u(\hat{y}^{t,k}_A)$
        \ENDFOR
        \STATE{Observe $y^t \in \cY$.}
    \ENDFOR
\end{algorithmic}
\caption{A reduction from a decision conversation swap regret and decision conversation calibration algorithm to a decision cross calibration algorithm}  \label{alg:algorithm-decisions}
\end{algorithm}

\begin{theorem}\label{thm:decision-alg}
    Fix a utility function $u:\cA\times\cY\to[0,1]$. Fix a policy class $\cC$. With probability $1-\rho$, Algorithm \ref{alg:algorithm-decisions}, instantiated with the algorithm of Theorem \ref{thm:decision-sr-alg} and $\cC$, obtains $(f^S, \cC)$-decision conversation swap regret and $f$-decision conversation calibration for:
    \[
    f^S(\tau) \leq O\left( L|\cA|^2\ln(d|\cA||\cC|T) + L|\cA| \sqrt{T \ln\left(\frac{dK|\cA||\cC|T}{\rho}\right)} \right)
    \]
    and
    \[
    f(\tau) \leq O\left( \ln(d|\cA||\cC|T) + \sqrt{T \ln\left(\frac{dK|\cA||\cC|T}{\rho}\right)} \right)
    \]
    for any $\tau\in[T]$.
\end{theorem}

To end this section, we instantiate Theorem \ref{thm:decision-final-regret} with the algorithmic bounds. As before, we face a tradeoff in the choice of $K$, the length of the conversation. We show that for appropriately chosen $K$, we guarantee sublinear regret bounds with respect to $\cC_J$. 

\begin{theorem}\label{thm:decision-final-bounds}
    Fix a utility function $u:\cA\times\cY\to[0,1]$. Let $\cC_J$ be a policy class over the joint feature space $\cX$. Let $\cC_A = \{c_A:\cX_A\to\cA\}$ and $\cC_B = \{c_B:\cX_B\to\cA\}$ be policy classes over $\cX_A$ and $\cX_B$ respectively. Suppose Alice and Bob interact via Protocol \ref{alg:collaboration-protocol-decisions}. 
    If:
    \begin{itemize}
        \item Both Alice and Bob use Algorithm \ref{alg:algorithm-decisions}, instantiated with the algorithm of Theorem \ref{thm:decision-sr-alg} and policy classes $\cC_A$ and $\cC_B$ respectively
        \item $\cC_A$ and $\cC_B$ jointly satisfy the $w(\cdot)$-weak learning condition with respect to $\cC_J$
    \end{itemize}
    Then, with probability $1-\rho$, the transcript $\pi^{1:T,K}$ on the last round $K$ satisfies:
    \begin{align*}
        \max_{c_J\in\cC_J} \sum_{t=1}^T u(c_J(x^t), y^t) - \sum_{t=1}^T u(a^{t,K}, y^t) &\leq 2Tw\inv\left( O\left( \frac{L|\cA|^3 \ln\left( \frac{dK|\cA||\cC_A||\cC_B|T}{\rho} \right)}{T^{1/4}} + \frac{1}{\sqrt{K-1}} \right) \right) \\ & \ \ \ \ + O\left( (K-1)L|\cA|^2 \ln\left( \frac{dK|\cA||\cC_A||\cC_B|T}{\rho} \right) \sqrt{T}  \right)
    \end{align*}
    whenever the inverse of $w$ exists.

    Moreover, if $K=\omega(1)$ and $K=o(\sqrt{T})$, then the transcript $\pi^{1:T,K}$ satisfies, for some constant $\alpha\in(0,1)$:
    \begin{align*}
        \max_{c_J\in\cC_J} \sum_{t=1}^T u(c_J(x^t), y^t) - \sum_{t=1}^T u(a^{t,K}, y^t) &\leq 2Tw\inv\left( O\left( \frac{L|\cA|^3 \ln\left( \frac{d|\cA||\cC_A||\cC_B|T}{\rho} \right)}{T^{1/4}} + o(1) \right) \right) \\ & \ \ \ \ + O\left( L|\cA|^2 \ln\left( \frac{d|\cA||\cC_A||\cC_B|T}{\rho} \right) T^{\alpha}  \right) \\
        &\leq o(T)
    \end{align*}
    That is, the transcript at the last round achieves sublinear external regret with respect to $\cC_J$. 
\end{theorem}

\section{Collaboration in the Batch Setting}
\label{sec:batch}
Thus far, we have studied the \textit{online} setting, in which participants jointly predict the label on a new adversarially chosen example every day. However, we can also study this form of  collaborative learning in the simpler \textit{distributional} or \textit{batch} setting, where Alice and Bob both receive different features $x_A$ and $x_B$  drawn from a distribution. They will train on a sample of such data (paired with labels) at training time, and then at test time (when labels are unavailable) will be evaluated on examples drawn from the same distribution. This setting is strictly easier than the online adversarial setting, and hence admits (morally if not notationally) simpler algorithms which we develop in this section. 

At a high level, the algorithm here will proceed over $R$ rounds that we index by $r$. In the training phase, Alice and Bob will iteratively build their models as follows:

\begin{itemize}
\item Bob will begin by generating an initial model and sending his model's initial predictions for all of the points in the training set, $P^0$, to Alice. These predictions will be discretized to a finite range.
\item In the next round, Alice will refine her model according to Bob's predictions:
\begin{itemize}
\item First, she will bucket her data into level sets \textit{according to Bob's predictions}. ``Level set $v$'' corresponds to all the points in the training set for which Bob predicted $v$. 
\item On each level set $v$ in parallel, Alice will run an internal boosting procedure which we call \lsboost\ with respect to her hypothesis class (defined only on her own features), generating a model $\tilde{f}_A^{1,v}$. This internal boosting process is equivalent to the LSBoost algorithm from \cite{globus-harris2023multicalibration}. In essence, it repeatedly performs squared error regression over $\cH_A$ on Alice's own level sets, until doing so no longer substantially improves squared error. This procedure results in a (discretized) ensemble of models from $\cH_A$ defined in parallel for each of the $v$ level sets.
\item For each level set $v$, Alice will look at the error of her resulting model on that level set $f_A^{1,v}$, and compare it to the error of Bob's (constant) predictor $v$ constrained to that level set. Depending on whether her predictions improve substantially over Bob's, she will either set $f_A^{1,v}$ to $\tilde{f}_A^{1,v}$ or to the constant predictor $v$ (i.e. ``agreeing'' with Bob's predictions on that levelset). 
\item She will define her final predictor at the end of round $1$, $f_A^1$, as an ensemble of these models such that if a datapoint $x = (x_A,x_B)$ is given predicted label $v$ by Bob's initial predictor, $f_A^1(x_A)=f_A^{1,v}(x_A)$.
\end{itemize}
\item She will then evaluate $f_A^1$ on every point in the training sample and send the resulting predictions $P^1$ to Bob. 
\item In the next round, Bob will run a symmetric procedure using Alice's predictions $P^1$. They will continue in this manner in rounds until the predictions have converged to agreement.
\end{itemize}

During this process, Alice and Bob will separately maintain transcripts of the models which they have iteratively built across the rounds of communication. At test time, to make a prediction on a new datapoint with features  $x = (x_A, x_B)$ partitioned across Alice and Bob, they will again engage in an interactive conversation, at each round making predictions according to the models recorded in the transcript that was generated during training.  This will proceed as follows:

\begin{itemize}
\item Bob will look at his model transcript, extract his initial model, and evaluate it on $x_B$. He will then send the prediction to Alice. 
\item Alice will extract from her transcript the model $f^{1,v^*}$ corresponding to the value of Bob's prediction $v^*$, and send her prediction $f^{1,v^*}(x_A)$ to Bob.
\item They will proceed in this manner across rounds until they have evaluated the final models stored in their transcripts, whose predictions they will output.
\end{itemize}

\subsection{Preliminaries for the Batch Setting}

Formally, as in the online setting, Alice and Bob  have feature spaces $\cX_A$ and $\cX_B$ and there is a real-valued outcome space $\cY$. We now additionally assume that there is a joint distribution $\cD \in \Delta (\cX_A \times \cX_B \times \cY)$ from which examples are drawn. We will write $\cD_{A}$ to denote the marginal distribution over $(\cX_A,\cY)$ and $\cD_{B}$ to denote the marginal distribution over $(\cX_B, \cY)$.

\subsubsection{Training Phase}

In the training phase, a finite training set $S = \{(x_A^i, x_B^i, y_i)\}_{i \in [n]} \sim \cD^n$ of size $n$ is sampled i.i.d, where we write $[n]$ to denote $\{1,\ldots,n\}$. Alice is given $S_A = \{(x_A^i, y^i)\}_{i \in [n]}$ and Bob is given $S_B = \{(x_B^i, y^i)\}_{i \in [n]}$. Importantly, $i$ here indexes over the \textit{same instances} whose features are split between parties: $x^i = (x_A^i, x_B^i)$. Over their rounds of communication, Alice and Bob's models will be generated by ensembling hypotheses $h_A$ and $h_B$ respectively in hypothesis classes $\cH_A$ and $\cH_B$, where $h_A:\cX_A \rightarrow \mathbb{R}$ and $h_B:\cX_B \rightarrow \mathbb{R}$. In particular, we will assume that they generate these hypotheses via access to a \textit{squared error regression oracle}:


\begin{definition}
\label{def:batch-oracle}
We say $\cO_{\cH}: \Delta (\cX \times \cY))\rightarrow (\cX \rightarrow \cY)$ is a squared error regression oracle for a class of real-valued functions $\cH$ if for every distribution $\cD \in \Delta (\cX \times \cY),$ $\cO_{\cH}$ outputs the squared-error minimizing function $h \in \cH$ over the distribution. I.e., if $h = \cO_{\cH}(\cD)$ then 
\[
h \in \arg\min_{h' \in \cH} \E_{(x,y) \sim \cD} \left[ (h'(x)-y)^2 \right].
\]

When we feed such an oracle a sample $S = (x_i, y_i)_{i\in[n]}$, we will interpret these expectations as over the sample.
\end{definition}

Across their interactions, Alice and Bob will round their predictions to some discretization, defined by a discretization parameter $m \in \bbZ^+$. We will write $[1/m] := \{0, \frac{1}{m}, \ldots, \frac{m-1}{m}, 1\}$ be a discretization of the range $[0,1]$ into multiples of $1/m$. They will round their predictions as follows: 

\begin{definition}[$\round(h; m)$]
    \label{def:round}
    Let $\cF$ be the collection of all real valued functions from features $\cX \rightarrow \bbR$. Then $\round$ is a function $\round:\cF \times \bbZ^+ \rightarrow \cF$ where $\round(h;m)$ outputs a function $\rh$ such that 
    \[
    \rh(x) = \min_{v \in [1/m]} \vert h(x) - v \vert.
    \]
\end{definition}

During  training, Alice and Bob will separately generate \textit{model transcripts} of the models they have generated so far, which they will use to construct predictions of the model out of sample. In essence, these model transcripts are simply a collection of models in $\cH_A$ and $\cH_B$ respectively, with the exception that in some rounds, their algorithm will generate $\bot$ instead of a model (indicating that they are deferring to their counter-party's prediction). 

\begin{definition}[Transcript]
Let $\cH_A$ be Alice's hypothesis class and let $m \in \bbZ^+$. Over her $R$ rounds of interaction with Bob, she will within each round run an internal algorithm in parallel $m$ times. This internal algorithm will either return $\bot$ or run for at most $K \in \bbZ^+$ phases. Over the course of these interactions she will generate her \textit{model transcript}, which is an object over both her interactions with Bob and her internal algorithm:
\[
\Pi_A^R = \{\pi_A^0, \ldots, \pi_A^R\} \in \left( \{\bot\} \cup \cH_A^{Km} \right)^{m R},
\]
where for each round \(r \in [R]\), we have
\[
\pi_A^r = \{\pi_A^{r,v}\}_{v \in [1/m]}
\]
and each of these sub-transcripts $\pi^{r,v}_A$ describes the (at most) $K$ phases of each of Alice's internal algorithm:\footnote{The internal algorithm will run for a variable number of phases across the rounds of the collaborative algorithm between Alice and Bob, but we can assume this variable number of phases is bounded by $K$. For the sake of notation, we can imagine instantiations with fewer phases to be padded with $\bot$ to make them length $K$.}
\[
\pi^{r,v}_A \in \left\{ \bot \right\} \cup \left\{ (\pi^{r,v,k}_A)_{k \in [K]} \right\},
\]
and each
\[
\pi^{r,v,k}_A = (h_A^{r,v,k,v'})_{v' \in [1/m]}, \quad \text{with} \quad h_A^{r,v,k,v'} \in \cH_A.
\]
Bob's transcript $\Pi_B^R$ will be defined analogously.
\end{definition}

Alice and Bob act in alternating rounds, Alice in even rounds and Bob in odd ones. At the end of each round $r$, the active player sends the other their current predictions on the training set (which are all discretized to lie in $[1/m])$. 

\begin{definition}[Prediction at round $r$]
We will write $P^r \in [1/m]^n$ to be the $n$ predictions generated at round $r$ for each $x^i \in S$. If $r$ is odd, $P^r = P^r_A$ are Alice's predictions, and if $r$ is even, $P^r = P^r_B$ are Bob's predictions. In our analyses, we will denote the $i$th prediction in the vector $P^r$ as $P^{r,i}$.

At the end of $R$ rounds, Alice and Bob will know a collection of predictions
\[
\cP^R = (P^0, \ldots, P^R)
    = \begin{cases}
    P^0_B, \ldots, P^{R-1}_B, P^R_A & \text{if } R \text{ is even}, \\
    P^0_B, \ldots, P^{R-1}_A, P^R_B & \text{else}.
\end{cases}
\]
\end{definition}

\begin{remark}
    Note that the dimension of these predictions $P^r$ is different than in the online setting. There, only a \textup{single} prediction $\hy^{r,k}$ is communicated between the players in their conversation. Here, we have a \textup{set} of $n$ predictions communicated in each round --- one for each point in the training set. 
\end{remark}



At round $r$, Alice will generate a model $f_A^r$. In the training algorithm defined in Section \ref{sec:batch-train}, this model will be only well-defined defined for the training sample; in Section \ref{sec:batch-test} we will discuss how to generate predictions on new data using the training transcript.

\begin{definition}[Model at round $r$]
At round $r$ of training, Alice will generate a model $f_A^r: \cX_A \times [1/m] \rightarrow [1/m]$ which is based on her datapoint and Bob's prediction from the previous round. In general this model will be invoked in contexts where Bob's prediction $v$ is clearly defined so we will write 
\[
f_A^r(\aliceData) = f_A^r(\aliceData,v).
\]
Bob's model $f_B^r$ will be defined analogously.
\end{definition}

\begin{definition}
\label{def:batch-functions}
At the final round $R$ of our collaboration algorithm COLLABORATE (Algorithm \ref{protocol:batch}), Alice and Bob will have two models $f^R_A$ and $f^R_B$ which will agree for all datapoints on both the training sample and at test time, so we can equivalently consider them as represented by a single model $f^R$. We will write $\cF^R$ to be the space of models which may be output by the collaboration algorithm on samples of size $n$, i.e., 
\[
\cF^R = \{f^R \vert f^R \leftarrow \text{COLLABORATE}((S_A,S_B),\aliceOracle,\bobOracle,m)\}_{(S_A,S_B) = \{(\aliceData^i,y^i),(\bobData^i,y^i)\}_{i \in [n]}},
\]
where $S_A$ and $S_B$ have been generated from a joint sample $S \in (\cX_A \times \cX_B \times \cY)^n.$
\end{definition}

\subsubsection{Test Time Evaluation}
Once Alice and Bob have completed training, they will have models $f^R_A$ and $f^R_B$ and model transcripts $\Pi^R_A$ and $\Pi^R_B$. However, their final models will be recursively defined in terms of their predictions in previous rounds. Thus, in order to evaluate $f^R$ on a new sample $(x_A,x_B)$, they will have to again interact over $R$ rounds, sending each other their predictions $\hy^r$ at each round, which will be computed based on their model transcripts $\Pi^R_A$ and $\Pi^R_B$. Note that here, since the prediction is for a single datapoint rather than a set of datapoints as it is in the training phase, we revert to the prediction notation used in the rest of the paper ($\hat{y}^r$ rather than $P^r$). This algorithm is formally described in Section \ref{sec:batch-test}.



\subsection{Batch Collaboration Algorithm}

Our algorithm will make use of level sets of Alice and Bob's model's (discretized) predictions on their own data as well as the level sets of \textit{each other}'s models. 

\begin{definition}[Level Sets]
\label{def:level-set}
Let $S_A$ be Alice's sample. Let Alice's predictions at round $r$ for each point in her sample $S_A$ be $\alicePreds{r} = \{P_A^{r,1}, \ldots, P_A^{r,n}\}$ and Bob's predictions at round $r$ be $\bobPreds{r} = \{P_B^{r,1}, \ldots, P_B^{r,n}\}$. Let $v \in \bbR$. We will say that 
\begin{align*}
\ls(S_A, P_A^r, v) &= \{\aliceData^i \vert P_A^{r,i} =v\}_{i \in [n]}\\
    &= \{\aliceData^i \vert \aliceModel{r}(\aliceData^i)=v\}_{i \in [n]}
\end{align*}
are Alice's $v$th level set on her own predictions. Similarly, we will call Alice's $v$th level set on \textit{Bob}'s predictions
\begin{align*}
\ls(S_A, P_B^r, v) &= \{\aliceData^i \vert P_B^{r,i} =v\}_{i \in [n]}\\
    &= \{\aliceData^i \vert \bobModel{r}(\bobData^i)=v\}_{i \in [n]}.
\end{align*}
\end{definition}

\begin{remark}
Note that the transcript at round $r$ is directly computable based only on Alice and Bob's knowledge of their and the other players' predictions $\alicePreds{r}$ and $\bobPreds{r}$---neither player has to recompute $\aliceModel{r}$ or $\bobModel{r}$, nor do they need access to the other players' features. 
\end{remark}

In general, for subroutines we use a subscript $\bullet$ to refer to either $A$ or $B$, depending on whose inputs the subroutine was called on, and a subscript $\circ$ to refer to the other player. With this notation in place, we can proceed to the algorithms. 

\subsubsection{Training Algorithm}
\label{sec:batch-train}

While training, Alice and Bob will run Algorithm \ref{protocol:batch}, COLLABORATE, on their training samples $(S_A, S_B)$. This algorithm proceeds in rounds, with Alice and Bob alternating who sends whom their most current predictions. In each round, the current player will call a subroutine \crossboost\ (Algorithm \ref{protocol:crossboost}), in which that player boosts their predictions in parallel on each of their datasets' level sets as defined \textit{by the other players' predictions}. This ``internal" boosting step which is done in parallel on each of these level sets is itself a boosting algorithm, which we call \lsboost\ (Algorithm \ref{protocol:ls-boost}), and is equivalent to the level set boosting algorithm from \citep{globus-harris2023multicalibration}: we restate it here as our parametrization is slightly different and to make our notational choices clear for the sake of our later analysis. At the end of the process, Alice and Bob will have a collection of individual model transcripts, which they will later use to evaluate the final model on new datapoints.

\begin{algorithm}[H]
\begin{spacing}{1.15}
\begin{algorithmic}
\STATE{ {\bf Alice's Input:} $\aliceOracle, \aliceSample, m$}
\STATE{ {\bf Bob's Input:} $\bobOracle, \bobSample, m$}
\STATE Let $h^0_B \in \bobOracle(S_B)$ and $f^0_B = \round(h^0_B; m)$.
\STATE Let $P^{-1} = \bot$ and $ \bobPreds{0} = \{\bobModel{0}(\bobData)\}_{(\bobData,y) \in \bobSample}$.
\STATE Bob sends $P^0 = \bobPreds{0}$ to Alice.
\STATE Let $r=0, \Pi_A^0 = \emptyset,$ and $\Pi_B^0 = \{\pi_B^0\}=\{\{f_B^0\}\}$.
\WHILE{$P^r \ne P^{r-1}$}
    \IF{$r$ is even}
        \STATE Alice plays, boosting her predictions on Bob's predictor's level sets:
        \begin{spacing}{0.8}
        \STATE \[
        \aliceModel{r+1}, \aliceTranscript{r+1} = \crossboost(\aliceSample, \aliceOracle, \bobPreds{r}, m)
        \]
        \end{spacing}
        \STATE Alice generates her predictions for this round, $\alicePreds{r+1} = \{ \aliceModel{r+1}(\aliceData)\}_{(\aliceData,y) \in \aliceSample}$
        \STATE Alice sends her updated predictions $P^{r+1} = \alicePreds{r+1}$ to Bob.
        \STATE Alice updates her model transcript, setting $\Pi_A^{r+1} = \Pi_A^{r} \cup \{\aliceTranscript{r+1}\}$.
        \STATE Bob does nothing, and sets $\bobModel{r+1}=\bobModel{r}$ and $\Pi_B^{r+1}=\Pi_B^{r}$.
    \ELSE
        \STATE Bob plays analogously, boosting his predictions on Alice's predictor's level sets:
        \begin{spacing}{0.8}
        \STATE \[
        \bobModel{r+1}, \bobTranscript{r+1}= \crossboost(\bobSample, \bobOracle, \alicePreds{r}, m)
        \]
        \end{spacing}
        \STATE Bob generates his predictions for this round, $\bobPreds{r+1} = \{ \bobModel{r+1}(\bobData)\}_{(\bobData,y) \in \aliceSample}$
        \STATE Bob sends his updated predictions $P^{r+1} = \bobPreds{r+1}$ to Alice.
        \STATE Bob updates his model transcript, setting $\Pi_B^{r+1} = \Pi_B^r \cup \{\bobTranscript{r+1}\}.$
        \STATE Alice does nothing, and sets $\aliceModel{r+1}=\aliceModel{r}$.
    \ENDIF
    \STATE $r = r+1$.
\ENDWHILE 
\STATE {\bf Alice's Output:} $\aliceModel{r}, \Pi_B^r$
\STATE {\bf Bob's Output:} $\bobModel{r}, \Pi_B^r$
\end{algorithmic}
\end{spacing}
\caption{COLLABORATE: Batch Collaboration Algorithm for Training}
\label{protocol:batch}
\end{algorithm}

\begin{algorithm}[H]
\caption{\textsc{CROSS-BOOST}}
\label{protocol:crossboost}
\begin{spacing}{1.15}
\begin{algorithmic}
\STATE \textbf{Input :} $\playerSample, \playerOracle, \playerPreds{r}, m$
\FOR {each $v \in [1/m]$}
   
    \STATE The player generates their $v$th level set on the other players' predictions $\playerPreds{r}$,
    \begin{spacing}{0.7}
    \[
    \playerSampleLS{r+1}{v} = \ls(\playerSample, \playerPreds{r}, v)
    \]
    \end{spacing}
    \STATE Using only their data constrained to this level set, they run the internal boosting algorithm, and evaluate their updated model's performance: 
    \begin{spacing}{0.7}
    \begin{align*}
    \playerModelLStilde{r+1}{v}, \tilde{\pi}^{r+1,v} &= \lsboost(\playerSampleLS{r+1}{v},
    \playerOracle, m)\\
    \widetilde{\err}^{r+1,v} &= \E_{(\playerData, y) \in \playerSampleLS{r+1}{v}}\left[(\playerModelLStilde{r+1}{v}(\playerData)-y)^2\right]
    \end{align*}
    \end{spacing}
    \STATE They then compare their updated model's performance to their counter-party's constant predictor, and determine which of the two to use as their final model:
    \STATE Let 
    \begin{spacing}{0.4}
    \[
    \err^v = \E_{(\playerData, y) \in \playerSampleLS{r+1}{v}}\left[(v-y)^2\right]
    \]
    \end{spacing}
    \IF{$(\err^v - \widetilde{\err}^{r+1,v}) > 1/m^2$}
    \STATE $\playerModelLS{r+1}{v}(\playerData) = \playerModelLStilde{r+1}{v}$
    \STATE $\transcript{r+1,v} = \tilde{\pi}^{r+1,v}$
    \ELSE 
    \STATE $\playerModelLS{r+1}{v}(\playerData) = v$
    \STATE $\transcript{r+1,v} = \bot$
    \ENDIF
\ENDFOR
\STATE 
\STATE The player then ensembles their models on each of the level sets of the others' predictions and updates their transcript for the round:
\STATE \[
         \playerModel{r+1}(\playerData) = \sum_{v \in [1/m]} \1[x \in \playerSampleLS{r+1}{v} ] \cdot \playerModelLS{r+1}{v}(\playerData),
        \]
\STATE $\transcript{r+1} = \{ \transcript{r+1,v} \}_{v \in 1/m}$
\STATE \textbf{Output: } $\playerModel{r+1}, \transcript{r+1}$
\end{algorithmic}
\end{spacing}
\end{algorithm}

\begin{algorithm}[H]
\caption{\lsboost \quad [\cite{globus-harris2023multicalibration}]}
\begin{spacing}{1.2}
\begin{algorithmic}
\STATE \textbf{Input: } $\playerSample, \playerOracle, m$ 
\STATE Let $k=0$
\STATE Let $h_{\bullet}^{r,v,0} = \playerOracle(\playerSample)$ and $\lsSubtranscript{0}=\{h_{\bullet}^{r,v,0}\}$
\STATE Let $\playerModelBoost{k} = \round(h_{\bullet}^{r,v,0}; m^2)$ 
\STATE Let $\err_{-1} = \infty$ and $\err_0 = \E_{(\playerData,y) \sim \playerSample}[(h_{\bullet}^{r,v,0}(\playerData - y)^2]$
\WHILE{$\err_{k-1} - \err_{k} \ge 1/m^2$}
\FOR{each $v' \in [1/m^2]$}
\STATE $\playerSampleLSBoost{k+1} = LS(\playerSample, \playerModelBoost{k}, v')$
\STATE Let $\playerOracleModel{k+1} = \playerOracle(\playerSampleLSBoost{k+1})$.
\ENDFOR
\STATE The player ensembles their models:
\begin{spacing}{0.8}
\[\tilde{f}^{r,v,k+1}_{\bullet}(\playerData) = \sum_{v' \in [1/m]} \1[\playerModelBoost{k}(\playerData)=v'] \cdot \playerOracleModel{k+1}(\playerData)\]
\[\playerModelBoost{k+1}(\playerData) = \round(\tilde{f}^{r,v,k+1}; m^2)\]
\end{spacing}
\STATE Let $\err_{k+1} = \E_{(\playerData,y) \sim \playerSample}[(\tilde{f}^{r,v,k+1}_\bullet(\playerData - y)^2]$ and $k = k+1$.
\STATE Let $\lsSubtranscript{k+1}=\{\playerOracleModel{k+1}\}_{v'\in[1/m^2]}$.
\STATE Let $k = k+1$.
\ENDWHILE
\STATE Let $\lsTranscript{k} = (\lsSubtranscript{0}, \ldots, \lsSubtranscript{k-1})$
\STATE \textbf{Output:} $\playerModelBoost{k-1}, \lsTranscript{k}$
\end{algorithmic}
\end{spacing}
\label{protocol:ls-boost}
\end{algorithm}

\subsubsection{Test-time Evaluation of Collaborative Model}
\label{sec:batch-test}

Upon receiving a fresh datapoint $(x_A,x_B)$ from the distribution, Alice and Bob will use their model transcripts from training and a $R$-round interaction to evaluate $f^R$ on the new datapoint. This is described in detail in Algorithm \ref{protocol:oos-batch-collab}, which itself has two subroutines, \crossboosteval\ (Algorithm \ref{protocol:crossboost-eval}) and \lsboosteval\ (Algorithm \ref{protocol:oos-ls-boost}).

\begin{algorithm}[H]
\caption{Test-time Batch Collaboration Algorithm}
\begin{spacing}{1.2}
\begin{algorithmic}
\STATE \textbf{Alice's Input: } $\aliceData, \ \Pi_A^R=\{\pi_A^1, \ldots, \pi_A^R\},m$.
\STATE \textbf{Bob's Input: } $\bobData, \ \Pi_B^R=\{\pi_B^0, \ldots, \pi_B^R\},m$.

\STATE Bob extracts $f_B^0 = \round(h^0_B; m))$ from $\pi_B^0 = \{f_B^0\}.$
\STATE Bob evaluates $\hy^0 = f_B^0(\bobData)$, and sends it to Alice.
\STATE Let $r = 0$.
\WHILE{$r < R$}
\IF{$r$ is even}
\STATE Alice updates her prediction and sends it to Bob:
\STATE She extracts $\pi^{r+1}_A$ from $\Pi_A^R$ 
\STATE From her transcript from the round and Bob's predictions $\hy^{r}$, she reconstructs $f^{r+1}_A$ and evaluates it on $\aliceData$, generating her prediction $\hy^{r+1}$ for this round:
\begin{spacing}{0.7}
\[
\hy^{r+1}=\crossboosteval(\aliceData, \hy^{r}, \pi^{r+1}_A,m)
\]
\end{spacing}
\STATE \quad  She sends her updated prediction $\hy^{r+1}$ to Bob.
\STATE \quad  Bob does nothing.
\ELSE 
\STATE Bob updates his prediction and sends it to Alice:
\STATE He extracts $\pi^{r+1}_B$ from $\Pi_B^R$ 
\STATE He reconstructs $f^{r+1}_B$ and evaluates it on $\bobData$, generating his prediction $\hy^{r+1}$ for this round:
\begin{spacing}{0.7}
\[
\hy^{r+1}=\crossboosteval(\bobData, \hy^{r}, \pi^{r+1}_B,m)
\]
\end{spacing}
\STATE \quad  He sends his updated prediction $\hy^{r+1}$ to Alice.
\STATE \quad  Alice does nothing.
\ENDIF
\STATE $r = r+1$
\ENDWHILE
\STATE \textbf{Alice's Output:} $\hy^R$
\STATE \textbf{Bob's Output:} $\hy^R$
\end{algorithmic}
\end{spacing}
\label{protocol:oos-batch-collab}
\end{algorithm}

\begin{algorithm}[H]
\caption{\crossboosteval: Test time evaluation of \crossboost}
\label{protocol:crossboost-eval}
\begin{algorithmic}
\STATE \textbf{Input:} $\playerData, \hy^{r-1},\transcript{r}=\{\lsTranscript{v}\}_{v \in [1/m]}, m.$
\STATE Let $v^* =  \hy^{r-1}$ be the value of the other player's predictions on $\playerData$. 
\STATE The player extracts $\pi^{r,v^*}$ from $\transcript{r}$.
\IF{$\pi^{r,v^*}=\bot$}
\STATE $\hy^r = \playerModel{r}(\playerData) = v^*$.
\ELSE
\STATE $\hy^r = \lsboosteval(\playerData, \pi^{r,v^*}, m)$
\ENDIF
\STATE \textbf{Output: $\hy^r$} 
\end{algorithmic}
\end{algorithm}

\begin{algorithm}[H]
\caption{\lsboosteval: Test time evaluation of \lsboost}
\begin{algorithmic}
\STATE \textbf{Input:} $\playerData, \lsTranscript{r} = \{\lsSubtranscript{0}, \ldots, \lsSubtranscript{K}\}, m.$
\STATE The player extracts $\lsSubtranscript{0} = \{h_{\bullet}^{r,v,0}\} $ from $\lsTranscript{r}$.
\STATE Let $v^*_0 = \playerModelBoost{0}(\playerData) = \round(h_{\bullet}^{r,v,0}; m^2)(\playerData).$
\STATE Let $k=0$.
\WHILE{$k<K$}
\STATE The player extracts $\lsSubtranscript{k+1} = \{h_{\bullet}^{r,v,k+1,v'}\}_{v' \in [1/m]}$ from $\lsTranscript{r}$.
\IF{$\lsSubtranscript{k+1} \ne \bot$} 
\STATE Let $v^*_{k+1} = \playerModelBoost{k+1}(\playerData) = \round(h_\bullet^{r,v,k+1,v^*_k}(\playerData); m^2)$.
\STATE Let $k = k+1$
\ELSE
\STATE \textbf{Output: } $v^*_{k} = \playerModelBoost{k}(\playerData)$.
\ENDIF
\ENDWHILE
\STATE \textbf{Output: } $v^*_{K} = \playerModelBoost{K}(\playerData)$.
\end{algorithmic}
\label{protocol:oos-ls-boost}
\end{algorithm}

\subsection{Algorithm Analysis}
\label{sec:batch-alg-analysis}

We will first show that the COLLABORATE algorithm is guaranteed to converge in a small number of rounds. We will then show that if Alice and Bob's model classes satisfy a joint weak learning condition with respect to $\cH_J$, then the output of the COLLABORATE algorithm will have low regret with respect to $\cH_J$, and finally will demonstrate that it generalizes out of sample.

First, we state our convergence guarantee, the proof of which may be found in Appendix \ref{app:batch-convergence}. 

\begin{theorem}
\label{thm:batch-convergence}
In training, the subprocess \lsboost\ converges after $K=m^2$ (sub)rounds, and the COLLABORATE Algorithm \ref{protocol:batch} converges after $R=m^2$ rounds on the training sample $S$.
\end{theorem}

We now prove an in-sample accuracy theorem for COLLABORATE. The proof of this statement follows from our Boosting Lemma \ref{lem:weak-learning-nodist} and a series of Lemmas, the formal statements and proofs of which may be found in Appendix \ref{app:batch-boosting}: 

\begin{itemize}
\item Any time that \lsboost\ is invoked, the resulting model will have small swap regret with respect to the players' own hypothesis class on the subset of data it was called on. (Lemma \ref{lem:lsboost-regret})
\item For any invocation of the \crossboost\ algorithm, either a model from \lsboost\ will be used or a constant predictor from the other player will be. If a model from \lsboost\ was used, it will have small swap regret on that subsample. And if not, the regret of the constant predictor which is used instead cannot be too much bigger, because the player only decided to use this constant predictor because the improvement from using \lsboost\ instead was small. Summing over the players' level sets gives a swap-regret guarantee on the entire model generated by \crossboost\ with respect to the players' own hypothesis class and their sample. (Lemma \ref{lem:crossboost-regret}) 
\item Because the final predictions by Alice and Bob generated by COLLABORATE always agree, the final predictions have low swap regret on $\cH_A \cup \cH_B$. (Corollary \ref{cor:batch-regret})
\item Hence, if $\cH_A$ and $\cH_B$ satisfy the weak learning condition with respect to $\cH_J$, we can directly apply the boosting result from Lemma \ref{lem:weak-learning-nodist}.
\end{itemize}

\begin{theorem}
\label{thm:boosting-batch}

    Let $\cH_J$ be a hypothesis class over the joint feature space $\cX$, and let $\cH_A = \{h_A:\cX_A\to\cY\}$ and $\cH_B = \{h_B:\cX_B\to\cY\}$ be hypothesis classes over $\cX_A$ and $\cX_B$ respectively. Let $f^R$ be the final model output by COLLABORATE. Then, if $\cH_A$ and $\cH_B$ jointly satisfy the $w(\cdot)$-weak learning condition with respect to $\cH_J$,
    \[
    \E_S[(f^R(x) - y)^2] - \min_{h_J\in\cH_J}\E_S[(h_J(x)-y)^2] \leq 2w\inv \left(\frac{3}{m}\right),
    \]
    whenever the inverse of $w$ exists. 
\end{theorem}

\begin{proof}
We know from Corollary \ref{cor:batch-regret} that the final models $f_A^R$ and $f_B^R$ output by Algorithm \ref{protocol:batch} have $(3/m, \cH_A \cup \cH_B)$- swap regret on the sample $S$. By assumption, $\cH_A$ and $\cH_B$ jointly satisfy the $w(\cdot)$-weak learning condition with respect to $\cH_J$. So, we can directly apply boosting Lemma \ref{lem:weak-learning-nodist}, which will guarantee that 
\[
    \E_S[(f^R(x) - y)^2] - \min_{h_J\in\cH_J}\E_S[(h_J(x)-y)^2] \leq 2w\inv \left(\frac{3}{m}\right).
\]
\end{proof}

This gives us in-sample accuracy guarantees. Ultimately we are interested in out of sample accuracy guarantees. We conclude this section by stating a slightly informal version of our generalization theorem. The formal statement and its proof are in Appendix \ref{app:batch-generalization}.

\begin{theorem}
\label{thm:batch-generalization}
Let $\varepsilon, \delta > 0$ and let $\cF$ be the class of models output from Algorithm \ref{protocol:batch} as described in Definition \ref{def:batch-functions}. Let $d$ be the pseudodimension of Alice and Bob's joint hypothesis class $\cH_J$. Let $S = \{(\aliceData, \bobData,y_i)\}_{i\in[n]}\sim \cD^n$ be a sample of $n$ iid points drawn from $\cD$. Then, if  
    \[
    n \ge O\left(\frac{m^7 d \log(m d) + \log(1/\delta)}{\varepsilon^2} \right),
    \]
    \[
    \Pr\left( \max_{f \in \cF} \left\vert \E_{(\aliceData, \bobData, y)\sim \cD} \left[(y-f(x))^2\right]-\E_{(\aliceData, \bobData, y)\sim S} \left[(y-f(x))^2\right] \right\vert \ge \epsilon \right) \le \delta.
    \]
\end{theorem}

\section{Lifting to the One-Shot Bayesian Setting}
\label{sec:bayes}
Our paper primarily concerns itself with information aggregation in frequentist settings --- both the online adversarial setting studied in Sections \ref{sec:online-alg} and \ref{sec:action} in which there is no distribution at all, and the batch setting studied in Section \ref{sec:batch} in which there is a distribution, but the learners have no prior knowledge of it except through a training sample. However, the theorems we prove can be lifted to the one-shot Bayesian setting studied by \cite{aumann1976,aaronson2004complexity,bo2023agreement}, which extends and generalizes the information aggregation result from \cite{bo2023agreement} in the original setting of Aumann's agreement theorem. We generalize the information aggregation theorem of \cite{bo2023agreement} in two ways: first, our weak learning condition is strictly weaker than the information substitutes condition given by \cite{bo2023agreement} --- for example, as we have shown, our weak learning condition is satisfied by \emph{linear} functions, whereas the information substitutes condition is not (as we demonstrate in Section \ref{sec:lower-bounds}). Second, our information aggregation theorems are agnostic in the sense that we can guarantee that \emph{independently of the prior distribution}, Bayesians with a common prior must agree on predictions that are as accurate as the best model on their joint feature space in any hypothesis class with bounded  fat shattering dimension, so long as the hypothesis class satisfies our weak learning assumption. In contrast \cite{bo2023agreement} apply their information substitutes condition only to the Bayes optimal predictors on $\cX_A, \cX_B$, and $\cX$ respectively. 

Rather than the online adversarial setting we study in Sections \ref{sec:online-alg} and \ref{sec:action}, we assume (as we do in Section \ref{sec:batch}) that instances are drawn  from $\cD$: $(x_A, x_B, y) \sim \cD$, where $\cD$ is a joint distribution over $\cX_A \times \cX_B \times \cY$. However, unlike in Section \ref{sec:batch} we now assume that this distribution is known to both Alice and Bob as their (common) prior distribution. We now model Alice and Bob as perfect Bayesians, who at each round of conversation, form a posterior distribution conditional on all of their observations thus far (both the features visible to them and the transcript of the conversation so far) and communicate their posterior expectation of $y$. For simplicity, rather than communicating these expectations to arbitrary precision, Alice and Bob communicate expectations rounded to multiples of some discretization parameter $m \in \mathbb{N}$ (which guarantees among other things that the communication requires only a bounded number of bits). Let $[\frac{1}{m}]$ represent the discretization of the unit interval into $m$ grid points: $\{0, \frac{1}{m}, \frac{2}{m}, \ldots, 1\}$. We denote a prediction $\hat{y}$ that is rounded to the nearest multiple of $\frac{1}{m}$ as $\bar{y}$.

\begin{definition}[Bayesian Learner]
\label{def:bayesian-learner}
Fix a joint distribution $\cD \in \Delta (\cX_A \times \cX_B \times \cY)$ over features observable to Alice, features observable to Bob, and labels. We say that Alice (resp., Bob) is a Bayesian Learner if for all $t, k > 0$, given observable features $x_A^t$, prediction transcript $\pi^{1:t-1}$, and conversation $C^t_{1:k-1}$, they make a prediction as 
\[ \hat{y}^{t, k}_A = \mathbb{E}_{\cD} [ Y | x^t_A, \pi^{1:t-1}, C^t_{1:k-1}]. \]
\end{definition}

\begin{protocol}[ht]
\begin{algorithmic}
    \STATE{ {\bf Input} $(\cD, \cY, K)$}
    \FOR{each day $t = 1, \ldots$}
        \STATE Receive $x^t = (x^t_A,x^t_B, y^t) \sim \cD$. Alice sees $x^t_A$ and Bob sees $x^t_B$.
        \FOR{each round $k = 1, 2, \ldots,K$}
            \IF { $k$ is odd}
                \STATE Alice predicts $\ymk{t}{k} \in \cY$, and sends Bob $\bar{y}^{t, k}_A$ (i.e. the rounded version of $\ymk{t}{k}$)
            \ENDIF
            \IF{ $k$ is even}
                \STATE Bob  predicts $\yhk{t}{k} $, and sends Alice $\bar{y}^{t, k}_B$. 
                \ENDIF
        \ENDFOR
        \STATE{Alice and Bob observe $y^t \in \cY$.}
    \ENDFOR

\end{algorithmic}
\caption{\textsc{Bayesian $K$-round Collaboration Protocol}}  \label{alg:bayes-collaboration}
\end{protocol}

Our argument will proceed as follows:
\begin{enumerate}
    \item First, we observe that the predictions of a Bayesian are always unbiased at the time they are made. Among other things, this implies that a Bayesian always has no \emph{expected} conversation swap regret with respect to \emph{any} benchmark policy. 
    \item A consequence of this is that a Bayesian's average realized conversation swap regret tends to zero as the number of days of interaction tends to infinity, for any benchmark class for which the realized squared error uniformly converges to the expected squared error with sufficiently many samples. This is the case for any benchmark class of policies with finite fat shattering dimension \citep{anthony1999neural}.
    \item Thus, if we imagine sampling $T$ instances $(x_A,x_B,y) \sim \cD$  from the prior distribution and two Bayesians collaborating on these instances, in the limit as $T\rightarrow \infty$, we can apply our information aggregation theorems with respect to any benchmark class that satisfies our weak learning condition and has bounded fat shattering dimension to bound the expected squared error of the final predictions.
    \item Finally, we observe that since the examples are drawn i.i.d. and Bayesians will not condition on the history of past instances (as they are independent from the current instance),  the distribution on the sequence of interactions is permutation invariant. Thus we can bound the expected squared error of the prediction arrived at for the \emph{first} example, and hence our theorems apply even when $T = 1$.
\end{enumerate}

The broad strokes of this proof strategy mirror how \cite{collina2025tractable} lifted their sequential agreement theorems to the one-shot Bayesian setting. Since we aim for the stronger goal of information aggregation, we must now reason about swap regret with respect to an infinite benchmark class (rather than simple calibration).

\subsection{Bayesians and Conversation Swap Regret}

We first want to establish that Bayesians will have low conversation swap regret (Definition \ref{def:CSR}) when they participate in a sequential collaboration protocol (Protocol \ref{alg:bayes-collaboration}). Then, in the following section, we can proceed by instantiating Theorem \ref{thm:last-round}. In fact, Bayesians always have zero  \emph{expected} swap regret with respect to any fixed class of benchmark functions. To bound their realized swap regret, we need to uniformly bound the loss with respect to its expectation across every function in the benchmark class $\cH_B$, which is the step that causes us to require that $\cH_B$ has bounded  fat shattering dimension. 

\begin{theorem} \label{thm:bayesians-swap}
    Fix $\delta, \epsilon, m > 0$. Suppose the fat shattering dimension of $\cH_A$ is finite at any scale $\eps$. Fix transcript $\pi^{1:T,1:K}$. 
    Let $v$ range over values in $[\frac{1}{m}]$ and $g_B(T)$ be some bucketing.
    If Alice is a Bayesian learner with discretization $m$, 
    with probability $1 - \delta$, Alice's sequence of predictions resulting from Protocol \ref{alg:bayes-collaboration} has low conversation swap regret with respect to bucketing $g_B(T)$ and function class $\cH_A$: for all odd rounds $k$ and buckets $i \in \{1, \ldots, \frac{1}{g_B(T)} \}$ such that $\Pr{[ \bar{y}_B \in B_i]} > 0$, if $|T_B(k-1, i)|> \frac{ C_{\cH}^{\eps/256}  \ln(\frac{1}{\eps}) + \ln(\frac{1}{\delta}) } {\eps^2 }$: 
    \begin{align*}
        \sum_{t \in T_B(k-1, i)} &(\bar{y}^{t,k} - y^t)^2 - \sum_v{\min_{h \in \cH_A}} \sum_{t \in T_B(k-1, i)} \mathbb{I}[\bar{y}^{t,k} = v] (h(x^t) - y^t)^2 \\
        &\leq 2\sqrt{2T\ln{\frac{g_B(T)K}{\delta}}} + \frac{T}{m^2} + mT\eps ,
    \end{align*}
    where $T_B(k-1,i) = \left\{t  ~|~ \bar{y}^{t,k-1}_B \in B_i(1/g_B(T))\right\}$ is the 
    subsequence of days where the predictions of Bob at the previous round fall in bucket $i$, $C^{\eps}_{\cH_A}$ is the fat shattering dimension of $\cH_A$ at scale $\eps$, and $K$ is the number of rounds on each day. A symmetric condition holds for Bob. 
\end{theorem}

Before proceeding to the proof of Theorem \ref{thm:bayesians-swap}, we first formalize a simple observation: if we resample the label every day after the $j^\text{th}$ round of conversation \emph{from the posterior distribution on the label conditional on the transcript of interaction so far},  this does not change the distribution of transcripts. This allows us to conduct all of our subsequent analysis under this resampling thought experiment. 

\begin{lemma}[Lemma 6.3 of \cite{collina2025tractable}] \label{lem:resampling-transcript} 
    Let $\cD$ be a probability distribution over  $\cX_m \times \cX_h \times \cY$ and fix a day $t \in [T]$. 
    Fix a transcript through day $t-1$: $\pi^{1:t-1}$.
    \begin{itemize}
        \item Consider an interaction at day $t$ under Protocol \ref{alg:bayes-collaboration}.
        Let $\pi^t$ be the transcript of day $t$ from this interaction.
        \item Fix an arbitrary round $j$. Consider an interaction when $(x_A, x_B, y^t)$ is sampled from $\cD$ at the beginning of day $t$ and then Alice and Bob correspond according to Protocol \ref{alg:bayes-collaboration} until round $j$. Then, in round $j$, the outcome is resampled from the posterior distribution conditional on the information observed so far: $y' \sim \cD_{\cY} | x^t_A,  {\pi}^{1:t-1}, C^{t}_{1:j-1}, \pmk{t}{j}$, where $\cD_{\cY}$ is the marginal distribution on $\cY$. Let $\bar{\pi}^t_{j}$ be the transcript of day $t$ from this interaction, with $y^t$ replaced with $y'$.
    \end{itemize}
    For all rounds $k$,
     \begin{align*}
        \Pr_{\cD}[\pi^{t, 1:k}] = \Pr_{\cD}[\bar{\pi}^{t,1:k}_{j}].
    \end{align*}
\end{lemma}

Proofs in this section are deferred to Appendix \ref{app:bayes}.

Now, we analyze the expected conversation swap regret of Alice and Bob. Recall that in the definition of swap regret (Definition \ref{def:CSR}), we compare the squared error of Alice's (or Bob's) predictions to the predictions of the best comparator function in the benchmark class, separately for each level set of their prediction. Here, since predictions are restricted to the discretization $[\frac{1}{m}]$, we  have $m$ level sets. We first want to argue that for any possible swap function (i.e. selection of $m$ functions from $\cH_A$, one for each levelset),  Alice's expected swap regret is small.

\begin{lemma} \label{lem:bayes-expected-regret}
Fix some bucketing function $g_B(\cdot)$. If Alice is a Bayesian as in Protocol \ref{alg:bayes-collaboration}, she has low expected conversation swap regret with respect to any fixed swap function $\{h_0, h_{\frac{1}{m}}, \ldots,  h_1\} \in \{\cH_A\}^m$, where $h_v$ is the function she compares to her prediction $v \in [\frac{1}{m}]$. For all odd rounds $k$ and buckets $i \in \{1, \ldots, \frac{1}{g_B(T)}\}$: 
\begin{align*}
    \max_{\{h_0, h_{\frac{1}{m}}, \ldots,  h_1\} \in \{\cH_A\}^m} \E_{\cD} \left[(\bar{y}^{t, k}_A - y^{t})^{2} - \sum_{v} \mathbb{I}[\bar{y}^{t, k}_A = v](h_v(x^{t}) - y^{t})^{2} \right] \leq \frac{1}{m^2}.
\end{align*}
\end{lemma}

Having established that Bayesians have low \emph{expected} swap regret with respect to any fixed set of swap functions, we now want to establish that they have low \emph{realized} swap regret with high probabiilty over sufficiently long interactions, for large families of swap functions. We do this by applying two concentration arguments. The first (which establishes that the realized squared error of each sequence of predictions made by Alice and Bob are close to their expected squared error) is just an application of Azuma's inequality:

\begin{lemma} \label{lem:bayes-loss-concentration}
    Fix $T, \delta> 0$ and bucketing $g_B(T)$. Let $\pi^{1:T,1:K}$ be the transcript after running Protocol \ref{alg:bayes-collaboration} for $T$ days. For all even rounds $k$ and buckets $i \in \{1, \ldots, \frac{1}{g_B(T)}\}$, with probability $1-\delta$, 
    \begin{align*}
        \sum_{t=1}^T (\bar{y}^{t, k}_A - y^t)^2 -\E_{\cD}[(\bar{y}_A^{:, k}-y^t)^2 | \pi^{1:t-1}] \leq 2\sqrt{2T\ln{\frac{1}{\delta}}}.
    \end{align*}
\end{lemma}

To argue that Bayesians have low realized swap regret with respect to a (possibly infinite) benchmark class of functions, we next need to argue that the squared error for every function in the benchmark class (across each of the levelsets of our predictions) concentrates uniformly around its expectation. To do this we recall  the fat shattering dimension, which captures the capacity of real-valued function classes \citep{anthony1999neural}. Full details are in Appendix \ref{app:bayes}.

\begin{lemma} \label{lem:bayes-benchmark-concentration}
    Fix $\eps, \delta > 0$. Let $|T_B(k-1, i)| > \frac{ C_{\cH}^{\eps/256}  \ln(\frac{1}{\eps}) + \ln(\frac{1}{\delta}) } {\eps^2 } $,  where $C^{\eps}_{\cH_A}$ is  the fat shattering dimension of $\cH_A$ at scale $\eps$. Fix bucketing $g_B(T)$. Let $\pi^{1:T,1:K}$ be the transcript after running Protocol \ref{alg:bayes-collaboration} for $T$ days. For all even rounds $k$, buckets $i \in \{1, \ldots, \frac{1}{g_B(T)}\}$ for Bob's prediction in round $k-1$, and level set $v \in [\frac{1}{m}]$ of Alice's prediction in round $k$, with probability $1-\delta$, 
    \begin{align*}
        \sup_{h \in \cH_A} \left|\frac{1}{|T_B(k-1, i)|} \sum_{t \in T_B(k-1, i)} \mathbb{I}[\bar{y}^{t, k}_A = v] (h(x^t) - y)^2 - \E_{\cD}[ \mathbb{I}[\bar{y}_A(x) = v](h(x) - y)^2 | \pi^{1:t-1}] \right|
        & \leq \eps.
    \end{align*}
\end{lemma}

Finally, we can proceed to the proof of Theorem \ref{thm:bayesians-swap}, which gives a high probability bound on the realized swap regret on the predictions made by Bayesians in Protocol \ref{alg:bayes-collaboration}. The proof is deferred to Appendix \ref{app:bayes}, but follows directly from Lemmas \ref{lem:bayes-expected-regret}, \ref{lem:bayes-loss-concentration}, and \ref{lem:bayes-benchmark-concentration}.

\subsection{Online to One-Shot Reduction}

In this section, we show that if an instance is drawn from a common prior and both agents are Bayesian, then our theorems which guarantee information aggregation with high probability on all instances over an arbitrarily long sequence of length $T$ hold in fact for a \emph{single} conversation with high probability.

We can imagine an arbitrarily long sequence of conversations over many days. Each conversation on any given day continues for exactly $K$ rounds. We have shown that Bayesians satisfy our notion of conversation swap regret with parameters growing sublinearly with $T$.
In a Bayesian setting, since instances are drawn i.i.d. from a fixed prior, Bayesians need not condition on information from prior days. Thus, instances drawn each day (and subsequently, conversations each day) are distributed identically. Therefore, the theorems we give which apply to the average cumulative regret over the course of a subsequence of length $T$ also holds in expectation over the draw from the prior, on any single instance. Since we don't actually need to run the protocol for $T$ rounds to get predictions on the first round, we can take $T\rightarrow \infty$ (as it is just a thought experiment). 

\begin{protocol}[ht]
\begin{algorithmic}
    \STATE{ {\bf Input} $\cD \in \Delta(\cX_h \times \cX_m \times \cY)$, instance for which you want information aggregation: $(x^*_h, x^*_m, y^*) \sim \cD$ }
    \STATE{ {\bf Parameter} number of samples: $T$}
  \STATE{Fix $(x^1_h, x^1_m, y^1) \sim \cD$}
  \STATE{For $t \in \{2,\ldots,T\}$ draw $(x^t_h, x^t_m, y^t) \sim \cD$}
    \FOR{each day $t = 1, \ldots,T$}
        \STATE Alice observes $x^t_A$ and Bob observes $x^t_B$.
        \FOR{each round $k = 1, 2, \ldots,L$}
            \IF { $k$ is odd}
                \STATE Alice predicts $\ymk{t}{k} $, and sends Bob $\bar{y}^{t, k}_A$
            \ENDIF
            \IF{ $k$ is even}
                \STATE Bob predicts $\yhk{t}{k} \in \cY$, and sends Alice $\bar{y}^{t, k}_B$
            \ENDIF
        \ENDFOR
        \STATE{Alice and Bob observe $y^t \in \cY$}
    \ENDFOR

\end{algorithmic}
\caption{\textsc{Online-to-One-Shot: General Bayesian Collaboration Protocol}}  \label{alg:one-shot}
\end{protocol}

\begin{theorem} \label{thm:one-shot}
        Let $\cH_J$ be a hypothesis class over the joint feature space $\cX$.
    Let $\cH_A = \{h_A:\cX_A\to\cY\}$ and $\cH_B = \{h_B:\cX_B\to\cY\}$ be hypothesis classes over $\cX_A$ and $\cX_B$.  
    Consider instance $(x_A, x_B, y) \sim \cD$. If
    \begin{itemize}
        \item Alice and Bob are both Bayesian learners
        \item  $\cH_A$ and $\cH_B$ have finite fat shattering dimension at every scale, and $\cH_A$ and $\cH_B$ jointly satisfy the $w(\cdot)$-weak learning condition with respect to $\cH_J$, for continuous $w(\cdot)$ such that $w(\gamma) > 0$ for all $\gamma > 0$,
    \end{itemize} then, if they engage in $K$ rounds of conversation on a single instance  $(x_A, x_B, y)$, the prediction in round $K$ will have regret to the best function in $ \cH_J$ bounded by:
     \begin{align*}
        \E[(\hat{y}^{1, K} - y)^2] - \min_{h_j \in \cH_J} \E[(h_j(x) - y)^2] \leq O\left(w\inv\left(K^{-\frac{1}{3}}\right)\right).
    \end{align*} 
\end{theorem}

We can instantiate the above result for  bounded norm linear functions, which satisfy our weak learning guarantee (Theorem \ref{thm:linear-weak-learning}).
\begin{remark}
    If $\cH_A$ and $\cH_B$ are the classes of  linear functions with bounded norm parameter vectors: $\cH_A = \{x_A \to \theta^T x : \|\theta\|_2 < C\}$ and $\cH_B = \{x_B \to \theta^T x : \|\theta\|_2 < C\}$ and $\cH_J$ is the Minkowski sum of $\cH_A$ and $\cH_B$, then for an arbitrary prior distribution, when Bayesian learners engage in a conversation of length $K$:
   \[ \E[(\hat{y}^{1, k} - y)^2] - \min_{h_j \in \cH_J} \E[(h_j(x) - y)^2] \leq   O(C K^{-\frac{1}{6}}). \]
\end{remark}

\section{Lower Bounds: Necessity of Interaction, Weak Learning and Swap Regret}
\label{sec:lower-bounds}

Lastly, we provide qualitative lower bounds to motivate the design choices in our collaborative learning protocols. We demonstrate the necessity of interaction between parties, the necessity of a condition like our weak learning assumption for achieving information aggregation guarantees, and the necessity of using a stronger criterion than external regret (like swap regret) within the protocol.

\paragraph{Interaction is Necessary.} One might wonder if interaction is necessary, especially when the underlying function classes are simple, like linear functions, which satisfy our weak learning condition (Theorem~\ref{thm:linear-weak-learning}). Perhaps some non-adaptive combination of the optimal \emph{linear} predictors $h_A^*(x_A)$ and $h_B^*(x_B)$ is sufficient to achieve performance competitive with the optimal joint linear predictor $h_J^*(x)$. The following example, adapted from the proof of Theorem~\ref{thm:bounded-necessary}, shows this is not the case. It demonstrates that even when the Bayes optimal predictors are themselves linear for Alice, Bob, and the joint feature space, the information required for optimal joint prediction might not be recoverable just by combining the optimal individual linear predictors.

\begin{theorem}[Interaction Necessity for Linear Functions]
\label{thm:interaction-necessity}
There exists a joint distribution $\cD$ over $\cX_A \times \cX_B \times \cY$ and classes $\cH_A, \cH_B, \cH_J$ corresponding to (bounded) linear functions over $\cX_A$, $\cX_B$, and $\cX = \cX_A \times \cX_B$ respectively, such that for any  $f: \cY \times \cY \rightarrow \cY$, 
\[ \E_{\cD}\left[(f(h_A^*(x_A), h_B^*(x_B)) - y)^2\right] > \min_{h_J \in \cH_J}\E_{\cD}\left[(h_J(x) - y)^2\right], \]
where $h_A^*$, $h_B^*$ are the optimal linear predictors in $\cH_A$, $\cH_B$ respectively.
\end{theorem}
The proof uses a similar construction as \Cref{thm:bounded-necessary} where the label has no correlation with Alice's features, and weak correlation with Bob's feature. However, subtracting Alice's features from Bob's features gives the signal exactly. Since the optimal predictor on Alice's features is 0 (no correlation to the label), there is no way for any aggregation function to subtract Alice's features from Bob's to get the performance of the joint predictor.

\paragraph{The Weak Learning Condition is Necessary for Boosting.}
Our boosting result (Theorem~\ref{thm:weak-learning}) shows that if $\cH_A$ and $\cH_B$ satisfy the weak learning condition with respect to $\cH_J$, then achieving low swap regret with respect to $\cH_A \cup \cH_B$ implies low external regret with respect to $\cH_J$. We now show this condition is necessary: if a triple $(\cH_A, \cH_B, \cH_J)$ fails the weak learning condition, there exist distributions and prediction sequences with no swap regret to $\cH_A \cup \cH_B$ but positive external regret to $\cH_J$.

\begin{theorem}[Necessity of Weak Learning for Boosting]
\label{thm:weak-learning-necessity}
For any triple of function classes $(\cH_A, \cH_B, \cH_J)$ that fails to satisfy the $w(\cdot)$-weak learning condition (Definition~\ref{def:joint-weak}) for any strictly increasing function $w$, there exists a sequence of examples $(x_A^t, x_B^t, y^t)_{t=1}^T$ and predictions $\hat{y}^{1:T}$ such that, as $T \rightarrow \infty$, the sequence $\hat{y}^{1:T}$ has 0 swap regret with respect to $\cH_A$ and $\cH_B$, but has positive external regret with respect to $\cH_J$.
\end{theorem}
The proof follows from observing that if the weak-learning condition is not satisfied for any $w(\cdot)$ then there is a distribution such that the joint predictor gets a non-zero gain over the constant predictor but both Alice and Bob do not improve over the constant predictor. Now predicting according to the best constant predictor guarantees no swap-regret to either Alice or Bob, but has non-zero external regret to the joint predictor, since there is a joint predictor better than the constant predictor on the distribution.

\paragraph{Weak Learning is Weaker than Information Substitutes.}
We show that our weak learning condition is strictly weaker than the ``information substitutes'' condition studied by  \cite{bo2023agreement}. The concept of information substitutes, in the context of Bayesian agreement, fundamentally concerns the diminishing marginal value of information. When applied to predictors, it says that the improvement gained by adding Bob's information (or signal) is smaller if the Alice's information is already available, and vice-versa.

To translate this concept for comparing function classes $(\cH_A, \cH_B, \cH_J)$, we need a measure of the ``value'' provided by each parties features when used by functions in one of these classes. In prediction tasks with squared error loss, a natural measure of value is the reduction in expected squared error compared to a baseline constant predictor. This gives us the following condition:
\begin{definition}[Information Substitutes for function classes]
\label{def:joint-information-sub}
    Let $\cH_A: \cX_A \rightarrow \cY$ and $\cH_B: \cX_B \rightarrow \cY$ be hypothesis classes for Alice and Bob, respectively, and let $\cH_J$ be a hypothesis class of  over the joint features $\cH_J: \cX_A \times \cX_B \rightarrow \cY$. We say model classes $\cH_A$ and $\cH_B$ satisfy information substitutes with respect to $\cH_J$ if, for all distributions $\cD$, 
    \[
    \min_{h_A \in \cH_A}\E[(h_A(x) - y)^2] - \min_{h_J \in \cH_J}\E[(h_J(x) - y)^2] \le \min_{c \in \mathbb{R}}\E[(c - y)^2] - \min_{h_B \in \cH_B}\E[(h_B(x) - y)^2]
    \]
\end{definition}
Information substitutes, as defined here, imposes a stronger, quantitative relationship on the magnitudes of the maximum achievable gains compared to weak-learning which asks if any positive gain with the joint features implies some positive gain for either individual class.  
\begin{lemma}\label{lem:weak-is-weaker}
If model classes $\cH_A$ and $\cH_B$ satisfy information substitutes with respect to $\cH_J$, they also jointly satisfy the $w(\cdot)$-weak learning condition with respect to $\cH_J$ for $w(\gamma) = \gamma/2$.
\end{lemma}

Combining this with \Cref{thm:quadratic} gives us that our weak-learning condition is significantly weaker than the information substitutes condition.
\begin{corollary}\label{cor:linear-weak-IS}
There exists $\cH_A, \cH_B, \cH_J$ that satisfy the $w(\cdot)$-weak learnability condition for $w(\gamma) = \Theta(\gamma^2)$ but do not satisfy the information substitutes condition. In fact, the class of bounded linear functions over $\cX_A = \cX_B = [-1,1]$ witnesses this gap.
\end{corollary}

\paragraph{External Regret is Insufficient.}
Our protocol aims to produce predictions $p$ that have low swap regret with respect to $\cH_A$ and $\cH_B$. One might ask if the weaker condition of low \emph{external} regret would suffice. That is, if $p$ has low external regret to $\cH_A$ and low external regret to $\cH_B$, does it follow (under the weak learning condition) that $p$ has low external regret to $\cH_J$? The following example shows the answer is no, even for linear functions where the weak learning condition holds.

\begin{theorem}
    \label{lem:swap-necessary}
    There exists a joint distribution $\cD$ over $\cX_A \times \cX_B \times \cY$ and classes $\cH_A, \cH_B, \cH_J$ corresponding to linear functions over $\cX_A, \cX_B,$ and $\cX_A  \times \cX_B$, respectively, such that there exists a sequence of examples $(x_A^t, x_B^t, y^t)_{t=1}^T$ and predictions $\hat{y}^{1:T}$ such that, as $T \rightarrow \infty$, the sequence $\hat{y}^{1:T}$ no external regret to $\cH_A$ and $\cH_B$, but has positive external regret with respect to $\cH_J$. 
\end{theorem}


\section{Discussion and Future Work}
\label{sec:discussion}
We present efficient protocols for collaborative information aggregation, enabling two parties with distinct feature spaces—even if mutually illegible—to provably achieve the accuracy of joint feature access without sharing their features. Our protocols require the two parties to operate only on their own feature spaces and communication occurs solely through label predictions or best-response actions, making our framework practical in modern AI systems, particularly human-AI interaction and multi-modal settings, where challenges like privacy, data modality differences, and computational overheads often render feature sharing impractical. Moreover, our protocols underscore the fundamental role of interaction to achieve performance that surpasses that of the individual parties, or simple non-interactive aggregation methods, opening up a new avenue of research in collaborative learning.

Our work naturally leaves open several questions. Theoretically, extending the analysis of the weak learning condition beyond the linear-like classes and Minkowski sum structure would broaden the applicability of our framework to more complex function classes encountered in practice.  Additionally, our online guarantees hold against worst-case adversarial sequences, hence, exploring settings under beyond-worst-case assumptions—for instance, leveraging models like smoothed analysis \cite{haghtalab2024smoothed} or incorporating mechanisms such as selective prediction \cite{goel2023adversarial}—could potentially yield tighter bounds and reduce the number of communication rounds. 

From a practical and safety perspective, the current protocols assume honest participation of both parties. A crucial direction, particularly for human-AI collaboration, involves designing protocols inherently robust to strategic manipulation, mitigating risks where a capable AI might deceptively steer outcomes towards misaligned objectives. Ensuring trustworthiness in these interactions would require designing strategy-proof protocols within our collaborative framework. 

Empirically evaluating the feasibility of our protocols is an important direction. While empirical evaluations in realistic human-AI settings may be challenging, evaluations in the multi-modal setting should be a good test ground for understanding the practical challenges of scalability, performance, and communication efficiency, to guide further development of the framework.



\subsection*{Acknowledgments.} We thank Avrim Blum for enlightening discussions about the relationship between Agreement and co-training, Kate Donahue for pointing out the connection to ``complementarity'' in the human/machine interaction literature and Damek Davis for identifying the necessity of boundedness for linear functions to satisfy weak-learning. This research was supported in part by an IBM PhD fellowship, the Hans Sigrist prize, the Simons collaboration on Algorithmic Fairness, and NSF grant CCF-1934876.

\bibliographystyle{plainnat}
\bibliography{refs}

\begin{thebibliography}{56}
\providecommand{\natexlab}[1]{#1}
\providecommand{\url}[1]{\texttt{#1}}
\expandafter\ifx\csname urlstyle\endcsname\relax
  \providecommand{\doi}[1]{doi: #1}\else
  \providecommand{\doi}{doi: \begingroup \urlstyle{rm}\Url}\fi

\bibitem[Aaronson(2005)]{aaronson2004complexity}
Scott Aaronson.
\newblock The complexity of agreement.
\newblock In \emph{Proceedings of the thirty-seventh annual ACM symposium on Theory of computing}, pages 634--643, 2005.

\bibitem[Alur et~al.(2024)Alur, Raghavan, and Shah]{alur2024human}
Rohan Alur, Manish Raghavan, and Devavrat Shah.
\newblock Human expertise in algorithmic prediction.
\newblock \emph{Advances in Neural Information Processing Systems}, 37:\penalty0 138088--138129, 2024.

\bibitem[Anthony and Bartlett(1999)]{anthony1999neural}
M.~Anthony and P.L. Bartlett.
\newblock \emph{Neural Network Learning: Theoretical Foundations}.
\newblock Cambridge University Press, 1999.
\newblock ISBN 9780521573535.
\newblock URL \url{https://books.google.com/books?id=UH6XRoEQ4h8C}.

\bibitem[Arunachaleswaran et~al.(2025)Arunachaleswaran, Collina, Roth, and Shi]{arunachaleswaran2025elementary}
Eshwar~Ram Arunachaleswaran, Natalie Collina, Aaron Roth, and Mirah Shi.
\newblock An elementary predictor obtaining $2\sqrt \textit{T} + 1$ distance to calibration.
\newblock In \emph{Proceedings of the 2025 Annual ACM-SIAM Symposium on Discrete Algorithms (SODA)}, pages 1366--1370. SIAM, 2025.

\bibitem[Aumann(1976)]{aumann1976}
Robert~J. Aumann.
\newblock {Agreeing to Disagree}.
\newblock \emph{The Annals of Statistics}, 4\penalty0 (6):\penalty0 1236 -- 1239, 1976.
\newblock \doi{10.1214/aos/1176343654}.
\newblock URL \url{https://doi.org/10.1214/aos/1176343654}.

\bibitem[Azoury and Warmuth(1999)]{azoury1999relative}
Katy~S. Azoury and M.~K. Warmuth.
\newblock Relative loss bounds for on-line density estimation with the exponential family of distributions.
\newblock In \emph{Proceedings of the Fifteenth Conference on Uncertainty in Artificial Intelligence}, UAI'99, page 31–40, San Francisco, CA, USA, 1999. Morgan Kaufmann Publishers Inc.
\newblock ISBN 1558606149.

\bibitem[Azoury and Warmuth(2001)]{azoury2001relative}
Katy~S Azoury and Manfred~K Warmuth.
\newblock Relative loss bounds for on-line density estimation with the exponential family of distributions.
\newblock \emph{Machine learning}, 43:\penalty0 211--246, 2001.

\bibitem[Balcan et~al.(2004)Balcan, Blum, and Yang]{balcan2004co}
Maria-Florina Balcan, Avrim Blum, and Ke~Yang.
\newblock Co-training and expansion: Towards bridging theory and practice.
\newblock \emph{Advances in neural information processing systems}, 17, 2004.

\bibitem[Balcan et~al.(2012)Balcan, Blum, Fine, and Mansour]{balcan2012distributed}
Maria~Florina Balcan, Avrim Blum, Shai Fine, and Yishay Mansour.
\newblock Distributed learning, communication complexity and privacy.
\newblock In \emph{Conference on Learning Theory}, pages 26--1. JMLR Workshop and Conference Proceedings, 2012.

\bibitem[Baltru{\v{s}}aitis et~al.(2018)Baltru{\v{s}}aitis, Ahuja, and Morency]{baltruvsaitis2018multimodal}
Tadas Baltru{\v{s}}aitis, Chaitanya Ahuja, and Louis-Philippe Morency.
\newblock Multimodal machine learning: A survey and taxonomy.
\newblock \emph{IEEE transactions on pattern analysis and machine intelligence}, 41\penalty0 (2):\penalty0 423--443, 2018.

\bibitem[Bansal et~al.(2021)Bansal, Wu, Zhou, Fok, Nushi, Kamar, Ribeiro, and Weld]{bansal2021does}
Gagan Bansal, Tongshuang Wu, Joyce Zhou, Raymond Fok, Besmira Nushi, Ece Kamar, Marco~Tulio Ribeiro, and Daniel Weld.
\newblock Does the whole exceed its parts? the effect of ai explanations on complementary team performance.
\newblock In \emph{Proceedings of the 2021 CHI conference on human factors in computing systems}, pages 1--16, 2021.

\bibitem[B{\l}asiok et~al.(2023)B{\l}asiok, Gopalan, Hu, and Nakkiran]{blasiok2023unifying}
Jaros{\l}aw B{\l}asiok, Parikshit Gopalan, Lunjia Hu, and Preetum Nakkiran.
\newblock A unifying theory of distance from calibration.
\newblock In \emph{Proceedings of the 55th Annual ACM Symposium on Theory of Computing}, pages 1727--1740, 2023.

\bibitem[Blum and Mansour(2007)]{blum2007external}
Avrim Blum and Yishay Mansour.
\newblock From external to internal regret.
\newblock \emph{Journal of Machine Learning Research}, 8\penalty0 (6), 2007.

\bibitem[Blum and Mitchell(1998)]{blum1998combining}
Avrim Blum and Tom Mitchell.
\newblock Combining labeled and unlabeled data with co-training.
\newblock In \emph{Proceedings of the eleventh annual conference on Computational learning theory}, pages 92--100, 1998.

\bibitem[Blum et~al.(2017)Blum, Haghtalab, Procaccia, and Qiao]{blum2017collaborative}
Avrim Blum, Nika Haghtalab, Ariel~D Procaccia, and Mingda Qiao.
\newblock Collaborative pac learning.
\newblock \emph{Advances in Neural Information Processing Systems}, 30, 2017.

\bibitem[Blum et~al.(2021)Blum, Haghtalab, Phillips, and Shao]{blum2021one}
Avrim Blum, Nika Haghtalab, Richard~Lanas Phillips, and Han Shao.
\newblock One for one, or all for all: Equilibria and optimality of collaboration in federated learning.
\newblock In \emph{International Conference on Machine Learning}, pages 1005--1014. PMLR, 2021.

\bibitem[Camara et~al.(2020)Camara, Hartline, and Johnsen]{camara2020mechanisms}
Modibo~K Camara, Jason~D Hartline, and Aleck Johnsen.
\newblock Mechanisms for a no-regret agent: Beyond the common prior.
\newblock In \emph{2020 ieee 61st annual symposium on foundations of computer science (focs)}, pages 259--270. IEEE, 2020.

\bibitem[Cheng et~al.(2021)Cheng, Fan, Jin, Liu, Chen, Papadopoulos, and Yang]{cheng2021secureboost}
Kewei Cheng, Tao Fan, Yilun Jin, Yang Liu, Tianjian Chen, Dimitrios Papadopoulos, and Qiang Yang.
\newblock Secureboost: A lossless federated learning framework.
\newblock \emph{IEEE intelligent systems}, 36\penalty0 (6):\penalty0 87--98, 2021.

\bibitem[Collina et~al.(2024)Collina, Roth, and Shao]{collina2023efficient}
Natalie Collina, Aaron Roth, and Han Shao.
\newblock Efficient prior-free mechanisms for no-regret agents.
\newblock \emph{ACM Conference on Economics and Computation (EC)}, 2024.

\bibitem[Collina et~al.(2025)Collina, Goel, Gupta, and Roth]{collina2025tractable}
Natalie Collina, Surbhi Goel, Varun Gupta, and Aaron Roth.
\newblock Tractable agreement protocols.
\newblock In \emph{Proceedings of the 57th Annual ACM Symposium on Theory of Computing (STOC)}, 2025.

\bibitem[Donahue and Kleinberg(2021)]{donahue2021model}
Kate Donahue and Jon Kleinberg.
\newblock Model-sharing games: Analyzing federated learning under voluntary participation.
\newblock In \emph{Proceedings of the AAAI Conference on Artificial Intelligence}, volume~35, pages 5303--5311, 2021.

\bibitem[Donahue et~al.(2022)Donahue, Chouldechova, and Kenthapadi]{donahue2022human}
Kate Donahue, Alexandra Chouldechova, and Krishnaram Kenthapadi.
\newblock Human-algorithm collaboration: Achieving complementarity and avoiding unfairness.
\newblock In \emph{Proceedings of the 2022 ACM Conference on Fairness, Accountability, and Transparency}, pages 1639--1656, 2022.

\bibitem[Foster and Vohra(1999)]{FV99}
Dean~P. Foster and Rakesh Vohra.
\newblock Regret in the on-line decision problem.
\newblock \emph{Games and Economic Behavior}, 29\penalty0 (1):\penalty0 7--35, 1999.
\newblock ISSN 0899-8256.
\newblock \doi{https://doi.org/10.1006/game.1999.0740}.
\newblock URL \url{https://www.sciencedirect.com/science/article/pii/S0899825699907406}.

\bibitem[Frongillo et~al.(2023)Frongillo, Neyman, and Waggoner]{bo2023agreement}
Rafael Frongillo, Eric Neyman, and Bo~Waggoner.
\newblock Agreement implies accuracy for substitutable signals.
\newblock In \emph{Proceedings of the 24th ACM Conference on Economics and Computation}, pages 702--733, 2023.

\bibitem[Garg et~al.(2024)Garg, Jung, Reingold, and Roth]{garg2024oracle}
Sumegha Garg, Christopher Jung, Omer Reingold, and Aaron Roth.
\newblock Oracle efficient online multicalibration and omniprediction.
\newblock In \emph{Proceedings of the 2024 Annual ACM-SIAM Symposium on Discrete Algorithms (SODA)}, pages 2725--2792. SIAM, 2024.

\bibitem[Geanakoplos and Polemarchakis(1982)]{geanakoplos1982we}
John~D Geanakoplos and Heraklis~M Polemarchakis.
\newblock We can't disagree forever.
\newblock \emph{Journal of Economic theory}, 28\penalty0 (1):\penalty0 192--200, 1982.

\bibitem[Globus-Harris et~al.(2023)Globus-Harris, Harrison, Kearns, Roth, and Sorrell]{globus-harris2023multicalibration}
Ira Globus-Harris, Declan Harrison, Michael Kearns, Aaron Roth, and Jessica Sorrell.
\newblock Multicalibration as boosting for regression.
\newblock In \emph{Proceedings of the 40th International Conference on Machine Learning}, ICML'23. JMLR.org, 2023.

\bibitem[Goel et~al.(2023)Goel, Hanneke, Moran, and Shetty]{goel2023adversarial}
Surbhi Goel, Steve Hanneke, Shay Moran, and Abhishek Shetty.
\newblock Adversarial resilience in sequential prediction via abstention.
\newblock \emph{Advances in Neural Information Processing Systems}, 36:\penalty0 8027--8047, 2023.

\bibitem[Goh et~al.(2024)Goh, Gallo, Hom, Strong, Weng, Kerman, Cool, Kanjee, Parsons, Ahuja, et~al.]{goh2024large}
Ethan Goh, Robert Gallo, Jason Hom, Eric Strong, Yingjie Weng, Hannah Kerman, Jos{\'e}phine~A Cool, Zahir Kanjee, Andrew~S Parsons, Neera Ahuja, et~al.
\newblock Large language model influence on diagnostic reasoning: a randomized clinical trial.
\newblock \emph{JAMA Network Open}, 7\penalty0 (10):\penalty0 e2440969--e2440969, 2024.

\bibitem[Gomez et~al.(2025)Gomez, Cho, Ke, Huang, and Unberath]{gomez2025human}
Catalina Gomez, Sue~Min Cho, Shichang Ke, Chien-Ming Huang, and Mathias Unberath.
\newblock Human-ai collaboration is not very collaborative yet: A taxonomy of interaction patterns in ai-assisted decision making from a systematic review.
\newblock \emph{Frontiers in Computer Science}, 6:\penalty0 1521066, 2025.

\bibitem[Gopalan et~al.(2023)Gopalan, Hu, Kim, Reingold, and Wieder]{gopalan2023loss}
Parikshit Gopalan, Lunjia Hu, Michael~P Kim, Omer Reingold, and Udi Wieder.
\newblock Loss minimization through the lens of outcome indistinguishability.
\newblock In \emph{14th Innovations in Theoretical Computer Science Conference (ITCS 2023)}, 2023.

\bibitem[Green and Chen(2019)]{green2019principles}
Ben Green and Yiling Chen.
\newblock The principles and limits of algorithm-in-the-loop decision making.
\newblock \emph{Proceedings of the ACM on human-computer interaction}, 3\penalty0 (CSCW):\penalty0 1--24, 2019.

\bibitem[Haghtalab et~al.(2022)Haghtalab, Jordan, and Zhao]{haghtalab2022demand}
Nika Haghtalab, Michael Jordan, and Eric Zhao.
\newblock On-demand sampling: Learning optimally from multiple distributions.
\newblock \emph{Advances in Neural Information Processing Systems}, 35:\penalty0 406--419, 2022.

\bibitem[Haghtalab et~al.(2023)Haghtalab, Jordan, and Zhao]{haghtalab2023unifying}
Nika Haghtalab, Michael Jordan, and Eric Zhao.
\newblock A unifying perspective on multi-calibration: Game dynamics for multi-objective learning.
\newblock \emph{Advances in Neural Information Processing Systems}, 36:\penalty0 72464--72506, 2023.

\bibitem[Haghtalab et~al.(2024)Haghtalab, Roughgarden, and Shetty]{haghtalab2024smoothed}
Nika Haghtalab, Tim Roughgarden, and Abhishek Shetty.
\newblock Smoothed analysis with adaptive adversaries.
\newblock \emph{Journal of the ACM}, 71\penalty0 (3):\penalty0 1--34, 2024.

\bibitem[Hardy et~al.(2017)Hardy, Henecka, Ivey-Law, Nock, Patrini, Smith, and Thorne]{hardy2017private}
Stephen Hardy, Wilko Henecka, Hamish Ivey-Law, Richard Nock, Giorgio Patrini, Guillaume Smith, and Brian Thorne.
\newblock Private federated learning on vertically partitioned data via entity resolution and additively homomorphic encryption.
\newblock \emph{arXiv preprint arXiv:1711.10677}, 2017.

\bibitem[H{\'e}bert-Johnson et~al.(2018)H{\'e}bert-Johnson, Kim, Reingold, and Rothblum]{hebert2018multicalibration}
Ursula H{\'e}bert-Johnson, Michael Kim, Omer Reingold, and Guy Rothblum.
\newblock Multicalibration: Calibration for the (computationally-identifiable) masses.
\newblock In \emph{International Conference on Machine Learning}, pages 1939--1948. PMLR, 2018.

\bibitem[Kong and Schoenebeck(2023)]{kong2023false}
Yuqing Kong and Grant Schoenebeck.
\newblock False consensus, information theory, and prediction markets.
\newblock In \emph{14th Innovations in Theoretical Computer Science Conference (ITCS 2023)}, volume 251, page~81. Schloss Dagstuhl--Leibniz-Zentrum f $\{$$\backslash$" u$\}$ r Informatik, 2023.

\bibitem[Li and Tang(2024)]{li2024multimodal}
Songtao Li and Hao Tang.
\newblock Multimodal alignment and fusion: A survey.
\newblock \emph{arXiv preprint arXiv:2411.17040}, 2024.

\bibitem[Lu et~al.(2025)Lu, Roth, and Shi]{lu2025sample}
Jiuyao Lu, Aaron Roth, and Mirah Shi.
\newblock Sample efficient omniprediction and downstream swap regret for non-linear losses.
\newblock \emph{arXiv preprint arXiv:2502.12564}, 2025.

\bibitem[Natarajan(1989)]{natarajan1989learning}
Balas~K Natarajan.
\newblock On learning sets and functions.
\newblock \emph{Machine Learning}, 4\penalty0 (1):\penalty0 67--97, 1989.

\bibitem[Noarov et~al.(2023)Noarov, Ramalingam, Roth, and Xie]{noarov2023high}
Georgy Noarov, Ramya Ramalingam, Aaron Roth, and Stephan Xie.
\newblock High-dimensional prediction for sequential decision making, 2023.

\bibitem[Noti et~al.(2025)Noti, Donahue, Kleinberg, and Oren]{noti2025ai}
Gali Noti, Kate Donahue, Jon Kleinberg, and Sigal Oren.
\newblock Ai-assisted decision making with human learning.
\newblock \emph{arXiv preprint arXiv:2502.13062}, 2025.

\bibitem[Peng et~al.(2024)Peng, Garg, and Kleinberg]{peng2024no}
Kenny Peng, Nikhil Garg, and Jon Kleinberg.
\newblock A no free lunch theorem for human-ai collaboration.
\newblock \emph{arXiv preprint arXiv:2411.15230}, 2024.

\bibitem[Pollard(2012)]{pollard2012convergence}
David Pollard.
\newblock \emph{Convergence of stochastic processes}.
\newblock Springer Science \& Business Media, 2012.

\bibitem[Qiao and Zheng(2024)]{qiao2024distance}
Mingda Qiao and Letian Zheng.
\newblock On the distance from calibration in sequential prediction.
\newblock In \emph{The Thirty Seventh Annual Conference on Learning Theory}, pages 4307--4357. PMLR, 2024.

\bibitem[Rakhlin et~al.(2014)Rakhlin, Sridharan, and Tewari]{rakhlin2014sequential}
Alexander Rakhlin, Karthik Sridharan, and Ambuj Tewari.
\newblock Sequential complexities and uniform martingale laws of large numbers.
\newblock \emph{Probability Theory and Related Fields}, 161, 02 2014.
\newblock \doi{10.1007/s00440-013-0545-5}.

\bibitem[Rakhlin et~al.(2015)Rakhlin, Sridharan, and Tewari]{rakhlin2015online}
Alexander Rakhlin, Karthik Sridharan, and Ambuj Tewari.
\newblock Online learning via sequential complexities.
\newblock \emph{J. Mach. Learn. Res.}, 16\penalty0 (1):\penalty0 155–186, January 2015.
\newblock ISSN 1532-4435.

\bibitem[Roth et~al.(2023)Roth, Tolbert, and Weinstein]{roth2023reconciling}
Aaron Roth, Alexander Tolbert, and Scott Weinstein.
\newblock Reconciling individual probability forecasts.
\newblock In \emph{Proceedings of the 2023 ACM Conference on Fairness, Accountability, and Transparency}, pages 101--110, 2023.

\bibitem[Shalev-Shwartz and Ben-David(2014)]{shalev2014understanding}
Shai Shalev-Shwartz and Shai Ben-David.
\newblock \emph{Understanding machine learning: From theory to algorithms}.
\newblock Cambridge university press, 2014.

\bibitem[Vapnik and Chervonenkis(1971)]{vcdim}
V.N. Vapnik and A.~YA. Chervonenkis.
\newblock On the uniform convergence of relative frequencies of events to their probabilities, 1971.

\bibitem[Vovk(2006)]{vovk2006line}
Vladimir Vovk.
\newblock On-line regression competitive with reproducing kernel hilbert spaces.
\newblock In \emph{International Conference on Theory and Applications of Models of Computation}, pages 452--463. Springer, 2006.

\bibitem[Vovk(2001)]{vovk2001competitive}
Volodya Vovk.
\newblock Competitive on-line statistics.
\newblock \emph{International Statistical Review}, 69\penalty0 (2):\penalty0 213--248, 2001.

\bibitem[Wei et~al.(2022)Wei, Li, Ma, Ding, Wei, Wu, Chen, and Ranbaduge]{wei2022vertical}
Kang Wei, Jun Li, Chuan Ma, Ming Ding, Sha Wei, Fan Wu, Guihai Chen, and Thilina Ranbaduge.
\newblock Vertical federated learning: Challenges, methodologies and experiments.
\newblock \emph{arXiv preprint arXiv:2202.04309}, 2022.

\bibitem[Zhang et~al.(2024)Zhang, Zhan, Chen, Du, and Lee]{zhang2024optimal}
Zihan Zhang, Wenhao Zhan, Yuxin Chen, Simon~S Du, and Jason~D Lee.
\newblock Optimal multi-distribution learning.
\newblock In \emph{The Thirty Seventh Annual Conference on Learning Theory}, pages 5220--5223. PMLR, 2024.

\bibitem[Zhao et~al.(2021)Zhao, Kim, Sahoo, Ma, and Ermon]{zhao2021calibrating}
Shengjia Zhao, Michael Kim, Roshni Sahoo, Tengyu Ma, and Stefano Ermon.
\newblock Calibrating predictions to decisions: A novel approach to multi-class calibration.
\newblock \emph{Advances in Neural Information Processing Systems}, 34:\penalty0 22313--22324, 2021.

\end{thebibliography}
\appendix

\section{Proofs of Tightness of \Cref{thm:linear-weak-learning} from Section~\ref{sec:boost}}

We first give the formal proofs for the necessity of boundedness for weak-learning and the tightness of quadratic guarantees. Then we show why some assumption on the joint class like the Minowski sum one we make is necessary to get weak-learnability.
\subsection{Proof of Theorem~\ref{thm:bounded-necessary}}
\begin{proof}
Let $\cX_A = \cX_B = [-1,1]$ and $\cF_A  = \{x_A \mapsto w_A x_A: w_A \in \mathbb{R}\}$ and $\cF_B  = \{x_B \mapsto w_B x_B: w_B \in \mathbb{R}\}$. Note that $\cF_A$ and $\cF_B$ are star-shaped since they are linear functions, but unbounded since we have no bounds on the weights. For any strictly increasing function $w$, we will construct a distribution such that the $w(\cdot)$-weak-learnability condition does not hold for these function classes.

Consider the following joint distribution $\cD_\rho$ over $\cX_A \times \cX_B \times \cY$ for any $\rho \ge 1$:
\begin{align*}
x_A  = \frac{1}{2} \xi_A,~ x_B = x_A + \frac{\xi_2}{2\rho} \text{ and } y = \xi_B \text{ for } \xi_A, \xi_B \sim_{\text{unif}} \{-1, +1\}.
\end{align*}

Observe that the optimal constant predictor $c^* = \E[Y] = 0$, giving $\min_{c \in \mathbb{R}}\E[(c - y)^2] = \E[\xi_B^2] = 1$ and the optimal joint predictor is $h_J^*(x) = 2\rho x_B - 2\rho x_A = y$, yielding $\min_{h_J \in \cH_J}\E\left[(h_J(x) - y)^2\right] = 0$. This implies that 
\[
\min_{c \in \mathbb{R}} \E[(c - y)^2] - \min_{h_J\in \cH_J} \E[(h_J(x) - y)^2] = 1.
\]
We will show that despite this, the improvement over the constant function for the optimal predictor on either feature alone is much smaller. Observe that the the label $y$ does not depend on $x_A$, hence the optimal predictor over $\cX_A$ is $h_A^*(x) = 0$ which implies $\min_{h_A \in \cH_A}\E\left[(h_A(x_A) - y)^2\right] = \E[y^2] = 1$. This implies,
\[
\min_{c \in \mathbb{R}} \E[(c - y)^2] - \min_{h_A\in \cH_A} \E[(h_A(x_A) - y)^2] = 0 \le w(0) < w(1).
\]
Here the last follows from $w(0) \in [0,1]$ and $w$ being strictly increasing.

The label $y$ does have correlation with $x_B$, and a simple calculation gives us that the optimal linear predictor over $\cX_B$ has form $h_B^*(x_B) = w_B x_B$ where
\begin{align*}
w_B &= \frac{\E[x_By]}{\E[x_B^2]}= \frac{\frac{\E[\xi_A\xi_B]}{2} + \frac{\E[\xi_B^2]}{2 \rho}}{\frac{\E[\xi_A^2]}{4} + \frac{\E[\xi_B^2]}{4 \rho^2}} = \frac{2 \rho}{\rho^2 + 1}.
\end{align*}
This gives us
\begin{align*}
\E\left[(h_B^*(x_B) - y)^2\right] & = \E\left[\left(\frac{2\rho}{\rho^2+1}x_B - \xi_B\right)^2\right] \\
&= \E\left[\left(\frac{\rho}{\rho^2+1}\xi_A - \frac{\rho^2}{\rho^2+1}\xi_B\right)^2\right] \\
&= \frac{\rho^2}{(\rho^2+1)^2}\E[\xi_A^2] + \frac{\rho^4}{(\rho^2+1)^2}\E[\xi_B^2] = \frac{\rho^2}{\rho^2+1}
\end{align*}
This in turn implies:
\begin{align*}
\min_{c \in \mathbb{R}} \E[(c - y)^2] - \min_{h_B\in \cH_B} \E[(h_B(x_B) - y)^2]  &= 1 - \frac{\rho^2}{\rho^2+1} = \frac{1}{\rho^2+1}.
\end{align*}
$w(\cdot)$-weak learnability would require us to have $\frac{1}{\rho^2+1} \ge w(1)$. However, we can always choose $\rho$ large enough to make this not hold. In particular, any $\rho > \sqrt{\frac{1 - w(1)}{w(1)}}$ will violate this condition. Note that since $w$ is strictly increasing, we will be guaranteed that $w(1) > w(0) \ge 0$, so such a $\rho$ exists. Therefore, for every fixed $w$, we can always construct a distribution that does not satisfy our weak-learnability guarantee.
\end{proof}

\subsection{Proof of Theorem~\ref{thm:quadratic}}
\begin{proof}
Let $\cX_A = \cX_B = [-1,1]$ and $\cF_A  = \{x_A \mapsto w_A x_A: w_A \in \mathbb{R}, |w_A| \le 1\}$ and $\cF_B  = \{x_B \mapsto w_B x_B: w_B \in \mathbb{R}, |w_B| \le 1\}$. Note that $\cF_A$ and $\cF_B$ are star-shaped since they are linear functions, and 1-bounded since both the input and weights are bounded by 1. For any strictly increasing function $w$, we will construct a distribution such that the $w(\cdot)$-weak-learnability condition does not hold for these function classes with respect to $\cH_J = \{h_A + h_B: h_A \in \cH_A, h_B \in \cH_B\}$. 

We will consider the same joint distribution as
in the proof of \cref{thm:bounded-necessary}. We will further assume that $\rho \ge 1$. 

Recall that the optimal joint predictor was $h_J(x) = 2 \rho x_A - 2 \rho x_B$ which required elements from the base classes to have norm $2 \rho$ which grows with increasing $\rho$. In our bounded class, however, the optimal predictor is the scaled down version of this predictor to adhere to our norm constraints: $h^*_J(x) = x_A - x_B = \frac{y}{2 \rho}$. This gives us,
\begin{align*}
    \E[(h_J^*(x) - y)^2] = \E\left[\left(\frac{y}{2 \rho} - y\right)^2\right] = \left(1 - \frac{1}{2\rho}\right)^2 \E[y^2] = \frac{(2 \rho - 1)^2}{4 \rho^2}.
\end{align*}
Which in turn implies, that the gain of the joint predictor over the constant function is
\[
\min_{c \in \mathbb{R}} \E[(c - y)^2] - \min_{h_J\in \cH_J} \E[(h_J(x) - y)^2] = 1 - \frac{(2 \rho - 1)^2}{4 \rho^2} = \frac{4 \rho - 1}{4 \rho^2} \in \left[\frac{3}{4\rho}, \frac{1}{\rho}\right].
\]
Here the last follows from using the fact that $\rho \ge 1$.

Recall that the optimal predictor over $\cX_A$ is $h_A^*(x_A) = 0$ which still belongs to our bounded class, and its gain over the constant predictor was 0. The optimal predictor over $\cX_B$ in the unbounded case is $h_B^*(x_B) = \frac{2 \rho}{\rho^2 + 1}x_B$. Since $\rho^2 + 1 \ge 2 \rho$ for all $\rho$, the norm of this predictor is actually bounded by 1. Therefore, for our bounded class, this remains an optimal predictor. The gain of this predictor over the constant predictor is
\begin{align*}
    &\min_{c \in \mathbb{R}} \E[(c - y)^2] - \min_{h_B\in \cH_B} \E[(h_B(x_B) - y)^2]  = \frac{1}{\rho^2+1} \in \left[\frac{1}{2\rho^2}, \frac{1}{\rho^2}\right]
\end{align*}
Here the last follows from using the fact that $\rho \ge 1$. 

Therefore, for $\cD_\rho$, the gain from the joint predictor over a constant is $\Theta(1/\rho)$ and from the best individual predictor over constant is $\Theta(1/\rho^2)$ implying that there is no $w(\gamma) = \omega(\gamma^2)$ for this distribution that satisfies weak-learnability.
\end{proof}

Finally, we establish that it is necessary to make some assumption on $\cH_J$, such as the Minowski sum structure we use---multiplicative rather than additive combinations would not work:
\begin{theorem}
There exists classes $\cF_A = \{f_A: \cX_A \to \mathbb{R}\}$ and $\cF_B = \{f_B: \cX_B \to \mathbb{R}\}$ that are star-shaped and 1-bounded over some domain $\cX_A, \cX_B$ such that $\cH_A = \{f_A + b_A : f_A \in \cF_A, b_A \in \mathbb{R}\}$ and $\cH_B = \{f_B + b_B : f_B \in \cF_B, b_B \in \mathbb{R}\}$ but do not jointly satisfy $w(\cdot)$-weak learning with respect to $\cH_J = \{h_A \cdot h_B: h_A \in \cH_A, h_B \in \cH_B\}$ for any strictly increasing $w$.
\end{theorem}
\begin{proof}
We will consider the function classes as in the proof of \Cref{thm:quadratic}, that is, $\cX_A = \cX_B = [-1,1]$ and $\cF_A  = \{x_A \mapsto w_A x_A: w_A \in \mathbb{R}, |w_A| \le 1\}$ and $\cF_B  = \{x_B \mapsto w_B x_B: w_B \in \mathbb{R}, |w_B| \le 1\}$. We know that this class is $1$-bounded and star shaped.

Now consider the following joint distribution over $\cX_A \times \cX_B \times \cY$:
\begin{align*}
x_A & \sim_{\text{unif}} \{-1, +1\}, x_B \sim_{\text{unif}} \{-1, +1\} \text{ independent of } x_A, \text{ and }y = x_A x_B 
\end{align*}
The best constant predictor on this is $\E[y] = 0$. This has loss $\E[y^2] = 1$. The best joint predictor for this distribution is $h_J^*(x) = x_A x_B$ which can be constructed using $h_A(x_A) = x_A$ and $h_B(x_B) = x_B$. Since this perfectly predicts the label, this has loss 0, therefore its gain over the constant predictor is 1. However, the optimal predictor on either function alone is $h_A^*(x_A) = h_B^*(x_B) = 0$. This is because the label is uniformly random given only information of either $x_A$ or $x_B$. This implies that the gain of the best predictor over the constant predictor is 0. This violates the weak-learning condition for any strictly increasing $w$ ($w(1) > w(0) \ge 0$).
\end{proof}


\section{Additional Material from Section \ref{sec:online-alg}}\label{sec:main_app}
\subsection{Calibration Preliminaries} \label{sec:prelims-calibration}
In this section we give the basic calibration definitions that we work with in our proofs. 

The standard measure of calibration of some sequence of predictions $\hat{y}^{1:T}$ to outcomes $y^{1:T}$ in a sequential prediction setting is \emph{expected calibration error}, defined as follows.

\begin{definition}[Expected Calibration Error] Given a sequence of predictions $\hat{y}^{1:T}$ and outcomes $y^{1:T}$, their expected calibration error is,
\[
\ECE(\hat{y}^{1:T}, y^{1:T}) = \sum_{p \in [0,1]} \left| \sum_{t=1}^T \mathbbm{1}[\hat{y}^t = p] (\hat{y}^t - y^t) \right|
\]

Here the outer sum is over the values $p$ that appear in the sequence $\hat{y}^{1:T}$. 
\end{definition}

We will sometimes measure calibration error of a sequence instead using \emph{distance to calibration}, first defined by \cite{blasiok2023unifying} (we here use the definition given by \cite{qiao2024distance} in the sequential setting). Distance to calibration measures the  $\ell_1$ distance between a sequence of predictions and the closest sequence of  \textit{perfectly calibrated} predictions. 
\begin{definition}[Distance to Calibration] Given a sequence of predictions $\hat{y}^{1:T}$ and outcomes $y^{1:T}$, the distance to calibration is,
\[
\CalDist(\hat{y}^{1:T}, y^{1:T}) = \min_{q^{1:T} \in \cC(y^{1:T})} \left\|\hat{y}^{1:T} - q^{1:T}\right\|_1
\]    
where $\cC(y^{1:T}) = \{ q^{1:T} : \ECE(q^{1:T}, y^{1:T}) = 0 \}$ is the set of predictions that are perfectly calibrated against outcomes $y^{1:T}$. 
\end{definition}

Calibration has a close relationship to squared error, which we will use as a potential function in some of our analyses. Below we define the squared error of a sequence of predictions relative to a sequence of outcomes:

\begin{definition}[Squared Error] Given a sequence of predictions $\hat{y}^{1:T}$ and outcomes $y^{1:T}$, the squared error between them is,
  \begin{align*}
       \SQE (\hat{y}^{1:T},y^{1:T}):= \sum_{t \in [T]}(\hat{y}^{t} - y^{t})^{2}. 
  \end{align*} 
  We will overload this notation for the special case of constant sequences $\hat{y}^{1} = \ldots = \hat{y}^{T} = p$:
    \begin{align*}
       \SQE (p,y^{1:T}):= \sum_{t \in [T]}(p - y^{t})^{2}.
  \end{align*} 
\end{definition}

\subsection{Conversation Calibration}\label{app:conv_calib}
Here we formally define the notion of calibration introduced in~\cite{collina2025tractable}, called \emph{conversation calibration}. This notion is defined over a transcript of days to $1...T$ and varied-length rounds. An agent is \emph{conversation calibrated} if for every round $k$, the sequence of predictions (over days $t$) that they make at round $k$ of conversation is calibrated not just marginally, but \emph{conditionally} on the value of the prediction that the other agent made at round $k-1$. We will condition on \emph{bucketings} of predictions.

\begin{definition}[Bucketing of the Prediction Space] \label{def:bucketing}
For bucket coarseness parameter $n$, let $B_n(i)= \left[\frac{i-1}{n}, \frac{i}{n} \right)$ and $B_n(n) = \left[\frac{n-1}{n}, 1 \right]$ form a set $\cB_n$ of $n$ buckets of width $1/n$ that partition the unit interval.
\end{definition}

\begin{definition}[Conversation-Calibrated Predictions]
\label{def:conversation-calibration}
Fix an error function $f:\{1, \ldots, T\} \rightarrow \mathbb{R}$ and bucketing function $g: \{1, \ldots, T\} \rightarrow (0,1]$. Given a prediction transcript $\pi^{1:T}$ resulting from an interaction in the Collaboration Protocol, Bob is $(f, g)$-conversation-calibrated if for all even rounds $k$ and buckets $i \in \{1, \ldots, 1/g(T)\}$:
\begin{align*}
    \CalDist(\phk{T_A(k-1, i)},y^{T_A(k-1, i)}) \leq f(|T_A(k-1, i)|),
\end{align*} 
where $T_A(k-1,i) = \left\{t  ~|~ \pmk{t}{k-1} \in B_i(1/g(T))\right\}$ is the subsequence of days where the predictions of Alice at the previous round fall in bucket $i$.

Symmetrically, Alice is $(f, g)$-conversation-calibrated if for all odd rounds $k$ and buckets $i \in \{1, \ldots, 1/g(T)\}$:
\begin{align*}
    \CalDist(\pmk{T_B(k-1, i)},y^{T_B(k-1, i)}) \leq f(|T_B(k-1, i)|),
\end{align*}
where $T_B(k-1,i) = \left\{t  ~|~ \pmk{t}{k-1} \in B_i(1/g(T))\right\}$ is the subsequence of days where the predictions of Bob at the previous round fall in bucket $i$.
\end{definition}

We also introduce a function that checks whether, on a given day $t$ and given even round $k$, the prediction $\hat{y}^{t,k}$ is within $\epsilon$ of the prediction in the previous round $\hat{y}^{t,k-1}$. Formally, we define
\begin{definition}[Agreement Condition $A_{\pi^{1:T}}(t,k, \epsilon)$ and Disagreement Subsequence $D(T^{k})$]
       \[
    A_{\pi^{1:T}}(t,k):= \begin{cases}
        \mathbb{I}[|\hat{y}_{A}^{t,k} -\hat{y}_{A}^{t,k-1}| \leq \epsilon] &\text{if }\ell \text{ is odd},\\
         \mathbb{I}[|\hat{y}_{B}^{t,k} - \hat{y}_{B}^{t,k-1}| \leq \epsilon] &\text{if }\ell \text{ is even}.
    \end{cases}
    \]

    Furthermore, let $D(T^{k})$ be the subset of days $t$ such that $A_{\pi^{1:T}}(t,k) = 0$.
\end{definition}

We are now ready to discuss the relationship between conversation calibration and conversation swap regret. 

\begin{theorem}
If $\cH$ contains all constant functions, then $(f,g,\cH)$-Conversation Swap Regret implies $(f',g)$-Conversation Calibration, where $f'(T) = \sqrt{T \cdot f(T)}$. \label{thm:CSR2CC} 
\end{theorem}
\begin{proof}
Assume that Bob satisfies  $(f,g,\cH)$-Conversation Swap Regret. Let $T_{A}(k-1,i)$ be the subsequence of days where the predictions of Alice in round $k-1$ fall in bucket $i$. 
As $\cH$ contains all constant functions, $(f,g,\cH)$-Conversation Swap Regret directly implies that 
\begin{align*}
   &  \sum_{t \in T_{A}(k-1,i)}(\hat{y}_{k}^{t} - y^{t})^{2} -
    \sum_{v}\min_{h \in \cH_{B}}\left( \sum_{t \in T_{A}(k-1,i)}\mathbb{I}[\hat{y}_{k}^{t} = v](h(x^{t}) - y^{t})^{2}\right) \leq f(|T_{A}(k-1,i)|) \\
    & \implies \sum_{t \in T_{A}(k-1,i)}(\hat{y}_{k}^{t} - y^{t})^{2} -
    \sum_{v}\min_{x^{*} \in [0,1]}\left( \sum_{t \in T_{A}(k-1,i)}\mathbb{I}[\hat{y}_{k}^{t} = v](x^{*} - y^{t})^{2}\right) \leq f(|T_{A}(k-1,i)|) \\
     & \implies \sum_{v} \left(\sum_{t \in T_{A}(k-1,i)}\mathbb{I}[\hat{y}_{k}^{t} = v](\hat{y}_{k}^{t} - y^{t})^{2} -
   \min_{x_{v}^{*} \in [0,1]} \sum_{t \in T_{A}(k-1,i)}\mathbb{I}[\hat{y}_{k}^{t} = v](x_{v}^{*} - y^{t})^{2}\right) \leq f(|T_{A}(k-1,i)|) \\
    & \implies \sum_{v} \left(\sum_{t \in T_{A}(k-1,i)}\mathbb{I}[\hat{y}_{k}^{t} = v](\hat{y}_{k}^{t} - y^{t})^{2} -
    \sum_{t \in T_{A}(k-1,i)}\mathbb{I}[\hat{y}_{k}^{t} = v](x_{v}^{a} - y^{t})^{2}\right) \leq f(|T_{A}(k-1,i)|) \tag{Where $x_{v}^{a}$ is the average on the level set} \\
    & \implies \sum_{t \in T_{A}(k-1,i)}\sum_{v} \mathbb{I}[\hat{y}_{k}^{t} = v](\hat{y}_{k}^{t}  - x^{a}_{v})^{2} \leq f(|T_{A}(k-1,i)|) \tag{By Lemma~\ref{lem:squares_diff}}
    \end{align*}

Note that, by Cauchy-Schwartz, we have that $\sqrt{\sum_{t \in T_{A}(k-1,i)}\sum_{v}\mathbb{I}[\hat{y}_{k}^{t} = v](\hat{y}_{k}^{t}  - x^{a}_{v})^{2}}\sqrt{|T_{A}(k-1,i)|} \geq \sum_{t \in T_{A}(k-1,i)}\sum_{v}\mathbb{I}[\hat{y}_{k}^{t} = v]|\hat{y}_{k}^{t}  - x^{a}_{v}|$, and therefore that $\sum_{t \in T_{A}(k-1,i)}\sum_{v}\mathbb{I}[\hat{y}_{k}^{t} = v](\hat{y}_{k}^{t}  - x^{a}_{v})^{2} \geq \frac{(\sum_{t \in T_{A}(k-1,i)}\sum_{v}\mathbb{I}[\hat{y}_{k}^{t} = v]|\hat{y}_{k}^{t}  - x^{a}_{v}|)^{2}}{|T_{A}(k-1,i)|}$. Thus, we can write

\begin{align*}
   & \frac{(\sum_{t \in T_{A}(k-1,i)}\sum_{v}\mathbb{I}[\hat{y}_{k}^{t} = v]|\hat{y}_{k}^{t}  - x^{a}_{v}|)^{2}}{|T_{A}(k-1,i)|} \leq f(|T_{A}(k-1,i)|) \\
    & \implies \frac{\sum_{t \in T_{A}(k-1,i)}\sum_{v}\mathbb{I}[\hat{y}_{k}^{t} = v]|\hat{y}_{k}^{t}  - x^{a}_{v}|}{\sqrt{|T_{A}(k-1,i)|}} \leq \sqrt{f(|T_{A}(k-1,i)|)} \tag{Taking the square root of both sides} \\
    & \implies \sum_{t \in T_{A}(k-1,i)}\sum_{v}\mathbb{I}[\hat{y}_{k}^{t} = v]|\hat{y}_{k}^{t}  - x^{a}_{v}| \leq \sqrt{f(|T_{A}(k-1,i)|)\cdot|T_{A}(k-1,i)|}  \\
    & \implies  ECE(\hat{y}^{T_{A}(k-1,i)}_{k},y^{T_{A}(k-1,i)}) \leq \sqrt{f(|T_{A}(k-1,i)|)\cdot|T_{A}(k-1,i)|} \tag{As the LHS is exactly ECE} \\
    & \implies CalDist(\hat{y}^{T_{A}(k-1,i)}_{k},y^{T_{A}(k-1,i)}) \leq \sqrt{f(|T_{A}(k-1,i)|)\cdot|T_{A}(k-1,i)|} \tag{As ECE upper bounds CalDist}
\end{align*}

As Conversation Swap Regret holds true for all rounds, this implies $\sqrt{|T_{A}(k-1,i)| \cdot f(|T_{A}(k-1,i)|)}$-conversation calibration. 
The proof holds symmetrically for Alice.
\end{proof}

\begin{theorem} \label{thm:collect}
    If a sequence $\hat{y}_{k}$ has $(f, g, \cH)$-Conversation Swap Regret, then
    \begin{align*}
\sum_{t =1}^{T}(\hat{y}_{k}^{t} - y^{t})^{2} -
    \sum_{v}\min_{h \in \cH}\left( \sum_{t=1}^{T}\mathbb{I}[\hat{y}_{k}^{t} = v](h(x^{t}) - y^{t})^{2}\right) \leq \frac{f(g(T) T)}{g(T)}.
\end{align*}
\end{theorem}
\begin{proof}
\begin{align*}
    & \sum_{t =1}^{T}(\hat{y}_{k}^{t} - y^{t})^{2} -
    \sum_{v}\min_{h \in \cH_{B}}\left( \sum_{t=1}^{T}\mathbb{I}[\hat{y}_{k}^{t} = v](h(x^{t}) - y^{t})^{2}\right) = \\
    & \sum_{i}\sum_{t \in T_{A}(k-1,i)}(\hat{y}_{k}^{t} - y^{t})^{2} -
    \sum_{v}\min_{h \in \cH_{B}}\left( \sum_{i}\sum_{t \in T_{A}(k-1,i)}\mathbb{I}[\hat{y}_{k}^{t} = v](h(x^{t}) - y^{t})^{2}\right) \\
    & \leq \sum_{i}\sum_{t \in T_{A}(k-1,i)}(\hat{y}_{k}^{t} - y^{t})^{2} -
    \sum_{i}\sum_{v}\min_{h \in \cH_{B}}\left( \sum_{t \in T_{A}(k-1,i)}\mathbb{I}[\hat{y}_{k}^{t} = v](h(x^{t}) - y^{t})^{2}\right) \tag{As by moving the sum over $i$ out of the min we are only strengthening the benchmark} \\
    & = \sum_{i} \left( \sum_{t \in T_{A}(k-1,i)}(\hat{y}_{k}^{t} - y^{t})^{2} -
    \sum_{v}\min_{h \in \cH_{B}}\left( \sum_{t \in T_{A}(k-1,i)}\mathbb{I}[\hat{y}_{k}^{t} = v](h(x^{t}) - y^{t})^{2}\right) \right) \\
    & = \sum_{i} \left( f(|T_{A}(k-1,i)|) \right) \tag{By the Conversation Swap Regret Condition} \\
    & \leq \frac{f(g(T) T)}{g(T)} \tag{By the assumption that $f$ is concave}
\end{align*}
\end{proof}

\subsection{Additional Online Preliminaries}

\begin{definition}[$\cZ-$valued Tree] \label{def:z-valued-tree}
    A $\cZ-$valued tree {\bf z} of depth $n$ is a rooted complete binary tree with nodes labeled by elements of $\cZ$. We identify the tree {\bf{z}} with the sequence $({\bf z}_1, \ldots, {\bf z}_n)$ of labeling functions ${\bf{z}}_i: \{ \pm1\}^{i-1} \to \cZ$ which provide the labels for each node. Here, $\bf{z}_1 \in \cZ$ is the root of the tree, while ${\bf z}_i, i > 1$ is the label of the node obtained by following the path of length $i-1$ from the root, with $+1$ indicating `right' and $-1$ indicating `left.'
\end{definition}

\begin{definition}
\label{def:tree-shattering}
A $\cZ-$valued tree {\bf z} of depth $d$ is shattered by a function class $\cF \subseteq \{\pm 1\}^{\cZ}$ if 
\begin{align*}
    \forall \ \eps \in \{\pm 1\}^d, \ \exists \ f \in \cF \text{ s.t. } \forall \ t \in [d], \ f({\bf z}_t(\eps)) = \eps_t.
\end{align*}
\end{definition}

\begin{definition}[Sequential Fat Shattering Dimension \citep{rakhlin2014sequential}]
\label{def:sequential-fat}
A $\cZ-$valued binary tree {\bf z} of depth $d$ is $\alpha-$shattered by a function class $\cF \subseteq \mathbb{R}^{\cZ}$ if there exists an $\mathbb{R}-$valued tree {\bf s} of depth $d$ such that
\begin{align*}
    \forall \ \eps \in \{\pm 1\}^d, \ \exists \ f \in \cF \text{ s.t. } \forall \ t \in [d], \ \eps_t(f({\bf z}_t(\eps)) - {\bf s}_t(\eps)) \geq \alpha/2.
\end{align*}
The sequential fat shattering dimension  $\textsc{fat}_{\alpha}(\cF, \cZ)$ at scale $\alpha$ is the maximal $d$ such that $\cF$ $\alpha-$shatters a $\cZ-$valued tree of depth $d$.
\end{definition}

\subsection{Proof of Theorem~\ref{thm:rounds}}

\begin{lemma}[Lemma A.1 from~\cite{collina2025tractable}] If $m = \frac{1}{T}\sum_{t=1}^{T}y^{t}$, then for any constant $x$,
\begin{equation}
\SQE(x,y^{1:T}) - \SQE(m,y^{1:T})  = \sum_{t=1}^{T}(x-m)^{2}
\end{equation} 
\label{lem:squares_diff}
\end{lemma}

\begin{lemma}[Lemma A.2 from~\cite{collina2025tractable}]
Let $T_{k}^{i,p_{h}} = \{t: \phk{t}{k} = p_{h} \textit{ and } \pmk{t}{k-1} \in B_{i}(\frac{1}{g(T)})\}$ be the subsequence of days such that the predicts $p_{h}$ in round $k$ and the model predicts in bucket $B_{i}(\frac{1}{g(T)})$ in round $k-1$. If the human is $(\cdot, g_{h}(T))$-conversation calibrated, then

\begin{equation}
    \sum_{t \in T_{k}^{i,p_{h}}}(\pmk{t}{k-1} - y^{t})^{2} - \sum_{t \in T_{k}^{i,p_{h}}}(i \cdot g_{h}(T) - y^{t})^{2}  \geq - g_{h}(T) \cdot |T_{k}^{i,p_{h}}|
\end{equation}
\label{lem:v1}
\end{lemma}

\begin{lemma}[Lemma A.3 from~\cite{collina2025tractable}]
    Consider any sequence of predictions and labels $p^{1:T}, y^{1:T}$ and some other sequence of predictions $q^{1:T}$ such that $||p^{1:T} - q^{1:T}|| \leq \gamma$. Then, $$\sum_{t=1}^T(q^{t} - y^{t})^{2} - \sum_{t=1}^T(p^{t} - y^{t})^{2} \leq 3\gamma$$ \label{lem:bound_error_diff}
\end{lemma}

\begin{lemma}\label{lem:mh} 
   If Bob is $(0, g_{B}(T))$-conversation-calibrated, then for any even $k$, 
   \begin{align*} 
    \SQE(\phk{T}{k},y^{1:T}) \leq \SQE(\pmk{T}{k-1},y^{1:T}) -
   (\epsilon - g_{B}(T))^{2}|D(T^{k})| + g_B(T)T
   \end{align*}
   And if Alice is $(0, g_{A}(T))$-conversation-calibrated, for any odd $k$,
   \begin{align*} 
    \SQE(\pmk{T}{k},y^{1:T}) \leq \SQE(\phk{T}{k-1},y^{1:T}) -
   (\epsilon - g_{A}(T))^{2} |D(T^{k})| + g_A(T)T
   \end{align*}
\end{lemma}

\begin{proof}
Let $T_{k}^{i,p_{h}} = \{t: t \in \Tk{k} \text{ and } \phk{t}{k} = p_{h} \text{ and } \pmk{t}{k-1} \in B_{i}(\frac{1}{g(T)})\}$ be the subsequence of days such that Bob predicts $p_{h}$ in round $k$ and Alice predicts in bucket $B_{i}(\frac{1}{g(T)})$ in round $k-1$. Let $m_{k}^{i,p_h} = \frac{\sum_{t \in T_{k}^{i,p_{h}}}y^{t}}{|T_{k}^{i,p_{h}}|} $ be the true mean on this subsequence. The difference in squared errors over this subsequence can be written as: 

\begin{align*}
 & \sum_{t \in T_{k}^{i,p_{h}}}(\pmk{t}{k-1} - y^{t})^{2} - \sum_{t \in T_{k}^{i,p_{h}}}(\phk{t}{k} - y^{t})^{2} 
 \\ & = \left[\sum_{t \in T_{k}^{i,p_{h}}}(\pmk{t}{k-1} - y^{t})^{2} - \sum_{t \in T_{k}^{i,p_{h}}}(m_{k}^{i,p_h} - y^{t})^{2} \right] - \left[\sum_{t \in T_{k}^{i,p_{h}}}(\phk{t}{k} - y^{t})^{2}  - \sum_{t \in T_{k}^{i,p_{h}}}(m_{k}^{i,p_h} - y^{t})^{2} \right]  \tag{Adding and subtracting $\sum_{t \in T_{k}^{i,p_{h}}}(m_{k}^{i,p_h} - y^{t})^{2}$}
 \\ & \geq \left[\sum_{t \in T_{k}^{i,p_{h}}}(i \cdot g_{B}(T) - y^{t})^{2} -|T_{k}^{i,p_{h}}| \cdot g_{B}(T)  - \sum_{t \in T_{k}^{i,p_{h}}}(m_{k}^{i,p_h} - y^{t})^{2} \right] - \\ &  \left[\sum_{t \in T_{k}^{i,p_{h}}}(\phk{t}{k} - y^{t})^{2}  - \sum_{t \in T_{k}^{i,p_{h}}}(m_{k}^{i,p_h} - y^{t})^{2} \right] \tag{By Lemma~\ref{lem:v1}}
  \\ & =\left[\sum_{t \in T_{k}^{i,p_{h}}}(i \cdot g_{B}(T) - m_{k}^{i,p_{h}})^{2} - |T_{k}^{i,p_{h}}| \cdot g_{B}(T) \right] -  \left[\sum_{t \in T_{k}^{i,p_{h}}}(\phk{t}{k} - y^{t})^{2}  - \sum_{t \in T_{k}^{i,p_{h}}}(m_{k}^{i,p_h} - y^{t})^{2} \right]  \tag{By Lemma~\ref{lem:squares_diff}}
    \\ & = \left[\sum_{t \in T_{k}^{i,p_{h}}}(i \cdot g_{B}(T) - m_{k}^{i,p_{h}})^{2} - |T_{k}^{i,p_{h}}| \cdot g_{B}(T)   \right] - \left[\sum_{t \in T_{k}^{i,p_{h}}}(p_h - y^{t})^{2}  - \sum_{t \in T_{k}^{i,p_{h}}}(m_{k}^{i,p_h} - y^{t})^{2} \right] \tag{As by definition of $T_{k}^{i,p_{h}}$, $\phk{t}{k} = p_{h}$} 
      \\ & \geq \left[\sum_{t \in T_{k}^{i,p_{h}}}(i \cdot g_{B}(T) - m_{k}^{i,p_{h}})^{2} - |T_{k}^{i,p_{h}}| \cdot g_{B}(T) \right]-  \left[\sum_{t \in T_{k}^{i,p_{h}}}(p_h - m_{k}^{i,p_{h}})^{2} \right]  \tag{By Lemma~\ref{lem:squares_diff}}
           \\ & \geq - |T_{k}^{i,p_{h}}| \cdot g_{B}(T)  +  \sum_{t \in T_{k}^{i,p_{h}}}(i \cdot g_{B}(T) - p_{h})^{2} \tag{As Bob is $(0, g_{B}(T))$-conversation calibrated, $p_{h} = m_{k}^{i,p_{h}}$}
    \end{align*}
Using this analysis, we can write the difference in squared errors over the entire sequence $\phk{T}{k}$ and $\pmk{T}{k-1}$ as follows, where the first term comes from summing the above expression over all $i, p_h$:
    \begin{align*}
        &\SQE(\bar{p}^{T,k-1}_A, y^{1:T}) - \SQE(\bar{p}^{T,k}_B, y^{1:T}) \\
        & = \sum_{\forall i, p_h} \left( \sum_{t \in T_{k}^{i,p_{h}}}(\pmk{t}{k-1} - y^{t})^{2} - \sum_{t \in T_{k}^{i,p_{h}}}(\phk{t}{k} - y^{t})^{2} \right)   \\
        & = \sum_{\forall i, p_{h}} \left(- |T_{k}^{i,p_{h}}| \cdot g_{B}(T)  +  \sum_{t \in T_{k}^{i,p_{h}}}(i \cdot g_{B}(T) - p_{h})^{2}  \right) \tag{by the analysis above} \\
        & = -g_{B}(T)T + \sum_{\forall i, p_{h}} \sum_{t \in T_{k}^{i,p_{h}}}(i \cdot g_{B}(T) - p_{h})^{2}  \tag{As $g_{B}(T)$ is independent of $i$ and $p_{h}$, and $\sum_{\forall i, p_{h}}\left|T_{k}^{i,p_{h}}\right| = T$} \\
        & \geq -g_{B}(T)T + \sum_{\forall i, p_{h}} \sum_{t \in T_{k}^{i,p_{h}}}\mathbbm{1}[|i \cdot g_{B}(T) - \phk{t}{k}| \geq \epsilon - g_{B}(T)](i \cdot g_{B}(T) - p_{h})^{2}   \\
        & \geq -g_{B}(T)T + (\epsilon - g_{B}(T))^{2} \sum_{\forall i, p_{h}}\sum_{t \in T_{k}^{i,p_{h}}} \mathbbm{1}[|i \cdot g_{B}(T) - \phk{t}{k}| \geq \epsilon - g_{B}(T)]  
    \end{align*}

Note that, for all days in the subsequence $T_{k}^{i,p_{h}}$, in round $k-1$ Alice predicted in bucket $B_{i}(\frac{1}{g_{B}(T)}) = i \cdot g_{B}(T)$, and therefore in each of these days, by the definition of our bucketing, $\pmk{t}{k-1} \geq (i-1) \cdot g_{B}(T)$ and $\pmk{t}{k-1} \leq i \cdot g_{B}(T)$. So consider any round $t \in T_{k}^{i,p_{h}}$. If $|\phk{t}{k} - \pmk{t}{k-1}| \geq \epsilon$, then we have:

\begin{align*}
    |\phk{t}{k} - \pmk{t}{k-1}| &\le  |\phk{t}{k} - i\cdot g_{B}(T)| + |i\cdot g_{B}(T) - \pmk{t}{k-1}|\\
    &= |\phk{t}{k} - i\cdot g_{B}(T)| + i\cdot g_{B}(T) - \pmk{t}{k-1}\\
    &\le |\phk{t}{k} - i\cdot g_{B}(T)|  + i\cdot g_{B}(T) - (i-1) \cdot g_{B}(T)\\
    &= |\phk{t}{k} - i\cdot g_{B}(T)| + g_{B}(T),\\
    \implies |\phk{t}{k} - i\cdot g_{B}(T)| & \ge |\phk{t}{k} - \pmk{t}{k-1}| - g_{B}(T) \ge \epsilon - g_{B}(T).
\end{align*}

Thus, if $|\phk{t}{k} - \pmk{t}{k-1}| \geq \epsilon$, then $|i \cdot g_{B}(T) - \phk{t}{k}| \geq \epsilon - g_{B}(T)$, $\forall t \in T_{k}^{i,p_{h}}$. Therefore the set of days for which the former condition holds is a subset of the latter condition, and we can write
    
\begin{align*}
    & -g_{B}(T)T + (\epsilon - g_{B}(T))^{2} \sum_{\forall i, p_{h}} \mathbbm{1}[|i \cdot g_{B}(T) - p_{h}| \geq \epsilon - g_{B}(T)] \cdot \left|T_{k}^{i,p_{h}}\right|  \\
    & \geq -g_{B}(T)T + (\epsilon - g_{B}(T))^{2} \sum_{\forall i, p_{h}} \sum_{t \in T_{k}^{i,p_{h}}} \mathbbm{1}[|\phk{t}{k} - \pmk{t}{k-1}| \geq \epsilon] \\
    & = -g_{B}(T)T + (\epsilon - g_{B}(T))^{2} |D(T^{k})| \tag{As on every day and round where there is not agreement, Bob and Alice disagreed by at least $\epsilon$}
\end{align*}

As Bob and Alice are perfectly symmetrical, we also obtain the symmetrical result for Alice.
\end{proof}

\begin{theorem} If Bob is $(f_{B}(\cdot), g_{B}(\cdot))$-conversation-calibrated, then after engaging in the collaboration protocol for $T$ days:
\begin{align*}
    \SQE(\phk{T}{k}, y^{1:T}) \leq  \SQE(\pmk{T}{k-1}, y^{1:T}) -(\epsilon - g_{B}(T))^{2} |D(T^{k})| + g_{B}(T) T + 3\frac{f_{B}(g_{B}(T) \cdot T)}{g_{B}(T)}
\end{align*}
And if Alice is $(f_{A}(\cdot), g_{A}(\cdot))$-conversation-calibrated, then after engaging in the collaboration protocol for $T$ days: 
\begin{align*}
    \SQE(\pmk{T}{k}, y^{1:T}) \leq  \SQE(\phk{T}{k-1}, y^{1:T}) -(\epsilon - g_{A}(T))^{2}|D(T^{k})| + g_{A}(T) T + 3\frac{f_{A}(g_{A}(T) \cdot T)}{g_{A}(T)}
\end{align*}

\label{thm:cases}
\end{theorem}
\begin{proof}

Let $T_{m}({k,i}) = \{t: \pmk{t}{k-1} \in B_i\left(\frac{1}{g_{B}(T)}\right)\}$ be the subsequence of days in which Alices predicts in bucket $B_{i}(\frac{1}{g_{B}(T)})$ at round $k-1$. 

Note that Bob has distance to calibration of $f_{B}(|T_m(k, i)|)$ on every such subsequence defined this way. Therefore, for predictions $p_{h}^{1:T,k}$ from Bob at round $k$: 

\begin{align*} 
\CalDist(p_{h}^{\Tk{k},k}, y^{1:T}) & =
 \min_{q^{1:T} \in C(y^{1:T})}\|p_{h}^{\Tk{k},k} - q^{1:T}\|_{1} \\
 \\ & \leq \sum_{i =1}^{\frac{1}{g_{B}(T)}}\min_{q^{1:|T_m(k, i)|} \in C^{ T_m(k, i) }(y^{1:T})}\|p^{1:T} - q_{v}^{1:T}\|_{1}\\
& \leq \sum_{i=1}^{\frac{1}{g_{B}(T)}} f_{B}(|T_m(k, i)|) \tag{By the calibration distance of Bob}\\ 
& \leq \frac{f_{B}(g_{B}(T) \cdot |\Tk{k}|)}{g_{B}(T)} \tag{By the assumption that $f_{B}$ is concave} \\
& \leq \frac{f_{B}(g_{B}(T) \cdot T)}{g_{B}(T)} 
\end{align*}

Let $q^{k}$ be the set of perfectly calibrated predictions that are $f_{B}(|T_m(k, i)|)$-close to $p_{h}^{1:T,k}$. Then, we have that 

\begin{align*}
& \SQErr(p_{h}^{T, k}, y^{1:T}) - \SQErr(p_{m}^{T, k-1}, y^{1:T}) \\ & \leq \SQErr(q^{k}, y^{1:T}) - \SQErr(p_{h}^{T, k-1}, y^{1:T}) + 3\frac{f_{B}(g_{B}(T) \cdot T)}{g_{B}(T)} \tag{By Lemma~\ref{lem:bound_error_diff}} \\
&  \leq - (\epsilon - g_{B}(T))^{2}|D(T^{k})| + g_{B}(T) T + 3\frac{f_{B}(g_{B}(T) \cdot T)}{g_{B}(T)} \tag{By Lemma~\ref{lem:mh}}.
\end{align*}

As Bob and Alice are symmetric, we also obtain the symmetric result for Alice.
\end{proof}

\begin{proof}[Proof of Theorem~\ref{thm:rounds}]
   By composing the two results in Theorem~\ref{thm:cases}, we see that 
\begin{align*}
    &\SQErr(\phk{T}{k-2},y^{1:T}) - \SQErr(\phk{T}{k}, y^{1:T})  \\
    &\geq(\epsilon - g_{B}(T))^{2} |D(T^{k})| + (\epsilon - g_{A}(T))^{2}|D(T^{k-1})|  - g_{A}(T) T - 3\frac{f_{A}(g_{A}(T) \cdot T)}{g_{A}(T)} - g_{B}(T) T - 3\frac{f_{B}(g_{B}(T) \cdot T)}{g_{B}(T)}\\
    &\geq(\epsilon - g_{B}(T))^{2}|D(T^{k})| + (\epsilon - g_{A}(T))^{2}|D(T^{k-1})|  - (g_{A}(T) + g_{B}(T))T - 3\left(\frac{f_{A}(g_{A}(T) \cdot T)}{g_{A}(T)} + \frac{f_{B}(g_{B}(T) \cdot T)}{g_{B}(T)}\right).
\end{align*}

Now we can apply the above expression recursively for $k$ rounds in order to bound the total number of days of disagreement: 

\begin{align*}
   &\SQErr(\phk{T}{k}, y^{1:T}) \\
   &\leq \SQErr(\phk{T}{2},y^{1:T}) - (\epsilon - g_{A}(T))^{2}\left(\sum_{k=1, k \text{ odd}}^{k} |D(T^{k})| \right) - (\epsilon - g_{B}(T))^{2}\left(\sum_{k=1, k \text{ even}}^{k} |D(T^{k})| \right) \\ & \qquad + (g_{A}(T) + g_{B}(T)) rT  + 3\left(\frac{f_{A}(g_{A}(T) \cdot T)}{g_{A}(T)} + \frac{f_{B}(g_{B}(T) \cdot T)}{g_{B}(T)}\right)\left(\sum_{k=1,k \text{ even}}^{k}1\right) \\
      & \leq \SQErr(\phk{T}{2},y^{1:T}) - ((\epsilon - g_{A}(T))^{2} + (\epsilon - g_{B}(T))^{2})\left(\sum_{k=1}^{k} |D(T^{k})| \right) \\ & \qquad + (g_{A}(T) + g_{B}(T)) rT  + 3\left(\frac{f_{A}(g_{A}(T) \cdot T)}{g_{A}(T)} + \frac{f_{B}(g_{B}(T) \cdot T)}{g_{B}(T)}\right)\frac{k}{2} \\
& \leq \SQErr(\phk{T}{2},y^{1:T}) - 2\epsilon^{2}\left(\sum_{k=1}^{k} |D(T^{k})| \right)  + 3k(g_{A}(T) + g_{B}(T))T + 3k\left(\frac{f_{A}(g_{A}(T) \cdot T)}{g_{A}(T)} + \frac{f_{B}(g_{B}(T) \cdot T)}{g_{B}(T)}\right) 
\\ & = \SQErr(\phk{T}{2},y^{1:T}) - 2\epsilon^{2}\left(\sum_{k=1}^{k} |D(T^{k})| \right) + 3kT \beta(T,f_{A},f_{B})
\end{align*}
Finally we can compose this expression with one more instantiation of Theorem \ref{thm:cases}:
\begin{align*}
    \SQE(\phk{T}{2}, y^{1:T}) &\leq  \SQE(\pmk{T}{1}, y^{1:T}) -(\epsilon - g_{B}(T))^{2} |D(T^{1})| + g_{B}(T) T + 3\frac{f_{B}(g_{B}(T) \cdot T)}{g_{B}(T)} \\
      & \leq \SQE(\pmk{T}{1}, y^{1:T}) -\epsilon^{2}|D(T^{1})| + T\beta(T,f_{A},f_{B})
\end{align*}

and get a final expression of:
\begin{align*}
    \SQErr(\phk{T}{k}, y^{1:T}) & \leq \SQE(\pmk{T}{1}, y^{1:T}) - 2\epsilon^{2} \left(\sum_{k=1}^{k} |D(T^{k})| \right) + 3kT \beta(T,f_{A,f_{B}})
\end{align*}

Note also that $\SQE(\pmk{T}{1}, y^{1:T}) \leq T$ and  $\SQE(\pmk{T}{k}, y^{1:T}) \geq 0$. Therefore, we have that 

\begin{align*}
    0 & \leq T - 2\epsilon^{2} \left(\sum_{k=1}^{k} |D(T^{k})| \right) + rT \beta(T,f_{A},f_{B}) \\
  &   \implies  \sum_{k=1}^{k} |D(T^{k})| \leq \frac{T + rT \beta(T,f_{A},f_{B})}{2\epsilon^{2}}
\end{align*}

Thus, the round between $1$ and $k$ with the smallest number of disagreements has no more than $\frac{T + rT\beta(T,f_{A},f_{B})}{2r\epsilon^{2}}$ disagreements. Let $k$ be the index of this round. As there are $T$ predictions total in round $k$, the fraction of predictions in the round that are disagreements is
$$\frac{T + rT\beta(T,f_{A},f_{B})}{2rT\epsilon^{2}} = \frac{1}{2r\epsilon^{2}} + \frac{\beta(T,f_{A},f_{B})}{2\epsilon^{2}}$$
\end{proof}

\subsection{Proof of Theorem \ref{thm:main}}

\begin{lemma} If the sequence of real-valued predictions $a^{1:T}$ is $(\epsilon, \delta)$-close to the sequence $b^{1:T}$, and $a,b,y$ are all bounded above by $1$, then $$\sum_{t=1}^{T}(a^{t} - y^{t})^2 - \sum_{t=1}^{T}(b^{t} - y^{t})^2 \leq 4 (\delta + \epsilon) T$$ \label{lem:diffbound}
\end{lemma} 
\begin{proof}
    \begin{align*}
       & \sum_{t=1}^{T}(a^{t} - y^{t})^2 - \sum_{t=1}^{T}(b^{t} - y^{t})^2  \\
        & = \sum_{t=1}^{T} \mathbbm{1}[|a^{t} - b^{t}| \geq \epsilon] \left((a^{t} - y^{t})^2 - (b^{t} - y^{t})^2\right) + \sum_{t=1}^{T} \mathbbm{1}[|a^{t} - b^{t}| < \epsilon] \left((a^{t} - y^{t})^2 - (b^{t} - y^{t})^2\right) \\
        & \leq \sum_{t=1}^{T} \mathbbm{1}[|a^{t} - b^{t}| \geq \epsilon] \left(|a^{t} - b^{t}|\cdot |a^{t} + b^{t}| + 2 |y^{t}| \cdot |a^{t} - b^{t}|\right)  \\ & +\sum_{t=1}^{T} \mathbbm{1}[|a^{t} - b^{t}| < \epsilon] \left(|a^{t} - b^{t}|\cdot |a^{t} + b^{t}| + 2 |y^{t}| \cdot |a^{t} - b^{t}|\right) \\
        & \leq \sum_{t=1}^{T} \mathbbm{1}[|a^{t} - b^{t}| \geq \epsilon] \left(|a^{t} + b^{t}| + 2 |y^{t}|  \right) + \sum_{t=1}^{T} \mathbbm{1}[|a^{t} - b^{t}| < \epsilon] \left(\epsilon \cdot |a^{t} + b^{t}| + 2 |y^{t}| \cdot \epsilon \right) \\
        & \leq \sum_{t=1}^{T} \mathbbm{1}[|a^{t} - b^{t}| \geq \epsilon] \left(4 \right) + \sum_{t=1}^{T} \mathbbm{1}[|a^{t} - b^{t}| < \epsilon] \left(4 \cdot \epsilon \right) \tag{By the upper bounds on the values} \\
        & \leq 4 \delta T + 4 \epsilon (1 - \delta) T \leq 4  T (\delta + \epsilon)
    \end{align*}
\end{proof}

\begin{proof}[Proof of Theorem~\ref{thm:main}]
By Theorem~\ref{thm:CSR2CC}, Alice is $(f'_{A},g_{A})$-conversation calibrated and Bob is $(f'_{B},g_{B})$-conversation calibrated, where $f'_{A}(x) = \sqrt{x \cdot f_{A}(x)}$, and symmetrically for $f'_{b}$. Thus, by Theorem~\ref{thm:rounds}, after the collaboration protocol is run for $K$ rounds, there is at least one round $k + 1 > 1$ where the fraction of predictions that are $\epsilon$-far from the previous round is at most $\frac{1}{2K\epsilon^{2}} + \frac{\beta(T,f'_{A},f'_{B})}{2\epsilon^{2}}$, where $\beta(T,f'_{A},f'_{B})= 3\left(g_A(T) + g_{B}(T) + \frac{f'_{A}(g_{A}(T) \cdot T)}{g_{A}(T) \cdot T} + \frac{f'_{B}(g_{B}(T) \cdot T)}{g_{B}(T) \cdot T}\right)$.
Consider the round before round $k+1$, round $k$. 

First consider the case where $k$ is an even round. Then, by definition, the predictions $\hat{y}_{k}^{1},\ldots,\hat{y}_{k}^{T}$ in this round have $(f_{B}, g_{B}, H_{B})$-conversation swap regret. We will now define a sequence of predictions $\bar{y}$ which is $g_{B}T$-far in $L_{1}$ distance from $\hat{y}_{k}^{1},\ldots,\hat{y}_{k}^{T}$, and show that $\bar{y}$ has low swap regret to $\cH_{A} \cup \cH_{B}$. This sequence is generated by combining level sets of $\hat{y}_{k}^{1},\ldots,\hat{y}_{k}^{T}$ such that each level set is mapped to the closest value in $\{\frac{1}{g_{A}(T)},\ldots,1\}$. We will first compute the swap regret of $\bar{y}$ with respect to $\cH_{B}$:

\begin{align*}
   &  \sum_{t=1}^{T}(\bar{y}^{t} - y^{t})^{2} - \sum_{v}\min_{h \in \cH_{B}} \left( \sum_{t = 1}^{T}\1[\bar{y}^{t} = v](h(x^{t}) - y^{t})^{2}\right)  \\
    & \leq \sum_{t=1}^{T}(\bar{y}^{t} - y^{t})^{2} - \sum_{v}\min_{h \in \cH_{B}} \left( \sum_{t = 1}^{T}\1[\hat{y}_{k}^{t} = v](h(x^{t}) - y^{t})^{2}\right) \tag{As $\bar{y}$ has strictly coarser level sets than $\hat{y}_{k}$, here we are only strengthening the benchmark} \\
    & =  \sum_{t=1}^{T}(\hat{y}_{k}^{t} - y^{t})^{2} - \sum_{v}\min_{h \in \cH_{B}} \left( \sum_{t = 1}^{T}\1[\hat{y}_{k}^{t} = v](h(x^{t}) - y^{t})^{2}\right) + \left(\sum_{t=1}^{T}(\bar{y}^{t} - y^{t})^{2} - \sum_{t=1}^{T}(\hat{y}_{k}^{t} - y^{t})^{2} \right) \\
    & \leq \frac{f_{B}(g_{B}(T) T)}{g_{B}(T)} + \left(\sum_{t=1}^{T}(\bar{y}^{t} - y^{t})^{2} - \sum_{t=1}^{T}(\hat{y}_{k}^{t} - y^{t})^{2} \right) \tag{By Theorem~\ref{thm:collect}} \\
    & = \frac{f_{B}(g_{B}(T) T)}{g_{B}(T)} + \sum_{t=1}^{T}\left( (\bar{y}^{t})^{2} - (\hat{y}_{k}^{t})^{2} + 2y^{t}(y_{k}^{t} - \bar{y}^{t})  \right)  \\
     & \leq \frac{f_{B}(g_{B}(T) T)}{g_{B}(T)} + \sum_{t=1}^{T}\left( |\bar{y}^{t} - \hat{y}_{k}^{t}|\cdot |\bar{y}^{t} + \hat{y}_{k}^{t}| + 2|y^{t}| \cdot|y_{k}^{t} - \bar{y}^{t}|  \right) \\
    & = \frac{f_{B}(g_{B}(T) T)}{g_{B}(T)} + \sum_{t=1}^{T}\left( \frac{g_{A}(T)}{2} \cdot |\bar{y}^{t} + \hat{y}_{k}^{t}|+ 2|y^{t}| \cdot \frac{g_{A}(T)}{2}  \right) \tag{By construction of $\bar{y}$} \\
    & \leq \frac{f_{B}(g_{B}(T) T)}{g_{B}(T)} + \sum_{t=1}^{T}\left( \frac{3}{2}g_{A}(T) \right) \tag{As $y \leq 1$} = \frac{f_{B}(g_{B}(T)  \cdot T)}{g_{B}(T)} +  \frac{3g_{A}(T) \cdot T}{2} 
\end{align*}

Next, we will compute the swap regret of $\bar{y}$ with respect to $\cH_{A}$. Here, we crucially use the fact that the sequence $\hat{y}_{k+1}$ has high agreement with $\hat{y}_{k}$, and furthermore that $\hat{y}_{k+1}$ has low swap regret to $\cH_{A}$ \emph{exactly} on the level sets of $\bar{y}$. Let $T_{B}(k,i)$ be the subsequence of days on which Bob predicts in bucket $i$ in round $k$. 

\begin{align*}
   &  \sum_{t=1}^{T}(\bar{y}^{t} - y^{t})^{2} - \sum_{v}\min_{h \in \cH_{A}} \left( \sum_{t = 1}^{T}\1[\bar{y}^{t} = v](h(x^{t}) - y^{t})^{2}\right)  \\
    & =  \left(\sum_{t=1}^{T}(\bar{y}^{t} - y^{t})^{2} - \sum_{t=1}^{T}(\hat{y}_{k+1}^{t} - y^{t})^{2}\right) + \sum_{t=1}^{T}(\hat{y}_{k+1}^{t} - y^{t})^{2} - \sum_{v}\min_{h \in \cH_{A}} \left( \sum_{t = 1}^{T}\1[\bar{y}^{t} = v](h(x^{t}) - y^{t})^{2}\right)  \\
    & = 4(\epsilon + g_{A}(T) +\frac{1}{2K\epsilon^{2}} + \frac{\beta(T,f'_{A},f'_{B})}{2\epsilon^{2}})T + \sum_{t=1}^{T}(\hat{y}_{k+1}^{t} - y^{t})^{2} - \sum_{v}\min_{h \in \cH_{A}} \left( \sum_{t = 1}^{T}\1[\bar{y}^{t} = v](h(x^{t}) - y^{t})^{2}\right) \tag{By~\ref{lem:diffbound}, and the fact that $\hat{y}_{k+1}$ is $(\epsilon + g_{A}(T), \frac{1}{2K\epsilon^{2}} + \frac{\beta(T,f'_{A},f'_{B})}{2\epsilon^{2}})$-close to $\bar{y}$} \\
   & =  4(\epsilon + g_{A}(T) +\frac{1}{2K\epsilon^{2}} + \frac{\beta(T,f'_{A},f'_{B})}{2\epsilon^{2}})T +\sum_{i} \sum_{t \in T_{B}(r,i)}(\hat{y}_{k+1}^{t} - y^{t})^{2} - \sum_{v}\min_{h \in \cH_{A}} \left( \sum_{t = 1}^{T}\1[\bar{y}^{t} = v](h(x^{t}) - y^{t})^{2}\right) \\
   & =  4(\epsilon + g_{A}(T) +\frac{1}{2K\epsilon^{2}} + \frac{\beta(T,f'_{A},f'_{B})}{2\epsilon^{2}})T +\sum_{i} \sum_{t \in T_{B}(r,i)}(\hat{y}_{k+1}^{t} - y^{t})^{2} - \sum_{i}\min_{h \in \cH_{A}} \left( \sum_{t \in  T_{B}(r,i)}(h(x^{t}) - y^{t})^{2}\right) \tag{As $\bar{y}$ attains a particular value exactly when $t \in T_{B}(r,i)$; that is, when Bob predicts in bucket $i$ in round $k$.} \\
    & \leq  4(\epsilon + g_{A}(T) +\frac{1}{2K\epsilon^{2}} + \frac{\beta(T,f'_{A},f'_{B})}{2\epsilon^{2}})T + \sum_{i} \sum_{t \in T_{B}(r,i)}(\hat{y}_{k+1}^{t} - y^{t})^{2} \\
    & \quad\quad\quad\quad\quad\quad- \sum_{i} \sum_{v}\min_{h \in \cH_{A}} \left( \sum_{t \in  T_{B}(r,i)}\mathbbm{1}[\hat{y}_{k+1}^{t} = v](h(x^{t}) - y^{t})^{2}\right) \tag{As we are only making the benchmark more powerful} \\
     & \leq  4(\epsilon + g_{A}(T) +\frac{1}{2K\epsilon^{2}} + \frac{\beta(T,f'_{A},f'_{B})}{2\epsilon^{2}})T  \\
     & \quad\quad\quad\quad\quad\quad + \sum_{i} \left( \sum_{t \in T_{B}(r,i)}(\hat{y}_{k+1}^{t} - y^{t})^{2} - \sum_{v}\min_{h \in \cH_{A}} \left( \sum_{t \in  T_{B}(r,i)}\mathbbm{1}[\hat{y}_{k+1}^{t} = v](h(x^{t}) - y^{t})^{2}\right) \right)\\
     & \leq  4(\epsilon + g_{A}(T) +\frac{1}{2K\epsilon^{2}} + \frac{\beta(T,f'_{A},f'_{B})}{2\epsilon^{2}})T + \sum_{i} \left( f_{A}(|T_{B}(r,i)|) \right) \tag{By the conversation swap regret of Alice}\\
    & = 4(\epsilon + g_{A}(T) + \frac{1}{2K\epsilon^{2}} + \frac{\beta(T,f'_{A},f'_{B})}{2\epsilon^{2}})T + \frac{f_{A}(g_{A}(T)\cdot T)}{g_{A}(T)} \tag{By the concavity of $f_{A}$}
\end{align*}

Thus, $\bar{y}$ simultaneously has $(4(\epsilon + g_{A}(T) + \frac{1}{2K\epsilon^{2}} + \frac{\beta(T,f'_{A},f'_{B})}{2\epsilon^{2}})T + \frac{f_{A}(g_{A}(T)\cdot T)}{g_{A}(T)}, \cH_{A})$-Swap Regret and $(\frac{3g_{A}(T)\cdot T}{2} + \frac{f_{B}(g_{B}(T)\cdot T)}{g_{B}(T)}, \cH_{B})$-Swap Regret. Thus, it has at most 
$$\left(4T(\epsilon + g_{A}(T) + \frac{1}{2K\epsilon^{2}} + \frac{\beta(T,f'_{A},f'_{B})}{2\epsilon^{2}}) + \frac{f_{A}(g_{A}(T)\cdot T)}{g_{A}(T)} + \frac{3g_{A}(T)\cdot T}{2} + \frac{f_{B}(g_{B}(T)\cdot T)}{g_{B}(T)}, \cH_{A} \cup \cH_{B}\right)$$-Swap Regret. 

 Note that we can select the agreement parameter $\epsilon$ here however we like in order to minimize the swap regret. In particular, we would like to pick $\epsilon$ to minimize the expression $\epsilon + \frac{1}{2K\epsilon^{2}} + \frac{\beta(T,f'_{A},f'_{B})}{2\epsilon^{2}} = \epsilon + \frac{\beta(T,f'_{A},f'_{B}) + 1/K}{2\epsilon^{2}}$. 
By setting $\epsilon = (\frac{\beta(T,f'_{A},f'_{B}) + 1/K}{2})^{\frac{1}{3}}$, we get that  

\begin{align*}
    \epsilon + \frac{\beta(T,f'_{A},f'_{B}) + 1/K}{2\epsilon^{2}} & = \\
    & (\frac{\beta(T,f'_{A},f'_{B}) + 1/K}{2})^{\frac{1}{3}} + \frac{\beta(T,f'_{A},f'_{B}) + 1/K}{2(\frac{\beta(T,f'_{A},f'_{B}) + 1/K}{2})^{\frac{2}{3}}} \\
    & = 2(\frac{\beta(T,f'_{A},f'_{B}) + 1/K}{2})^{\frac{1}{3}} 
\end{align*}

Plugging this back into the swap regret expression, we get that, if $k$ is an even round, 
$\bar{y}_{k}$ has at most

\begin{align*}
   (8T(\frac{\beta(T,f'_{A},f'_{B}) + 1/K}{2})^{\frac{1}{3}} + \frac{11}{2}Tg_{A}(T) + \frac{f_{A}(g_{A}(T)\cdot T)}{g_{A}(T)} +  \frac{f_{B}(g_{B}(T)\cdot T)}{g_{B}(T)}, \cH_{A} \cup \cH_{B})
\end{align*}-swap regret. 

In the case where $k$ is an odd round, by a symmetric argument in which we define $\bar{y}_{k}$ by combining level sets of $\hat{y}_{k}$ to map to the closest value in $g_{B}(T)$, $\bar{y}_{k}$ has

\begin{align*}
    (8T(\frac{\beta(T,f'_{A},f'_{B}) + 1/K}{2})^{\frac{1}{3}} + \frac{11}{2}Tg_{B}(T) + \frac{f_{B}(g_{B}(T)\cdot T)}{g_{B}(T)} +  \frac{f_{A}(g_{A}(T)\cdot T)}{g_{A}(T)}, \cH_{A} \cup \cH_{B})
\end{align*}-swap regret.

Thus, in all cases, the swap regret of $\bar{y}_{k}$ with respect to $\cH_{A} \cup \cH_{B}$ is at most 
\begin{align*}
  &  8T(\frac{\beta(T,f'_{A},f'_{B}) + 1/K}{2})^{\frac{1}{3}} + \frac{11}{2}Tg_{B}(T) + \frac{11}{2}Tg_{A}(T) + \frac{f_{B}(g_{B}(T)\cdot T)}{g_{B}(T)} +  \frac{f_{A}(g_{A}(T)\cdot T)}{g_{A}(T)} \\
    & = 8T(\frac{\beta(T,f'_{A},f'_{B}) + 1/K}{2})^{\frac{1}{3}} + \frac{11}{2}T\beta(T,f_{A},f_{B})
\end{align*}

Note that $\hat{y}_{k}$ is close in $L_{1}$ distance to $\bar{y}$, as we have only modified each entry by either at most $\frac{g_{A}(T)}{2}$ or at most $\frac{g_{B}(T)}{2}$, depending if it was an even or odd round. Therefore, $\hat{y}_{k}$ has at most
\begin{align*}
    (\frac{T}{2}(g_{A}(T) + g_{B}(T)), 8T(\frac{\beta(T, f'_{A},f'_{B}) + 1/K}{2})^{\frac{1}{3}} + \frac{11}{2}T\beta(T,f_{A},f_{B}), \cH_{A} \cup \cH_{B})
\end{align*}-distance to swap regret.

\end{proof}

\subsection{Proof of  Theorem~\ref{thm:last-round}}
\begin{proof}[Proof of Theorem~\ref{thm:last-round}]
By Theorem~\ref{thm:main}, if Alice has $(f_{A}, g_{A}, \cH_A)$-conversation swap regret and Bob has $(f_{B}, g_{B}, \cH_B)$-conversation swap regret, there exists a round $k$ of the protocol that has $(\frac{T}{2}(g_{A}(T) + g_{B}(T)), 8T(\frac{\beta(T,f'_{A},f'_{B}) + 1/K}{2})^{\frac{1}{3}} + \frac{11}{2}T\beta(T,f_{A},f_{B}), \cH_{A} \cup \cH_{B})$-distance to swap regret, where $\beta(T,f_{A},f_{B}) = \frac{f_{A}(g_{A}(T)\cdot T)}{Tg_{A}(T)} + \frac{f_{B}(g_{B}(T)\cdot T)}{Tg_{B}(T)} + g_{A}(T) + g_{B}(T)$, $f_{A}'(x) = \sqrt{x \cdot f_{A}(x)}$ and $f_{B}'(x) = \sqrt{x \cdot f_{B}(x)}$. Then by the fact that $\cH_A$ and $\cH_B$ jointly satisfy the $w(\cdot)$-weak learning condition with respect to $\cH_J$ and via Theorem~\ref{thm:weak-learning}, instantiating $f^{S} = 8T(\frac{\beta(T,f'_{A},f'_{B}) + 1/K}{2})^{\frac{1}{3}} + \frac{11}{2}T\beta(T,f_{A},f_{B})$ and $f^{D} = \frac{T}{2}(g_{A}(T) + g_{B}(T))$, we have that for the predictions $\hat{y}^{k,t}$ in round $k$: 
\begin{align*}
 & \sum_{t=1}^T (\hat{y}^{k,t}-y^t)^2 - \min_{h_J\in\cH_J} \sum_{t=1}^T (h_J(x^t) - y^t)^2 \\ & \leq 2Tw\inv\left(\frac{8T(\frac{\beta(T,f'_{A},f'_{B}) + 1/K}{2})^{\frac{1}{3}} + \frac{11}{2}T\beta(T,f_{A},f_{B})}{T}\right) + 3\frac{T}{2}(g_{A}(T) + g_{B}(T)) \\
 & = 2Tw\inv\left(8(\frac{\beta(T,f'_{A},f'_{B}) + 1/K}{2})^{\frac{1}{3}} + \frac{11}{2}\beta(T,f_{A},f_{B})\right) + 3\frac{T}{2}(g_{A}(T) + g_{B}(T)) \\
\end{align*}

By Theorem~\ref{thm:cases}, we can upper bound the increase in squared error from round $i$ to round $i+2$ by $3 T\beta(T,f'_{A},f'_{B})$. The maximum number of rounds between $k$ and $K$ is $K$. Therefore, we have that 

\begin{align*}
&\sum_{t=1}^T (\hat{y}^{K,t}-y^t)^2 \leq \sum_{t=1}^T (\hat{y}^{k,t}-y^t)^2 + 3TK\beta(T,f'_{A},f'_{B}) \\    
\end{align*}

Combining the above results, we have that

\begin{align*}
&\sum_{t=1}^T (\hat{y}^{K,t}-y^t)^2 - \min_{h_J\in\cH_J} \sum_{t=1}^T (h_J(x^t) - y^t)^2 \leq \\ & 2Tw^{-1}\left(8(\frac{\beta(T,f'_{A},f'_{B}) + 1/K}{2})^{\frac{1}{3}} + \frac{11}{2}\beta(T,f_{A},f_{B})\right) + 3\frac{T}{2}(g_{A}(T) + g_{B}(T)) + 3TK\beta(T,f'_{A},f'_{B}) \\    
\end{align*}
\end{proof}

\subsection{Proof of Theorem~\ref{thm:alg-conversation-sr}}
\begin{proof}[Proof of Theorem~\ref{thm:alg-conversation-sr}]
    Let $\rho' = \frac{2g(T)\rho}{K}$. Let $M$ be the algorithm given by the reduction in Theorem \ref{thm:external-to-swap}, given an online algorithm $M_0$ that achieves external regret with respect to $\cH$ bounded by $r(\tau)$ for any $\tau\in[T]$. In particular, Theorem \ref{thm:external-to-swap} guarantees that with probability $1-\rho'$, $M$ achieves $(f, \cH)$-swap regret for:
    \[
    f(\tau) \leq m \cdot r\left(\frac{\tau}{m}\right)+ \frac{3\tau}{m} + m + \max(8B,2\sqrt{B}) \cdot m \cdot C_{\cH} \cdot \sqrt{\tau \log\left(\frac{2mK}{g(T)\rho}\right)}
    \]

    By construction, on every odd round $k$, a separate copy $M_{k,i}$ is run for every subsequence on which the previous prediction falls into bucket $i$. By a union bound, the probability that any one of the copies fails is $\frac{K}{2}\cdot \frac{1}{g(T)}\cdot\rho' = \rho$. Then, since conversation swap regret measures the swap regret conditioned on subsequences on which the previous prediction falls into bucket $i$ (as parameterized by $g$), with probability $1-\rho$, Algorithm \ref{alg:csr-algorithm} also satisfies $(f, g, \cH)$-conversation swap regret.

\end{proof}

\subsection{Proof of Theorem \ref{thm:sublinear-to-sublinear}}

\begin{lemma}\label{lem:limit}
    If $w$ is continuous and strictly convex, $w(0) = 0$ and $\lim_{T \rightarrow \infty}s(T) = 0$, then $\lim_{T \rightarrow \infty}w^{-1}(s(T)) = 0$. 
\end{lemma}
\begin{proof}
Note that as $w$ is strictly monotone, $w^{-1}$ is defined everywhere in the range of $(0,c)$, where $c = \lim_{x \rightarrow \infty}\inf(w(x))$ and $c > 0$. As $w(0) = 0$, it must be the case that $w^{-1}(0) = 0$. Furthermore, as $w$ is continuous, $w^{-1}$ must be continuous. Now, we can proceed to reason about $w^{-1}$:
\begin{align*}
    \lim_{T \rightarrow \infty}w^{-1}(x(T)) & =  w^{-1}\lim_{T \rightarrow \infty}(x(T)) \tag{By the continuity of $w^{-1}$}\\
    & = f(0) \tag{By the fact that $\lim_{T \rightarrow \infty}s(T) = 0$} \\
    & = 0 \tag{By the fact that $w^{-1}(0) = 0$}
\end{align*}
\end{proof}

\begin{proofof}{Theorem \ref{thm:sublinear-to-sublinear}}
    Let $\rho' = \rho/2$. We set our parameters to be sublinear in $T$. Specifically, set $m=T^{1/4}$ and $1/g_A(T)=1/g_B(T)=T^{\alpha_g}$ for some constant $\alpha_g\in(0,1)$.
    By Theorem \ref{thm:alg-conversation-sr}, there is an algorithm that achieves, with probability $1-\rho'$, $(f_A,g_A,\cH_A)$-conversation swap regret for, for any $\tau\in[T]$:
    \begin{align*}
        f_A(\tau) &\leq m \cdot r_A\left(\frac{\tau}{m}\right)+ \frac{3\tau}{m} + m + \max(8B,2\sqrt{B}) \cdot m \cdot C_{\cH} \cdot \sqrt{\tau \log\left(\frac{2mK}{g_A(T)\rho'}\right)} \tag{by Theorem \ref{thm:alg-conversation-sr}}\\
        &\leq m \cdot r_A\left(\frac{T}{m}\right)+ \frac{3T}{m} + m + \max(8B,2\sqrt{B}) \cdot m \cdot C_{\cH} \cdot \sqrt{T \log\left(\frac{4mK}{g_A(T)\rho}\right)} \\
        &\leq T^{1/4} \cdot \tilde{O}((T^{3/4})^{\alpha_A}) + 3T^{3/4} + T^{1/4} + \max(8B,2\sqrt{B}) \cdot C_\cH \cdot T^{3/4} \sqrt{\log\left(\frac{4KT^{1/4+\alpha_g}}{\rho}\right)} \\
        &\leq \tilde{O}\left(T^{\alpha_1} \sqrt{\log \left( \frac{K}{\rho} \right)}\right)
    \end{align*}
    for $\alpha_1 = \max\{1/4+3/4\cdot\alpha_A, 3/4\} \in(0,1)$. Since Bob's expression is symmetric, Theorem \ref{thm:alg-conversation-sr} similarly implies that there is an algorithm that achieves, with probability $1-\rho'$, $(f_B,g_B,\cH_B)$-conversation swap regret for:
    \begin{align*}
        f_B(\tau) \leq \tilde{O}\left( T^{\alpha_2} \sqrt{\log \left( \frac{K}{\rho} \right)}\right)
    \end{align*}
    for $\alpha_2=\max\{1/4+3/4\cdot\alpha_B, 3/4\}\in(0,1)$. Thus, by a union bound, with probability $1-2\rho'=1-\rho$, Alice has $(f_A,g_A,\cH_A)$-conversation swap regret and Bob has $(f_B,g_B,\cH_B)$-conversation swap regret. 

    Now, by Theorem \ref{thm:last-round}, the transcript on the last round has regret bounded by:   
    \begin{align*}
        &\sum_{t=1}^T (\hat{y}^{t,K}-y^t)^2 - \min_{h_J\in\cH_J} \sum_{t=1}^T (h_J(x^t) - y^t)^2 \\
        &\leq 2Tw\inv\left(8\left(\frac{\beta(T,f'_A,f'_B) + 1/K}{2}\right)^{\frac{1}{3}} + \frac{11}{2}\beta(T,f_A,f_B)\right) + 3\frac{T}{2}(g_{A}(T) + g_{B}(T)) + 3TK\beta(T,f'_A,f'_B)
    \end{align*}
    where for $\tau\leq T$:
    \begin{align*}
        f'_A(\tau) &= \sqrt{\tau \cdot f_A(\tau)} \leq \sqrt{ T \cdot \tilde{O}\left( T^{\alpha_1} \sqrt{\log \left( \frac{K}{\rho} \right)}\right)} \leq \tilde{O}\left( T^{(1+\alpha_1)/2}\log^{1/4} \left( \frac{K}{\rho} \right)\right), \\
        f'_B(\tau) &= \sqrt{\tau \cdot f_A(\tau)} \leq \sqrt{ T \cdot \tilde{O}\left( T^{\alpha_2} \sqrt{\log \left( \frac{K}{\rho} \right)}\right)} \leq \tilde{O}\left( T^{(1+\alpha_2)/2} \log^{1/4} \left( \frac{K}{\rho} \right)\right)
    \end{align*}
    and thus:
    \begin{align*}
        \beta(T, f_A, f_B) &= \frac{f_A(g_{A}(T)T)}{Tg_{A}(T)} + \frac{f_B(g_{B}(T)T)}{Tg_{B}(T)} + g_{A}(T) + g_{B}(T) \\
        &\leq \tilde{O}\left( (T^{\alpha_1+\alpha_g-1}+T^{\alpha_2+\alpha_g-1}) \sqrt{\log \left( \frac{K}{\rho} \right)} + T^{-\alpha_g} \right), \\
        \beta(T, f'_A, f'_B) &= \frac{f'_A(g_{A}(T)T)}{Tg_{A}(T)} + \frac{f'_B(g_{B}(T)T)}{Tg_{B}(T)} + g_{A}(T) + g_{B}(T) \\ 
        &\leq \tilde{O}\left( (T^{\alpha_1/2+\alpha_g-1/2}+T^{\alpha_2/2+\alpha_g-1/2}) \log ^{1/4}\left( \frac{K}{\rho} \right) + T^{-\alpha_g} \right)
    \end{align*}
    Suppose $\alpha_g < \min\{1/2-\alpha_1/2, 1/2-\alpha_2/2\}$. Then: $$T^{\alpha_1+\alpha_g-1}, T^{\alpha_2+\alpha_g-1}, T^{\alpha_1/2+\alpha_g-1/2}, T^{\alpha_2/2+\alpha_g-1/2} \leq T^{-\alpha}$$
    for some constant $\alpha\in(0,1)$.
    Hence, plugging in to the expression above, we have that:
    \begin{align*}
        &\sum_{t=1}^T (\hat{y}^{t,K}-y^t)^2 - \min_{h_J\in\cH_J} \sum_{t=1}^T (h_J(x^t) - y^t)^2 \\
        &\leq 2Tw\inv\left(\tilde{O}\left( T^{-\alpha} \sqrt{\log \left( \frac{K}{\rho} \right)} + T^{-\alpha_g}  +\frac{1}{K} \right)^{1/3} + \tilde{O}\left( T^{-\alpha} \sqrt{\log \left( \frac{K}{\rho} \right)} + T^{-\alpha_g} \right) \right) \\ & \ \ \ \ + O(T^{1-\alpha_g}) + T \cdot \tilde{O}\left( T^{-\alpha} K\log ^{1/4}\left( \frac{K}{\rho} \right) + KT^{-\alpha_g} \right) \\
        &\leq 2Tw\inv\left(\tilde{O}\left( T^{-\alpha/3} \log^{1/6} \left( \frac{K}{\rho} \right) + T^{-\alpha_g/3}   +\frac{1}{K^{1/3}} \right) + \tilde{O}\left( T^{-\alpha} \sqrt{\log \left( \frac{K}{\rho} \right)} + T^{-\alpha_g} \right) \right) \\ & \ \ \ \ + O(T^{1-\alpha_g}) + T\cdot \tilde{O}\left( T^{-\alpha} K\log ^{1/4}\left( \frac{K}{\rho} \right) + KT^{-\alpha_g} \right) \tag{by concavity of the cube root function} \\
        &\leq 2Tw\inv\left(\tilde{O}\left( T^{-\alpha/3} \sqrt{\log \left( \frac{K}{\rho} \right)} + T^{-\alpha_g/3}   +\frac{1}{K^{1/3}} \right) \right) + O(T^{1-\alpha_g}) + \tilde{O}\left( KT^{1-\alpha} \log ^{1/4}\left( \frac{K}{\rho} \right) + KT^{1-\alpha_g} \right) \\
        &\leq 2Tw\inv\left(\tilde{O}\left( T^{-\alpha'} \sqrt{\log \left( \frac{K}{\rho} \right)} +\frac{1}{K^{1/3}} \right) \right) + O(T^{1-\alpha_g}) + \tilde{O}\left( KT^{1-\alpha''} \log ^{1/4}\left( \frac{K}{\rho} \right)\right)
    \end{align*}
    where $\alpha' = \min\{\alpha/3, \alpha_g/3\}\in(0,1)$ and $\alpha'' = \min\{\alpha, \alpha_g\}\in(0,1)$. This proves the first part of the theorem. 

    To argue the second part, suppose $K=\omega(1)$ and $K=O(T^{\alpha''-\eps})$ for $\eps>0$. Then:
    \begin{align*}
        &\sum_{t=1}^T (\hat{y}^{t,K}-y^t)^2 - \min_{h_J\in\cH_J} \sum_{t=1}^T (h_J(x^t) - y^t)^2 \\ 
        &\leq 2Tw\inv\left(\tilde{O}\left( T^{-\alpha'} \sqrt{\log \left( \frac{K}{\rho} \right)} +\frac{1}{K^{1/3}} \right) \right) + O(T^{1-\alpha_g}) + \tilde{O}\left( KT^{1-\alpha''} \log ^{1/4}\left( \frac{K}{\rho} \right)\right) \\
        &= 2Tw\inv\left(\tilde{O}\left( T^{-\alpha'} \sqrt{\log \left( \frac{T}{\rho} \right)} \right) + o\left(1 \right) \right) + O(T^{1-\alpha_g}) + \tilde{O}\left( T^{1-\eps} \log ^{1/4}\left( \frac{T}{\rho} \right)\right)
    \end{align*}
    Now, observe that any function $\tilde{O}\left( T^{-\alpha'} \sqrt{\log \left( \frac{T}{\rho} \right)} \right) + o\left(1 \right)  \to 0 $ as $T\to \infty$. Thus, by Lemma \ref{lem:limit}, $w\inv\left(\tilde{O}\left( T^{-\alpha'} \sqrt{\log \left( \frac{T}{\rho} \right)} \right) + o\left(1 \right) \right) \to 0$ as $T\to\infty$. In particular, this implies that $$Tw\inv\left(\tilde{O}\left(T^{-\alpha'} \sqrt{\log \left( \frac{T}{\rho} \right)} + o(1) \right)\right) = o(T)$$ i.e. is sublinear in $T$. Notice that since $w$ is strictly increasing, $w\inv$ exists for sufficiently large $T$ (larger than a constant). Therefore, for sufficiently large $T$, the regret is bounded by: 
    \begin{align*}
        \sum_{t=1}^T (\hat{y}^{t,K}-y^t)^2 - \min_{h_J\in\cH_J} \sum_{t=1}^T (h_J(x^t) - y^t)^2 \leq o(T) + O(T^{1-\alpha_g}) + \tilde{O}\left( T^{1-\eps} \log ^{1/4}\left( \frac{T}{\rho} \right)\right) 
    \end{align*}
    which completes the proof.

\end{proofof}

\subsection{Proof of Theorem~\ref{thm:linear}}

\begin{theorem}\citep{rakhlin2015online}
    Let $\cX=\{x\in\mathbb{R}^d : \|x\|_2\leq 1\}$ and $\cH = \{x\mapsto\langle\theta,x\rangle : \|\theta\|_2 \leq C\}$ be the set of all linear functions with bounded norm. $\cH$ has finite sequential fat-shattering dimension.
\end{theorem}

\begin{corollary}\label{cor:linear-alg}
     Let $\cX=\{x\in\mathbb{R}^d : \|x\|_2\leq 1\}$ and $\cH = \{x\mapsto\langle\theta,x\rangle : \|\theta\|_2 \leq C\}$ be the set of all linear functions with bounded norm. Fix any bucketing function $g$.  There exists an online algorithm that, with probability $1-\rho$, achieves $(\tilde{O}\left( \max(C^2,C) d\log\left(\frac{K}{g(T)\rho}\right)T^{3/4}\right) , g, \cH)$-conversation swap regret. 
\end{corollary}
\begin{proof}
    We have that $\langle\theta,x\rangle^2 \leq \|\theta\|_2^2\|x\|_2^2\leq C^2$. Therefore, by setting $m = T^{\frac{1}{4}}$ and instantiating Theorem \ref{thm:alg-conversation-sr} with the external regret algorithm of Theorem \ref{thm:linear-external-regret}, we have that, with probability $1-\rho$, Algorithm \ref{alg:csr-algorithm} achieves $(f,g,\cH)$-conversation swap regret for:
    \begin{align*}
        & f(|T(k-1),i)|) \\
        &\leq T^{\frac{1}{4}} r\left(\frac{|T(k-1,i)|}{T^{\frac{1}{4}}}\right)+ \frac{3|T(k-1,i)|}{T^{\frac{1}{4}}} + T^{\frac{1}{4}} + \max(8C^2,2C) T^{\frac{1}{4}} C_{\cH} \sqrt{|T(k-1,i)|\log\left(\frac{2KT^{\frac{1}{4}}}{g(T)\rho}\right)} \tag{by Theorem \ref{thm:alg-conversation-sr} and our setting of $m$}\\
        &\leq T^{\frac{1}{4}} \left(2d\ln\left(\frac{|T(k-1,i)|}{T^{\frac{1}{4}}}+1\right) + C^2 \right)+ \frac{3|T(k-1,i)|}{T^{\frac{1}{4}}} + T^{\frac{1}{4}} \\ & \ \ \ \ + \max(8C^2,2C)  T^{\frac{1}{4}}  C_{\cH} \sqrt{|T(k-1,i)|\log\left(\frac{2KT^{\frac{1}{4}}}{g(T)\rho}\right)} \tag{by Theorem \ref{thm:linear-external-regret}} \\
        &\leq \tilde{O}\left( \max(C^2,C) d\log\left(\frac{K}{g(T)\rho}\right)T^{3/4}\right) \\
        &= \tilde{O}\left( \max(C^2,C) d\log\left(\frac{K}{g(T)\rho}\right)T^{3/4}\right) 
    \end{align*}

\end{proof}

\begin{proof}[Proof of Theorem~\ref{thm:linear}]
    Let $\rho' = \rho/2$. By Corollary \ref{cor:linear-alg}, for any bucketing function $g_A$, there is an algorithm that achieves, with probability $1-\rho'$,$(f_A, g_A, \cH_A)$-conversation swap regret for:
    \[
    f_A(|T_B(k-1,i)|) \leq \tilde{O}\left(\max(C^{2},C)d\log\left(\frac{K}{g_A(T)\rho}\right)T^{3/4}\right) 
    \]
    Likewise, for any bucketing function $g_B$, there is an algorithm that achieves, with probability $1-\rho'$, $(f_B, g_B, \cH_B)$-conversation swap regret for:
    \[
    f_B(|T_A(k-1,i)|) \leq \tilde{O}\left(\max(C^{2},C)d\log\left(\frac{K}{g_B(T)\rho}\right)T^{3/4}\right) 
    \]
    Thus by a union bound, with probability $1-2\rho' = 1-\rho$, Alice has $(f_A, g_A, \cH_A)$-conversation swap regret and Bob has $(f_B, g_B, \cH_B)$-conversation swap regret. Let $g_{A} = g_{B} = T^{-\frac{1}{8}} $.  

    Now, by Theorem \ref{thm:linear-weak-learning}, $\cH_A$ and $\cH_B$ jointly satisfy the $w(\cdot)$-weak learning condition with respect to $\cH_J$ for $w(\gamma) = \frac{\gamma^2}{16C^2}$. In particular, we have that $w\inv(\gamma) = 4C\gamma^{1/2}$ for $\gamma \leq \frac{1}{16C^2}$. Therefore, by Theorem \ref{thm:last-round}, we have that the transcript $\pi^{1:T,K}$ at the last round satisfies:
    \begin{align*}
        &\sum_{t=1}^T (\hat{y}^{t,K}-y^t)^2 - \min_{h_J\in\cH_J} \sum_{t=1}^T (h_J(x^t) - y^t)^2 \\
        &\leq  
    2Tw\inv\left(8(\frac{\beta(T,f'_{A},f'_{B}) + 1/K}{2})^{\frac{1}{3}} + \frac{11}{2}\beta(T,f_{A},f_{B})\right) + 3\frac{T}{2}(g_{A}(T) + g_{B}(T)) + 3TK\beta(T,f'_{A},f'_{B}) \\
    & = \tilde{O}\left( T ((\beta(T,f'_{A},f'_{B}) + 1/K)^{\frac{1}{3}} + \beta(T,f_{A},f_{B}))^{\frac{1}{2}} + T(g_{A}(T) + g_{B}(T)) + TK\beta(T,f'_{A},f'_{B})\right) \tag{By Theorem~\ref{thm:linear-weak-learning}} \\
     & = \tilde{O}\left(  T(\beta(T,f'_{A},f'_{B}) + 1/K)^{\frac{1}{6}} + T\beta^{\frac{1}{2}}(T,f) + Tg_{A}(T) + Tg_{B}(T) + TK\beta(T,f'_{A},f'_{B})\right)  \\
    & = \tilde{O}\left( T\beta^{\frac{1}{6}}(T,f') + TK^{-\frac{1}{6}} + T\beta^{\frac{1}{2}}(T,f) + Tg_{A}(T) + Tg_{B}(T) + TK\beta(T,f'_{A},f'_{B})\right)  \\
    & = \tilde{O}\left( T\beta^{\frac{1}{6}}(T,f') + TK^{-\frac{1}{6}} + T\beta^{\frac{1}{2}}(T,f) + T^{\frac{7}{8}} + TK\beta(T,f'_{A},f'_{B})\right)  \tag{Instantiating $g_{A}$ and $g_{B}$} \\
    \\ 
    \end{align*}
    Here, $\beta(T,f_{A},f_{B}) = \frac{f(g_{A}(T)\cdot T)}{Tg_{A}(T)} + \frac{f(g_{B}(T)\cdot T)}{Tg_{B}(T)} + g_{A}(T) + g_{B}(T)$, and $f'(x) = \sqrt{x \cdot f(x)}$.

    Plugging in $f_A$ and $f_B$, we have that:
    \begin{align*}
        \beta(T, f_A, f_B) &\leq \tilde{O}\left(\frac{d\log\left(\frac{K}{g_A(T)\rho}\right)}{g_A(T)T^{1/4}} + \frac{d\log\left(\frac{K}{g_B(T)\rho}\right)}{g_B(T)T^{1/4}} + g_A(T) + g_B(T) \right) \\
        & = \tilde{O}\left(d\log\left(\frac{KT^{\frac{1}{8}}}{\rho}\right)T^{-1/8} + T^{-1/8} \right) \\
        & = \tilde{O}\left(d\log\left(\frac{KT^{\frac{1}{8}}}{\rho}\right)T^{-1/8} \right)
    \end{align*}

   Moveover:
    \begin{align*}
        \beta(T,f'_A,f'_B) & \leq \tilde{O}\left( \frac{\sqrt{g_{A}(T) \cdot T \cdot f(g_{A}(T) \cdot T)}}{T g_{A}(T)} + \frac{\sqrt{g_{B}(T) \cdot T \cdot f(g_{B}(T) \cdot T)}}{T g_{B}(T)}  + g_{A}(T) + g_{B}(T) \right) \\
        & =  \tilde{O}\left(\sqrt{\frac{f(g_{A}(T) \cdot T)}{T g_{A}(T)}} + \sqrt{\frac{f(g_{B}(T) \cdot T)}{T g_{B}(T)}}  + g_{A}(T) + g_{B}(T) \right) \\
        & =  \tilde{O}\left(\sqrt{\frac{\max(C^{2},C)d\log\left(\frac{K}{g_A(T)\rho}\right)T^{3/4}}{T g_{A}(T)}} + \sqrt{\frac{\max(C^{2},C)d\log\left(\frac{K}{g_B(T)\rho}\right)T^{3/4}}{T g_{B}(T)}}  + g_{A}(T) + g_{B}(T) \right) \\
        & =  \tilde{O}\Bigg(T^{-\frac{1}{8}}g_{A}^{-\frac{1}{2}}(T)\sqrt{\max(C^{2},C)d\log\left(\frac{K}{g_A(T)\rho}\right)} \\ & \quad  \quad \quad\quad +  T^{-\frac{1}{8}}g_{B}^{-\frac{1}{2}}(T)\sqrt{\max(C^{2},C)d\log\left(\frac{K}{g_B(T)\rho}\right)}  + g_{A}(T) + g_{B}(T) \Bigg) \\
        & =  \tilde{O}\left(T^{-\frac{1}{8}}T^{1/16}\sqrt{\max(C^{2},C)d\log\left(\frac{KT^{1/8}}{\rho}\right)} + T^{-1/8}\right) \\
         & =  \tilde{O}\left(T^{-1/16}\sqrt{\max(C^{2},C)d\log\left(\frac{KT^{1/8}}{\rho}\right)}\right)
    \end{align*}

Plugging these expressions into the regret bound of the final round, we get that:
\begin{align*}
   &  \tilde{O}\left( T\beta^{\frac{1}{6}}(T,f'_{A},f'_{B}) + TK^{-\frac{1}{6}} + T\beta^{\frac{1}{2}}(T,f_{A},f_{B}) + T^{\frac{7}{8}} + TK\beta(T,f'_{A},f'_{B})\right) \\
   & = \tilde{O} ( T (T^{-1/8}\sqrt{\max(C^{2},C)d\log\left(\frac{KT^{1/8}}{\rho}\right)})^{1/6} + TK^{-\frac{1}{6}} + T(\max(C^{2},C)d\log\left(\frac{KT^{\frac{1}{8}}}{\rho}\right)T^{-1/16})^{1/2} \\ & + T^{\frac{7}{8}} + KT^{\frac{7}{8}}\sqrt{\max(C^{2},C)d\log\left(\frac{KT^{1/8}}{\rho}\right)}) \\
   & = \tilde{O} ( T^{47/48}(\max(C^{2},C)d\log\left(\frac{KT^{1/8}}{\rho}\right))^{1/12} + TK^{-\frac{1}{6}} + T^{31/32}(\max(C^{2},C)d\log\left(\frac{KT^{\frac{1}{8}}}{\rho}\right))^{1/2} \\ & + T^{\frac{7}{8}} + KT^{\frac{7}{8}}\sqrt{\max(C^{2},C)d\log\left(\frac{KT^{1/8}}{\rho}\right)}) \\
   & =  \tilde{O} \left( T^{47/48}\sqrt{\max(C^{2},C)d\log\left(\frac{KT^{1/8}}{\rho}\right)} + TK^{-\frac{1}{6}} + KT^{\frac{7}{8}}\sqrt{\max(C^{2},C)d\log\left(\frac{KT^{1/8}}{\rho}\right)}\right) \\
\end{align*}
 
\end{proof}

\section{Additional Material from Section \ref{sec:action}}\label{app:action-proofs}

\subsection{Proof of Theorem \ref{thm:decision-weak-learning}}
\begin{proof}[Proof of Theorem \ref{thm:decision-weak-learning}]
    Let $a^t = \BR(\hat{y}^t)$. We show the contrapositive. Suppose there exists a collection $\{c_{J,a}\}_{a\in\cA} \subseteq\cC_J$ such that:
    \[
    \sum_{a\in\cA} \sum_{t=1}^T \1[a^t=a] u(c_{J,a}(x^t), y^t) > \sum_{t=1}^T  u(a^t, y^t) + 2Tw\inv\left(\frac{f^S(T)}{T}\right)
    \]
    Equivalently,
    \[
    \frac{1}{T} \sum_{a\in\cA} \sum_{t=1}^T \1[a^t=a] u(c_{J,a}(x^t), y^t) > \frac{1}{T} \sum_{a\in\cA} \sum_{t=1}^T \1[a^t=a] u(a, y^t) + 2w\inv\left(\frac{f^S(T)}{T}\right)
    \]
    Since $\pi^{1:T}$ has $(f^S,\cC_A\cup\cC_B)$-decision swap regret, and $\cC_A$ and $\cC_B$ contain the set of all constant functions (Assumption \ref{assumption:decision-constant}), the decision swap regret with respect to the collection of best constant actions is:
    \begin{align*}
        \frac{1}{T} \sum_{a\in\cA} \max_{a^*\in\cA} \sum_{t=1}^T \1[a^t=a] u(a^*, y^t) - \frac{1}{T} \sum_{a\in\cA} \sum_{t=1}^T \1[a^t=a] u(a, y^t) \leq \frac{f^S(T)}{T} \leq w\inv\left(\frac{f^S(T)}{T} \right)
    \end{align*}
    where the second inequality uses the fact that $w(\gamma)\leq \gamma$, and so $\gamma\leq w\inv(\gamma)$. Then, since the utility of actions $a^t$ is close to the utility of the collection of best constant actions, we have that:
    \begin{align*}
        \frac{1}{T} \sum_{a\in\cA} \sum_{t=1}^T \1[a^t=a] u(c_{J,a}(x^t), y^t) &> \frac{1}{T} \sum_{a\in\cA} \sum_{t=1}^T \1[a^t=a] u(a, y^t) + 2w\inv\left(\frac{f^S(T)}{T}\right) \\
        &\geq \frac{1}{T} \sum_{a\in\cA} \max_{a^*\in\cA} \sum_{t=1}^T \1[a^t=a] u(a^*, y^t) - w\inv\left(\frac{f^S(T)}{T} \right) + 2w\inv\left(\frac{f^S(T)}{T}\right) \\
        &= \frac{1}{T} \sum_{a\in\cA} \max_{a^*\in\cA} \sum_{t=1}^T \1[a^t=a] u(a^*, y^t) + w\inv\left(\frac{f^S(T)}{T}\right)
    \end{align*}
    Let $S_a = \{t: a^t=a \}$ and $$\gamma_a = \frac{1}{|S_a|} \sum_{t=1}^T \1[a^t=a] u(c_{J,a}(x^t), y^t) - \max_{a^*\in\cA} \frac{1}{|S_a|} \sum_{t=1}^T \1[a^t=a] u(a^*, y^t) $$ 
    Then, we can rewrite the expression above as:
    \begin{align*}
        & \frac{1}{T} \sum_{a\in\cA} \sum_{t=1}^T \1[a^t=a] u(c_{J,a}(x^t), y^t) - \frac{1}{T} \sum_{a\in\cA} \max_{a^*\in\cA} \sum_{t=1}^T \1[a^t=a] u(a^*, y^t) \\
        &= \frac{1}{T} \sum_{a\in\cA} |S_a| \cdot \frac{1}{|S_a|} \sum_{t=1}^T \1[a^t=a] u(c_{J,a}(x^t), y^t) - \frac{1}{T} \sum_{a\in\cA} |S_a| \max_{a^*\in\cA}  \frac{1}{|S_a|} \sum_{t=1}^T \1[a^t=a] u(a^*, y^t) \\
        &= \frac{1}{T} \sum_{a\in\cA} |S_a| \gamma_a \\
        &> w\inv\left(\frac{f^S(T)}{T}\right)
    \end{align*}
    Observe that since $\cC_J$ contains the set of all constant functions (Assumption \ref{assumption:decision-constant}), there is always a choice of $\{c_{J,a}\}_{a\in\cA}$ such that $\gamma_a$ is non-negative for all $a$. Thus, we can invoke the weak learning condition: on any subsequence $S_a$ for which $c_{J,a}$ improves over the best constant action by $\gamma_a$, there is some $c_a \in \cC_A\cup\cC_B$ that improves over the best constant action by $w(\gamma_a)$. Specifically, there exists a collection $\{c_a\}_{a\in\cA}\subseteq \cC_A\cup\cC_B$ such that:
    \begin{align*}
        & \frac{1}{T} \sum_{a\in\cA} \sum_{t=1}^T \1[a^t=a] u(c_{a}(x^t), y^t) - \frac{1}{T} \sum_{a\in\cA} \max_{a^*\in\cA} \sum_{t=1}^T \1[a^t=a] u(a^*, y^t) \\
        &= \frac{1}{T} \sum_{a\in\cA} |S_a| \cdot \frac{1}{|S_a|} \sum_{t=1}^T \1[a^t=a] u(c_{a}(x^t), y^t) - \frac{1}{T} \sum_{a\in\cA} |S_a| \max_{a^*\in\cA}  \frac{1}{|S_a|} \sum_{t=1}^T \1[a^t=a] u(a^*, y^t) \\
        &\geq \frac{1}{T} \sum_{a\in\cA} |S_a| w(\gamma_a) \tag{by the $w$-weak learning condition} \\
        &\geq w\left( \frac{1}{T} \sum_{a\in\cA} |S_a| \gamma_a \right) \tag{by convexity of $w$ and Jensen's inequality} \\
        &> w\left( w\inv\left(\frac{f^S(T)}{T}\right) \right) \tag{by monotonicity of $w$} \\
        &= \frac{f^S(T)}{T}
    \end{align*}
    In particular, this implies that:
    \begin{align*}
        \sum_{a\in\cA} \max_{c_a^*\in\cC_A\cup\cC_B} \sum_{t=1}^T \1[a^t=a] u(c_a^*(x^t), y^t) &\geq \sum_{a\in\cA} \sum_{t=1}^T \1[a^t=a] u(c_a(x^t), y^t) \\
        &> \sum_{a\in\cA} \max_{a^*\in\cA} \sum_{t=1}^T \1[a^t=a] u(a^*, y^t) + f^S(T)\\
        &\geq \sum_{a\in\cA} \sum_{t=1}^T \1[a^t=a] u(a, y^t) + f^S(T) \\
        &= \sum_{t=1}^T u(a^t, y^t) + f^S(T)
    \end{align*}
    which violates the $(f^S, \cC_A\cup\cC_B)$-decision swap regret condition. This completes the proof. 
    
\end{proof}

\subsection{Proof of Theorem \ref{thm:decision-online}}

We begin by introducing the key lemmas we will use. In what follows, we denote $a^{t,k}_A = \BR_u(\hat{y}^{t,k}_A)$ and $a^{t,k+1}_B=\BR_u(\hat{y}^{t,k+1}_B)$ for all $k\in[K]$ and $t\in[T]$. The first lemma shows how to convert a decision conversation swap regret guarantee into a decision swap regret guarantee for the sequence of predictions on any round $k$. Observe that decision conversation swap regret stronger than decision swap regret, since it additionally conditions on the action chosen by the other party in the previous round. 
\begin{lemma}\label{lem:dec-sr}
    If Alice has $(f_A^S, \cC_A)$-decision conversation swap regret, then for all odd $k \in [K]$, the transcript $\pi^{1:T,k}$ satisfies $(f'_A, \cC_A)$-decision swap regret, where:
    \[
    f'_A(T) \leq |\cA|f^S_A\left(\frac{T}{|\cA|}\right)
    \]
    A symmetric statement holds for Bob. 
\end{lemma}
\begin{proof}
    We can compute the decision swap regret with respect to $\cC_A$ over round $k$:
    \begin{align*}
        & \sum_{a\in\cA} \max_{c\in\cC_A} \sum_{t=1}^T \1[a^{t,k}_A = a] u(c(x^t_A),y^t) - \sum_{t=1}^T u(a^{t,k}_A,y^t) \\
        &= \sum_{a\in\cA} \max_{c\in\cC_A} \sum_{a'\in\cA} \sum_{t\in T_B(k-1,a')} \1[a^{t,k}_A = a] u(c(x^t_A),y^t) - \sum_{a'\in\cA} \sum_{t\in T_B(k-1,a')} u(a^{t,k}_A,y^t) \\
        &\leq \sum_{a'\in\cA} \left( \sum_{a\in\cA}  \max_{c\in\cC_A}  \left( \sum_{t\in T_B(k-1,a')} \1[a^{t,k}_A = a] u(c(x^t_A),y^t)\right) - \sum_{t\in T_B(k-1,a')} u(a^{t,k}_A,y^t) \right) \tag{by the fact that moving the max inside the sum only strengthens the benchmark}\\
        &\leq \sum_{a'\in\cA} f^S_A(|T_B(k-1,a')|) \tag{by $(f^S_A,\cC_A)$-decision conversation swap regret}\\
        &\leq |\cA|f^S_A\left(\frac{T}{|\cA|}\right) \tag{by concavity of $f^S_A$}
    \end{align*}
\end{proof}



We next argue that if Alice and Bob communicate for sufficiently many rounds, there will exist some round where they $\eps$-agree on a large fraction of days. To do this, we use a result from \citet{collina2025tractable} showing that the utility must increase on any round they disagree. 

\begin{lemma}[Lemma 5.4 of \citet{collina2025tractable}]\label{lem:agreement-action}
    If Bob is $f_B$-decision conversation calibrated, then after engaging in Protocol \ref{alg:collaboration-protocol-decisions} for $T$ days, for all odd rounds $k\in[K]$, we have:
    \[
    \sum_{t=1}^T u(a^{t,k+1}_B, y^t) - \sum_{t=1}^T u(a^{t,k}_A, y^t) \geq \eps |D(T^{k+1})| - 2L|\cA|^2 f_B\left( \frac{T}{|\cA|^2} \right)
    \]
    where $D(T^{k+1})$ is the subset of days over round $k+1$ such that Alice and Bob $\eps$-disagree, i.e.:
    \[
    \left| u(a^{t,k}_A, \hat{y}^{t,k}_A) - u(a^{t,k+1}_B, \hat{y}^{t,k}_A) \right| > \eps
    \]
    or
    \[
    \left| u(a^{t,k}_A, \hat{y}^{t,k+1}_B) - u(a^{t,k+1}_B, \hat{y}^{t,k+1}_B) \right| > \eps
    \]
    Furthermore, if Alice is $f_A$-decision conversation calibrated, then after engaging in Protocol \ref{alg:collaboration-protocol-decisions} for $T$ days, for all even rounds $k\in[K]$, we have:
    \[
    \sum_{t=1}^T u(a^{t,k+1}_A, y^t) - \sum_{t=1}^T u(a^{t,k}_B, y^t) \geq \eps |D(T^{k+1})| - 2L|\cA|^2 f_A\left( \frac{T}{|\cA|^2} \right)
    \]
\end{lemma}

\begin{remark}
    Lemma 5.4 of \citet{collina2025tractable} is stated for a slightly different setting where, every day, the conversation protocol halts after both parties $\eps$-agree (whereas our protocol runs for a fixed number of rounds). There, the decrease in utility is a function of the number of days the protocol advances to the next round. This is equivalent to the number of days Alice and Bob $\eps$-disagree, and so the result translates straightforwardly to our setting.
\end{remark}

\begin{lemma}\label{lem:agreement-round-action}
    After engaging in Protocol \ref{alg:collaboration-protocol-decisions} for $T$ days, each with $K$ rounds, there is at least one round $k$ (without loss, assume $k$ odd) such that the fraction of days Alice and Bob $\eps$-agree, i.e.:
    \[
    \left| u(a^{t,k}_A, \hat{y}^{t,k}_A) - u(a^{t,k+1}_B, \hat{y}^{t,k}_A) \right| \leq \eps
    \]
    and
    \[
    \left| u(a^{t,k}_A, \hat{y}^{t,k+1}_B) - u(a^{t,k+1}_B, \hat{y}^{t,k+1}_B) \right| \leq \eps,
    \]
    is at least $1- \left(\frac{1}{(K-1) \eps} + \frac{\beta(T)}{\eps}\right)$, where $\beta(T) = \frac{L|\cA|^2}{T}\left( f_A\left( \frac{T}{|\cA|^2} \right) + f_B\left( \frac{T}{|\cA|^2} \right) \right)$. 
\end{lemma}
\begin{proof}
    Using Lemma \ref{lem:agreement-action}, we can calculate the difference in utility over two rounds:
    \begin{align*}
        & \sum_{t=1}^T u(a^{t,k+2}_A, y^t) - \sum_{t=1}^T u(a^{t,k}_A, y^t) \\
        &= \sum_{t=1}^T u(a^{t,k+2}_A, y^t) - \sum_{t=1}^T u(a^{t,k+1}_B, y^t) + \sum_{t=1}^T u(a^{t,k+1}_B, y^t) - \sum_{t=1}^T u(a^{t,k}_A, y^t) \\
        &\geq \eps |D(T^{k+2})| - 2L|\cA|^2 f_A\left( \frac{T}{|\cA|^2} \right) + \eps |D(T^{k+1})| - 2L|\cA|^2 f_B\left( \frac{T}{|\cA|^2} \right) \tag{by Lemma \ref{lem:agreement-action}}\\
        &= \eps (|D(T^{k+2})| + |D(T^{k+1})|) - 2T\beta(T) \tag{by definition of $\beta(T)$}
    \end{align*}
    Now, to calculate the difference in utility over $K$ rounds (we assume without loss that $K$ is odd; we obtain the same result if $K$ is even), we iteratively apply the above $(K-1)/2$ times:
    \begin{align*}
        \sum_{t=1}^T u(a^{t,K}_A, y^t) - \sum_{t=1}^T u(a^{t,1}_A, y^t) 
        &\geq \eps \sum_{k=2}^K |D(T^k)| - \frac{K-1}{2}\cdot 2T\beta(T) \\
        &= \eps \sum_{k=2}^K |D(T^k)| - (K-1) T\beta(T)
    \end{align*}
    Observe that since utilities are bounded between $[0,1]$, the left hand side of this expression is at most $T$. Thus, rearranging, we have that the total number of $\eps$-disagreements is at most:
    \begin{align*}
        \sum_{k=2}^K |D(T^k)| &\leq \frac{T+(K-1)T\beta(T)}{\eps}
    \end{align*}
    Therefore, there must exist some round $k^*$ with a number of $\eps$-disagreements at most:
    \[
    |D(T^{k^*})| \leq \frac{T+(K-1)T\beta(T)}{(K-1) \eps} = \frac{T}{(K-1) \eps} + \frac{T\beta(T)}{\eps}
    \]
    That is, on round $k^*$, the fraction of $\eps$-disagreements over $T$ days is at most:
    \[
    \frac{|D(T^{k^*})|}{T} \leq \frac{1}{(K-1) \eps} + \frac{\beta(T)}{\eps}
    \]
    which proves the claim. 
\end{proof}

Finally, we show that on any round where Alice and Bob $\eps$-agree, the utilities under their best response actions do not differ by too much.

\begin{lemma}\label{lem:agreement-regret-action}
    Suppose that on some odd round $k\in[K]$, on at least $1-\delta$ fraction of days $t\in[T]$, we have:
    \[
    \left| u(a^{t,k}_A, \hat{y}^{t,k+1}_B) - u(a^{t,k+1}_B, \hat{y}^{t,k+1}_B) \right| \leq \eps
    \]
    If Bob is $f_B$-decision conversation calibrated, then:
    \[
    \sum_{t=1}^T u(a^{t,k+1}_B, y^t) - \sum_{t=1}^T u(a^{t,k}_A, y^t) \leq (\eps + \delta) T + L|\cA|^2 f_B\left(\frac{T}{|\cA|^2}\right)
    \]
    A symmetric statement holds for even round $k$ and Alice. 
\end{lemma}
\begin{proof}
    We can compute:
    \begin{align*}
        &\sum_{t=1}^T u(a^{t,k+1}_B, y^t) - \sum_{t=1}^T u(a^{t,k}_A, y^t)  \\
        &\leq \sum_{t=1}^T u(a^{t,k+1}_B, \hat{y}^{t,k+1}_B) - \sum_{t=1}^T u(a^{t,k}_A, y^t) \tag{by definition of best response to $\hat{y}^{t,k+1}_B$} \\
        &= \sum_{t=1}^T u(a^{t,k+1}_B, \hat{y}^{t,k+1}_B) - \sum_{a\in\cA}\sum_{a'\in\cA} \sum_{t=1}^T \1[a^{t,k+1}_B=a, a^{t,k}_A=a']  u(a', y^t) \\
        &= \sum_{t=1}^T u(a^{t,k+1}_B, \hat{y}^{t,k+1}_B) - \sum_{a\in\cA}\sum_{a'\in\cA} u\left(a', \sum_{t=1}^T \1[a^{t,k+1}_B=a, a^{t,k}_A=a'] y^t\right)  \tag{by linearity of $u$} \\
        &\leq \sum_{t=1}^T u(a^{t,k+1}_B, \hat{y}^{t,k+1}_B) - \sum_{a\in\cA}\sum_{a'\in\cA} u\left(a', \sum_{t=1}^T \1[a^{t,k+1}_B=a, a^{t,k}_A=a'] \hat{y}^{t,k+1}_B \right) + L|\cA|^2 f_B\left(\frac{T}{|\cA|^2}\right) \\
        &= \sum_{t=1}^T  u(a^{t,k+1}_B, \hat{y}^{t,k+1}_B) - \sum_{t=1}^T u(a^{t,k}_A, \hat{y}^{t,k+1}_B)  + L|\cA|^2 f_B\left(\frac{T}{|\cA|^2}\right) \tag{by linearity of $u$} \\
        &= \sum_{t=1}^T \1\left[\left| u(a^{t,k}_A, \hat{y}^{t,k+1}_B) - u(a^{t,k+1}_B, \hat{y}^{t,k+1}_B) \right| \leq \eps\right] \left( u(a^{t,k+1}_B, \hat{y}^{t,k+1}_B) - u(a^{t,k}_A, \hat{y}^{t,k+1}_B) \right) \\ & \ \ \ \ + \sum_{t=1}^T \1\left[\left| u(a^{t,k}_A, \hat{y}^{t,k+1}_B) - u(a^{t,k+1}_B, \hat{y}^{t,k+1}_B) \right| > \eps\right] \left( u(a^{t,k+1}_B, \hat{y}^{t,k+1}_B) - u(a^{t,k}_A, \hat{y}^{t,k+1}_B) \right) \\ & \ \ \ \ + L|\cA|^2 f_B\left(\frac{T}{|\cA|^2}\right) \\
        &\leq \eps(1-\delta)T + \delta T + L|\cA|^2 f_B\left(\frac{T}{|\cA|^2}\right) \tag{by assumption}\\
        &\leq (\eps+\delta) T + L|\cA|^2 f_B\left(\frac{T}{|\cA|^2}\right) \tag{since $\eps,\delta\geq0$}
    \end{align*}
    Here, the first inequality uses $f_B$-decision conversation calibration and the fact that $u$ is $L$-Lipschitz; we can see that for any $a'\in\cA$:
    \begin{align*}
        & \left| \sum_{a\in\cA}\sum_{a'\in\cA} \left( u\left(a', \sum_{t=1}^T \1[a^{t,k+1}_B=a, a^{t,k}_A=a'] \hat{y}^{t,k+1}_B \right) - u\left(a', \sum_{t=1}^T \1[a^{t,k+1}_B=a, a^{t,k}_A=a'] y^t \right) \right) \right| \\
        &\leq \sum_{a\in\cA}\sum_{a'\in\cA} \left| u\left(a', \sum_{t=1}^T \1[a^{t,k+1}_B=a, a^{t,k}_A=a'] \hat{y}^{t,k+1}_B \right) - u\left(a', \sum_{t=1}^T \1[a^{t,k+1}_B=a, a^{t,k}_A=a'] y^t \right) \right| \\
        &\leq \sum_{a\in\cA}\sum_{a'\in\cA} L \left\| \sum_{t=1}^T \1[a^{t,k+1}_B=a, a^{t,k}_A=a'] (\hat{y}^{t,k+1}_B - y^t) \right\|_\infty \tag{by $L$-Lipschitzness} \\
        &\leq L \sum_{a\in\cA}\sum_{a'\in\cA} f_B(|T(k+1,a,a')|) \tag{by $f_B$-decision conversation calibration} \\
        &\leq L|\cA|^2 f_B\left(\frac{T}{|\cA|^2}\right) \tag{by concavity of $f_B$}
    \end{align*}
    The second inequality follows from the fact that on at least $1-\delta$ fraction of the days, the difference in utility is at most $\eps$. On the remaining days, the difference in utility is at most 1. 
\end{proof}
    
Putting this all together, we can prove Theorem \ref{thm:decision-online}.
\begin{proofof}{Theorem \ref{thm:decision-online}}
    Let $\eps = \left(\frac{1}{K-1}+\beta(T)\right)^{1/2}$.
    By Lemma \ref{lem:agreement-round-action}, there exists a round $k$ such that on $1-\left(\frac{1}{(K-1) \eps} + \frac{\beta(T)}{\eps}\right)$ fraction of the days, Alice and Bob's actions are $\eps$-approximate best responses to each others' predictions. First, consider the case where $k$ is odd, i.e. Alice communicates on round $k$. We have that:
    \[
    \left| u(a^{t,k}_A, \hat{y}^{t,k}_A) - u(a^{t,k+1}_B, \hat{y}^{t,k}_A) \right| \leq \eps
    \]
    and
    \[
    \left| u(a^{t,k}_A, \hat{y}^{t,k+1}_B) - u(a^{t,k+1}_B, \hat{y}^{t,k+1}_B) \right| \leq \eps
    \]

    Since Alice has $(f^S_A,\cC_A)$-decision conversation swap regret, by Lemma \ref{lem:dec-sr}, the transcript at round $k$ satisfies $\left( |\cA|f^S_A\left(\frac{T}{|\cA|}\right), \cC_A\right)$-decision swap regret. Next, we show that the transcript at round $k$ additionally has bounded decision swap regret with respect to $\cC_B$. 
    
    We can calculate the decision swap regret as:
    \begin{align*}
        & \sum_{a\in\cA} \max_{c\in\cC_B} \sum_{t=1}^T \1[a^{t,k}_A=a] u(c(x^t_B), y^t) - \sum_{t=1}^T u(a^{t,k}_A, y^t) \\
        &\leq \sum_{a\in\cA} \max_{c\in\cC_B} \sum_{t=1}^T \1[a^{t,k}_A=a] u(c(x^t_B), y^t) - \sum_{t=1}^T u(a^{t,k+1}_B, y^t) + \left(\eps + \frac{1}{(K-1) \eps} + \frac{\beta(T)}{\eps}\right)T + L|\cA|^2 f_B\left(\frac{T}{|\cA|^2}\right) \tag{by Lemma \ref{lem:agreement-regret-action}} \\
        &= \sum_{a\in\cA} \max_{c\in\cC_B} \sum_{t=1}^T \1[a^{t,k}_A=a] u(c(x^t_B), y^t) - \sum_{t=1}^T u(a^{t,k+1}_B, y^t) + 2T\left(\frac{1}{(K-1)}+\beta(T)\right)^{1/2} + L|\cA|^2 f_B\left(\frac{T}{|\cA|^2}\right) \tag{by our setting of $\eps$} \\
        &= \sum_{a\in\cA} \max_{c\in\cC_B} \sum_{t=1}^T \1[a^{t,k}_A=a] \left( u(c(x^t_B), y^t) -  u(a^{t,k+1}_B, y^t) \right) + 2T\left(\frac{1}{(K-1)}+\beta(T)\right)^{1/2} + L|\cA|^2 f_B\left(\frac{T}{|\cA|^2}\right) \\
        &= \sum_{a\in\cA} \max_{c\in\cC_B} \sum_{a'\in\cA} \sum_{t=1}^T \1[a^{t,k}_A=a, a^{t,k+1}_B=a'] \left( u(c(x^t_B), y^t) -  u(a^{t,k+1}_B, y^t) \right) \\ & \ \ \ \ + 2T\left(\frac{1}{(K-1)}+\beta(T)\right)^{1/2} + L|\cA|^2 f_B\left(\frac{T}{|\cA|^2}\right) \\
        &\leq \sum_{a\in\cA} \sum_{a'\in\cA} \max_{c\in\cC_B}  \sum_{t=1}^T \1[a^{t,k}_A=a, a^{t,k+1}_B=a'] \left( u(c(x^t_B), y^t) -  u(a^{t,k+1}_B, y^t) \right) \\ & \ \ \ \ + 2T\left(\frac{1}{(K-1)}+\beta(T)\right)^{1/2} + L|\cA|^2 f_B\left(\frac{T}{|\cA|^2}\right) \\
        &\leq \sum_{a\in\cA} f^S_B(|T_A(k, a)|) + 2T\left(\frac{1}{(K-1)}+\beta(T)\right)^{1/2} + L|\cA|^2 f_B\left(\frac{T}{|\cA|^2}\right) \tag{by $(f^S_B,\cC_B)$-decision conversation swap regret} \\
        &\leq |\cA| f^S_B\left(\frac{T}{|\cA|}\right) + 2T\left(\frac{1}{(K-1)}+\beta(T)\right)^{1/2} + L|\cA|^2 f_B\left(\frac{T}{|\cA|^2}\right) \tag{by concavity of $f^S_B$}
    \end{align*}
    Here, the second inequality holds, since moving the max inside the sum can only make the quantity larger.

    For brevity, let: $$\lambda_{A}^{odd} \coloneqq |\cA|f^S_A\left(\frac{T}{|\cA|}\right) \ \ \mathrm{and} \ \  \lambda_B^{odd} \coloneqq |\cA| f^S_B\left(\frac{T}{|\cA|}\right) + 2T\left(\frac{1}{(K-1)}+\beta(T)\right)^{1/2} + L|\cA|^2 f_B\left(\frac{T}{|\cA|^2}\right).$$ Hence, the transcript at round $k$ simultaneously has $(\lambda_A^{odd}, \cC_A)$-decision swap regret and $(\lambda_B^{odd}, \cC_B)$-decision swap regret. Therefore, it has $(\max\{\lambda_A^{odd}, \lambda_B^{odd}\}, \cC_A\cup\cC_B)$-decision swap regret. 

    Now, consider the case where $k$ is even. Since all statements hold symmetrically, we have that, for:
    \[
    \lambda_{A}^{even} \coloneqq |\cA| f^S_A\left(\frac{T}{|\cA|}\right) + 2T\left(\frac{1}{(K-1)}+\beta(T)\right)^{1/2} + L|\cA|^2 f_A\left(\frac{T}{|\cA|^2}\right) \ \ \mathrm{and} \ \  \lambda_B^{even} \coloneqq |\cA|f^S_B\left(\frac{T}{|\cA|}\right),
    \]
    the transcript at round $k$ simultaneously has $(\lambda_A^{even}, \cC_A)$-decision swap regret and $(\lambda_B^{even}, \cC_B)$-decision swap regret, and therefore $(\max\{\lambda_A^{even}, \lambda_B^{even}\}, \cC_A\cup\cC_B)$-decision swap regret. 

    Since $\lambda^{even}\geq \lambda^{odd}$ and $\lambda^{odd}\geq \lambda^{even}$, we can conclude that there exists a round $k$ such that the transcript at round $k$ has  $(\max\{\lambda_A^{even}, \lambda_B^{odd}\}, \cC_A\cup\cC_B)$-decision swap regret.
\end{proofof}

\subsection{Proof of Theorem \ref{thm:decision-final-regret}}

\begin{proof}[Proof of Theorem \ref{thm:decision-final-regret}]
    By Theorem \ref{thm:decision-online}, there exists a round $k^*$ of the protocol such that the transcript $\pi^{1:T,k^*}$ has $(\max\{\lambda_A, \lambda_B\},\cC_A\cup\cC_B)$-decision swap regret. Then, since $\cC_A$ and $\cC_B$ satisfy the $w(\cdot)$-weak learning condition, Theorem \ref{thm:decision-weak-learning} gives us that $\pi^{1:T,k^*}$ has $\left(2Tw\inv\left(\frac{\max\{\lambda_A, \lambda_B\}}{T}\right), \cC_J\right)$-decision swap regret. This proves the first part of the theorem.

    To prove the second part, we use Lemma \ref{lem:agreement-action}, which bounds the decrease in utility from every round $k$ to $k+1$. We have that over two rounds, the change in utility is: 
    \begin{align*}
        & \sum_{t=1}^T u(a^{t,k}_A, y^t) - \sum_{t=1}^T u(a^{t,k+2}_A, y^t) \\
        &= \sum_{t=1}^T u(a^{t,k}_A, y^t) - \sum_{t=1}^T u(a^{t,k+1}_B, y^t) + \sum_{t=1}^T u(a^{t,k+1}_B, y^t) - \sum_{t=1}^T u(a^{t,k+2}_A, y^t) \\
        &\leq  2L|\cA|^2 f_B\left( \frac{T}{|\cA|^2} \right) - \eps |D(T^{k+1})| +  2L|\cA|^2 f_A\left( \frac{T}{|\cA|^2} \right) - \eps |D(T^{k+2})| \tag{by Lemma \ref{lem:agreement-action}}\\
        &\leq  2L|\cA|^2 f_B\left( \frac{T}{|\cA|^2} \right) +  2L|\cA|^2 f_A\left( \frac{T}{|\cA|^2} \right) \\
        &= 2T\beta(T) \tag{by definition of $\beta(T)$}
    \end{align*}
    Thus, we can bound the decrease in utility by applying this expression iteratively from round $k^*$ to the last round $K$. There are at most $K-1$ rounds between $k^*$ and $K$, and so applying this expression $(K-1)/2$ times bounds the decrease in utility, i.e.: 
    \begin{align*}
        \sum_{t=1}^T u(a^{t,k^*}, y^t) - \sum_{t=1}^T u(a^{t,K}, y^t) \leq \frac{K-1}{2}\cdot 2T\beta(T) = (K-1)T\beta(T)
    \end{align*}
    Therefore, we can bound the external regret of the last round:
    \begin{align*}
        &\max_{c_J\in\cC_J} \sum_{t=1}^T u(c_J(x^t), y^t) - \sum_{t=1}^T u(a^{t,K}, y^t) \\
        &\leq \max_{c_J\in\cC_J} \sum_{t=1}^T u(c_J(x^t), y^t) - \sum_{t=1}^T u(a^{t,k^*}, y^t) + (K-1)T\beta(T) \\
        &\leq 2Tw\inv\left(\frac{\max\{\lambda_A, \lambda_B\}}{T}\right) + (K-1)T\beta(T)
    \end{align*}
    Here, the last line follows from the fact that external regret is upper bounded by decision swap regret, and we have previously bound the decision swap regret of the transcript at round $k^*$. This completes the proof. 
\end{proof}

\subsection{Proof of Theorem \ref{thm:decision-alg}}

\begin{proof}[Proof of Theorem \ref{thm:decision-alg}]
    Let $M$ be the algorithm of Theorem \ref{thm:decision-sr-alg}. Let $\rho'=\frac{2\rho}{K|\cA|}$. By Theorem \ref{thm:decision-sr-alg}, with probability $1-\rho'$, $M$ produces predictions that are $f$-decision calibrated for:
    \[
    f(\tau) \leq O\left( \ln(d|\cA||\cC|T) + \sqrt{T\ln\left(\frac{d|\cA||\cC|T}{\rho'}\right)} \right)
    \]
    Moreover, plugging the guarantees of $M$ into Theorem \ref{thm:decision-sr-bound}, we have that with probability $1-\rho'$, $M$ obtains $(f^S,\cC)$-decision swap regret for:
    \[
    f^S(\tau) \leq O\left( L|\cA|^2\ln(d|\cA||\cC|T) + L|\cA| \sqrt{T \ln\left(\frac{d|\cA||\cC|T}{\rho'}\right)} \right)
    \]
    By construction, on every odd round $k$, a separate copy $M_{k,a}$ is run for every subsequence on which the action from the previous round $a^{t,k-1}$ is $a$. By a union bound, the probability that any one of the copies fails is at most $\frac{K}{2}|\cA|\rho'=\rho$. Therefore, since decision conversation calibration asks for decision calibration on every such subsequence, with probability $1-\rho$, Algorithm \ref{alg:algorithm-decisions} is also $(f,\cC)$-decision conversation calibrated. Likewise, since decision conversation swap regret measures the decision swap regret on every such subsequence, with probability $1-\rho$, Algorithm \ref{alg:algorithm-decisions} also achieves $(f^S,\cC)$-decision conversation swap regret.
\end{proof}

\subsection{Proof of Theorem \ref{thm:decision-final-bounds}}

\begin{proof}[Proof of Theorem \ref{thm:decision-final-bounds}]
    Let $\rho' = \rho/2$. By Theorem \ref{thm:decision-alg}, Algorithm \ref{alg:algorithm-decisions} achieves, with probability $1-\rho'$, $(f^S, \cC_A)$-decision conversation swap regret and $f_A$-decision conversation calibration for:
    \[
    f^S_A(\tau) \leq O\left( L|\cA|^2\ln(d|\cA||\cC_A|T) + L|\cA| \sqrt{T \ln\left(\frac{dK|\cA||\cC_A|T}{\rho}\right)} \right)
    \]
    and
    \[
    f_A(\tau) \leq O\left( \ln(d|\cA||\cC_A|T) + \sqrt{T \ln\left(\frac{dK|\cA||\cC_A|T}{\rho}\right)} \right)
    \]
    for any $\tau\in[T]$. Likewise, Algorithm \ref{alg:algorithm-decisions} achieves, with probability $1-\rho'$, $(f^S_B, \cC_B)$-decision conversation swap regret and $f_B$-decision conversation calibration for:
    \[
    f^S_B(\tau) \leq O\left( L|\cA|^2\ln(d|\cA||\cC_B|T) + L|\cA| \sqrt{T \ln\left(\frac{dK|\cA||\cC_B|T}{\rho}\right)} \right)
    \]
    and
    \[
    f_B(\tau) \leq O\left( \ln(d|\cA||\cC_B|T) + \sqrt{T \ln\left(\frac{dK|\cA||\cC_B|T}{\rho}\right)} \right)
    \]
    Thus, by a union bound, if Alice and Bob both use Algorithm \ref{alg:algorithm-decisions} to interact, then with probability $1-2\rho'=1-\rho$, Alice has $(f^S, \cC_A)$-decision conversation swap regret and $f_A$-decision conversation calibration, and Bob has $(f^S_B, \cC_B)$-decision conversation swap regret and $f_B$-decision conversation calibration. 

    Then, by Theorem \ref{thm:decision-final-regret}, the transcript $\pi^{1:T,K}$ on the last round satisfies:
    \begin{align*}
        \max_{c_J\in\cC_J} \sum_{t=1}^T u(c_J(x^t), y^t) - \sum_{t=1}^T u(a^{t,K}, y^t) \leq 2Tw\inv\left(\frac{\max\{\lambda_A, \lambda_B\}}{T}\right) + (K-1)T\beta(T)
    \end{align*}
    where:
    \begin{align*}
        \beta(T) &= \frac{L|\cA|^2}{T}\left( f_A\left( \frac{T}{|\cA|^2} \right) + f_B\left( \frac{T}{|\cA|^2} \right) \right) \\ 
        &\leq O\left( \frac{L|\cA|^2\ln(d|\cA||\cC_A|T)}{T} + L|\cA|^2 \sqrt{\frac{ \ln\left(\frac{dK|\cA||\cC_A|T}{\rho}\right)}{T}} + \frac{L|\cA|^2\ln(d|\cA||\cC_B|T)}{T} + L|\cA|^2 \sqrt{\frac{ \ln\left(\frac{dK|\cA||\cC_B|T}{\rho}\right)}{T}} \right) \\
        &\leq O\left( \frac{L|\cA|^2\ln(d|\cA||\cC_A||\cC_B|T)}{T} + L|\cA|^2 \sqrt{\frac{ \ln\left(\frac{dK|\cA||\cC_A||\cC_B|T}{\rho}\right)}{T}} \right) \tag{by Cauchy-Schwartz}
    \end{align*}
    and thus:
    \begin{align*}
        \lambda_{A} &\leq |\cA| f^S_A\left(\frac{T}{|\cA|}\right) + L|\cA|^2 f_A\left(\frac{T}{|\cA|^2}\right) + 2T\left(\frac{1}{(K-1)}+\beta(T)\right)^{1/2} \\
        &\leq O\Bigg( L|\cA|^3\ln(d|\cA||\cC_A|T) + L|\cA|^2\sqrt{T\ln\left( \frac{dK|\cA||\cC_A|T}{\rho} \right)} \\ & \ \ \ \ +  \frac{T}{\sqrt{K-1}} + |\cA|\sqrt{TL\ln(d|\cA||\cC_A||\cC_B|T)} +|\cA|\sqrt{L} \ln^{1/4} \left( \frac{dK|\cA||\cC_A||\cC_B|T}{\rho} \right) T^{3/4} \Bigg) \tag{by concavity of the square root function} \\
        &\leq O\left( L|\cA|^3 \ln(d|\cA||\cC_A||\cC_B|T) + L|\cA|^2 \sqrt{\ln\left( \frac{dK|\cA||\cC_A||\cC_B|T}{\rho} \right)} T^{3/4} + \frac{T}{\sqrt{K-1}} \right)
    \end{align*}
    Since the expression for $\lambda_B$ is symmetric, we have that:
    \[
    \lambda_B \leq O\left( L|\cA|^3 \ln(d|\cA||\cC_A||\cC_B|T) + L|\cA|^2 \sqrt{\ln\left( \frac{dK|\cA||\cC_A||\cC_B|T}{\rho} \right)} T^{3/4} + \frac{T}{\sqrt{K-1}} \right)
    \]
    Hence, plugging this in, we can compute:
    \begin{align*}
        &\max_{c_J\in\cC_J} \sum_{t=1}^T u(c_J(x^t), y^t) - \sum_{t=1}^T u(a^{t,K}, y^t) \\
        &\leq 2Tw\inv\left(\frac{\max\{\lambda_A, \lambda_B\}}{T}\right) + (K-1)T\beta(T) \\
        &\leq 2Tw\inv\left( O\left( \frac{L|\cA|^3 \ln(d|\cA||\cC_A||\cC_B|T)}{T} + \frac{L|\cA|^2 \sqrt{\ln\left( \frac{dK|\cA||\cC_A||\cC_B|T}{\rho} \right)}}{T^{1/4}} + \frac{1}{\sqrt{K-1}} \right) \right) \\ & \ \ \ \ + O\left( (K-1)L|\cA|^2\ln(d|\cA||\cC_A||\cC_B|T) + (K-1)L|\cA|^2 \sqrt{T\ln\left( \frac{dK|\cA||\cC_A||\cC_B|T}{\rho} \right)}   \right) \\
        &\leq 2Tw\inv\left( O\left( \frac{L|\cA|^3 \ln\left( \frac{dK|\cA||\cC_A||\cC_B|T}{\rho} \right)}{T^{1/4}} + \frac{1}{\sqrt{K-1}} \right) \right) \\ & \ \ \ \ + O\left( (K-1)L|\cA|^2 \ln\left( \frac{dK|\cA||\cC_A||\cC_B|T}{\rho} \right) \sqrt{T}  \right)
    \end{align*}
    which proves the first part of the theorem.

    To prove the second part, suppose $K=\omega(1)$ and $K=o(\sqrt{T})$. Then, we can compute:
    \begin{align*}
        &\max_{c_J\in\cC_J} \sum_{t=1}^T u(c_J(x^t), y^t) - \sum_{t=1}^T u(a^{t,K}, y^t) \\
        &\leq 2Tw\inv\left( O\left( \frac{L|\cA|^3 \ln\left( \frac{d|\cA||\cC_A||\cC_B|T}{\rho} \right)}{T^{1/4}} + o(1) \right) \right) + O\left( L|\cA|^2 \ln\left( \frac{d|\cA||\cC_A||\cC_B|T}{\rho} \right) T^{\alpha}  \right)
    \end{align*}
    for some constant $\alpha\in(0,1)$. Now, observe that any function $O\left( \frac{L|\cA|^3 \ln\left( \frac{d|\cA||\cC_A||\cC_B|T}{\rho} \right)}{T^{1/4}} + o(1) \right) \to 0$ as $T\to\infty$. Hence, by Lemma \ref{lem:limit}, $w\inv\left( O\left( \frac{L|\cA|^3 \ln\left( \frac{d|\cA||\cC_A||\cC_B|T}{\rho} \right)}{T^{1/4}} + o(1) \right) \right) \to 0$ as $T\to\infty$ and thus,
    \[
    T w\inv\left( O\left( \frac{L|\cA|^3 \ln\left( \frac{d|\cA||\cC_A||\cC_B|T}{\rho} \right)}{T^{1/4}} + o(1) \right) \right) = o(T)
    \]
    Notice that since $w$ is strictly increasing, $w\inv$ exists for sufficiently large $T$ (larger than a constant). Therefore, for sufficiently large $T$, the regret is bounded by:
    \begin{align*}
        \max_{c_J\in\cC_J} \sum_{t=1}^T u(c_J(x^t), y^t) - \sum_{t=1}^T u(a^{t,K}, y^t) \leq o(T) + O\left( L|\cA|^2 \ln\left( \frac{d|\cA||\cC_A||\cC_B|T}{\rho} \right) T^{\alpha}  \right)
    \end{align*}
    which completes the proof. 
\end{proof}

\section{Details from Section \ref{sec:batch}}
\subsection{COLLABORATE and \lsboost\ Halt}
\label{app:batch-convergence}

\begin{theorem}
In training, the sub-process \lsboost\ converges after $K=m^2$ (sub)rounds, and the COLLABORATE Algorithm \ref{protocol:batch} converges after $R=m^2$ rounds on the training sample $S$.
\end{theorem}

\begin{proof}
To begin, assume that \lsboost\ always terminated after at most $K$ rounds. At round $r$, let $\err^r$ refer to the empirical squared error of the predictions $P^r$ generated at round $r$:
\[
\err^r = \frac{1}{n}\sum_{i \in [n]}(P^{r,i}-y^i)^2.
\]

Consider what happens at round $r$ of Algorithm \ref{protocol:batch} when \crossboost \ is called. The \crossboost\ algorithm has two kinds of updates that can occur on the level sets of the other players' predictions: either the current player can choose to update their predictor to the output of \lsboost\ or they can set their predictions on that level set to be equivalent to Bob's. Note that if, at any round, they choose on all their level sets to use Bob's predictions, Algorithm \ref{protocol:batch} will halt, because $P^{r+1}=P^r$. 

Say that instead they choose to use the output of \lsboost\ on at least one level set $v^*$. Then, on this level set their predictions will be equal to $\tilde{f}^{r+1,v^*}(\playerData)$. Note that the player's level sets on the other players' predictions are disjoint, and that squared error is always non-negative. So,

\begin{align*}
\err^{r+1} - \err^r &= \err^{r+1}-\sum_{v \in [1/m]} \vert \playerSampleLS{r+1}{v} \vert \cdot \err^v \\
&= \frac{1}{n} \sum_{v \in [1/m]}\vert \playerSampleLS{r+1}{v} \vert \left(\sum_{\playerData^i \in \playerSampleLS{r+1}{v}}\left(P^{r+1,i}-y^i\right)^2\right)  - \sum_{v \in [1/m]} \vert \playerSampleLS{r+1}{v}\vert \cdot \err^v \\
&= \frac{1}{n} \sum_{v \in [1/m]}\vert \playerSampleLS{r+1}{v}\vert \left(\sum_{\playerData^i \in \playerSampleLS{r+1}{v} }\left(P^{r+1,i}-y^i\right)^2 - \err^v\right)\\
&\ge \frac{\vert S^{v^*}\vert}{n} \left(\sum_{\playerData^i \in \playerSampleLS{r+1}{v^*}} \left(P^{r+1,i}-y^i\right)^2 - \err^{v^*}\right) \\
&= \frac{\vert S^{v^*}\vert}{n} \left(\sum_{\playerData^i \in \playerSampleLS{r+1}{v^*}} \left(\tilde{f}^{r+1,v^*}(\playerData^i-y^i\right)^2 - \err^{v^*}\right) \\
&= \widetilde{\err}^{r+1,v^*}-\err^{v^*}\\
&\ge 1/m^2
\end{align*}

Thus, at every round $r$ in which they do not halt, they must improve the squared error of their predictions by at least $\alpha$. In the worst case, $\err^0=1$, i.e. Bob's initial predictions are maximally incorrect.  Squared error can never decrease below zero, so they must halt after at most $R=m^2$ rounds.

It remains to show that \lsboost\ also terminates. This follows a similar potential argument on the squared error as above. Modulo notational changes in our halting condition, the complete proof is equivalent to that of the halting condition proved as part of Theorem 4.3 in \cite{globus-harris2023multicalibration}.
\end{proof}

\subsection{Proof of In-Sample Accuracy Guarantee}
\label{app:batch-boosting}

We begin by proving a series of swap regret guarantees, first with respect to the individual runs of \lsboost, then with respect to runs of \crossboost, and finally with respect to the final model output by COLLABORATE.

\begin{lemma}
\label{lem:lsboost-regret}
Let $f^{r,v,K}$ be the model output by a run of \lsboost\ on a player's sample $S^\bullet$, and let $\cH_\bullet$ be that player's own hypothesis class. Then, every time \lsboost\ is run by a player, the final model has $(2/m^2,\cH_\bullet)$-swap regret on the sample $S^\bullet$ it was run on:
\[
2/m^2 \ge \sampleExp \left[\left(\roundedModel - y \right)^2 \right] - \min_{h \in \cH_\bullet} \E_{(\playerData, y) \sim S_\bullet} \left[ \1[f^{r,v,K}(\playerData)=v](h(x)-y)^2 \right].
\]
\end{lemma}

\begin{proof}
Say that \lsboost\ is run on a sample $S_\bullet$ and outputs the model from round $K$. Recall that in \lsboost\, if the output model is the model from round $K$, then in fact the algorithm ran for $K+1$ rounds, and the stopping condition at the final round $K+1$ is in terms of the error of the \textit{unrounded} predictors $\tilde{f}^{r,v,K}$ and $\tilde{f}^{r,v,K+1}$ which were generated at that round and the previous one. So, since the algorithm halted, 
\begin{align*}
    1/m^2 &> \err_K - \err_{K+1} \\
    &= \E_{(\playerData, y) \sim S_\bullet} \left[ \left(\tilde{f}_\bullet^{r,v,K}(\playerData)-y \right)^2 \right] - \E_{(\playerData, y) \sim S_\bullet} \left[\left(\tilde{f}_\bullet^{r,v,K+1}(\playerData)-y \right)^2 \right]\\
    &= \E_{(\playerData, y) \sim S_\bullet} \left[ \left(\tilde{f}_\bullet^{r,v,K}(\playerData)-y \right)^2 \right] - 
    \E_{(\playerData, y) \sim S_\bullet} \left[\left( \sum_{v' \in [1/m]} \1\left[f^{r,v,K}(x)=v'\right] \cdot h^{r,v,K+1,v'}(x) -y \right)^2 \right]\\
    &\ge \E_{(\playerData, y) \sim S_\bullet} \left[ \left(\tilde{f}_\bullet^{r,v,K}(\playerData)-y \right)^2 \right] - 
    \sum_{v'\in[1/m]} \E_{(\playerData, y) \sim S_\bullet} \left[ \1[f^{r,v,K}(x)=v']\left(h^{r,v,K+1,v'}(x)-y \right)^2 \right]\\
    & \megamegaquad \text{(by Cauchy-Schwartz)}\\
    &= \E_{(\playerData, y) \sim S_\bullet} \left[ \left(\tilde{f}_\bullet^{r,v,K}(\playerData)-y \right)^2 \right] - \min_{h \in \cH_\bullet} \left(\E_{(\playerData, y) \sim S_\bullet} \left[ \1[f^{r,v,K}(\playerData)=v](h(x)-y)^2 \right]\right). \quad\quad (\text{Eq. } 1)\\
    & \megamegaquad \text{(by the definition of $h^{r,v,K+1,v'} \in \playerOracle$)}
\end{align*}

This expression is nearly the swap regret statement we want, except we need to bound the swap regret of our \textit{rounded} predictor $f^{r,v,K}$, rather than our unrounded $\tilde{f}^{r,v,K}$. However, note that pointwise, from the definition of $\round$, $\vert f^{r,v,K}(x) - \tilde{f}^{r,v,K}(x) \vert \le 1/(2m^2).$ Hence, 

\begin{align*}
\sampleExp \left[ \left(f_\bullet^{r,v,K} (\playerData)-y \right)^2 \right] &= 
    \sampleExp \left[ \left(\round\left(\tilde{f}_\bullet^{r,v,K}(\playerData);m^2 \right)-y \right)^2 \right] \\
&= \sampleExp \left[ \left(\roundStuff\right)^2  \right] \\ 
&\miniquad - 2\sampleExp \left[\roundStuff \cdot y\right] + \sampleExp \left[y^2 \right] \\
&\le \sampleExp \left[\left(\unroundedModel + \frac{1}{2m^2} \right)^2 -2\left(\unroundedModel-\frac{1}{2m^2}\right)y + y^2 \right]\\
& \le \sampleExp \left[\left(\unroundedModel - y \right)^2 \right] +\frac{3}{4m^2}
\end{align*}

Combining this with the bound in Equation 1 gives us that 

\begin{align*}
1/m^2 &> \E_{(\playerData, y) \sim S_\bullet} \bigg[ \big(\tilde{f}_\bullet^{r,v,K}(\playerData) -y \big)^2 \bigg] - \min_{h \in \cH_\bullet} \left(\E_{(\playerData, y) \sim S_\bullet} \left[\1[f^{r,v,K}(\playerData)=v](h(x)-y)^2 \right]\right) \\
& \ge \sampleExp \left[\left(\roundedModel - y \right)^2 \right] - \min_{h \in \cH_\bullet} \left( \E_{(\playerData, y) \sim S_\bullet} \left[\1[f^{r,v,K}(\playerData)=v](h(x)-y)^2 \right]\right) - \frac{3}{4m^2}
\end{align*}

\noindent And hence $\roundedModel$ has at most $2/m^2 > 1/m^2 + 3/4m^2$ swap-regret on $\playerSample$ with respect to $\cH_\bullet$.
\end{proof}

\begin{lemma}
\label{lem:crossboost-regret}
Let $\playerModel{r}$ be the model generated by \crossboost\ at round $r$ on the player's sample $\playerSample$. Then $\playerModel{r}$ will have $(3/m, \cH_\bullet)$-swap regret on $\playerSample$.
\end{lemma}

\begin{proof}
Recall that in $\crossboost$, the player will bucket their sample into level sets based on the other players' predictions, which we call $\playerSampleLS{r}{v}$. Their final model $\playerModel{r}$ will be an ensemble of models $\playerModelLS{r}{v}$ generated on each level set $v$. On some of these level sets, their model will equal to a model $\playerModelLStilde{r}{v}$ which is output by $\lsboost$. On these level sets, we can directly invoke the swap regret guarantee from Lemma \ref{lem:lsboost-regret}. However, if the \lsboost\ process did not sufficiently improve their squared error on $\playerSampleLS{r}{v}$, they will instead set $\playerModelLS{r}{v}$ to always predict the other players' constant prediction $v$. We will first show that for any level set $v$ where this happens, there is low swap regret with respect to the sample $\playerSampleLS{r}{v}$. 

As in the statement of Algorithm \ref{protocol:crossboost}, let 
\begin{align*}
\err^v &= \sampleExpLS \left[(v-y)^2\right], \text{and}\\
\widetilde{\err}^{r,v} &= \sampleExpLS \left[(\playerModelLStilde{r}{v}(\playerData)-y)^2\right].
\end{align*}
Since the player chose to use Bob's predictor $v$, we know that $\err^v-\widetilde{err}^v \le 1/m^2.$ But this means 
\begin{align*}
1/m^2 &> \sampleExpLS \left[(v-y)^2\right] - \sampleExpLS \left[(\playerModelLStilde{r}{v}(\playerData)-y)^2\right] \\
&= \left( \sampleExpLS \left[(v-y)^2\right] - \sum_{v' \in [1/m]} \min_{h \in \cH_\bullet} \left( \sampleExpLS \left[ \1[\playerModelLStilde{r}{v}(\playerData)=v'](h(x)-y)^2\right] \right) \right) \\
& \miniquad - \left( \sampleExpLS \left[(\playerModelLStilde{r}{v}(\playerData)-y)^2\right] - \sum_{v' \in [1/m]} \min_{h \in \cH_\bullet}\left(\sampleExpLS \left[  \1[\playerModelLStilde{r}{v}(\playerData)=v'](h(x)-y)^2\right] \right) \right) \\ 
&>  \sampleExpLS \left[(v-y)^2\right] -  \sum_{v' \in [1/m]} \min_{h \in \cH_\bullet} \left(\sampleExpLS \left[ \1[\playerModelLStilde{r}{v}(\playerData)=v'](h(x)-y)^2\right]\right) - 2/m^2\\
&\megamegaquad \text{(By Lemma \ref{lem:lsboost-regret})}
\end{align*}
Recall that low swap regret always implies low external regret. And for constant predictors, swap regret and external regret are equivalent statements. So this inequality in turn implies that on the subsample $\playerSampleLS{r}{v}$ where they used Bob's constant prediction $v$ instead of $\playerModelLStilde{r}{v}$,
\begin{align*}
3/m^2 &> \sampleExpLS \left[(v-y)^2\right] - \min_{h \in \cH_\bullet} \left( \sampleExpLS \left[ (h(x)-y)^2\right]\right),\\
&= \sampleExpLS \left[(\playerModelLStilde{r}{v}(\playerData)-y)^2\right] - \sum_{v' \in [1/m]} \min_{h \in \cH_\bullet} \left(\sampleExpLS \left[\1[\playerModelLStilde{r}{v}(\playerData) = v'](h(x)-y)^2\right]\right).\\ 
\end{align*}

In other words, the player will have at most $(3/m^2, \cH_\bullet)$-swap regret with respect to the subsample $\playerSampleLS{r}{v}$ on any subsample where they chose to follow the other players' prediction, which will be a constant predictor on this subsample. We will now combine these marginal guarantees which are with respect to the subsamples $\playerSampleLS{r}{v}$ into a swap regret guarantee on the entire sample $\playerSample$. 

On any level set $\playerSampleLS{r}{v}$ where $\playerModel{r}$ evaluates to $\playerModelLStilde{r}{v}$, they will have $(2/m^2, \cH_\bullet)$-swap regret with respect to $\playerSampleLS{r}{v}$, and on any level set $\playerSampleLS{r}{v}$ where $\playerModel{r} = v$, they will have at most $(3/m^2, \cH_\bullet)$-swap regret with respect to $\playerSampleLS{r}{v}$. So in the worst case they will have swapped out to the other players' predictions on each level set, and 

\begin{align*}
\sampleExp \big[(\playerModel{r}(\playerData)&-y)^2\big] - \sum_{v \in [1/m]} \min_{h \in \cH} \left(\sampleExp \left[ \1[\playerModel{r}=v](h(x)-y)^2\right] \right)\\
&\le \E_{(\playerData,y)\sim\playerSample} \Bigg[\sum_{v \in [1/m]} \Pr\left(\playerData \in \playerSampleLS{r}{v}\right) \bigg( \sampleExpLS \big[(\playerModel{r}(\playerData)-y)^2\big] \\
& \megaquad - \min_{h \in \cH}\left( \sampleExpLS \left[ \1[\playerModel{r}(\playerData)=v](h(x)-y)^2\right]\right)\bigg)\Bigg],\\
&\le m \left(\frac{3}{m^2}\right),\\
&= 3/m.
\end{align*}
\end{proof}

\begin{corollary}
\label{cor:batch-regret}
Let $f^R_A$ and $f^R_B$ be the models output by Alice and Bob after running Algorithm \ref{protocol:batch}, which halted after $r$ rounds. Then the models will have $(3/m, \cH_A \cup \cH_B)$-swap regret with respect to the shared sample $S$. 
\end{corollary}

\begin{proof}
    Note that at the final round, Alice and Bob's predictions will agree, because otherwise Algorithm \ref{protocol:batch} will not have terminated. We know from Lemma \ref{lem:crossboost-regret} that their models on their respective samples $\aliceSample$ and $\bobSample$ will have $(3/m, \cH_A)$ and $(3/m, \cH_B)$-swap regret respectively. So, since they also agree at this round, it must be the case that they have swap regret bounded by $3/m$ with respect to $\cH_A \cup \cH_B$.
\end{proof}

This gives us all of the technical machinery needed for the proof of Theorem \ref{thm:boosting-batch}, as stated in Section \ref{sec:batch-alg-analysis}.

\subsection{Details of Generalization Guarantee in Batch Setting}
\label{app:batch-generalization}

We will now state our generalization guarantee. As the models generated by the COLLABORATE algorithm only include $m$ possible values for the final predictor, we will leverage a multiclass uniform convergence theorem which relies on the  pseudodimension of $\cH_J$. We will then in turn use a bound on the Natarajan dimension of $\cH_J$ to bound its pseudodimension, applying a lemma that states that if a model may be written as a decision rule over binary classifiers, then its Natarajan dimension is bounded above by its pseudodimension. Writing our models as such decision rules will require a small technical assumption that $\cH_J$ is ``closed" with respect to $\cH_A$ and $\cH_B$, i.e. that $\cH_J$ contains a function equivalent to any function in $\cH_A$ or $\cH_B$ but defined over its input space $\cX_A \times \cX_B$.

\begin{definition}[Closure of $\cH_J$ with respect to $\cH_A$ and $\cH_B$] We will say that $\cH_J$ is closed with respect to $\cH_A$ and $\cH_B$ if for any $h_A \in \cH_A$ there exists some $h \in \cH_J$ such that $h(x)=h((x_A,x_B))=h_A(x_A)$ and for any $h_B \in \cH_B$ there exists some $h \in \cH_J$ such that $h(x)=h((x_A, x_B))=h_B(x_B).$
\end{definition}

We now state the generalization theorem. 

\begin{theorem}
\label{thm:batch-generalization-formal}
Let $\varepsilon, \delta > 0$ and let $\cF$ be the class of models output from Algorithm \ref{protocol:batch} for any joint input distribution $\cD$. Let $d$ be the pseudodimension of Alice and Bob's joint hypothesis class $\cH_J$, and assume that $\cH_J$ is closed with respect to $\cH_A$ and $\cH_B$. Let $S = \{(\aliceData, \bobData,y_i)\}_{i\in[n]}\sim \cD^n$ be a sample of $n$ iid points drawn from $\cD$. Then, if  
    \[
    n \ge O\left(\frac{m^7 d \log(m d) + \log(1/\delta)}{\varepsilon^2} \right),
    \]
    \[
    \Pr\left( \min_{f \in \cF} \left\vert \E_{(\aliceData, \bobData, y)\sim \cD} \left[(y-f(x))^2\right]-\E_{(\aliceData, \bobData, y)\sim S} \left[(y-f(x)^2\right] \right\vert \ge \epsilon \right) \le \delta.
    \]
\end{theorem}

\subsubsection{Definitions and Referenced Theorem Statements for Proof of Generalization}
In order to prove this statement, we will need to rely on a variety of different definitions and standard results from the machine learning theory literature. 

\begin{definition}[VC-dimension]\cite{vcdim}
Let $\cH$ be a class of binary classifiers $h: \cX \rightarrow \{0,1\}$. Let $S = \{x_1, \ldots, x_n\}$ and let $\Pi_\cH(S) = \{ \left( h(x_1), \ldots, h(x_n) \right): h \in \cH\} \subseteq \{0,1\}^m$. We say that $S$ is shattered by $\cH$ if $\Pi_\cH(S) = \{0,1\}^n$. The Vapnik-Chervonenkis (VC) dimension of $\cH$, denoted $\text{VCdim}(\cH)$, is the cardinality of the largest set $S$ shattered by $\cH$.
\end{definition}

\begin{definition}{Pseudodimension}[\cite{pollard2012convergence}]
Let $\cH$ be a class of functions from $\cX$ to $\bbR$. We say that a set $S = (x_1, \ldots, x_m, y_1, \ldots, y_m) \in \cX^m \times \bbR^m$ is pseudo-shattered by $\cH$ if for any $(b_1, \ldots, b_m) \in \{0,1\}^m$ there exists $h \in \cH$ such that $\forall i, h(x_i) > y \Longleftrightarrow b_i = 1$ The pseudodimension of $\cH,$ denoted $\pdim(\cH)$ is the largest integer $m$ for which $\cH$ pseudo-shatters some set $S$ of cardinality $m$.
\end{definition}

\begin{definition}[Shattering for multiclass functions]\cite{natarajan1989learning, shalev2014understanding}
A set $C \subseteq \cX$ is shattered by $\cH$ if there exists two functions $f_0, f_1:C \rightarrow [k]$ such that 
\begin{enumerate}
\item For every $x \in C, f_0(x) \ne f_1(x)$.
\item For every $B \subseteq C$ there exists a function $h \in \cH$ such that 
\[
\forall x \in B, h(x) = f_0(x) \text{ and } \forall x \in C \ B, h(x)=f_1(x).
\]
\end{enumerate}

\end{definition}

\begin{definition}[Natarajan dimension]\cite{natarajan1989learning, shalev2014understanding}
The Natarajan dimension of $\cH$, denoted $\ndim(\cH),$ is the maximal size of a shattered set $C \subseteq \cX$.
\end{definition}

\begin{theorem}[Multiclass uniform convergence]\cite{shalev2014understanding}
\label{thm:uc}
Let $\epsilon, \delta > 0$ and let $\cH$ be a class of functions $h: \cX \rightarrow [1/k]$ such that the Natarajan dimension of $\cH$ is $d$. Let $\cD \in \Delta (\cX \times [0,1]) $ be an arbitrary distribution and let $D = \{(x_1, y_1), \ldots, (x_n, y_n)\}_{(x_i, y_i) \sim \cD}$ be a sample of $n$ points from $\cD$.  Then for 
\[ n \ge O\left( \frac{d \log(k) + \log(1/\delta)}{\eps^2}\right),\]
\[
\Pr\left[\max_{h \in \cH}\left\vert \E_{(x,y)\sim \cD} [(y-h(x))^2] - \E_{(x,y) \sim D} [(y-h(x))^2] \right\vert \ge \epsilon]\right ] \le \delta. 
\]
\end{theorem}

\begin{lemma}\cite{shalev2014understanding}
\label{lem:ndim-bound}
Suppose we have $\ell$ binary classifiers from binary class $\cH_{\textup{bin}}$ and a rule $r: \{0,1\}^\ell \rightarrow [k]$ that determines a multiclass label according to the predictions of the $\ell$ binary classifiers. Define the hypothesis class corresponding to this rule as 
\[
\cH = \{ r(h_1(\cdot), \ldots, h_\ell(\cdot)) : (h_1, \ldots, h_\ell) \in (\cH_\text{bin})^\ell\}.
\]
Then, if $d = \textup{VCdim}(\cH_\textup{bin}),$
\[
\ndim(\cH) \le 3 \ell d \log (\ell d).
\]
\end{lemma}

\subsubsection{Generalization Proof}

First, we show that the models generated by the COLLABORATE algorithm may be written as decision rules over a polynomial number of binary predictors. 

\begin{lemma}
\label{lem:decision-rule}
    Let $K$ be an upper bound on the number of rounds that \lsboost\ ever runs for, and let $r$ be a round of COLLABORATE. Then we can write the player's model $\playerModel{r}$ at round $r$ as a decision rule $\playerDecisionRule{r}:\{0,1\}^\ell \rightarrow [1/m]$ over $\ell \le m + rKm^3$ binary predictors. Assuming that $\cH_J$ is closed with respect to $\cH_A$ and $\cH_B$, each of these binary predictor $g: \cX \rightarrow \{0,1\}$ will be a mapping from the full feature space $\cX = (\cX_A, \cX_B)$ induced by a function $h \in \cH_J$.
\end{lemma}

\begin{proof}

We proceed by induction, first showing that $f^r$ may be written as a decision rule over classifiers and then arguing that the number of total classifiers is bounded by $m + rKm^3$.

\paragraph{Base Case}
Consider the following $m$ binary classifiers, $g^{0,v}: \cX \rightarrow \{0,1\}$ defined for each $v \in [1/m]$ and $x = (\aliceData, \bobData) \in \cX$:
\begin{align*}
g^{0,v}(x) &= \begin{cases}
    1 & \text{if } \bobModel{0}(\bobData) = v,\\
    0 & \text{else}
\end{cases} \\
&= \begin{cases}
    1 & \text{if } \round(h^0_B;m)(\bobData) = v,\\
    0 & \text{else}
\end{cases} 
\end{align*}

We can then write the following decision rule
\[
\rho^0(\{g^{0,v}\}_{v \in [1/m]})(x) = \arg\max_{v \in [1/m]} v \cdot \1[g^{0,v}(x)=1] = \bobModel{0}(\bobData),
\]

\noindent which exactly reconstructs the starting model.

\paragraph{Induction step} Say that at round $r$ Bob has played, and his model $f^r_B$ may be written as a decision rule $\rho^r$. We will now show that Alice's model $f^{r+1}_A$ may be written as a decision rule recursively defined in terms of $\rho^r$. First, we will will fix $v$, and consider what happens internally to $\lsboost$:

\begin{quote}
\textbf{Base case} Consider the initial round of $\lsboost$, when $k=0.$ For each $v' \in [1/m]$, let $g^{r+1,v,0,v'}: \cX \rightarrow \{0,1\}$ be a classifier
\begin{align*}
g^{r+1,v,0,v'}(x) &= \begin{cases}
1 & \text{if } f_A^{r+1,v,0}(x_A)=v', \\
0 & \text{else, }
\end{cases} \\
 &= \begin{cases}
1 & \text{if } \round(h_A^{r+1,v,0};m)(x_A)=v', \\
0 & \text{else. }
\end{cases}
\end{align*}

As in our base case for the analysis for \crossboost, we can rewrite $f_A^{r+1,v,0}$ as a decision rule $\rho^{r,v,0}$ in terms of $g^{r,v,0,v'}$:
\[
\rho^{r,v,0}\left(\{g^{r+1,v,0,v'}\}_{v'\in [1/m]}\right)(x) = \arg\max_{v' \in [1/m]} v' \cdot \1[g^{r+1,v,0,v'}(x)=1]=\aliceModelBoost{r+1}{0}(x_A).
\]
\textbf{Induction step for \lsboost} Say that the claim holds at round $k$ of $\lsboost$, i.e.\ that there is a decision rule $\rho^{r+1,v,k}$ such that $\rho^{r+1,v,k}= \aliceModelBoost{r+1}{k}.$ Let $v', i \in [1/m]$ and define the following $m^2$ binary classifiers 
\[
\aliceG(x) = \begin{cases}
1 & \text{if } \round(\aliceOracleModel{k+1}(x_A);m)=i, \\
0 & \text{else}.
\end{cases}
\]

Then, we can write
\begin{align*}
\rho^{r+1,v,k+1}(\rho^{r+1,v,k}, \{\aliceG&\}_{(v',i)\in [1/m]})(x) \\
&= \sum_{(v',i) \in [1/m]} i \cdot \1[\rho^{r+1,v,k}(x)=v']\1[\aliceG(x)=1], \\
&= \sum_{v' \in [1/m]} \1[\aliceModelBoost{r+1}{k}(\aliceData)=v'] \cdot \sum_{i \in [1/m]} i
 \cdot \1[\round(\aliceOracleModel{k+1}(x_A);m)=i] \\
 &= \sum_{v' \in [1/m]} \1[\aliceModelBoost{r+1}{k}(\aliceData)=v'] \cdot \round(\aliceOracleModel{k+1}(x_A);m),\\
 &= \aliceModelBoost{r+1}{k+1}(\aliceData),
\end{align*}
\noindent which concludes the induction internal to \lsboost.

Following the induction argument in \cite{globus-harris2023multicalibration}, $\rho^{r+1,v,k+1}$ is a decision rule over a total of $m + (k+1)m^2$ classifiers. 
\end{quote}

\noindent Now, we wish to show that $\aliceModel{r+1}$ may be written as a decision rule $\rho^{r+1}$. Recall that in $\crossboost$, on each level set of Bob's prediction, the updated model $f^{r,v}$ will \textit{either} be equivalent to Bob's predictions or a model output by \lsboost\ will be evaluated on the point. Let $V_1 \subseteq [1/m]$ be the collection of level sets at round $r$ where Alice's updated model was equivalent to $\tilde{f}^{r+1,v}$ and let $V_2$ be the collection of level sets where her model used Bob's predictions. I.e.,

\begin{align*}
V_1 &= \{v \in [1/m]: f_A^{r+1,v}(\aliceData)=\tilde{f}_A^{r+1,v}(\aliceData)\}\\
V_2 &= \{v \in [1/m]: f_A^{r+1,v}(\aliceData)=v\}.
\end{align*}

Note that $[1/m] = V_1 \cup V_2$ and the two sets are disjoint. For $v \in V_1$, let $K_v$ be the total number of rounds that \lsboost\ ran for, and define

\[
\rho^{r,v} = \begin{cases}
    \rho^{r,v,K_v} &\text{if }v \in V_1 \\
    \rho^r &\text{if }v \in V_2.
\end{cases}
\]

Then we can write 

\begin{align*}
    \rho^{r+1}\left(\rho^r, \left\{\rho^{r,v}\right\}_{v\in[1/m]}\right)(x) &= \sum_{v \in [1/m]} \1[\rho^r(x)=v] \cdot \rho^{r,v}(x)\\
    &= \sum_{v \in V_1} \1[\rho^r(x)=v]\rho^{r,v,K_v}(x) + \sum_{v \in V_2} \1[\rho^r=v]\rho^r \\
    &= \sum_{v \in V_1} \1[x \in S^{r+1,v}]\cdot\tilde{f}(x_A) + \sum_{v \in V_2} \1[x \in S^{r+1,v}]\cdot v \\
    &= f^{r+1}_A.
\end{align*}

In other words, Alice's model at round $r+1$ may be written as a decision rule recursively defined in terms of her decision rules from \lsboost\ on the level sets where these models are used and on Bob's decision rule $\rho^r$. 

We now need to give an upper bound for the number of binary predictors which $\rho^r$ is comprised of. Let $K$ be the maximum number of rounds that \lsboost\ ever runs for. Note $\rho^0$ is made up of $m$ classifiers, and say that $\rho^r$ is made up of at most $m+r(m+Km^2)m$ classifiers. Note that for any $v \in V_2$ no new classifiers will be invoked. So in the worst case, $V_1 = [1/m]$, i.e. for each of Alice's level sets on Bob's predictor, \lsboost\ is invoked. Each of these runs will add at most $m+Km^2$ classifiers to the decision rule, so in total there will be at most $m(m+Km^2)$ new classifers added to the decision rule. Hence, $\rho^{r+1}$ will be comprised of at most 

\begin{align*}
\ell &= m + r(m+Km^2)+(m+Km^2)m \\
&= m + (r+1)m(m+Km^2) \\
&\le m + (r+1)(K+1)m^3
\end{align*}

\noindent classifiers.
\end{proof}

\begin{lemma}[The VC dimension of $\inducedLS$ is bounded by the pseudodimension of $\cH_J$]
\label{lem:VCtoPseudo}
Let $\inducedLS$ be the class of Boolean classifiers induced by $\round(h(x);m)$ for $h \in \cH_J$. I.e., for any $g \in \inducedLS$ there must be some $v \in [1/m]$ such that 
\[
g(x) = \begin{cases}
1 & \text{if } \round(h(x);m)=v,\\
0 & \text{else.}
\end{cases}
\]

\noindent Let $d'$ be the VC dimension of $\inducedLS$, and let $d$ be the pseudodimension of $\cH_J$. Then $d'<d$.

\end{lemma}

\begin{proof}
Let $d'$ be the VC dimension of $\inducedLS$, and let $d$ be the pseudodimension of $\cH_J$. First, consider the richer hypothesis class of the set of linear thresholds induced by $\round(h(x);m)$. We will call this class $\inducedThreshold$: i.e., for any $g \in \inducedThreshold$ there must be some $v \in [1/m]$ such that 
\[
g(x) = \begin{cases}
1 & \text{if } \round(h(x);m) \ge v,\\
0 & \text{else.}
\end{cases}
\]
Note that any function in $\cG_{\cH_J, m}^\le$ can be written as an (infinite) disjunction over functions in $\cG_{\cH_J, m}$. Hence, the VC dimension of $\cG_{\cH_J, m}^\le$, which we will call $d''$, must be greater than $d'$.

We will now show that the pseudodimension of $\cH_J$, $d$, bounds $d''$. Say for contradiction that it doesn't, and that $d < d''$. Since $d''>d$, it must be the case that any $d+1$ points in $\cX$ are shattered by some $g \in \inducedThreshold$. Say that the labels induced by $g$ on these $d+1$ points are $(b_1, \ldots, b_{d+1})$. By construction of $\inducedThreshold$, there must be some $v \in [1/m]$ such that $\round(h(x_i); m) \rightarrow b_i = 1$. From the definition of $\round$, this means there is some $i$ such that $h(x_i) > i \Leftrightarrow b_i = 1$. But this is the definition of pseudo-shattering, and hence $\cH_J$ must pseudo-shatter the $d+1$ points. Hence by contradiction $d'' < d$ and

\[
d' < d'' < d.
\]
\end{proof}

\begin{lemma}[Bound on Natarajan dimension of $\cF^R$]
\label{lem:ndim}
Let $\cF^R$ be the class of models that are output by Algorithm \ref{protocol:batch} after $r$ rounds, and let $d$ be the pseudodimension of $\cH_J$. Then, 
\[
\ndim(\cF^R) \le 3(m + m^7)d \log\left((m + m^7)d\right)
\]
\end{lemma}

\begin{proof}
Let $r$ be the number of outer rounds that COLLABORATE runs for and let $K$ be an upper bound on any internal run of \lsboost. We combine the results of Lemmas \ref{lem:decision-rule} and \ref{lem:VCtoPseudo}:  In Lemma \ref{lem:decision-rule}, we showed that $f^R$ may be written as a collection of decision rules over no more than $\ell=m + RKm^3$ predictors in $\inducedLS$. Let $d'=\VCdim(\inducedLS)$. Plugging this in to Lemma \ref{lem:ndim-bound} and using the bound from Lemma \ref{lem:VCtoPseudo},
\begin{align*}
\ndim(\cF^R) &\le 3(m + RKm^3)d' \log((2m + RKm^3)d')\\
&\le 3(m + RKm^3)d \log((m + RKm^3)d) &&\text{(By Lemma \ref{lem:VCtoPseudo})}
\end{align*}

We know from Theorem \ref{thm:batch-convergence} that Algorithm \ref{protocol:batch} will converge after no more than $R\le m^2$ rounds and the internal runs of \lsboost\ will run for no more than $m^2$ rounds. Plugging these in as bounds on $K$ and $R$, we get 
\[
\ndim(\cF^R) \le 3(m + m^7)d \log\left((m + m^7)d\right).
\]
\end{proof}

We now have all the components to prove our generalization theorem. 

\begin{proofof}{Theorem \ref{thm:batch-generalization-formal}} This follows directly from Theorem \ref{thm:uc} and Lemma \ref{lem:ndim} and suppressing the smaller terms.
\end{proofof}

\section{Details from Section \ref{sec:bayes}} 
\label{app:bayes}

\lemma{Let $\cH$ be a class of real-valued functions $h : \mathcal{X} \to \mathbb{R}$.
Let $y : \mathcal{X} \to \mathcal{Y}$ be a fixed labeling function, and fix a label $v \in \mathcal{Y}$. Let $\cH_{A}^*$ be defined such that for each $h \in \cH_{A}$, the corresponding function $h^* \in \cH_{A}^*$ is given by:
\[
h^*(x) = h(x) \cdot \mathbf{1}[y(x) = v].
\]

Then, 
\[
C_{\cH_{A}^{*}}^{\epsilon} \leq C_{\cH_{A}}^{\epsilon}
\]
In other words, for any scale $\epsilon$, the fat-shattering dimension of $\cH_{A}^{*}$ is at most the fat-shattering dimension of $\cH_{A}$. \label{lem:shattering_bound} 
}

\begin{proof}
Let $S = \{x_1, \dots, x_n\} \subseteq \mathcal{X}$ be a set of size $n$ that is $\epsilon$-shattered by $\cH_{A}^*$. That is, there exists a witness vector $\vec{r} = (r_1, \dots, r_n) \in \mathbb{R}^n$ such that for every binary vector $\vec{b} \in \{0,1\}^n$, there exists a function $h^* \in \cH_{A}^*$ satisfying:
\[
\forall i \in [n], \quad
\begin{cases}
h^*(x_i) > r_i + \epsilon & \text{if } b_i = 1, \\
h^*(x_i) < r_i - \epsilon & \text{if } b_i = 0.
\end{cases}
\]

But for any $h^* \in \cH_{A}^*$, we have $h^*(x) = h(x) \cdot \mathbf{1}[y(x) = v]$ for some $h \in \cH_{A}$. Therefore, $h^*(x_i) = 0$ whenever $y(x_i) \neq v$. In particular, if $x_i$ has $y(x_i) \neq v$, then the above inequalities cannot hold for any $r_i$ with nonzero margin $\epsilon$.

Hence, only points $x_i$ with $y(x_i) = v$ can be involved in the $\epsilon$-shattering. Let $S_v = \{x_i \in S : y(x_i) = v\}$. Then the shattering must occur over $S_v$, and the effective shattering occurs only over this subset.

Note that by construction, for each $h^* \in \cH_{A}^*$, there is an $h \in \cH_{A}$ such that $h^*(x_i) = h(x_i)$ for all $x_i \in S_v$. So the class $\mathcal{H}$, restricted to $S_v$, can realize the same shattering. Therefore:
\[
C_{\cH_{A}^{*}}^{\epsilon} \leq C_{\cH_{A}|S_{v}}^{\epsilon} \leq C_{\cH_{A}}^{\epsilon}
\]    
\end{proof}

\subsection{Proof of Lemma \ref{lem:bayes-expected-regret}}
\begin{proof}[Proof of Lemma \ref{lem:bayes-expected-regret}]
    Consider a modified interaction under Protocol \ref{alg:bayes-collaboration} where, at each day in round $j$ (if the conversation reaches round $j$), the outcome is resampled according to the information seen by Alice so far: $y' \sim \cD_y| x_A^t, \pi^{t-1}, C^t_{j-1}, p_B^{t, j}$. Let $\hat{\pi}^j$ be the transcript from this interaction. 
    
    First, we will show that $\Pr_{\cD}[\pi] = \Pr_{\cD}[\hat{\pi}^{j}]$, where $\pi$ is the transcript under the unmodified Protocol \ref{alg:bayes-collaboration}.
    
    Let $\hat{\pi}^{1:t,j}$ denote the transcript of this interaction up to day $t$. Note that this is distinct from $\bar{\pi}^{t,j}$, which denotes the transcript of an interaction only on day $t$ where the resampling only occurs in round $j$.
    We will proceed via induction over days. 
    \begin{itemize}[leftmargin=3ex]
        \item \textbf{Base Case}:  $\Pr_{\cD}[\pi^{1:1}] = \Pr_{\cD} [\hat{\pi}^{1:1,j}]$. 
        
        \textit{Proof}: 
        On day $t=1$, we have $\Pr_{\cD}[\pi^{1}] = \Pr_{\cD}[\bar{\pi}^{1,j}]$, by Lemma~\ref{lem:resampling-transcript}. Note that $\bar{\pi}^{1,j} = \bar{\pi}^{1:1,j} = \hat{\pi}^{1:1,j}$, and therefore $\Pr_{\cD}[\pi^{1:1}] = \Pr[\hat{\pi}^{1:1,j}]$.
        \item \textbf{Inductive Step}:  If $\Pr_{\cD}[\pi^{1:t}] = \Pr_{\cD}[\hat{\pi}^{1:t,j}]$, then $\Pr_{\cD}[\pi^{1:t+1}] = \Pr_{\cD}[\hat{\pi}^{1:t+1,j}]$.
        
        \textit{Proof}: 
        Observe that the state of the model algorithm in any round $t+1$ is a function only of the algorithm $M$ and the transcript until that round: $\pi^{1:t}$ or $\bar{\pi}^{1:t}$.
        By the Inductive Hypothesis, $\Pr_{\cD}[\pi^{1:t}] = \Pr_{\cD}[\hat{\pi}^{1:t,j}]$ -- and consequently, since the model algorithm $M$ is the fixed between both interactions, therefore,  $\Pr_{\cD}[\pi^{t+1,j}] = \Pr_{\cD}[\bar{\pi}^{t+1,j}]$.
        By Lemma~\ref{lem:resampling-transcript}, this is equal to $\Pr_{\cD}[\pi^{t+1}]$. As $\Pr_{\cD}[\hat{\pi}^{1:t,j}] = \Pr_{\cD}[\pi^{1:t}]$ and $\Pr_{\cD}[\bar{\pi}^{t+1,j}] = \Pr_{\cD}[\pi^{t+1}]$, we have that $\Pr_{\cD}[\pi^{1:t+1}] = \Pr_{\cD}[\hat{\pi}^{1:t+1,j}]$.
    \end{itemize}
    
    Now, all that remains to show is that Alice's sequence of predictions in $\hat{\pi}(j)$ has low expected regret with respect to $h$.
    Recall that Alice is a Bayesian Learner (Definition \ref{def:bayesian-learner}), which means that her prediction in round $k$ is deterministic after round $k-1$, and is the posterior mean of the distribution conditioned on the transcript up to day $t-1$, their features on day $t$, and the conversation of day $t$ through round $k-1$. 
    Since squared error is a proper scoring rule, it follows that predicting the mean of the sampling distribution, as Alice does, has lower squared error than predicting any other post-processing of the information available to her, and in particular, the function $h \in \cH_A$, which is defined only on Alice's features $x_A$, a subset of the information she has conditioned on. 
    Therefore, it follows that a perfect Bayesian will have 0 regret with respect to the swap function over her $[\frac{1}{m}]$ level sets defined by the $m$ fixed functions in $\cH_A$, notated as $\{h_0, h_{\frac{1}{m}}, \ldots, h_1\}$.
    \begin{align*}
\mathbb{E}_{\cD} [ (\hat{y} - y)^2] \leq
        \mathbb{E}_{\cD} \left[\sum_{v} \mathbb{I}[\hat{y} = v](h_v(x) - y)^{2} \right].
    \end{align*}
    However, since Alice and Bob are not \emph{perfect} Bayesians in Protocol \ref{alg:bayes-collaboration}, but instead round their prediction to the nearest multiple of $\frac{1}{m}$, their expected regret with respect to $h$ will depend on this discretization. 
    \begin{align*}
        \E_{\cD}[(\bar{y} - y)^2] &= \E_{\cD} [(\hat{y} - y + \bar{y} - \hat{y}))^2] \\
        &= \E_{\cD}[(\hat{y} - y)^2 + (\bar{y} - \hat{y})^2 + 2 (\hat{y} - y) (\bar{y} - y)].
    \end{align*}
    Since $|\hat{y} - y| < \frac{1}{m}$ and $\E_{\cD}[ 2 (\hat{y} - y) (\bar{y} - y) ] = 0$, 
    we have
    \begin{align*}
       \E_{\cD}[(\bar{y} - y)^2]  \leq \E_{\cD}[(\hat{y} - y)^2] + \frac{1}{m^2}. 
    \end{align*}
    We have shown the claim for an arbitrary set of $m$ functions in $\cH_A: \{h_0, h_{\frac{1}{m}}, \ldots, h_1\}$, and can thus conclude that it holds for any swap function with respect to $\cH_A$.
\end{proof}

\begin{theorem}[Azuma's Inequality] \label{thm:azuma}
    Let $\{ X_0, X_1, \ldots \}$ be a martingale sequence such that $|X_{i+1} - X_i| < c$ for all $i$, then,
    \begin{align*}
        \Pr[X_n - X_0 \geq \epsilon] \leq \exp{ \left(- \frac{\epsilon^2}{2c^2n} \right)}.
    \end{align*}
\end{theorem}

An immediate corollary of Theorem \ref{thm:azuma} follows from appropriately setting parameters.
\begin{corollary}
Letting $X_0 = 0, \eps = c\sqrt{2 n \ln \frac{1}{\delta}}$, then we have for any $\delta \in (0, 1)$, with probability $1 - \delta$, 
\begin{align*}
    X_n \leq  c \sqrt{2 n \ln \frac{1}{\delta} } .
\end{align*}
\end{corollary}

\begin{lemma} \label{lem:azuma-app}
    Let $E: \Pi \to [0, 1]$ represent any  conditioning event. 
    Consider the random process $\{\cZ^t\}$ adapted to the sequence of random variables $\pi^t$ for $t \geq 1$ and let
    \begin{align*}
        \cZ^t \coloneqq Z^{t-1} + E(\pi^{1:t-1}) \cdot \left( y^t(\pi^{1:t-1}) - \E_{y \sim \cD} [y | \pi^{1:t-1}] \right) 
    \end{align*}
    Then,
    \begin{align*}
        \sum_{t=1}^T  E(\pi^{1:t-1}) \cdot \left( y^t(\pi^{1:t-1}) - \E_{y \sim \cD} [y | \pi^{1:t-1}] \right) \leq 2 \sqrt{ 2T \ln \frac{1}{\delta} }, 
    \end{align*}
    with probability $1 - \delta$ over the randomness of $\cD$ and $\pi^{1:t-1}$. 
\end{lemma}
\begin{proof}
    First, observe that the above sequence is a martingale as
    $\E_{\cD} [  E(\pi^{1:t-1}) \cdot (y^t(\pi^{1:t-1}) - \E_{y \sim \cD} [ y | \pi^{1:t-1} ]  ] = 
    E(\pi^{1:t-1}) \cdot \E_{\cD} [ (y^t(\pi^{1:t-1}) - \E_{y \sim \cD} [ y | \pi^{1:t-1} ]  ] = 0$, since $E(\pi^{1:t-1})$ is a constant at the start day $t$ as it does not depend on the outcome $y^t$. Thus, $\E_{\cD} [ Z^{t+1} ] = Z^t$.
    Next, observe that since the outcomes $y \in [-1, 1]$, we have the bounded difference condition: $|Z^t - Z^{t-1}| < 2$ for all $t$.
    We can then instantiate Azuma's Inequality with $n = T$ and $c = 2$ to get the claim. 
\end{proof}

\subsection{Proof of Lemma \ref{lem:bayes-loss-concentration}}
\begin{proof}[Proof of Lemma \ref{lem:bayes-loss-concentration}]
    Fix bucket $i \in \{1, \ldots, \frac{1}{g_B(T)}\}$ of Bob's prediction in round $k-1$.
    Since Alice's prediction is deterministic of round $k$ is deterministic after round $k-1$, we can instantiate Lemma \ref{thm:azuma} with the event $E(\pi^{1:T}) = \mathbb{I}[y_B^{t, k-1} \in i]$ and have, that with probability $1 - \delta$, 
    \begin{align*}
        \left| \sum_{t = 1}^T E(\pi^{1:t-1}) ( \hat{y}^t_k - y^t(\pi^{1:t-1}))^2 - \mathbb{E}_{y \sim \cD} [ (y-y^t)^2 | \pi^{1:t-1}]\right| \leq 2 \sqrt{2T \ln{\frac{1}{\delta}}}. 
    \end{align*}
\end{proof}



\begin{definition}
    Let $\cH$ be a set of functions mapping from a domain $\cX$ to $\mathbb{R}$ and suppose that $S = \{x_1, \ldots, x_m\} \subseteq \cX$. Fix $\gamma > 0$. Then $S$ is $\gamma-$shattered by $\cH$ if there are real numbers $r_1, \ldots, r_m$, such that for each $b \in \{0, 1\}^m$ there is a function $h$ in $\cH$ satisfying, for all $i \in [m]$,

    \begin{align*}
        h(x_i) \geq r_i &+ \gamma \text{ if } b_i = 1  \\
        &\text{ and} \\
        h(x_i) \leq r_i &- \gamma \text{ if } b_i = 0.
    \end{align*}
    We say that $r = (r_1, \ldots, r_m)$ witnesses the shattering.
\end{definition}
\begin{definition}[Fat Shattering Dimension \citep{anthony1999neural}] 
    Suppose that $\cH$ is a set of functions from a domain $\cX$ to $\mathbb{R}$ and that $\gamma > 0$. Then $\cH$ has $\gamma-$dimension $d$ if $d$ is the maximum cardinality of subset $S$ of $\cX$ that is $\gamma-$shattered by $\cH$. If no such maximum exists, we say that $\cH$ has infinite $\gamma-$dimension. The $\gamma-$dimension of $\cH$ is denoted $\textsc{fat}_{\cH}(\gamma)$. This defines a function $\textsc{fat}_{\cH}: \mathbb{R}^+ \to \mathbb{N} \cup \{0, \infty\},$ which we call the fat shattering dimension of $\cH$. We say that $\cH$ has finite fat shattering dimension whenever it is the case that for all $\gamma > 0$, $\textsc{fat}_{\cH}(\gamma)$ is finite.
\end{definition}


\begin{theorem}[\cite{anthony1999neural}] \label{thm:fat-concentration-two}
Let $\cH$ be a hypothesis space of real-valued functions with finite fat-shattering dimension, then
  \begin{align*}
        \sup_{h \in \cH} \left| \frac{1}{T} \sum_{i=1}^T (h(x^i) - y^i)^2 - \E_{\cD} [(h(x^i) - y^i)^2] \right| \leq \eps.
    \end{align*}
for
\begin{align*}
    M(\eps, \delta) = O( \frac{ C_{\cH}^{\eps/256}  \ln(\frac{1}{\eps}) + \ln(\frac{1}{\delta}) } {\eps^2 } ),
\end{align*}
where $M(\eps, \delta)$ is the number of samples needed to reach $\eps$ uniform convergence with probability $1 - \delta$. 
\end{theorem}

\subsection{Proof of Lemma \ref{lem:bayes-benchmark-concentration}}
\begin{proof}[Proof of Lemma \ref{lem:bayes-benchmark-concentration}]
    Note that in a Bayesian setting, each $(x^t, y^t)$ are sampled i.i.d. from $\cD$ every day, which means that for a fixed round $k$, Bayesian predictions (and consequently the choice of the benchmark function and thus value $h(x^t)$) are independent across days.
    Secondly, note that by Lemma \ref{lem:shattering_bound}, we have that 
    for any scale $\eps > 0$, $C_{\cH_{A}^{*}}^{\epsilon} \leq C_{\cH_{A}}^{\epsilon}$ where 
    $\cH_A^*$ is the function class defined as 
    $\cH_A^* = \{h(x) \cdot \mathbf{1}[y(x) = v]: \ \forall \ v \in \cY, h \in \cH_A\}$.
    
    Thus, we can directly apply Theorem \ref{thm:fat-concentration-two}. 
     \begin{align*}
        \sup_{h \in \cH_A} \left| \frac{1}{T} \sum_{i=1}^T \ell(h(x^i), y^i) - \E_{\cD} [\ell(h(x^i), y^i)] \right| \leq \eps.
    \end{align*}
    This means, that on the subsequence $T_B(k-1, i)$, for some level set $v$ of Alice's prediction, we have:
     \begin{align*}
        \sup_{h \in \cH_A} | \frac{1}{|T_B(k-1, i)|} \sum_{t \in T_B(k-1, i)} \mathbb{I}[\bar{y}^{t, k}_A = v] (h(x^t) - y)^2 &- \E_{\cD} [\mathbb{I}[\bar{y}_A = v] (h(x) - y)^2 | \pi^{1:t-1}] |  \leq \eps.
    \end{align*}
\end{proof}

\subsection{Proof of Theorem \ref{thm:bayesians-swap}}
\begin{proof}[Proof of Theorem \ref{thm:bayesians-swap}]
    With probability $1-\delta$, we have that
    \begin{align*}
    \sum_{t \in T_B(k-1, i)} &(\bar{y}^{t,k} - y^t)^2 - \sum_v{\min_{h \in \cH_B}} \sum_{t \in T_B(k-1, i)} \mathbb{I}[\bar{y}^{t,k} = v] (h(x^t) - y^t)^2 \\
    & \leq \sum_{t \in T_B(k-1, i)} \E[(\bar{y}^{t,k}- y^t)^2 | \pi^{1:t-1}] + 2\sqrt{2T\ln{\frac{1}{\delta}}} - \sum_v{\min_{h \in \cH_B}} \sum_{t \in T_B(k-1, i)} \mathbb{I}[\bar{y}^{t,k} = v] (h(x^t) - y^t)^2 \\
    & \leq \sum_{t \in T_B(k-1, i)} \E[(\bar{y}^{t,k}- y^t)^2 | \pi^{1:t-1}] + 2\sqrt{2T\ln{\frac{1}{\delta}}} \\ & \quad - \sum_v{\min_{h \in \cH_B}} \sum_{t \in T_B(k-1, i)} \E_{\cD} [\mathbb{I}[\bar{y}^{t,k} = v] (h(x^t) - y^t)^2 | \pi^{1:t-1}] + mT\eps \\
    &\leq 2\sqrt{2T\ln{\frac{1}{\delta}}} + \frac{T}{m^2} + mT\eps,
    \end{align*}  
    where the first inequality comes from Lemma \ref{lem:bayes-loss-concentration}, the second from applying Lemma \ref{lem:bayes-benchmark-concentration} to each level set $v$ of Alice's prediction, and the third from Lemma \ref{lem:bayes-expected-regret}.
    The final statement comes from taking a union bound over all buckets $g_B(T)$ and rounds $K$. 
\end{proof}

\subsection{Proof of Theorem \ref{thm:one-shot}}

\begin{lemma} \label{lem:one-shot}
    Let $\cH_J$ be a hypothesis class over the joint feature space $\cX$.
    Let $\cH_A = \{h_A:\cX_A\to\cY\}$ and $\cH_B = \{h_B:\cX_B\to\cY\}$ be hypothesis classes over $\cX_A$ and $\cX_B$.  
 Consider instance $(x_A, x_B, y) \sim \cD$. If
    \begin{itemize}
        \item Alice and Bob are both Bayesian learners, with discretization $m =T^{\alpha_g}$, for $\alpha \in [0, \frac{1}{4}]$
        \item $\cH_A$ and $\cH_B$ jointly satisfy the $w(\cdot)$-weak learning condition with respect to $\cH_J$ for any continuous $w(\cdot)$ such that $w(\gamma) > 0$ for $\gamma > 0$,
    \end{itemize} then under Protocol \ref{alg:one-shot} the prediction in round $K$ will have low expected error with respect to $\cH_J$ on day 1, with probability $1 - \delta$:
     \begin{align*}
        \E[(\hat{y} - y)^2] - \min_{f_j \in \cH_J} \E[(f_j(x) - y)^2] \leq 
        \frac{\tilde{O}( 
        T^{\max(\frac{3}{4}, 1 - \alpha_g)}
        \sqrt{\ln{\frac{K}{\delta} } } )}{T}
    \end{align*} 
\end{lemma}

\begin{proof}[Proof of Lemma \ref{lem:one-shot}]
By Theorem \ref{thm:bayesians-swap}, we have that after $T$ rounds, with probability $1-\delta$, Alice will have $(2\sqrt{2T\ln{\frac{g_A(T)K}{\delta}}} + \frac{T}{m^2} + m\sqrt{32 T \ln{ \frac{4g_A(T)K}{\delta}}}, g_A(T), \cH_A)$ conversation swap regret (symmetrically for Bob).
We can instantiate this with parameters that are sublinear in $T$, specifically $m =T^{\frac{1}{4}}$ and $g_A(T) = T^{-\alpha_g}$ for some constant $\alpha_g \in (0, 1)$. 
Then, we know that Alice, with probability $1 - \delta'$, satisfy $(f_A, g_A, \cH_A)-$conversation swap regret, for:
\begin{align*}
        f_A(T) &\leq2\sqrt{2T\ln{\frac{g_A(T)K}{\delta'}}} + \frac{T}{m^2} + m\sqrt{32 T \ln{ \frac{4g_A(T)K}{\delta'}}} \tag{by Theorem \ref{thm:bayesians-swap}}\\
        &\leq 2\sqrt{2T\ln{\frac{T^{-\alpha_g}K}{\delta'}}} + \sqrt{T} + T^{\frac{3}{4}}\sqrt{32 \ln{ \frac{4T^{-\alpha_g}K}{\delta'}}} \\
        &\leq \tilde{O}\left(T^{\frac{3}{4}} \sqrt{\ln \left( \frac{K}{\delta'} \right)}\right).
    \end{align*}
Since guarantees for Bob are symmetric, the same expression holds for him.
Thus, by a union bound, with probability $\delta' = \delta/2$, with probability $1-\delta$, Alice and Bob simultaneously have $(f_A, g_A, \cH_A)$-conversation swap regret and $(f_B, g_B, \cH_B)$-conversation swap regret, respectively.
Protocol \ref{alg:one-shot} is simply a special case of Protocol \ref{alg:general-agreement}, in which $(x_A, x_B, y)$ are drawn from fixed distribution each day. Therefore, the guarantees from Theorem \ref{thm:last-round} hold, and we have that the predictions in round $K$ have low expected error with respect to $\cH_J$:
    
   \begin{align*}
       \sum_{t=1}^T (p_{K}^t-y^t)^2 &- \min_{h_J\in\cH_J} \sum_{t=1}^T (h_J(x^t) - y^t)^2 \leq \\ &     2Tw\inv\left(8(\frac{\beta(T,f'_{A},f'_{B}) + 1/K}{2})^{\frac{1}{3}} + \frac{1}{2}\beta(T,f_{A},f_{B})\right) + 3\frac{T}{2}(g_{A}(T) + g_{B}(T)) + 3KT\beta(T,f'_{A},f'_{B}),
    \end{align*} 
    where 
    \begin{align*}
        f'_A(T) = f'_B(T) = \sqrt{T \cdot f_A(T)} = \sqrt{T \cdot \tilde{O} \left( T^{\frac{1}{2}} \sqrt{\ln{\frac{K}{\delta}}} \right) } = \tilde{O}\left(T^{\frac{3}{4}}\ln^{\frac{1}{4}}\left(\frac{K}{\delta}\right) \right)
    \end{align*}
    and thus:
    \begin{align*}
        \beta(T, f_A, f_B) &= \frac{f_A(g_{A}(T)T)}{Tg_{A}(T)} + \frac{f_B(g_{B}(T)T)}{Tg_{B}(T)} + g_{A}(T) + g_{B}(T) \\
        &\leq \tilde{O}\left( (T^{\alpha_g - \frac{1}{4}}) \sqrt{\ln \left( \frac{K}{\delta} \right)} + T^{-\alpha_g} \right), \\
        \beta(T, f'_A, f'_B) &= \frac{f'_A(g_{A}(T)T)}{Tg_{A}(T)} + \frac{f'_B(g_{B}(T)T)}{Tg_{B}(T)} + g_{A}(T) + g_{B}(T) \\ 
        &\leq \tilde{O}\left( T^{-\frac{1}{4}} \ln^{\frac{1}{4}}{\frac{K}{\delta}} + T^{-\alpha_g} \right).
    \end{align*}
Returning to the expression from Theorem \ref{thm:last-round}, we see
   \begin{align*}
       \frac{1}{T} \sum_{t=1}^T (p_{K}^t-y^t)^2 &- \frac{1}{T} \min_{h_J\in\cH_J} \sum_{t=1}^T (h_J(x^t) - y^t)^2 \\ & \leq     2w\inv\left(8(\frac{\beta(T,f'_{A},f'_{B}) + 1/K}{2})^{\frac{1}{3}} + \frac{1}{2}\beta(T,f_{A},f_{B})\right) + 3\frac{1}{2}(g_{A}(T) + g_{B}(T)) + 3K\beta(T,f'_{A},f'_{B}) \\
       &\leq w\inv \left( \tilde{O}( 
       (T^{-\frac{1}{4}} \ln^{\frac{1}{4}}{\frac{K}{\delta}} + T^{-\alpha_g} + \frac{1}{K})^\frac{1}{3} \right) + \tilde{O} (T^{\alpha_g - \frac{1}{4}}) \sqrt{\ln \left( \frac{K}{\delta} \right)}  \\
       & \quad \quad \quad +
       K\tilde{O}(T^{\frac{-1}{4}} \ln^\frac{-1}{4}( \frac{K}{\delta}) + T^{-\alpha_g}).
    \end{align*} 
\end{proof}

\begin{proof}[Proof of Theorem \ref{thm:one-shot}]
By Lemma \ref{lem:one-shot} we have established that the cumulative regret grows as $o(T)$. 
The claim we want to show is about the expected regret only on a single day, which pertains $K$ rounds of conversation about our instance of interest.
In the Bayesian setting, since instances are drawn i.i.d. and Bayesian agents make predictions independently across days, only as a function of the draw from the prior at the beginning of that day - conversations are also identically and independently distributed. Therefore, to argue about the expected error on the instance on any single day, it suffices to reason about the average of the cumulative regret over all $T$ days. 
We can consider what would happen to the average expected regret in the limit as $T \to \infty$, 
\begin{align*} \lim_{T \to \infty} &\frac{  w\inv \left( \tilde{O}( 
       (T^{-\frac{1}{4}} \ln^{\frac{1}{4}}{\frac{K}{\delta}} + T^{-\alpha_g} + \frac{1}{K})^\frac{1}{3} \right) + \tilde{O} (T^{\alpha_g - \frac{1}{4}}) \sqrt{\ln \left( \frac{K}{\delta} \right)} +
       K\tilde{O}(T^{\frac{-1}{4}} \ln^\frac{-1}{4}( \frac{K}{\delta}) + T^{-\alpha_g})}{T} \\
       &= w\inv(O(\frac{1}{K})^{\frac{1}{3}}). \end{align*}
Thus, we have the claim.
\end{proof}

\section{Proofs of Lower Bounds from \Cref{sec:lower-bounds}}

\begin{proof}[Proof of \Cref{thm:interaction-necessity}]
We adapt the construction from the proof of Theorem~\ref{thm:bounded-necessary}. Define a joint distribution $\cD$ over $\cX_A \times \cX_B \times \cY$ where $\cX_A, \cX_B \subseteq \mathbb{R}$ as follows:
\begin{align*}
x_A = \xi_A, x_B = x_A + \xi_B = \xi_A + \xi_B, \text{ and } y = \xi_B = x_B - x_A,
\end{align*}
where $\xi_A, \xi_B$ are independent random variables uniformly distributed in $\{-1, +1\}$. 

We consider $\cH_A = \cH_B =  \{x \mapsto w x + b: |w| \le 1, |b| \le 1\}$ and $\cH_J = \{(x_A, x_B) \mapsto w_A x_A + w_B x_B + b: |w_A| \le 1, |w_B| \le 1, |b| \le 1\}$ to be the classes of bounded linear functions. Then we have the following:

\noindent\textit{Optimal Linear Predictor for Alice ($h_A^*$):} Since $y = \xi_B$ is independent of $x_A = \xi_A$, and $\E[y] = \E[\xi_B] = 0$, the optimal linear predictor is the constant predictor $h_A^*(x_A) = \E[y] = 0$. Its expected squared error is $\E[(0-y)^2] = \E[\xi_B^2] = 1$.

\noindent\textit{Optimal Linear Predictor for Bob ($h_B^*$):} We seek $h_B^*(x_B) = w_B x_B + c_B$. Since $\E[y]=0$ and $\E[x_B] = \E[x_A + \xi_B] = \E[x_A] + \E[\xi_B] = 0$, $c_B=0$. The optimal $w_B = \frac{\E[x_B y]}{\E[x_B^2]}$. We have,
    \begin{align*}
        \E[x_B y] &= \E[(\xi_A + \xi_B) \xi_B] = \E[\xi_A \xi_B] + \E[\xi_B^2] = 1. \\
        \E[x_B^2] &= \E[(\xi_A + \xi_B)^2] = \E[\xi_A^2 +  2\xi_A \xi_B + \xi_B^2] = 2.
    \end{align*}
    Thus, $w_B = \frac{1}{2}$ and $h_B^*(x_B) = x_B/2$. Its expected squared error is $\E[(h_B^*(x_B) - y)^2] = \E[(x_B/2 - y)^2] = \E[(\xi_B/2)^2] = 1/4$.

\noindent\textit{Optimal Linear Predictor for Joint Features ($h_J^*$):} The conditional expectation $\E[y | x_A, x_B]$ is the optimal predictor overall. Here, $y = \xi_B = x_B - x_A$. Since this relationship is linear, the optimal linear predictor is $h_J^*(x) = x_B - x_A$. Its expected squared error is $\E[(h_J^*(x) - y)^2] = \E[(y-y)^2] = 0$.

We have $h_A^*(x_A) = 0$ and $h_B^*(x_B) = x_B/2$. Any predictor $f(h_A^*(x_A), h_B^*(x_B))$ can only depend on $x_B$ since $h_A^*(x_A) = 0$ is constant. The best predictor for $y$ that is a function of $x_B$ is in this case exactly the optimal linear predictor $h_B^*(x_B) = x_B/2$, which achieves an error of $1/4$. Thus, the minimum error achievable using only $h_A^*$ and $h_B^*$ by any function $f$ is:
\[ \E_{\cD}[(f(h_A^*(x_A), h_B^*(x_B)) - y)^2] = \E_{\cD}[(h_B^*(x_B) - y)^2] \ge 1/4 > 0 = \E_{\cD}[(h_J^*(x) - y)^2]. \]
\end{proof}

\begin{proof}[Proof of \Cref{thm:weak-learning-necessity}]
Consider a triple $(\cH_A, \cH_B, \cH_J)$ that fails the $w(\cdot)$-weak learning condition for any strictly increasing $w$. This implies there exists a distribution $\cD$ such that for some $\gamma > 0$:
\[ \min_{c \in \Rset} \E_{\cD}[(c-y)^2] - \min_{h_J \in \cH_J} \E_{\cD} [(h_J(x)-y)^2] \ge \gamma, \]
but for all $h_A \in \cH_A$ and $h_B \in \cH_B$:
\[ \min_{c \in \Rset} \E_{\cD}[(c-y)^2] - \E_{\cD} [(h_A(x_A)-y)^2] < w(\gamma) \]
\[ \min_{c \in \Rset} \E_{\cD}[(c-y)^2] - \E_{\cD} [(h_B(x_B)-y)^2] < w(\gamma). \]
Since this must hold for any strictly increasing $w$ (and $w(0)=0$), it must be that the improvement over the constant predictor for both $\cH_A$ and $\cH_B$ is zero. That is, $\min_{h_A \in \cH_A} \E_{\cD}[(h_A(x_A) - y)^2] = \min_{h_A \in \cH_A} \E_{\cD}[(h_B(x_B) - y)^2] = \min_{c \in \Rset} \E_{\cD}[(c - y)^2]$. Let $c^* = \argmin_{c \in \Rset} \E_{\cD}[(c-y)^2]$ be the optimal constant predictor.

Now consider the sequence of examples $(x_A^t, x_B^t, y^t)_{t=1}^T$ be drawn i.i.d. from the distribution $\cD$ and the constant prediction sequence $\hat{y}^t = c^*$ for all $t=1, \ldots, T$. Since $\hat{y}^t=c^*$ for all $t$, the only relevant level set is $v=c^*$, swap regret with respect to $\cH_A$ reduces to:
\[ \frac{1}{T}\sum_{t=1}^T (\hat{y}^t - y^t)^2 - \min_{h_A \in \cH_A} \frac{1}{T} \sum_{t=1}^T (h_A(x_A^t) - y^t)^2. \]
As $T \rightarrow \infty$, by the law of large numbers, this reduces to
\[
\E_{\cD}[(c^* - y)^2] - \min_{h_A \in \cH_A} \E_{\cD}[(h_A(x_A) - y)^2] = 0.
\]
By the same argument, we get that the swap regret with respect to $\cH_B$ is also 0. However, the external regret (as $T \rightarrow \infty$) with respect to $\cH_J$,
\begin{align*}
    \E_{\cD}[(c^* - y)^2] - \min_{h_J \in \cH_J} \E_{\cD}[(h_J(x) - y)^2] \ge \gamma > 0.
\end{align*}
Here the inequality follows from our assumption. This implies that the external regret with respect to $\cH_J$ is positive, while swap regret with respect to both $\cH_A$ and $\cH_B$ is 0.
\end{proof}

\begin{proof}[Proof of \Cref{lem:weak-is-weaker}]
In order to prove that $\cH_A$ and $\cH_B$ satisfy weak-learnability, let us assume that for some distribution $\cD$ and $\gamma \in [0,1]$
\[
\min_{c \in \mathbb{R}}\E[(c - y)^2] - \min_{h_J \in H_J} \E[(h_J(x) - y)^2] \ge \gamma.
\]
Now we will show that, either
\[
\min_{c \in \mathbb{R}}\E[(c - y)^2] - \min_{h_A \in H_A} \E[(h_A(x_A) - y)^2] \ge \gamma/2,
\]
or
\[
\min_{c \in \mathbb{R}}\E[(c - y)^2] - \min_{h_B \in H_B} \E[(h_B(x_B) - y)^2] \ge \gamma/2.
\]
Since $\cH_A$ and $\cH_B$ satisfy information substitutes with respect to $\cH_J$, from the statement in \Cref{def:joint-information-sub}, we have
\begin{align*}
    \min_{h_A \in H_A} \E[(h_A(x_A) - y)^2] +  \min_{h_B \in H_B} \E[(h_B(x_B) - y)^2] &\le  \min_{c \in \mathbb{R}}\E[(c - y)^2] + \min_{h_J \in H_J} \E[(h_J(x) - y)^2].
\end{align*}
Substituting the assumption on the joint feature improving over the constant function, we get
\[
2 \min_{c \in \mathbb{R}}\E[(y - c)^2] - \min_{h_A \in H_A} \E[(h_A(x_A) - y)^2] -  \min_{h_B \in H_B} \E[(h_B(x_B) - y)^2] \ge \gamma.
\]
This implies that either $ \min_{c \in \mathbb{R}}\E[(y - c)^2] - \min_{h_A \in H_A} \E[(h_A(x_A) - y)^2]$ or $\min_{c \in \mathbb{R}}\E[(y - c)^2] - \min_{h_B \in H_B} \E[(h_B(x_B) - y)^2]$ must be $\ge \gamma/2$. This gives us the desired weak-learning condition.
\end{proof}

\begin{proof}[Proof of \Cref{cor:linear-weak-IS}]
Consider the class of bounded linear function over  $\cX_A = \cX_B = [-1,1]$ as defined in the proof of \Cref{thm:quadratic}. Suppose these classes satify the information substitutes condition, then by \Cref{lem:weak-is-weaker}, we know that they must satisfy $w(\cdot)$-weak learnability for $w(\gamma) = \gamma/2$. However, from \Cref{thm:quadratic}, we know that these function classes can not satisfy $w(\cdot)$-weak learnability for $w(\gamma) = \gamma/2$ giving us a contradiction. Therefore, these classes could not have satisfied the information substitutes condition.
\end{proof}

\begin{proof}[Proof of \Cref{lem:swap-necessary}]
    Consider the joint distribution $\cD$ over $\cX_A \times \cX_B \times \cY$ to be as follows:
    \[
    x_A \sim_{\text{i.i.d.}} \{0,1\}, x_B \sim_{\text{i.i.d.}} \{0,1\}, y = x_A x_B.
    \]
    Let the class of functions be bounded linear functions which satisfy our weak-learning condition.
    
    Observe that the the best linear predictor in $\cH_A$ is $h_A^*(x_A) = \E[y|x_A] = x_A/2$ and the best linear predictor in $\cH_B$ is $h_B^*(x_B) = \E[y|x_B] = x_B/2$. Observe that,
    \[
    \E[(h_A^*(x_A) - y)^2] = \E[(x_A/2 - x_A x_B)^2] = 1/8 = \E[(x_B/2 - x_A x_B)^2] = \E[(h_B^*(x_B) - y)^2].
    \]

    Now consider the sequence of examples $(x_A^t, x_B^t, y^t)_{t=1}^T$ be drawn i.i.d. from the distribution $\cD$ and the prediction sequence $\hat{y}^t = x_A^t/2$ for all $t=1, \ldots, T$. Observe that the external regret with respect to $\cH_A$ as $T \rightarrow \infty$ is
    \[
    \E_\cD[(x_A/2 - y)^2] - \min_{h_A \in \cH_A} \E_\cD[(h_A(x_A) - y)^2] = \E_\cD[(x_A/2 - y)^2] - \E_\cD[(x_A/2 - y)^2] =  0.
    \]
    Similarly the external regret with respect to $\cH_B$ as $T \rightarrow \infty$ is
    \[
    \E_\cD[(x_A/2 - y)^2] - \min_{h_B \in \cH_B} \E_\cD[(h_B(x_B) - y)^2] = \frac{1}{8} - \E[(h_B^*(x_B) - y)^2] = \frac{1}{8} - \frac{1}{8} = 0.
    \]
    Therefore the sequence of predictions has no external regret with respect to $\cH_A$ and $\cH_B$.

    However, the best linear predictor defined on both $\cX_A$ and $\cX_B$ is 
    $h_J^*(x) = (x_A + x_B)/2 - 1/4$. This has expected error over $\cD$, $\E_\cD[(h_J^*(x) - y)^2] = 1/16$. Thus, as $t \rightarrow \infty$, the predictions have external regret,
    \begin{align*}
        \E_\cD[(x_A/2 - y)^2] - \E[(h_J^*(x) - y)^2] &= \frac{1}{8} - \frac{1}{16} = \frac{1}{16} > 0.
    \end{align*}
\end{proof}

\end{document}